\documentclass[12pt,reqno]{amsart}
\renewcommand{\appendix}{\par
  \setcounter{section}{0}
  \setcounter{subsection}{0}
  \gdef\thesection{\Alph{section}}
}

\input{macrogilles_main}
\usepackage{amsmath,amsfonts,amsthm,amssymb}

\usepackage[T1]{fontenc}

\usepackage[hmarginratio={1:1}, includeheadfoot, margin=3cm]{geometry}

\usepackage[linecolor=blue!60!,backgroundcolor=blue!10!,textwidth=2.5cm,textsize=scriptsize,disable]{todonotes}

\usepackage[utf8]{inputenc}
\usepackage[T1]{fontenc}
\usepackage[main=english]{babel}%
\usepackage{graphicx}		
\usepackage{algorithm}		
\usepackage{algpseudocode} 	
\usepackage[normalem]{ulem} 
\algrenewcommand{\algorithmiccomment}[1]{\texttt{\# #1}} 
\usepackage{nicefrac} 		
\usepackage{enumitem} 		
\usepackage{booktabs}       
\usepackage{longtable}       
\usepackage{stmaryrd}
\usepackage[title]{appendix} 
\usepackage[export]{adjustbox} 
\usepackage{xspace}
\usepackage{upgreek} 
\usepackage{refcount}
\usepackage{csquotes} 

\usepackage{multirow}

\graphicspath{{Figures/}}

\usepackage{hyperref}

\def\mom{{\rm{MOB}}}

\def\ealpha{{e(\widehat{\mu}_1)}}

\newcommand{\NE}{\texttt{NE}}
\newcommand{\STBweight}{\texttt{STB weight}}
\newcommand{\STBopt}{\texttt{STB opt}}
\newcommand{\STBorth}{\texttt{STB orth}}
\newcommand{\STBegd}{\texttt{STB egd}}
\newcommand{\RKMSE}{\texttt{R-KMSE}}
\newcommand{\MTAconst}{\texttt{MTA const}}
\newcommand{\AGGo}{\texttt{AGG orth}}
\newcommand{\AGGe}{\texttt{AGG egd}}

\newcommand{\deff}{d^{\mathrm{e}}}
\newcommand{\detest}{d^*}
\newcommand{\deamm}{d^\bullet}

\newcommand{\estmu}{\wh{\mu}}
\newcommand{\muNE}{\estmu^\NE}
\newcommand{\Spx}{\cS}

\def\ealpha{{e(\muNE_1)}}
\newcommand{\Deltaq}{{\Delta q}}

\newcommand{\GS}{{\bf GS}\xspace}
\newcommand{\BS}{{\bf BS}\xspace}
\newcommand{\HT}{{\bf HT}\xspace}
\newcommand{\ECSS}{{\bf ECSS}\xspace}
\newcommand{\KC}{{\bf KC}\xspace}
\newcommand{\JS}{{\bf JS}\xspace}
\newcommand{\TSC}{{\bf TSC}\xspace}

\newcommand{\omvect}{{\bm{\om}}}
\newcommand{\muvect}{{\bm{\mu}}}

\newcommand{\rhovect}{{\bm{\rho}}}

\newcommand{\zetavect}{{\bm{\zeta}}}
\newcommand{\Cvect}{{\bm{\cC}}}
\newcommand{\diam}{\textup{\texttt{diam}}}

\newcommand{\singlemodel}{\cP_{\mathrm{single}}}
\newcommand{\multimodel}{\cP_{\mathrm{mult}}}

\newcommand{\vref}{V^*}

\newcommand{\rnrisk}{s}
\newcommand{\nrisk}{\rnrisk^2}
\newcommand{\invnrisk}{\rnrisk^{-2}}
\newcommand{\whnrisk}{\wh{\rnrisk}^2}

\newcommand{\taumink}{\tau_{\min}^k}
\newcommand{\taumino}{\tau_{\min}^\circ}

\newcommand{\cteW}{\varsigma}

\newcommand{\ratiobs}{\upphi}

\newcommand{\taubb}{\bar{\tau}_*}

\newcommand{\Rebuttal}[1]{ #1}
\newcommand{\JB}[1]{\Rebuttal{ #1}}
\newcommand{\GB}[1]{\Rebuttal{ #1}}

\newcommand{\HM}[1]{\Rebuttal{ #1}}
\newcommand{\Revision}[1]{{#1}}
\newcommand{\GBT}[1]{\Revision{#1}}

\newcommand{\ip}[2]{\langle #1 , #2 \rangle}

\makeatletter
\newcommand{\dwt}[1]{{%
  \mathpalette\double@widetilde{#1}%
}}
\newcommand{\double@widetilde}[2]{%
  \sbox\z@{$\m@th#1\widetilde{#2}$}%
  \ht\z@=.85\ht\z@
  \widetilde{\box\z@}%
}
\makeatother

\theoremstyle{definition}
\newtheorem*{assumption}{Assumption}
\newtheorem*{setting}{Setting}
 
	\usepackage[backend=bibtex, style=authoryear-comp, giveninits=true, uniquename=false, url=false, isbn=false, doi=false, dashed=false,  natbib=true, maxcitenames=2, maxbibnames=10, date=year, uniquelist=false]{biblatex}
	
	\DeclareNameAlias{sortname}{family-given} 
	\DeclareFieldFormat{pages}{#1} 
	
	\renewbibmacro*{volume+number+eid}{%
	  \printfield{volume}%
	  \setunit*{\addnbspace}
	  \printfield{number}%
	  \setunit{\addcomma\space}%
	  \printfield{eid}}
	\DeclareFieldFormat[article]{number}{\mkbibparens{#1}}
	\renewbibmacro{in:}{%
	  \ifboolexpr{%
	     test {\ifentrytype{article}}%
	     or
	     test {\ifentrytype{inproceedings}}%
	  }{}{\printtext{\bibstring{in}\intitlepunct}}%
	}
	
	\DeclareFieldFormat[article,inbook,incollection,inproceedings,patent,thesis,unpublished]{citetitle}{#1}
	\DeclareFieldFormat[article,inbook,incollection,inproceedings,patent,thesis,unpublished]{title}{#1}
	
\author{G. Blanchard$^{*,\dagger}$, J-B. Fermanian$^{*,\ddagger}$, H. Marienwald$^{*,\S}$}
\date{}
\thanks{\scriptsize
\noindent $^*$Authors contributed equally.}
\thanks{$^\dagger$Universit\'{e} Paris Saclay, Institut Math\'{e}matique d'Orsay, France. gilles.blanchard@universite-paris-saclay.fr}
\thanks{$^\ddagger$Universit\'{e} Paris Saclay, Institut Math\'{e}matique d'Orsay, France. jean-baptiste.fermanian@universite-paris-saclay.fr}
\thanks{$^\S$BIFOLD, Technische Universit{\"a}t Berlin, Germany. hannah.marienwald@campus.tu-berlin.de}

\begin{document}
	\title{Estimation of multiple mean vectors in high dimension}
	
	\begin{abstract}
		We endeavour to estimate numerous multi-dimensional means of various probability distributions on a common space based on independent samples. Our approach involves forming estimators through convex combinations of empirical means derived from these samples. We introduce two strategies to find appropriate
		data-dependent convex combination weights: a first one employing a testing procedure to identify neighbouring means with low variance, which results in a closed-form plug-in formula for the weights, and a second
		one determining weights via minimization of an upper confidence bound on the quadratic risk.
		Through theoretical analysis, we evaluate the improvement in quadratic risk offered by our methods compared to the empirical means. Our analysis focuses on a dimensional asymptotics perspective, showing that our methods asymptotically approach an oracle (minimax) improvement as the effective dimension of the data increases.
		We demonstrate the efficacy of our methods in estimating multiple kernel mean embeddings through experiments on both simulated and real-world datasets.
	\end{abstract}

	\maketitle
	\markleft{G. BLANCHARD, J-B. FERMANIAN, H. MARIENWALD}
	\section{Introduction}\label{sec:intro}
	We study the problem of jointly estimating multiple vector means of large dimension $d$
	of distinct probability distributions, from each of which we observe independent, i.i.d. samples.
	While the theme of multiple mean estimation has a long and rich history in statistics,
	the purpose of the present work is to focus specifically on some high-dimensional aspects (large or even infinite $d$, in a Hilbert space).

	The framework under examination is motivated by scenarios involving large volumes of high-dimensional data. These scenarios typically involve the categorization of independent samples into homogeneous units that may exhibit differences but also varying degrees of similarity. Examples include medical or educational records sourced from different institutions, or purchase histories organised by individual clients on an internet platform. This framework also intersects with the concepts of federated and personalised machine learning \citep{mcauley2022personalized, tan2022towards}.
	An application of particular interest within this framework is that of kernel mean embeddings \GBT{(KME)} of distributions \citep{muandet2017overview}.
	This involves estimating means of distributions after a formal mapping of the data into a Hilbert space.
	\HM{As a specific example, we examine the estimation of \GBT{KMEs for different} cell type distributions in a patient's blood sample based on empirical data provided by a flow cytometer. By jointly analysing samples from multiple patients, inter-patient similarities can be used to enhance the estimation while still providing patient-individualized estimates.}
	

	To formalize, let $\mu_1, \ldots, \mu_B$ be the vector means
	of distinct probability distributions $\mbp_1, \ldots, \mbp_B$ over $\mbr^d$ (an extension to Hilbert spaces is also discussed).
	The estimation of the means is based on a family of independent sample sets, $X_\bullet^{(1)}, \ldots , X_\bullet^{(B)}$, where each $X_\bullet^{(k)}$ with $k \in \intr{B} := \lbrace 1,\ldots,B \rbrace$ comprises of $N_k$ samples drawn i.i.d. from $\mbp_k$.
	Formally, the joint model is
	\begin{equation}
		\label{eq:mainmodel}
		\begin{cases}
			X^{(k)}_{\bullet}:=(X_i^{(k)})_{1\leq i\leq N_k} \stackrel{\text{i.i.d.}}{\sim} \mbp_k, \; k \in \intr{B}  ; \\
			(X^{(1)}_{\bullet},\ldots,X^{(B)}_{\bullet}) \text{ independent.}
		\end{cases}
	\end{equation}
	The distributions are assumed to be at least square-integrable.
	We refer to a set of samples $X_\bullet^{(k)}$ as {\em bag} and to $\mbp_k$ as {\em task}, in line with the domain of multi-task learning.
	Our aim is to define estimators $\estmu_k$ and analyse the risk given by the expected squared distance to the true means $\mu_k$.
	
	Evident candidates are empirical means taken separately for each bag, which we call {\em naive} estimators.
	The question we want to tackle is whether it is possible to improve over these individual naive estimators
	by exploiting similarities between tasks.
	We propose and study particular estimators $\estmu_k$ formed by a convex combination of naive
	estimators of ``related'' tasks.
	We insist that absolutely no information about the underlying similarity or task structure is assumed to be known {\em a priori}.
	Roughly speaking, we measure relatedness between tasks by estimating the distance between their means.
	
	The goal is to analyse the {\em relative} risk of the proposed estimators, i.e., the ratio of their risk to that of the corresponding naive estimator.
	The following questions will guide our estimator construction and analysis:
	\begin{enumerate}[label=(\alph*)]
		\item what would be the ideal ``oracle'' convex combination estimator, if some additional {\em a priori} information about task relatedness were known?
		\item can an empirical estimator approach the oracle relative risk from the data only, in a suitable asymptotical sense?
		\item is the oracle relative risk minimax optimal in a suitable asymptotical sense?
	\end{enumerate}
	Because we focus on the relative risk, the usual asymptotics of the sample size going to infinity is not the most relevant one (though we will assume that the sample sizes are ``large enough''). Rather, we will focus on {\em high-dimensional asymptotics} where the
	dimension grows large. More precisely, we mean a notion of {\em effective} dimension rather than ambient space dimension: the effective dimension
	of a task will be defined from spectral quantities related to its covariance matrix, as is common in high-dimensional statistics.
	
	
	{\bf Relation to previous work.}
	The problem of estimating multiple means 
	started in particular
	with the seminal work of Stein on the eponymous paradox and the James-Stein estimator \citep{js1961},
	continued with the empirical Bayes point of view on the latter \citep{efron1972empirical}, up to modern considerations
	on the topic \citep{brown2009nonparametric,jiang2009}. The topic of ``multitask learning'' also provides
	a more recent angle on the problem \citep{feldman2014revisiting,duan2023adaptive}. We defer a detailed discussion to Section~\ref{se:previouswork},
	but stress that most previous works analysed the {\em compound} (or cumulated) risk over all tasks and
	its behaviour in the asymptotics $B\rightarrow \infty$, in a one- or fixed-dimensional setting.
	By contrast, we will be interested in analyzing the
	individual risk separately for each task, and in ``high dimensional'' asymptotics.
	
	We start with a description of the considered setting in Section~\ref{se:setting}.
	Sections~\ref{se:testapproach} and \ref{se:Qaggreg} introduce two approaches to form convex combination estimators of the means, provide bounds on their relative risks, and a comparison of the two.
	A minimax analysis for suitable distribution classes is conducted in Section~\ref{se:minimax}.
	Finally, experiments on artificial and true data are presented in Section~\ref{se:application}.
	All proofs are provided in the Supplemental, wherein Supplemental~\ref{se:notation} contains a list of the used notation \Revision{and Supplemental~\ref{se:summary_results} a summary of the main results} for the reader's convenience.
	
	\section{Setting and notation}\label{se:setting}
	\subsection{Loss and risk}
	We consider the squared loss and expected risk with the Euclidean norm:
	\begin{equation}
		\label{eq:lossrisk}
		L_k(\estmu_k) := \norm{\estmu_k-\mu_k}^2\,; \qquad R_k(\estmu_k) := \e{L_k(\estmu_k)}.
	\end{equation}
	of an estimator $\estmu_k$ for $\mu_k$.
	The empirical mean $\muNE_k : = \frac{1}{N_k}\sum_{i=1}^{N_k} X_i^{(k)}$, called the {\em naive estimator}, serves as a reference.
	Due to the unbiasedness of the naive estimator, its variance is equal to its risk.
	More specifically, let the {\em naive risk} be denoted by
	\begin{equation}
		\label{eq:defnaiverisk}
		\nrisk_k := R_k(\muNE_k) = \frac{\tr \Sigma_k}{N_k},
	\end{equation}
	where $\Sigma_k$ is the covariance matrix of task $k$. 
	Then any estimator $\estmu_k$ is analysed in terms of its {\em relative risk} to the naive --- lower is better :
	\begin{equation}
		\label{eq:lossriskave}
		\frac{R_k(\estmu_k)}{\nrisk_k}.
	\end{equation}
	
	In contrast to the compound decision setting, our goal is to analyse the relative risk for each task separately.
	For this reason, the focus is on a specific task, say $k=1$ and $R_1(\estmu_1)/\nrisk_1$ without loss of generality.
	In Section \ref{sec:avgrisk} the relative risk averaged over tasks
	$\frac{1}{B} \sum_{k=1}^{B} {R_k(\estmu_k)}/{\nrisk_k}$ is considered.
	
	\subsection{High-dimensional asymptotics}\label{sec:highdimasy}
	Observe from~\eqref{eq:defnaiverisk} that the naive risk $\nrisk_1$ decreases at the parametric rate $\cO(N_1^{-1})$.
	We expect the risk of a competing estimator $\estmu_1$ to follow the same trend. As a consequence, the role of the sample size will cancel out in the relative risk.
	In order to state meaningful results,
	it is necessary to obtain sharp estimates of the other factors in the rate.
	
	To this end, we shift the perspective from a standard asymptotic view point, $N_1 \rightarrow \infty$, to high-dimensional asymptotics, emphasizing the behaviour of the risks as the dimensionality grows.
	There are different possible definitions of {\em effective dimensionality} of a distribution, generally
	linked to the spectral decay of the covariance matrix and ratios of its Schatten norms. The following ones will be relevant to
	our analysis:
	\begin{equation}\label{eq:effd}
		\deamm_k := \frac{(\tr \Sigma_k)^2}{\tr \Sigma_k^2}, \qquad \deff_k := \frac{\tr \Sigma_k}{\|\Sigma_k\|_{\infty}},
	\end{equation}
	where $\norm{\cdot}_\infty$ is the operator norm. Observe that in the isotropic setting $\Sigma_k \propto I_{d}$, the effective dimensions $\deamm_k$ and $\deff_k$ coincide with the ambient
	dimension $d$, as one would expect. In all cases it holds $1 \leq \sqrt{\deamm_k} \leq \deff_k \leq \deamm_k \leq d$. 
	In random matrix literature, $\deff$ is sometimes called intrinsic dimension~\citep{tropp2015introduction,hsu2012tail} or effective rank~\citep{Kol16}, and $\detest:= (\deff)^2/\deamm$ is known as the numerical or stable rank of $\Sigma$ \citep{rudelson2007sampling,tropp2015introduction}.
	Most notably, we uncover a ``blessing of dimensionality'' phenomenon:
	in a nutshell, we will show that the relative risks of our estimators asymptotically approach a suitable notion of oracle relative risk as the (effective) dimensionality increases.
	
	\subsection{Distributional assumptions}\label{sec:dist_ass}
	For our theoretical analysis, we consider the following different possible distributional assumptions:
	\begin{assumption}[\protect{\GS}, Gaussian setting] \label{ass:GS} For all $k \in \intr{B}$, the distribution $\mbp_k$ is $\cN(\mu_k,\Sigma_k)$.
	\end{assumption}
	\begin{assumption}[\protect{\BS}, Bounded setting] \label{ass:BS} For all $k \in \intr{B}$, $\mbp_k$ has support in the ball of radius $M$ centred at 0. \end{assumption}
	The (\BS) setting is of particular interest for the application to kernel mean embeddings, for which the assumption of a bounded kernel is very common.
	\GB{In fact, all of our results under (\BS) hold without change if $\mbr^d$ is replaced by a separable Hilbert space; in that case
		the covariance matrices $\Sigma_k$ become covariance operators, which are all trace-class as a consequence of the boundedness assumption.
		\footnote{\GB{We believe results in the Gaussian case can be also extended to a Hilbert space, assuming that the covariance operators are
				trace-class. Since classical Gaussian concentration results that we make use of are generally stated in a finite-dimensional space, we stick to
				this case for simplicity of presentation.}
	}}
	
	Supplemental~\ref{app:HT} covers another distributional assumption: heavy-tailed distributions with finite fourth moment.
	These results only hold for some of the proposed estimators (the testing approach, introduced in Section~\ref{se:testapproach}) and are, thus, not discussed further in the main text of this paper.
	
	\subsection{Simplifying settings}
	At times we will discuss unrealistic but simplifying settings to help with the exposition or to illuminate our theoretical findings.
	\begin{setting}[\protect{\ECSS}, Equal Covariance and Sample Sizes] For all $k \in \intr{B}$, $\Sigma_k=\Sigma$ and $N_k=N$, which implies that $\nrisk_k, \deamm_k, \deff_k$ do not depend on $k$.
	\end{setting}
	\begin{setting}[\protect{\KC}, Known Covariances]  For all $k \in \intr{B}$, $\Sigma_k$ is known.
		Consequently, all derived quantities $\tr \Sigma_k, \tr \Sigma_k^2, \deamm_k, \deff_k, \nrisk_k$ are also known.
	\end{setting}
	We will first derive the estimators assuming known covariances (\KC) but later provide estimates for covariance-related quantities if those are unknown.
	If the covariances and sample sizes are homogeneous (\ECSS) the risks are more transparent and interpretable which will help to illuminate our theoretical findings.
	We insist that the final algorithms neither assume (\KC) nor (\ECSS).

	\subsection{Naive estimator aggregation}\label{se:ne_aggregation}
	As announced earlier, without loss of generality we focus on estimating task $k=1$. Furthermore,
	we focus on estimators which can be written as convex combinations (aggregation) of naive estimators.
	Let $\Spx_B$ denote the $(B-1)$-dimensional simplex, and $\omvect=(\om_1,\ldots,\om_B) \in \Spx_B$ be a weight vector, then
	\begin{equation}
		\label{eq:defmuhatomega}
		\estmu_{\bm{\om}} := \sum_{k \in \intr{B}} \om_k \muNE_k \qquad \text{ s.t. } \qquad \sum_{k \in \intr{B}} \om_k = 1 \,\,\, \text{and} \,\,\, \forall k \in \intr{B}: \om_k \geq 0,
	\end{equation}
	whose loss and risk will be abbreviated as $L_1(\bm{\om})$ and $R_1(\bm{\om})$, respectively.
	While the weight vector $\omvect$ may be data-dependent later, for the present considerations we assume
	that the weights are {\em deterministic}. In this case, using independence of the naive estimators and the notation $\Delta_{k} := \mu_k -\mu_1$,
	we restate the risk $R_1(\bm{\om})$ by its bias-variance decomposition for a fixed $\bm{\om}$ as
	\begin{equation}
		\label{eq:riskone}
		R_1(\omvect) = \norm[3]{
			\sum_{k \in \intr{B}} \om_k(\mu_k
			- \mu_{1})}^2 + \sum_{k \in \intr{B}} \om_k^2 \nrisk_k
		=\sum_{k,k' \in \intr{B}} \om_k \om_{k'} \inner{\Delta_k,\Delta_{k'}} + \sum_{k \in \intr{B}} \om_k^2 \nrisk_k \, ,
	\end{equation}
	where the first term corresponds to the (squared) bias and the second to the variance.
	Intuitively, we want to give higher weights
	to tasks that are close (small task bias $\norm{\Delta_k}$) and can be accurately estimated (small naive risk $\nrisk_k$).
	At a  first glance, we could set as a goal to find suitable weights $\bm{\om}$ that minimise \eqref{eq:riskone};
	this, however, would require full knowledge
	of the Gram matrix $(\inner{\Delta_k,\Delta_{k'}})_{k,k' \in \intr{B}}$,
	in addition to the naive risks $\nrisk_k$.
	Estimation of the full Gram matrix, accurate enough to approach exact minimization of~\eqref{eq:riskone}, appears unattainable
	if the number of tasks $B$ is large and the Gram matrix becomes high-dimensional, which is the scenario we are interested in.
	For this reason, we will consider optimizing the risk given more limited information, which includes a subset of neighbouring tasks close to the
	target in relative sense
	but not their exact position. 
	We define the oracle risk as the minimiser of the worst-case risk of~\eqref{eq:riskone} as if this partial information was known to the oracle.
	
	We will consider two strategies to approach that oracle programme from data.
	In Section~\ref{se:testapproach} we aggregate only means close to the target which are identified by a test procedure.
	Minimization of an upper bound of the risk yields their weights.
	In Section~\ref{se:Qaggreg} we minimise directly an upper confidence bound of the aggregate risk~\eqref{eq:riskone}
	but have to take into account that the means that are further away induce a large uncertainty on the bias term.
	In both cases, we compare the obtained relative risk to that of the oracle. Additionally, we
	study the minimax risk under the oracle information in Section~\ref{se:minimax} and whether
	the proposed estimators match it.
	
	\section{A testing approach}
	\label{se:testapproach}
	A low-risk aggregation estimator \eqref{eq:defmuhatomega} combines naive estimations that --- at best --- provide a reduction in variance but add only a small bias, cf. \eqref{eq:riskone}.
	Our first approach explicitly controls the bias.
	We aim at identifying a subset of {\em neighbour} tasks whose means are sufficiently close to the target task.
	We then restrict the support of the weights to that subset and form a convex combination of neighbouring naive estimations.
	This approach and its analysis generalise ideas introduced in \citet{marienwald2021high}.
	Let us first introduce some additional notation.
	\begin{definition}[$\tau$-neighbouring tasks]
		Recall the notation $\Delta_k=\mu_k-\mu_1$.
		For a fixed $\tau>0$, let $V_\tau \subseteq \intr{B}$ denote the set of all {\em $\tau$-neighbouring tasks} (of task 1) as:
		\begin{equation}
			\label{eq:Vtau}
			V_\tau := \set{ k \in \intr{B} : \norm{\Delta_k}^2 \leq \tau \nrisk_1}.
		\end{equation}
		
		For $\tau=0$, for the sake of later notational coherence we define $V_0:=\set{1}$ which deviates from~\eqref{eq:Vtau} as $V_0$ does not contain
		any other tasks $k\neq 1$ even if $\Delta_k=0$.
	\end{definition}
	Note that this notion of $\tau$-neighbourhood is relative to the naive risk of task 1, and that $1 \in V_\tau$ always holds.
	\begin{definition}[Relative aggregated variance $\nu$] \label{def:relagvar}
		For a subset $U \subseteq \intr{B}$ of tasks, define their relative aggregated variance (to that of task 1) as:
		\begin{equation}\label{eq:def_s2v}
			\nu(U) := \frac{\nrisk(U)}{\nrisk_1}, \text{ with } \rnrisk^2(U) := \paren{ \sum_{k\in U}\frac{1}{\nrisk_k} }^{-1}.
		\end{equation}
	\end{definition}
	Observe that $s^2(U)$ is the variance of the optimal convex combination of unbiased, independent estimators that have
	different variances $\nrisk_k$ --- a classical problem of statistics. The quantity $\nu(U)$ is, again, relative to the naive risk of task 1.
	
	The quantity $\tau$ can be seen as the worst-case
	relative bias of a convex combination of their naive estimators for the goal of estimating $\mu_1$, while $\nu(V_\tau)$ is a best-case relative variance (i.e., all the tasks in $V_\tau$ would in fact have mean $\mu_1$).
	We introduce the following auxiliary function, which will capture an optimal trade-off between these two quantities.
	It provides a common reference value for the relative risks of our estimators and is of fundamental importance for the remainder of this manuscript.
	\begin{definition}
		Define the function $\cB :\mbr \times [0,1] \to [0,1]$ as
		\begin{equation}\label{eq:defBtaunu}
			\cB(\tau,\nu) := \paren{\frac{\tau}{1+\tau}} + \paren{\frac{1}{1+\tau}} \paren{\frac{\nu}{1+ \tau(1-\nu)}}.
		\end{equation}
		Observe that $\cB(0,\nu)=\nu$, $\cB(\tau,0) = \frac{\tau}{1+\tau}$, and $\cB$ is increasing in both of its variables.
	\end{definition}
	In the next section, we derive a form of optimal or ``oracle'' weights for combining naive estimators of
	tasks belonging to any given subset $V\subseteq V_\tau$, and identify $\cB$ as a bound on its relative risk.
	The following sections (\ref{sec:ortoemp} to~\ref{se:unknowncov}) are  concerned with approximating the oracle bound
	by estimating unknown quantities and using a plug-in principle.
	
	\subsection{Oracle procedure}
	\label{se:oracletestproc}
	For a fixed $\tau>0$, assume an oracle provides a set of neighbours $V$ with the guarantee that $V\subseteq V_\tau$ holds.
	We restrict our attention to convex combinations of naive estimators only in set $V$, i.e., estimators $\estmu_\omvect$
	as in~\eqref{eq:defmuhatomega} with $\bm{\om} \in \Spx_V$, where $\Spx_V$ will denote the set of convex weights of support included in $V$.
	Using the Cauchy-Schwartz inequality in~\eqref{eq:riskone} (with $\Delta_1=0$), for such aggregated estimators we obtain the risk bound
	\begin{equation}
		\label{eq:riskboundV}
		R_1(\omvect) \leq \tau \nrisk_1 (1-\om_1)^2 + \sum_{k \in V} \om_k^2 \nrisk_k, \,\, \text{ for all } \,\, \omvect \in \Spx_V,
	\end{equation}
	which can be optimised for $\omvect$.
	A bound on the oracle relative risk is presented next.
	
	\begin{lemma}\label{lem:oraclebound}
		Let $\tau  >0$ be fixed.
		For all $V \subseteq V_\tau$,  the weights $\omvect_V^* \in \Spx_V$ that minimise \eqref{eq:riskboundV} yield the bound
		\begin{equation}\label{eq:oracletestbound}
			\frac{R_1(\omvect^*_V)}{\nrisk_1} \leq \cB\paren{ \tau, \nu(V)}.
		\end{equation}
		The oracle weights $\omvect^*_V$
		are given by:
		\begin{equation}
			\label{eq:STBoptweights}
			\om_{V,k}^*(\tau,{\bm \rnrisk}) =
			(1-\lambda) \ind{k=1} + \lambda \frac{\nrisk(V)}{\nrisk_k},
			\text{ where } \lambda := \frac{1}{1+\tau(1- \nu(V))}.
		\end{equation}
	\end{lemma}
	It holds $\cB(\tau,\nu(V)) \in [\frac{\tau}{1+\tau}, 1]$, i.e., this bound cannot be better than $\frac{\tau}{1+\tau}$.
	We will call $\frac{\tau}{1+\tau}$ {\em best potential improvement} (that can be
	guaranteed by the oracle bound).
	The bound on the relative risk depends on the relative neighbourhood size $\tau$ and the relative aggregated variance $\nu(V)$.
	Because $\cB$ increases in both variables, small $\tau$ and $\nu(V)$ are beneficial.
	This coincides with what we noted from the bias-variance decomposition \eqref{eq:riskone}.
	If $\tau$ is fixed, it is of advantage to consider as many $\tau$-neighbours as possible so that $\nu(V)$ decreases, i.e., to take $V=V_\tau$.
	On the other hand, reducing the neighbourhood size $\tau$ reduces the bias but also leads to a
	smaller set of neighbours, ergo, a larger relative aggregated variance $\nu(V_\tau)$.
	Thus, there is a trade-off between both quantities.
	We may aim at a relative risk close to $\min_{\tau >0} \cB(\tau,V_\tau)$ but for the remainder of this section we assume $\tau>0$ fixed beforehand.
	
	The following observations enable additional insight into the involved quantities:
	\begin{enumerate}[label=(\alph*)]
		\item $\cB(0,\nu(V)) = \nu(V)$, i.e., when $\tau \searrow 0$, which implies that all tasks in $V$ have the same mean, the bound is given by the relative aggregated variance, as should be expected from the remark following Definition~\ref{def:relagvar}.
		\item $\cB(\tau,0) = \frac{\tau}{1+\tau}$,
		the best potential improvement is reached when $\nrisk(V) \searrow 0$.
		This happens if at least one of the $\tau$-neighbouring
		means is known with perfect precision and it becomes a ``reference point''. This scenario is comparable to the classical James-Stein
		setting, for which the origin is such a reference point and the James-Stein estimate
		improves most if the target is close to the origin (see Supplemental~\ref{se:JS} for a detailed discussion).
		However, $\nrisk(V) \searrow 0$ also happens when $\tau$-neighbours have a non-zero variance, but their number grows large.
		\item $\cB(\tau,\nu(V))$ remains unchanged if we replace a group of neighbours $V\setminus\set{1}$ by a single $\tau$-neighbour with variance $\nrisk(V\setminus \lbrace 1 \rbrace)$.
	\end{enumerate}
	
	\subsection{From an oracle to an empirical procedure} \label{sec:ortoemp}
	In practice, the oracle information about the relative neighbours is unavailable.
	However, we can hope to approach the oracle setting
	by estimating the set of $\tau$-neighbours $V_\tau$ and their risks $\nrisk_k$.
	We will assume that such estimates are independent of the samples used to compute $(\muNE_k)_{k\in\intr{B}}$.
	(To this end, one might resort to sample splitting.)
	The independence assumption of estimates is emphasised by a tilde notation: $(\wt{V},\wt{\bm{\rnrisk}}^2)$.
	
	The simplest is to plug in such estimates into the oracle formula~\eqref{eq:STBoptweights}. The next proposition quantifies
	how the relative risk of the plug-in procedure can be bounded, provided the estimation error is.
	\begin{lemma}
		\label{prop:boundSTB1}
		Let $\tau >0$ be fixed.
		Assume $\wt{V} \subseteq \intr{B} ,\wt{\bm{\rnrisk}}^2=(\wt{\rnrisk}^2_k)_{k \in \intr{B}} \in\mbr_+^B$ are possibly random but independent of the samples in model~\eqref{eq:mainmodel}.
		Let $\vref$ be some deterministic reference set, such that $1 \in \vref$.
		Let $(\wt{V},\wt{\bm{\rnrisk}}^2)$ be plugged in for $(V,\bm{\rnrisk}^2)$ into
		\eqref{eq:STBoptweights}, giving rise to weight vector $\wt{\omvect}$.
		Conditionally to the event
		\begin{align}\label{eq:eventA}
			\begin{cases}
				\vref \subseteq \wt{V} \subseteq V_\tau,\\
				\abs{\wt{\rnrisk}^2_k - \nrisk_k} \leq \eta \nrisk_k, \text{ for all } k \in \wt{V}, \text{ and some } \eta \in [0,1),
			\end{cases}
		\end{align}
		it holds
		\begin{equation}
			\label{eq:boundtest1}
			\frac{R_1(\wt{\omvect})}{\nrisk_1} \leq
			\paren{\frac{1+\eta}{1-\eta}} \cB\paren[1]{\tau,\nu(\wt{V})}
			\leq \paren{\frac{1+\eta}{1-\eta}} \cB\paren[1]{\tau,\nu(\vref)}.
		\end{equation}
	\end{lemma}
	Comparing the oracle relative risk bound \eqref{eq:oracletestbound} with that of the empirical procedure \eqref{eq:boundtest1},
	note first the requirement that all estimated neighbours are  $\tau$-neighbours ($\wt{V} \subseteq V_\tau$);
	secondly, the oracle risk is deteriorated by two factors: the excess factor $(1+\eta)/(1-\eta) \geq 1$ which quantifies
	what we lose due to estimation of the neighbours' risks; and the replacement of the set of true neighbours by the smaller
	set $\vref$, under the requirement that $\vref \subseteq \wt{V}$ holds. To summarise, we expect the risk of the
	empirical procedure to be close to the oracle risk if (1) the relative estimation error $\eta$ for naive risks is small, and
	(2) we can guarantee the ``sandwiching'' property $\vref \subseteq \wt{V} \subseteq V_\tau$, with $\vref$
	as large as possible; typically we would be satisfied with $\vref = V_{(1-\eps)\tau}$ for a small $\eps$.
	
	The next sections will introduce such estimates and the fulfillment of event \eqref{eq:eventA} under certain conditions, starting with
	the estimation of neighbour tasks.
	
	\subsection{Finding neighbours (known covariances)} \label{se:knownvar}
	For now let us assume (\KC); we will generalise to unknown covariances in the next section.
	Accordingly, the naive risks $\nrisk_k$ are known, so that $\eta=0$ in the context
	of~\eqref{eq:boundtest1}, and we focus
	on the estimation of the set of neighbours.
	We assume that we are doing so using independent ``tilde'' data
	$(\wt{X}_\bullet^{(k)})_{k\in \intr{B}}$ which are drawn from~\eqref{eq:mainmodel} but independent of $(X_\bullet^{(k)})_{k\in \intr{B}}$ (e.g., using sample splitting). For clarity $X_\bullet^{(k)}$ and $\wt{X}_\bullet^{(k)}$ are assumed to be of the same size $N_k$.
	Given the first requirement $\wt{V} \subseteq V_\tau$, it is natural to think of $\wt{V}$ as the output of a multiple test procedure
	(for which the null hypothesis for task $k$ is {\em not} being a $\tau$-neighbour, i.e., $\norm{\Delta_k} > \tau \nrisk_1$).
	
	Our approach is based on recent results for two-sample mean vector testing \citep{BlaFer23}.
	Assume $N_k\geq 2$ for all $k\in \intr{B}$.
	For $k \in \intr{B}\setminus \set{1}$, we form an unbiased estimator for $\norm{\Delta_k}^2$ based on the U-statistics
	\begin{equation}\label{eq:def_testU}
		\wt{U}_{k} := \sum_{\ell \in \set{1,k}}
		\sum_{\substack{i,j=1 \\ i \neq j}}^{N_\ell} \frac{\inner[1]{ \wt{X}^{(\ell)}_i , \wt{X}^{(\ell)}_j }}{N_\ell(N_\ell-1)}
		-  2 \sum_{i=1}^{N_1} \sum_{j=1}^{N_k} \frac{\inner[1]{ \wt{X}^{(1)}_i , \wt{X}^{(k)}_j }}{N_1N_k}\,.
	\end{equation}
	The following proposition is a direct consequence of \citet{BlaFer23}:
	\begin{proposition}
		\label{prop:simplifiedtestresult}
		Assume (\GS), (\KC) hold and let $\alpha\in (0,1)$, $\tau>0$  be fixed.
		\JB{Let $(\wt{X}_\bullet^{(k)})_{k\in \intr{B}}$ be a dataset drawn from \eqref{eq:mainmodel},
			$\wt{U}_k$ as in~\eqref{eq:def_testU}.}
		Let $\wt{T}^{(\tau)}_k$ be given by
		\begin{equation}
			\label{eq:defTk}
			\wt{T}^{(\tau)}_k := \ind{ \wt{U}_k \leq \tau \nrisk_1}.
		\end{equation}
		Put for $ k\in \intr{B}$
		\begin{equation}
			\label{eq:defqtest}
			\taumink
			:= 32 \paren{ \frac{1}{\sqrt{\deamm_1}} +  \frac{\nrisk_k/\nrisk_1}{\sqrt{\deamm_k}}} \log(8\alpha^{-1})\,,
		\end{equation}
		then it holds:
		\begin{align}
			\text{ if } \norm{\mu_1-\mu_k}^2 >
			\tau^+_k
			\nrisk_1: & \qquad \prob{\wt{T}^{(\tau)}_k=1} \leq \alpha; \label{eq:type1err} \\
			\text{ if } \norm{\mu_1-\mu_k}^2 \leq
			\tau^-_k
			\nrisk_1: & \qquad \prob{\wt{T}^{(\tau)}_k=0} \leq \alpha; \label{eq:type2err}
		\end{align}
		where $\tau^\pm_k = \paren[2]{\sqrt{\tau}\pm{\textstyle \sqrt{\taumink}}}_+^2$.
	\end{proposition}
	Equations~\eqref{eq:type1err}-\eqref{eq:type2err}
	can be understood as controls of the type I/II error level for the test of
	$\norm{\Delta_k}^2> \tau_{k}^+ \nrisk_1$ versus the alternative $\norm{\Delta_k}^2\leq \tau_{k}^- \nrisk_1$.
	It is possible to make the original null hypothesis $\norm{\Delta_k}^2> \tau \nrisk_1$ appear
	through notation translation ($\sqrt{\tau} \leftarrow \sqrt{\tau^-_k}$, $\sqrt{\tau^+_k} \leftarrow \sqrt{\tau}$, 
	if we assume additionally $\tau\geq \taumink $).
	We prefer to keep the above more symmetric form, also because the rejection set~\eqref{eq:defTk} has a simple form, used in practice.
	
	The test is able to identify mean differences very accurately relative to the target threshold $\tau \nrisk_1$
	if $\tau \gg \tau_{\min}^k$. Formula~\eqref{eq:defqtest} highlights the crucial role of the effective dimensionality for this minimal threshold of reliable detection. In the simplified
	(\ECSS) setting, this
	threshold is simply of order $1/\sqrt{\deamm_1}$.
	This reflects the known phenomenon that testing is more reliable than estimation in high dimensions; distances that can be detected might be of smaller order than the typical estimation error \JB{(see for e.g. \citealp{baraud2002non,blanchard2018minimax,BlaFer23})}.
	For fixed $\tau$ and increasing dimension,
	the inconclusive gap between the null and the alternative
	vanishes with increasing dimension --- a desirable property given the sandwiching property that we aim for (see~\eqref{eq:eventA}).
	
	In general non-(\ECSS) configurations, we still want to keep $\tau_{\min}^k$ small of order $1/\sqrt{\deamm_1}$.
	In view of the second term in~\eqref{eq:defqtest}, this suggests to
	only consider tasks with $\nrisk_k/\sqrt{\deamm_k} \leq \cteW \nrisk_1/\sqrt{\deamm_1}$ for some constant $\cteW\geq 1$.
	To this aim, denote the set of tasks satisfying this criterion as
	\begin{equation}
		\label{eq:setW}
		W_{(\cteW)} := \set{ k \in \intr{B} : \frac{\nrisk_k}{\sqrt{\deamm_k}} \leq \cteW \frac{\nrisk_1}{\sqrt{\deamm_1}}}
		= \set{ k \in \intr{B} : \frac{\norm{\Sigma_k}_2}{N_k} \leq \cteW \frac{\norm{\Sigma_1}_2}{N_1}},
	\end{equation}
	and correspondingly the set of whittled down neighbours as
	\begin{equation}
		\label{eq:Vtauc}
		V_{\tau,\cteW} := V_\tau \cap W_{(\cteW)} .
	\end{equation}
	\JB{The norm $\norm{\cdot}_p$ denotes the Schatten norm ($\norm{\Sigma}_p^p= \tr \Sigma^p$).} Note that since we are under (\KC), the set $W_{(\cteW)}$ is assumed to be fully known for now. (We will consider estimating it
	in the next section.)
	Then the following corollary makes the obtained sandwiching property explicit:
	\begin{corollary}
		\label{cor:testcor}
		Let $\cteW\geq 1$ be fixed.
		Assume (\GS) and (\KC) hold and let $\alpha\in (0,1)$.
		Then,
		defining
		\[
		\wt{V}_{\tau,\cteW} := \set{ k \in \intr{B}: \wt{T}^{(\tau)}_k=1} \cap W_{(\cteW)}
		\]
		(where $\wt{T}^{(\tau)}_k$ is as in~\eqref{eq:defTk}),
		with probability at least $1-\alpha$ it holds
		\begin{gather}
			\label{eq:testsandwichv}
			V_{\tau^-,\cteW} \subseteq \wt{V}_{\tau,\cteW} \subseteq V_{\tau^+},\\
			\notag  \qquad \text{ where }
			\tau^{\pm} :=\paren[1]{\sqrt{\tau} \pm \sqrt{\cteW\tau_{\min}^\circ}}_+^2,
			\;\; \tau^\circ_{\min} :=64 \log(8B\alpha^{-1})/\sqrt{\deamm_1}.
		\end{gather}
	\end{corollary}
	The sandwiching property~\eqref{eq:testsandwichv} provides a direct link to Lemma~\ref{prop:boundSTB1}.
	More specifically,
	Corollary~\ref{cor:testcor} together with Lemma~\ref{prop:boundSTB1} guarantee with high probability that the bound on the relative risk of the plug-in estimate $\estmu_{\wt{\omvect}}$ of~\eqref{eq:STBoptweights} using the estimated set of neighbours $\wt{V}_{\tau,\cteW}$
	is bounded by $\cB\paren[1]{\tau^+,{\nu(V_{\tau^-,\cteW})}}$ (recall $\eta = 0$ for now because of (\KC), and $\wt{V}_{0} := \set{1}$).
	Furthermore, for fixed $\tau$, if $\deamm_1/(\log B)^2 \rightarrow \infty$
	then $\taumino$ vanishes and it holds $\tau^- \approx \tau \approx \tau^+$.
	Under (\ECSS), we can simply take $\cteW=1$ and have $V_{\tau^-,\cteW} = V_{\tau^-}$, ensuring a relative risk very close to the oracle $\cB\paren[1]{\tau,{\nu(V_{\tau})}}$. In a general context, there is an additional trade-off through the choice of the constant $\cteW$.
	In both cases, closeness to the oracle relative risk {\em improves} with increasing effective dimensionality.
	
	\subsection{Unknown covariances} \label{se:unknowncov}
	In a realistic setting the covariances are unknown, especially in high dimensions.
	In this section, we estimate all quantities relevant for the fulfilment of Lemma~\ref{prop:boundSTB1}, using the
	same independent ``tilde'' data $(\wt{X}_\bullet^{(k)})_{k\in \intr{B}}$ as in the previous section.
	Observe that it is not necessary to estimate the full covariance matrices $\Sigma_k$, but only scalar quantities related to their Schatten norms.
	In particular, in the Gaussian setting we have the following result for the natural unbiased estimators of $\nrisk_k$:
	\begin{proposition}
		\label{prop:estnrisk}
		\JB{Let $(\wt{X}_\bullet^{(k)})_{k\in \intr{B}}$ be a dataset drawn from \eqref{eq:mainmodel} and} 
		\[\wt{\rnrisk}^2_k := \frac{1}{N_k(N_k-1)} \sum_{i=1}^{N_k} \norm[1]{\wt{X}^{(k)}_i - \wt{\mu}^\NE_k}^2 ,\]
		where $\wt{\mu}^\NE_k:=N_k^{-1} \sum_{i=1}^{N_k} \wt{X}^{(k)}_i $, and let $\alpha \in (0,1)$.
		Assume (\GS) holds. Then with probability at least $1-\alpha$:
		\begin{equation}
			\forall k \in \intr{B}: \qquad
			\abs{\wt{\rnrisk}^2_k - \nrisk_k}
			\leq \paren{4 \sqrt{2} \frac{\log (2B\alpha^{-1})}{\sqrt{\deamm_kN_k}}  } \nrisk_k.
		\end{equation}
	\end{proposition}
	When $N_k \gtrsim  \log^2(2B \alpha^{-1})$ for all $k$, the estimation of $\nrisk_k$ has relative accuracy of order $1/\sqrt{\deamm_k}$ with probability $(1-\alpha)$.
	This finding can be used for the fulfillment of the second requirement of condition~\eqref{eq:eventA}.
	It also allows to preserve the qualitative results of Proposition~\ref{prop:simplifiedtestresult} (up to numerical factors) for test~\eqref{eq:defTk} wherein  $\wt{\rnrisk}^2_1$ is plugged in
	for $\nrisk_1$.
	Finally, we also replace $\norm{\Sigma_k}_2$ in the definition~\eqref{eq:setW} of set $W_{(\cteW)}$ by suitable estimators;
	Proposition~\ref{prop:pluginall} in the Supplemental gives the details.
	It provides a quantitatively precise version of the sandwiching property
	analogous to~\eqref{eq:testsandwichv} with all unknown quantities are replaced by their proposed estimators. 
	
	We combine the obtained results in an illustrative example. 
	It shows a fully empirical algorithm that approximates the (whittled down) oracle $B(\tau,V_{\tau,\cteW})$
	(numerical constants are made explicit for concreteness but not meant to be sharp):
	\begin{theorem} \label{prop:fullemp}
		Assume (\GS) holds. \JB{Let  $(X_\bullet^{(k)})_{k\in \intr{B}}$ and $(\wt{X}_\bullet^{(k)})_{k\in \intr{B}}$ be two independent datasets drawn from \eqref{eq:mainmodel}} and $\alpha \in (0,1/3)$.
		Consider the following plug-in versions of the quantities appearing in~\eqref{eq:defTk}, \eqref{eq:setW}:
		\begin{equation}
			\label{eq:setWandtkdtilde}
			\wt{W}_{(\cteW)} :=  \set{ k \in \intr{B} : \frac{\wt{Z}_{k}^{(2)}}{N_k} \leq \cteW \frac{\wt{Z}_{1}^{(2)}}{N_1}}, \qquad
			\dwt{T{\text{\tiny}}}^{(\tau)}_k := \ind{ \wt{U}_k \leq \tau \wt{\rnrisk}^2_1},
		\end{equation}
		where $\wt{\rnrisk}^2_k$ as in Prop.~\ref{prop:estnrisk}, and $\wt{Z}_{k}^{(2)}$ estimates $\norm{\Sigma_k}_2$ as defined in~\eqref{eq:def_t_trsigma2} in the Supplemental.
		Define the set of estimated $\tau$-neighbours
		\begin{equation}
			\label{eq:setVdtilde}
			\dwt{V}_{\tau,\cteW} := \set[2]{ k \in \wt{W}_{(\cteW)}: \dwt{T{\text{\tiny}}}_k^{(\tau)}=1}.
		\end{equation}
		Assume $N_k \geq a(4 + \log(2B\alpha^{-1}))^4$ for all $k\in\intr{B}$,
		for a big enough numerical constant $a$
		($a=4400$ works).
		For fixed $\tau>0, \cteW\geq 1$, consider the
		weights $\wt{\omvect}^{\sharp}$ obtained by the {\em modified} plug-in $\paren[1]{\dwt{V}_{\wt{\tau},3\cteW},\wt{\bm{\rnrisk}}^2}$ for
		$(V,\bm{\nrisk})$ in~\eqref{eq:STBoptweights}, where
		\begin{equation} \label{eq:altplugin}
			\wt{\tau} := \paren{1 + \frac{1}{30{\textstyle \sqrt{\wt{\deamm_1}}}}}\paren{\sqrt{\tau}+ \sqrt{6\cteW\wt{\tau}_{\min}^{\circ}}}^2; \;\;\;\;\;
			\wt{\tau}_{\min}^{\circ} := \frac{64 \paren[1]{\log (8B\alpha^{-1})}}{\sqrt{\wt{\deamm_1}}}; \;\;\;\;\;
			\sqrt{\wt{\deamm_1}} := \frac{N_1 \wt{\rnrisk}^2_1}{ \wt{Z}_1^{(2)}}.
		\end{equation}
		Then with probability at least $1-3\alpha$ over the draw of the ``tilde'' sample $(\wt{X}_\bullet^{(k)})_{k\in \intr{B}}$, it holds
		\[
		\frac{R_1(\wt{\omvect}^\sharp)}{\nrisk_1} \leq
		\paren{1 + \frac{1}{7 \sqrt{\min_k \deamm_k}}}
		\paren{1+\frac{56\sqrt{\cteW\log(8B\alpha^{-1})}}{(\deamm_1)^{\frac{1}{4}}\sqrt{\tau}}}^2 \cB\paren[1]{\tau,\nu(V_{\tau,\cteW})},
		\]
		where the expected risk is with respect to the main sample $({X}_\bullet^{(k)})_{k\in \intr{B}}$.
	\end{theorem}
	
	\subsection{Discussion} \label{se:disctest}
	To summarise, for fixed values of $\tau,\cteW,B,(N_k)_{k \in \intr{B}}$, the bound on the relative risk of $\wt{\omvect}^\sharp$
	becomes arbitrarily close to the oracle bound in the high-dimensional asymptotics $\deamm_1 \rightarrow \infty$.
	We stress that this applies for {\em fixed} sample sizes
	$N_k$, provided $N_k \gtrsim  \log^4 B$. Consequently, the fully empirical procedure is (with high probability)
	not worse than the naive estimator \JB{on the main sample} up to a risk factor very close to 1 (since the oracle bound $\cB$ is always less than 1), and potentially
	performs much better if there are many true $\tau$-neighbouring tasks (again, as reflected by the oracle factor).
	The conclusion still holds true if $\tau,\cteW,B,(N_k)$ vary with $\deamm_1$ ($\tau \rightarrow 0$ and/or $B\rightarrow \infty$ being the most interesting situations)
	provided $\cteW\log(B)/\sqrt{\deamm_1} = o(\tau)$ holds and as $N_k \gtrsim  \log^4 B$ as before.

	\paragraph{\bf Beyond the Gaussian setting.} The results presented above hold under the Gaussian distributional assumptions (\GS). However, the required components --- specifically, concentration of estimators for distances between two means and for Schatten norms of the covariances ---
	can be extended with appropriate modifications to the bounded (\BS) and heavy-tailed (\HT)
	distributional settings. Detailed results are presented
	in Supplemental~\ref{ap:est_schatten_norm} and show the qualitative robustness
	of our approach beyond the Gaussian setting.

	\paragraph{\bf Beyond the testing approach.}
	The testing
	approach has two flaws. \GB{First, the theory requires to have two independent copies of the samples from the model~\eqref{eq:mainmodel}.
		In practice, this can be achieved by equally partitioning an initial dataset. This aspect has been brushed over in the preceding discussion, where the oracle risk improvement was with respect to the naive estimator using the main sample only. It is fairer to compare to the naive estimator on the joined
		(main+tilde) sample, in which case the risk improvement suffers an additional factor 2.
		This factor can be made asymptotically close to 1 in the dimension asymptotics by resorting to carefully chosen
		unequal sample splitting, but we do not pursue this here for simplicity. In our practical experiments, we deviate
		from the theoretical setting, computing all relevant quantities, as well as the final estimator on the same sample (see Section~\ref{se:application} and Supplementals~\ref{sec:apx_mnist} and \ref{sec:apx_methods}), which does not to affect the
		good performance of the method.}


Secondly, the issue of parameter selection of $\tau$ and $\cteW$ persists. As previously elucidated, the oracle relative risk $\cB$ exhibits a bias-variance trade-off: the aggregated variance decreases with an increase in the number of $\tau$-neighbours, consequently, with the worst-case relative bias $\tau$. Ideally, parameters should be adaptively chosen to strive for optimal oracle improvement $\min_{\tau \geq 0,\cteW \geq 1}\cB\paren[1]{\tau,\nu(V_{\tau,\cteW})}$. \GB{In practice, we resort to cross-validation but this is inconvenient.} The next section introduces an alternative approach pursuing this objective \GB{from a theoretically grounded point of view}. Additionally, Section~\ref{se:minimax} analyses whether the derived bounds are optimal.

\section{A ``$Q$-aggregation'' approach}
\label{se:Qaggreg}
In this section, we propose an alternative approach for forming the weights of the convex combination estimator~\eqref{eq:defmuhatomega}.
The weights are found by direct minimization of an upper confidence bound of the risk $R_1(\omvect)$, i.e.,
\begin{equation}\label{eq:def_w}
	\widehat{\omvect} \in \arg\min_{\omvect \in \Spx_B} \paren{\widehat{L}_1(\omvect) + u \widehat{Q}_1(\omvect)}\,,
\end{equation}
where $\wh{L}_1(\bm{\om})$ is an unbiased estimate of the risk.
The idea of this scheme bears resemblance to $Q$-aggregation \citep{LecRig14}, because the objective function will be
a quadratic function of $\bm{\om}$. The objective aims at taking into account all individual distances between the bags, rather
than selecting those less than a fixed threshold as in the testing approach.
The penalization term $\wh{Q}_1(\bm{\om})$ shall be a high probability upper bound on the difference between estimated and true {\em loss}
$(\wh{L}_1(\omvect) - L_1 (\omvect))$.
Observe that the penalization term also depends on the weight vector, since giving more weight to tasks that are further away
from the target (large $\norm{\Delta_k}$) will result in a larger variability of the risk estimate $\wh{L}_1(\omvect)$.
The parameter $u$ is a calibration constant.
Compared to the testing approach, one advantage is that it is not necessary to choose the parameters $\tau$ and $\cteW$.
Furthermore no sample splitting is needed.
On the other hand, the procedure is more computationally demanding since there is no closed form solution to~\eqref{eq:def_w}.
Instead, a solution $\wh{\bm{\om}}$ can be obtained by exponentiated gradient descent on the simplex \citep{kivinen1997exponentiated}.

We present specific choices for $\widehat{L}_1(\omvect), \widehat{Q}_1(\omvect)$ and an analysis of the relative risk of the resulting $Q$-aggregation estimator for (\GS) in Section~\ref{sec:QaggGS} and for (\BS) thereafter.
In contrast to \citet{LecRig14}, we focus on the effect of the dimension rather than that of the sample size, which provides a novel analysis.

\subsection{Gaussian setting}\label{sec:QaggGS}
Under assumption (\GS) we propose to use the following estimates to form the $Q$-aggregation estimator:
\begin{gather}
	\wh{L}_1(\omvect) = \norm[3]{ \sum_{k=2}^{B} \om_k (\muNE_k - \muNE_1) }^2 + (2\om_1-1) \whnrisk_1\,, \label{eq:defwhL}\displaybreak[0]\\
	\whnrisk_1 := \frac{1}{N_1(N_1-1)} \sum_{i=1}^{N_1} \norm{ X_i^{(1)} - \wh{\mu}_1^{\NE} }^2\, , \label{eq:def_whnrisk}\displaybreak[0]\\
	\widehat{Q}_1(\omvect) := \sum_{k=2}^{B} \omega_k \sqrt{\frac{\widehat{q}_{k}}{N_1}}, \quad \text{where} \quad \widehat{q}_{k} := \frac{1}{N_1-1} \sum_{i=1}^{N_1} \inner{ \muNE_1 - \muNE_k, X_i^{(1)} - \muNE_1 }^2\,. \label{eq:defwhq}
\end{gather}
It can be checked easily that $\whnrisk_1$ is an unbiased estimator of the naive risk $\nrisk_1$, and that
the estimator $\wh{L}_1(\omvect)$ is an unbiased estimate of the {\em conditional} risk $\e[1]{\wh{L}_1(\omvect) - L_1(\omvect)\big| X^{(k)}_\bullet, k \geq 2} = 0$. \JB{The estimator $\wh{Q}_1(\omvect)$ is supposed to upper bound the deviations of $\wh{L}_1(\omvect)$ around $L_1(\omvect)$, which are of order $\frac{1}{\sqrt{N_1}}\sum_{k=2}^{B}\omega_k \sqrt{(\wh{\mu}_k-\mu_1)^T\Sigma_1(\wh{\mu}_k-\mu_1)}$.}
With these choices we establish the following result for the relative risk of the $Q$-aggregation estimator:
\begin{theorem}\label{prop:qaggreg_gauss}
	Assume (\GS) holds, and let
	$u_0 \in \mbr_+$ be fixed such that 
	$\log (17B) \leq u_0 \leq (N_1-1)/2$.
	With $\wh{L}_1(\omvect)$ and $\wh{Q}_1(\omvect)$ as defined in \eqref{eq:defwhL},\eqref{eq:defwhq}, let
	\begin{equation}\label{eq:def_whomega}
		\wh{\omvect} \in \arg\min_{\omvect \in \Spx_B} \paren{ \wh{L}_1(\omvect) + 16 \sqrt{u_0}\, \wh{Q}_1(\omvect)}.
	\end{equation}
	Then it holds:
	\begin{equation} \label{eq:Qoracle}
		\frac{R_1(\wh{\omvect})}{\nrisk_1} \leq \frac{1}{\nrisk_1}\min_{\omvect \in \Spx_B} \brac[2]{{R_1(\omvect)}{}(1+ CBe^{-u_0/2})
			+ C{Q_1(\omvect)}{} \sqrt{ u_0 }}  + C\frac{u_0 }{\sqrt{\deamm_1}},
	\end{equation}
	where $C>0$ is an absolute constant, and (recalling $\Delta_k = \mu_k - \mu_1$)
	\begin{equation}\label{eq:def_Qw}
		Q_1(\omvect) := \sum_{k=2}^{B} \omega_k \sqrt{\frac{q_k}{N_1}}, \quad \text{with} \quad
		q_k := \Delta_k^T \Sigma_1\Delta_k + \frac{\tr \Sigma_1\Sigma_k}{N_k}\,.
	\end{equation}
\end{theorem}
The above bound~\eqref{eq:Qoracle} has the form of an ``oracle inequality'' relating the relative risk of the $Q$-aggregation approach to the minimum
of the attainable relative risk of any aggregation estimator with fixed weight $\omvect$ but with a penalization term $Q_1(\omvect)$.
The extra additive term (outside the minimum) vanishes in high effective dimension, but indicates that the relative risk bound cannot be better than $O(\log B/\sqrt{\deamm_1})$. \JB{The optimality of this term is examined in numerical experiments in Supplemental~\ref{se:optimality_uppbound}.} We also emphasise the requirement $\log B \lesssim N_1$ implicit in the condition on the calibration parameter $u_0$.
The effect of the penalization term $Q_1(\omvect)$ on the oracle bound~\eqref{eq:Qoracle} might appear obscure: depending on the weights
$\omvect$, the penalization might outweigh the main risk term $R_1(\omvect)$. It is noteworthy that this term penalises tasks with distant means (term $\Delta_k^T \Sigma_1\Delta_k$) or with high variance (term $\tr \Sigma_1\Sigma_k/N_k$).
To provide further clarification, we present the following corollary which bounds the relative risk of the $Q$-aggregation method
in terms of the relative risk of the oracle testing approach $\cB(\tau,\nu)$:
\begin{corollary}\label{cor:aistats}
	Assume (\GS) holds.
	Let $u_0 \in \mbr_+$ be fixed, such that $\log 17B \leq u_0 \leq (N_1-1)/2$, and $\wh{\omvect}$ as defined in \eqref{eq:def_whomega}.
	Then it holds:
	\begin{equation}\label{eq:aistats}
		\frac{R_1(\wh{\omvect})}{\rnrisk_1^2} \leq \paren{1+ CBe^{-u_0/2}}\inf_{\substack{\tau \geq 0 \\ \cteW\geq 1}} \brac{\cB\paren[1]{\tau,\nu(V_{\tau,\cteW})} + C \cteW \sqrt{\frac{u_0 }{\deff_1}} }\,.
	\end{equation}
	where $C>0$ is an absolute constant, $\cB(.,.)$, $\nu(.)$ are as defined in \eqref{eq:defBtaunu}, \eqref{eq:def_s2v} and $V_{\tau,\cteW}$ as in~\eqref{eq:setW}-\eqref{eq:Vtauc}.
\end{corollary}
As a simple illustration, assume the tasks satisfy (\ECSS) and have equal means ($\mu_k = \mu_1$ for $k\in \intr{B}$), but the estimator does not have this information.
The oracle merges all tasks and has relative risk $\inf_{\tau,\cteW} \cB\paren[1]{\tau,\nu(V_{\tau,\cteW})} = B^{-1}$ for $\tau \rightarrow 0, \cteW = 1$.
For $u_0 = \log 17B$, the relative risk of the $Q$-aggregation method \eqref{eq:aistats} becomes
\begin{equation*}
	\frac{R_1(\wh{\omega})}{\rnrisk_1^2} \leq  C \max \left\lbrace \frac{1}{B} , \sqrt{\frac{\log B}{\deff_1}}\right\rbrace \,, 
\end{equation*}
where $C \approx 1$ if $\deff_1$ and $B$ are large.
We observe again a blessing of dimensionality; the best improvement is obtained when $\deff_1$ is high ($\deff_1 \geq B^2\log B$ ensures
a relative risk bound of order $1/B$, which is the best improvement even if the information of equal means had been known).

\subsection{Comparison with the testing  approach} \label{se:discQagg}
Let us compare the bounds obtained for the test method (Theorem~\ref{prop:fullemp}) to that for the $Q$-aggregation approach (Corollary~\ref{cor:aistats}),
in high-dimensional asymptotics $\deamm_1, \deff_1 \rightarrow \infty$. \JB{It is not clear if the appearance of $\deff$ in \eqref{eq:aistats}, instead of $\deamm_1$, is fundamental or an artifact of the proof.}\footnote{\GB{The quantity $\sqrt{\deamm} = \norm{\Sigma}_1/\norm{\Sigma}_2$
		can be interpreted as the ratio of the estimation error to the optimal testing separation squared distance for vectors in high dimension (see \citealp{BlaFer23}).
		Its inverse is an intuitively natural notion to appear in a relative risk bound in our setting.}
} We start with an analysis of the conditions on the other parameters $\lbrace \tau,\cteW,B,(N_k)_{k \in \intr{B}} \rbrace$ under which the obtained bounds guarantee that
the relative risk of either method is bounded by the oracle bound $\cB\paren[1]{\tau,\nu(V_{\tau,\cteW})}$
up to a factor asymptotically converging to 1, a property which we call ``oracle-consistency'' for short.

Recall from Section~\ref{se:disctest} that the relative risk of the test method
is oracle-consistent
(as $\deamm_1 \rightarrow \infty$), provided $\cteW \log(B)/\sqrt{\deamm_1} = o(\tau)$ and $N_k \gtrsim  \log^4 B$ hold.
Aside from these conditions the parameters $\tau,\cteW,B,(N_k)$ can vary with $\deamm_1$.
On the other hand, \eqref{eq:aistats} shows that the $Q$-aggregation method is
oracle-consistent (as $\deff_1 \rightarrow \infty$) with respect to {\em any} $(\tau,\cteW)$ provided that $N_1 \gtrsim \log(B\deff_1)$,
and $\cteW\sqrt{\log(B\deff_1)/\deff_1} = o(\tau)$ (taking $u_0 =  2\log B\deff_1$). The additive terms in~\eqref{eq:aistats} are then
negligible compared to $\cB(\tau,\nu)$, due to $\cB(\tau,.)\geq \tau/(1+\tau)$.
Note also that it does not require any condition on $N_k$ for $k\neq 1$.

If $\deamm_1$ and $\deff_1$ are of the same order (e.g. in the isotropic setting), the above parameter conditions for
consistency of either method are very similar with only minor differences. One such difference is that
the test method is guaranteed to be oracle-consistent even if $B,\tau,\cteW,(N_k)$ are
fixed, i.e., must not change as $\deamm_1 \rightarrow \infty$; while we require $N_1\rightarrow \infty$
(though only at a logarithmic rate in $B,\deamm_1$) to
warrant oracle consistency of the aggregation estimator.
If $\deff_1$ is of order
$\sqrt{\deamm_1}$ (for example, for a slow power decrease of the eigenvalues $\lambda_i$ of the covariance, $\lambda_i = i^{-\alpha}$ for $1\leq i \leq d$ and $\alpha \in (1/2,1)$), then the oracle consistency conditions for the $Q$-aggregation method are narrower.

Still, one has to keep in mind that oracle-consistency for the testing approach only holds for the specific parameters $(\tau,\cteW)$ that must be provided by the user,
while the $Q$-aggregation method is oracle consistent with respect to any choice $(\tau,\cteW)$ satisfying the delineated conditions. In other words, the relative risk of the $Q$-aggregation method qualitatively enjoys the same asymptotic guarantees
as the testing approach with {\em optimally selected} $\tau$ and $\cteW$ subject to the above conditions \GBT{(statistical adaptivity property)}. This and the fact that the $Q$-aggregation
does not use data splitting is a strong argument in its favour.
On the other hand, the testing method has the advantage of being more flexible and easily adapts to non-Gaussian distributions, e.g., bounded or heavy-tailed distributions (see Supplemental~\ref{ap:est_schatten_norm}). With a modification of the penalization term, the $Q$-aggregation method can also be applied to bounded distributions, as shown next, but it currently does not accommodate heavy-tailed data distributions.

\HM{Supplemental \ref{sec:apx_mnist} provides a comparison of the proposed testing and $Q$-aggregation methods to the James-Stein estimator in terms of their empirical risk on illustrative toy data.}

\subsection{Bounded setting}\label{sec:QaggBS}
Our results for the $Q$-aggregation estimator can be extended to the bounded setting (\BS) where the data lie in a ball of radius $M$ centred in~$0$.
A precise value for $M$ is often known.
For example,
if the data lies in a reproducing kernel Hilbert space associated with a bounded kernel, $M^2$ will be the bound on the kernel.
The methodology we propose for (\BS) closely resembles the one outlined for the Gaussian setting.
It utilises the same estimates,~\eqref{eq:defwhL}-\eqref{eq:def_whnrisk}-\eqref{eq:defwhq}, for the risk estimation and its deviations.
In order to compensate the lack of regularity of bounded compared to Gaussian data, an additional penalization term $\wh{Q}^\BS(\omega)$ is introduced, which depends on $M$.

\begin{theorem}\label{prop:qaggreg_bounded}
	Assume (\BS).
	Let $u_0 \in \mbr_+$ with $2\log N_1 + \log (B)\leq u_0\leq \sqrt{N_1}$, and
	\begin{equation}\label{eq:def_w_bnd}
		\wh{\omvect} \in \arg\min_{\omvect \in \Spx_B}
		\paren{\wh{L}_1(\omvect) + 4\sqrt{2u_0} \wh{Q}_1(\omvect) + C_0 u_0 \wh{Q}_1^\BS(\omvect)},
	\end{equation}
	where $\wh{L}_1, \wh{Q}_1$ are defined in~\eqref{eq:defwhL},~\eqref{eq:defwhq} resp., $C_0 >1424$ works, and
	\begin{equation}\label{eq:def_whQBS}
		\wh{Q}_1^\BS(\omvect) = \frac{M}{N_1} \sum_{i=2}^B \omega_i \norm{\muNE_i-\muNE_1}\,.
	\end{equation}
	Assume $N_k \geq (\deamm_k)^\beta$ for some $\beta >0$ and all $k\in \intr{B}$, then:
	\begin{equation}
		\frac{R_1(\wh{\omvect})}{\nrisk_1} \leq \min_{\tau>0,\cteW\geq1} \paren[1]{ \cB\paren[1]{\tau,\nu\paren{V_{\tau,\cteW}}} + C\cteW \varepsilon}+ C \ratiobs_1 \varepsilon\,,
		\qquad \varepsilon := \max \left\lbrace \sqrt{\frac{u_0}{\deff_1}} \, , \, \frac{u_0}{(\deamm_1)^{\beta/2}}\right\rbrace,
	\end{equation}
	where $\cB(\cdot,\cdot)$, $\nu(\cdot)$ are as defined in \eqref{eq:defBtaunu}, \eqref{eq:def_s2v}, $V_{\tau,\cteW}$ as in~\eqref{eq:setW}-\eqref{eq:Vtauc}, $C$ an absolute constant, and $\ratiobs_1:= M^2/\tr \Sigma_1$.
\end{theorem}
The quantity $\beta$ reflects the trade-off between the requirement on the number of samples and the rate of convergence to the oracle bound.
A bound similar to that in the Gaussian case will only be obtained if a stricter condition on the bag sizes is met ($N_k \gtrsim \deamm_k$ instead of $N_1 \gtrsim \log{B}$ as in Corollary~\ref{cor:aistats}).
In contrast to~\eqref{eq:aistats}, there is no multiplicative constant in front of the bound, however, the additive term now involves the quantity
$\ratiobs_1$ (see Supplemental~\ref{ap:notes_gamma} for a discussion of this
quantity in the framework of kernel mean embedding (KME) estimation with a bounded kernel, which is our
primary motivation for analyzing the bounded setting).

\section{Minimax results}\label{se:minimax}
This section explores if the oracle relative risk upper bound $\cB\paren{\tau, \nu(V_\tau)}$ as defined in~\eqref{eq:oracletestbound}, which has been utilised as benchmark in previous sections, is optimal in a minimax sense. As before, we will first examine the estimation of a single mean. Subsequently, we extend the analysis to the compound relative risks averaged over tasks.

Our aim is to establish minimax bounds matching the upper bounds over distribution classes that are as restrictive as possible. Since a minimax lower bound on a distribution class also applies to every superclass containing it, bounds on restrictive classes are more insightful. 
To achieve this, we narrow down the distribution classes by fixing as many parameters as possible to arbitrary values. As employed throughout this manuscript, we will adopt a high-dimensional asymptotics viewpoint and focus on minimax statements as the effective dimension grows large.

\subsection{Single task relative risk}
We derive a lower minimax bound for a class of distributions that closely match the assumptions proposed
to introduce the oracle bound~\eqref{eq:oracletestbound}: a known subset of $\tau$-neighbours $V$ in arbitrary position, all other
parameters (sample sizes, covariances, \ldots) being fixed.
We additionally assume that all task covariance matrices are proportional to each other ("aligned"), which appears to
be the least favourable setting.
\begin{definition} \label{def:modelsingle}
	Let $\tau \in \mbr_+; B,V \in \mbn_{>0}$ with $B\geq V$, $\bm{\nrisk}=(\nrisk_1,\ldots,\nrisk_B)\in \mbr_+^B$,
	$(N_k)_{k \in \intr{B}} \in \mbn_{>0}^B$, and $\Sigma$
	a symmetric positive definite matrix of size $d$ with $\tr \Sigma=1$ be fixed. We denote by $\singlemodel(\tau,V,\Sigma,\bm{\nrisk})$ the set of joint distributions for tasks following model~\eqref{eq:mainmodel} such that:
	\begin{enumerate}[label=(\roman*)]
		\item The total number of bags is $B$ and the number of samples per bag is given by
		$(N_k)_{k\in \intr{B}}$. (Omitted from the distribution class notation for simplicity.)
		\item (\GS) holds.
		\item The task covariances are given by $\Sigma_k = N_k \nrisk_k \Sigma$ (i.e., all tasks have covariances proportional to $\Sigma$
		and the naive risks are specified by the vector $\bm{\nrisk}$).
		\item The mean vectors $(\mu_k)_{k \in \intr{B}}$ can vary freely subject to:
		\[
		\norm{\mu_1 - \mu_k}^2 \leq \tau \nrisk_1, k \in \intr{V}.\]
	\end{enumerate}
\end{definition}
A minimax lower bound, as by Theorem~\ref{prop:lowbnd_1} below, over that model holds over any larger model;
for instance, the model where $\Sigma_1$ is arbitrarily fixed and the other covariances may vary freely provided that the naive risks still match the prescribed $\bm{\nrisk}$.
\begin{theorem}\label{prop:lowbnd_1}
	It holds
	\begin{equation*}
		\inf_{\wh{\mu}_1} \sup_{\mbq \in \singlemodel(\tau,V,\Sigma,\bm{\nrisk})} \frac{R_1(\mbq,\wh{\mu}_1)}{\nrisk_1} \geq  \cB\paren[1]{\tau,\nu(\intr{V})} - \eps\paren[1]{\deff(\Sigma)} ,
	\end{equation*}
	where $\cB$ is defined in \eqref{eq:oracletestbound}, $\nu$ in~\eqref{eq:def_s2v},
	the infimum is over all estimators $\wh{\mu}_1$
	for $\mu_1$,
	and $R_1(\mbq,\wh{\mu}_1)$ indicates its risk~\eqref{eq:lossrisk} under distribution $\mbq$.
	The function $\eps(t)$ is {\em independent of any parameters} and satisfies $\eps(t) = O( (\log t)/ t)$
	as $t \rightarrow \infty$.
\end{theorem}
This minimax lower bound can be compared with the upper bounds obtained for the testing and $Q$-aggregation methods, Theorem~\ref{prop:fullemp} and Corollary~\ref{cor:aistats}, resp.
In the case of (\ECSS) (so that $V_{\tau,\cteW} = V_\tau$ for any $\cteW \geq 1$ and we can ignore the role of $\cteW$), the lower and upper bounds match.
This shows that the oracle relative risk $\cB(\tau, \nu(V_\tau))$ is indeed
minimax in the sense of high-dimensional asymptotics, provided that $\log(B) = o(\deff_1)$. Furthermore, the $Q$-aggregation method is {\em asymptotically minimax adaptive} over the parameter $\tau>0$. This can be seen as
a generalization of classical results on the James-Stein estimator
(see Supplemental~\ref{se:JS}). Observe also that for the upper and lower bounds the dimension-dependent remainder terms do not depend on other parameters, which makes the dimensional asymptotics uniform with respect to those parameters.

If (\ECSS) does not hold, there can be a discrepancy between the minimax lower bound and the obtained upper bounds due to the exclusion of high variance tasks in the latter ($V_\tau$ against $V_{\tau,\cteW}$). An unfavourable regime
illustrating this gap is the following:
suppose there are many tasks that are $\tau$-neighbours of the target ($\tau$ being fixed
independently of the dimension but $V \approx \deamm_1$) with significantly higher variances though ($\cteW=\nrisk_k / \nrisk_1  \approx V^{1/2}$ for all $2 \leq k\leq V$). In that scenario, the upper bounds of Theorem~\ref{prop:fullemp} and Corollary~\ref{cor:aistats} do not guarantee convergence to $\cB(\tau,\nu(V_\tau)) \approx \tau/(1+\tau)$, since
the remainder terms $\cteW/\sqrt{\deamm_1}$ (resp. $\cteW/\sqrt{\deff_1}$) do not converge to zero for
high-dimensional asymptotics.
This gap can amount to an arbitrary large factor since $\tau$ can be
arbitrarily small.
However, the scenario where a target task is surrounded by numerous neighbours with significantly higher variance can only arise for a small proportion of the tasks. This implies that this concern is alleviated when evaluating the relative risk averaged across all tasks, as shown next.

\subsection{Compound relative risk}
\label{sec:avgrisk}
We define the compound relative risk as the relative risk averaged over all tasks.
As we only studied upper bounds for a single task so far, we first derive new upper bounds
for the compound relative risk.
We then proceed to derive minimax bounds on restrictive distribution classes under which the task means exhibit a certain
clustering or covering structure.
\begin{definition}\label{def:clustering}
	Let $\muvect = (\mu_k)_{k \in \intr{B}}$ be a collection of vectors of $\mbr^d$, $J \in \mbn_{>0}$, and $\Cvect$ a $J$-partition of $\intr{B}$ (i.e., $\Cvect = (\cC_j)_{j \in \intr{J}}$ with $\cC_1 \sqcup \ldots \sqcup \cC_J = \intr{B}$). The diameters of the partition $\Cvect$ applied to
	$\muvect$ are defined as:
	\begin{equation}\label{eq:def_diameter}
		\diam(\Cvect,\muvect) = \paren{ \max_{k,\ell \in \cC_j} \norm{\mu_k - \mu_\ell} }_{j \in \intr{J}} \in \mbr^J_+ .
	\end{equation}
\end{definition}
We shall refer to parts as ``groups'' rather than clusters, because the partitioning can in principle be arbitrary. 
However, the intuition is that the set of vectors $\muvect$ exhibits more structure if it can be partitioned into a limited number of groups with small diameter. For instance, if it is strongly clustered, or supported on
a set of small metric entropy such as a low-dimensional manifold. 
The compound relative risk of the $Q$-aggregation approach can then be upper bounded from Theorem~\ref{prop:qaggreg_gauss} as follows:
\begin{corollary}\label{prop:uppbnd_all}
	Assume (\GS) holds, and let $u_0 \in \mbr_+$ such that $\log 17B\leq u_0 \leq (\min_k N_k-1)/2$.
	For $k \in \intr{B}$, define $\wh{L}_k(\omvect), \wh{Q}_k(\omvect)$ analogously to~\eqref{eq:defwhL},\eqref{eq:defwhq} and
	\begin{equation}\label{eq:def_omega_cpd}
		\wh{\omvect}_k \in \arg\min_{\omvect \in \Spx_B} \paren{ \wh{L}_k(\omvect) + 16 \sqrt{u_0}\, \wh{Q}_k(\omvect)}.
	\end{equation}
	Then it holds:
	\begin{equation} \label{eq:compoundoracle}
		\frac{1}{B} \sum_{k=1}^{B} \frac{R_k(\wh{\omvect}_k)}{\rnrisk_k^2} \leq   \paren{ 1+ CBe^{-u_0/2}} \min_{\Cvect } \paren{\cL^*\paren[1]{\bm{\rnrisk}, \Cvect,\diam (\Cvect,\muvect)} + C  \frac{u_0 }{\min_{k \in \intr{B}} (\deff_k)^{\nicefrac{1}{2}}}} , 
	\end{equation}
	where the minimum is taken over all partitions $\Cvect $ of $\intr{B}$, 
	$C$ is an absolute constant, and for $\zetavect\in \mbr^J_+$:
	\begin{equation}\label{eq:defLstar}
		\cL^*\paren{\bm{\rnrisk}, \Cvect,\bm{\zeta}} := \frac{1}{B} \sum_{j =1}^{J} \sum_{k \in C_j}\cB\paren{ \tau_{j,k}, \nu_{j,k}}, \qquad \tau_{j,k}:=\frac{\zeta_j^2}{\nrisk_k}, \;\; \nu_{j,k} :=
		\frac{\nrisk(\cC_j)}{\nrisk_k},
	\end{equation}
	and $\cB$ is defined in \eqref{eq:oracletestbound}.
\end{corollary}
Similarly to the estimation of a single mean, the bound on the compound relative risk depends on the maximum distance between tasks of the same group relative to the naive risk of
each task, and on the relative aggregated variances \eqref{eq:def_s2v} in each group.
Remarkably, the compound relative risk bound does not involve any ``whittling down''
of high-variance tasks as in the single task bound~\eqref{eq:Vtauc}, and holds under arbitrary
inhomogeneity of the tasks and sample sizes.

The quantity $\cL^*$ equates to an oracle compound relative risk and is minimax under high-dimensional asymptotics.
To show this, we extend the single task model~\ref{def:modelcompound} to a joint distribution class such that the tasks are divided into inhomogeneous groups.
\begin{definition}\label{def:modelcompound}
	Let $B \in \mbn_{>0}, \bm{\nrisk}=(\nrisk_1,\ldots,\nrisk_B)\in \mbr_+^B$,
	$(N_k)_{k \in \intr{B}} \in \mbn_{>0}^B$, and $\Sigma$
	a symmetric positive definite matrix of size $d$ with $\tr \Sigma=1$ be fixed.\\
	Let $J \in \mbn_{>0}$, $\Cvect$ be a $J$-partition of $\intr{B}$ and $\zetavect\in \mbr_+^J$.
	We define $\multimodel(\Cvect,\zetavect,\Sigma,\bm{\rnrisk})$ as the set of tasks according to model~\eqref{eq:mainmodel} with:
	\begin{itemize}
		\item[(i)-(iii)] as in Definition~\ref{def:modelsingle};
		\item[(iv)] The mean vectors $\muvect=(\mu_k)_{k \in \intr{B}}$ can vary freely subject to 
		\[ \muvect \in 
		\set{ \muvect \in \mbr^{d \times B}: \diam(\Cvect,\muvect)
			\leq\zetavect \text{ (coordinate-wise inequality) }}.
		\]
	\end{itemize}
\end{definition}

In words, $\multimodel(\Cvect,\zetavect,\Sigma,\bm{\rnrisk})$ is the set of Gaussian tasks
with fixed, aligned covariances and naive risks prescribed by the vector $\bs$,
such that the groups of mean vectors
given by partition $\Cvect$ have diameters bounded by the respective entries of vector $\zetavect$.
\begin{theorem}\label{prop:cpdlowbnd}
	Let $\bm{\rnrisk} \in \mbr_+^B$, $J \in \mbn_{>0}$, $\Cvect$ a $J$-partition of $\intr{B}$ and $\zetavect \in \mbr_+^J$ be fixed. It holds
	\begin{equation}\label{eq:cpdlowbnd}
		\lim_{\deff \to \infty} \sup_{\substack{\Sigma:\\ \deff(\Sigma) = \deff}} \inf_{\wh{\muvect}} \sup_{\mbq \in \multimodel(\Cvect,\bm{\zeta}, \Sigma,\bm{\rnrisk})} \frac{1}{B} \sum_{k =1}^{B} \frac{R_k(\mbq,\estmu_k)}{\rnrisk_k^2} \geq \cL^*\paren{\bm{\rnrisk}, \Cvect, \zetavect/2},
	\end{equation}
	where the infimum is over all joint estimators $\wh{\muvect}=(\wh{\mu}_1,\ldots,\wh{\mu}_B)$.
\end{theorem}
In particular, since it holds $\cL^*\paren{\bm{\rnrisk}, \Cvect, \zetavect/2} \geq \cL^*\paren{\bm{\rnrisk}, \Cvect, \zetavect}/4$, the upper bound matches the lower minimax
bound up to a fixed constant factor in a dimensional asymptotics sense (by choosing $u_0 = \log 17 B$ and provided that $\log B /(\min_k (\deff_k)^{-\nicefrac{1}{2}}) = o(\cL^*)$).
Moreover, \eqref{eq:compoundoracle} shows that
the $Q$-aggregation estimator is (up to that constant factor)
asymptotically minimax adaptive with respect to the choice of grouping $\Cvect$ of the task means, the corresponding group diameters, and the bag variances.

As in the single task case,
the minimax bound $\cL^*$ only depends on the bag sizes through the naive risks $\bm{\rnrisk}$:
bags with large variance and many samples are statistically equivalent to bags with
low variance and few samples.
Similarly, the improvement only depends on the relative aggregated variance of each group, not on the number of bags.
Lemma~\ref{prop:cpd_upbnd} gives an interpretable upper bound for $\cL^*$:

\begin{lemma}\label{prop:cpd_upbnd}
	Let $\bm{\rnrisk} \in \mbr_+^B$, $J \in \mbn_{>0}$, $\Cvect$ a $J$-partition of $\intr{B}$ and $\zetavect \in \mbr_+^J$, it holds:
	\begin{equation}\label{eq:cpd_upbnd_gen}
		\cL^*\paren{\bm{\rnrisk}, \Cvect,\zetavect}\leq
		\sum_{j=1}^{J} \frac{|\cC_j|}{B} \cdot
		\frac{\bar{\tau}_j + |\cC_j|^{-1}}{\bar{\tau}_j+1},
		\quad \ol{\tau}_{j}:=\frac{\zeta_j^2}{\ol{\rnrisk}^2(\cC_j)}, \;\; \ol{\rnrisk}^2(\cC_j) := \paren{ \frac{1}{|\cC_j|}\sum_{k\in \cC_j} \rnrisk_k^{-2}}^{-1},
	\end{equation}
	implying in particular:
	\begin{equation} \label{eq:estlglob}
		\cL^*\paren{\bm{\rnrisk}, \Cvect,\zetavect} \leq \min\paren{1,\frac{\taubb  }{1+\taubb} + \frac{J}{B}}, \qquad \taubb:=\sum_{j=1}^J \frac{\abs{\cC_j}}{B} \ol{\tau}_j.
	\end{equation}
	If all risks and diameters are equal, $\nrisk_k = \nrisk$ and $\zeta_j^2 = \zeta^2$ for all $k\in \intr{B}$ and $j\in \intr{J}$, then
	the bound of~\eqref{eq:estlglob} is sharp up to a factor at most $2.7$.
\end{lemma}

Bound \eqref{eq:estlglob} elucidates that the compound oracle relative risk $\cL^*$ is small when (i) there are few groups relative to the number of bags (i.e., $J/B$ small); and (ii) groups have on average a small squared diameter relative to the harmonic mean of the naive risks of its constituent tasks. 

Eq.~\eqref{eq:compoundoracle} implies that the compound risk is upper bounded by $\cL^*$ for any valid partitioning.
As an illustrative example we consider the (\ECSS) setting and  $\Cvect$ as a $\sqrt{\tau}\rnrisk$-covering of $\muvect$ for a given $\tau$. 
Then $\taubb = \tau$ and the number of groups $J$ is the covering number $N(\bm{\mu},\sqrt{\tau}s)$. 
This highlights that the $Q$-aggregation strategy will be very effective to reduce the compound risk if the set of true means can be covered by a relatively small number of balls, in comparison to the total number of tasks,
with a radius significantly smaller than the standard deviation of the naive estimates.

This bound takes a form akin to the findings presented in~\citet{marienwald2021high}, who examined the (\ECSS) setting only and used a testing strategy comparable to that of the previous section.
The parameter of their (and our) test, though, has to be fixed by the user.
In contrast, the $Q$-aggregation approach attains the oracle trade-off between the ``bias'' term $\tau / (1+\tau)$ and the ``variance'' term $N(\bm{\mu},\sqrt{\tau}s)/B$ without the need to specify $\tau$. 

Finally, observe that the first term $\tau / (1+\tau)$ resembles the best potential improvement and is reminiscent of the oracle improvement factor of the James-Stein estimator, which can be conceived as a special case; see Supplemental~\ref{se:JS} for additional details.

\section{Application: estimation of multiple kernel mean embeddings}\label{se:application}
We emphasise that our discussion and theoretical results include the case when $\mathcal{X}$ is a reproducing kernel Hilbert space (RKHS), in which case the mean corresponds to a kernel mean embedding (KME) \citep{smola2007hilbert, muandet2017overview}.
\HM{We showcase the efficacy of our methods for the estimation of multiple kernel mean embeddings here. For results on the estimation of real means on illustrative toy data, see Supplemental \ref{sec:apx_mnist}, where we also demonstrate the validity of the neighbouring test.}
Let $\mathcal{Z}$ be a measurable space enriched with a reproducing kernel $\kappa: \mathcal{Z} \times \mathcal{Z} \rightarrow \mathbb{R}$ and its corresponding RKHS $\mathcal{H}$.
The kernel mean embedding $\mu_{\mathbb{P}_Z} \in \mathcal{H}$ of distribution $\mathbb{P}_Z$ on $\mathcal{Z}$ and its empirical (naive) estimation $\widehat{\mu}_{\mathbb{P}_Z}$, which is based on the samples ${(Z_i )}_{1 \leq i \leq N_Z} \sim \mathbb{P}_Z$, are defined as
\begin{equation}\label{eq:KME_naiveestimation}
	\mu_{\mathbb{P}_Z}(\cdot) = \int_\mathcal{Z} \kappa(z, \cdot) \,\mathrm{d}\mathbb{P}_Z (z) \,\,\, , \,\,\, \widehat{\mu}_{\mathbb{P}_Z}(\cdot) = \frac{1}{N_Z}\sum_{i=1}^{N_Z} \kappa(Z_i, \cdot).
\end{equation}
The estimation of multiple KMEs is an instance of model \eqref{eq:mainmodel} once we identify $\mathcal{X} = \mathcal{H}$ and $X_i^{(k)} = \kappa(Z_i^{(k)},\cdot)$ for a bounded reproducing kernel $\kappa$; this allows a direct application of our theoretical results for the bounded setting.

For characteristic kernels, \JB{(see \citealp{fukumizu2007kernel,sriperumbudur2010hilbert} or \citealp{fukumizu2009kernel, sriperumbudur2011universality,nishiyama2016characteristic,szabo2018characteristic,simon2018kernel} for alternative descriptions)} the map from $\mathbb{P}$ to $\mu_\mathbb{P}$ is injective and contains information about all moments of $\mathbb{P}$, so that $\mu_\mathbb{P}$ provides a unique representation of $\mathbb{P}$.
Thus, KMEs can naturally be used to define a metric on probability distributions.
Let $\mathbb{P}, \mathbb{Q}$ denote distributions and their KMEs $\mu_\mathbb{P}, \mu_\mathbb{Q}$ respectively.
The maximum mean discrepancy (MMD) expresses the distance between $\mu_\mathbb{P}$ and $\mu_\mathbb{Q}$ in $\mathcal{H}$
\begin{align*}
	\text{MMD}^2(\mu_\mathbb{P}, \mu_\mathbb{Q}) &= {\Vert \mu_\mathbb{P} - \mu_\mathbb{Q} \Vert}_\mathcal{H}^2 \,\,\, , \\
	\widehat{\text{MMD}}^2(\mu_\mathbb{P}, \mu_\mathbb{Q}) &= \sum_{i \neq i'=1}^{N} \frac{\kappa(Z_i, Z_{i'})}{N (N - 1)}
	- 2 \sum_{i=1}^{N} \sum_{j=1}^{M} \frac{\kappa(Z_i, Y_j)}{N M}
	+ \sum_{j \neq j'}^{M} \frac{\kappa(Y_j, Y_{j'})}{M (M - 1)},
\end{align*}
where $\widehat{\text{MMD}}^2$ denotes an unbiased estimate based on the samples ${(Z_i)}_{1 \leq i \leq N} \sim \mathbb{P}$ and ${(Y_j)}_{1 \leq j \leq M} \sim \mathbb{Q}$.
For characteristic kernels it holds that $\text{MMD}^2(\mu_\mathbb{P}, \mu_\mathbb{Q}) = 0$ iff $\mathbb{P} = \mathbb{Q}$ \citep{gretton2012kernel}, which enables a large range of possible applications.

\subsection{Motivation and related work}
KMEs are employed for a variety of statistical tests, e.g., two-sample tests \citep{gretton2012kernel}, goodness-of-fit tests \citep{chwialkowski2016kernel}, and tests on statistical independence based on the Hilbert Schmidt independence criterion \citep{gretton2007kernel}.
It also finds application in machine learning, e.g., for unsupervised \citep{jegelka2009generalized} or supervised distributional learning \citep{muandet2012learning, szabo2016learning}, density estimation \citep{muandet2014kernel}, as part of the optimization criterion of the learning \citep{fakoor2020trade, brehmer2020flows}, and so on.
Due to the wide variety of kernel functions, kernel mean embeddings can in general be used on various data types and for structured data.
See \citet{muandet2017overview} for an in-depth overview on KMEs and their applications.

The success of applying the KME or the MMD relies heavily on the ability to accurately estimate the kernel mean based on sample data.
The naive empirical estimator \eqref{eq:KME_naiveestimation} was recently superseded by a James-Stein-like estimator \citep{muandet2014kernel}.
They showed that this estimator is admissible and consistent for a suitable choice of shrinkage.
Other single KME estimation strategies were proposed since then, e.g., non-linear shrinkage \citep{muandet2016stein}, an empirical Bayesian approach \citep{filippi2016bayesian}, and more robust estimations based on marginalised corrupted data \citep{xia2022sample}, or a MOM approach \citep{Ler19}.
To the best of our knowledge, there is no prior work on the improved estimation of {\em multiple} kernel mean embeddings except for \citet{marienwald2021high}.

\subsection{Description of the experiments}
We evaluate the estimation of multiple kernel mean embeddings on artificial and real-world data.

\subsubsection{Methods}
\HM{
	While the theoretical derivation necessitates sample splitting which results in independence between testing and estimation, this is not performed when applying the test-based methods in practice. All methods operate on the full data and function fully empirically without any prior knowledge. 
	A complete list and detailed description of all methods, including other SOTA approaches, can be found in Supplemental~\ref{sec:apx_methods}, and we only sketch the best performing ones here.
	Methods based on the testing procedure, which aggregate only neighbouring means, are abbreviated as \texttt{STB} (similarity test-based).
	
	The approaches differ in their definition of the convex weights.
	\STBopt{} calculates the plug-in version of the oracle weights \eqref{eq:STBoptweights}, where all quantities are replaced by their empirical estimates.
	\STBorth{} performs constrained risk minimization of \eqref{eq:riskone} but posits a heuristic orthogonality assumption, $\inner{\Delta_k,\Delta_{\ell}} = 0$ for all $k \neq \ell \in \intr{B}$, i.e., the mean differences are orthogonal. In other words, all remaining means $k, \ell$ spread pairwise orthogonal around mean $1$.
	A loose motivation for this heuristic is the principle that independent random centred vectors in high dimension are nearly orthogonal, which
	could apply if we assume that neighbouring means are in fact drawn locally and independently from a common prior distribution.
	This assumption might be unrealistic in practice but yields a closed-form solution for the weights.
	Finally, \STBegd{} minimises the $Q$-aggregation objective \JB{\eqref{eq:def_whomega}}
	and applies exponentiated gradient descent on the simplex \citep{kivinen1997exponentiated} to approximate the solution.
	\JB{For the estimation of kernel means, we do not include the theoretically motivated additional term $Q^{\BS}$ in the (BS) setting (see \eqref{eq:def_w_bnd}-\eqref{eq:def_whQBS}),
		instead we use a neighbouring test 
		as a prior ``safeguard'' with a purpose similar to the penalisation term $Q^{\BS}$, 
		namely to discard distant means.
		Note to this end we use a much larger value for $\tau$ than that required for the test-based methods.}
	Thus, \STBegd{} performs $Q$-aggregation on the set of candidates $\dwt{V}_{\tau,\cteW}$ (see \eqref{eq:setVdtilde}) passing the safeguard test. 
}

We compare their performances to the naive estimation (\NE), and we modify the multitask-averaging approach from \citet{feldman2014revisiting} (\MTAconst) so that it is applicable to the estimation of KMEs.
It assumes a constant similarity across tasks.
In Supplemental~\ref{sec:apx_expresults}, we further report the results of our previously proposed approach, \STBweight{} \citep{marienwald2021high}, which was not designed to handle inhomogenous data, and the regularised kernel mean shrinkage estimator \RKMSE, proposed in \citet{muandet2016stein}. \HM{It corresponds to a James-Stein-like kernel mean estimator,} that shrinks the estimation towards the origin and is performed separately on each bag.
\HM{Additionally, we also state the performance of the $Q$-aggregation approach without safeguard pre-selection of neighbours.
	Because their results are superseded by our main methods, we only report them in the Supplemental.}
In Supplemental~\ref{sec:apx_complexity} we discuss the computational complexity of all approaches.

The considered methods have data-dependent model parameters. \HM{See Supplemental~\ref{sec:apx_desc_methods}, \ref{apx:sec_defaultParams} for a discussion of their impact on the method's performance and guidelines on how to set them.} Their optimal values might be found by cross-validation whenever possible.
We also provide default parameter choices that we observed to perform well in most situations (see \ref{apx:sec_defaultParams} for a discussion and Table~\ref{tab:apx_default} for an overview).

\subsubsection{Experimental metric}
In the kernel case, the true KME $\mu$ is \JB{in general} unknown even for synthesised data.
We use a (naive) estimation based on an independent sample of the same distribution as approximation.
Because this proxy is computed on a very large sample, it can be assumed to have low risk and to be more accurate than the estimation performed by any method on much smaller bags.
The squared MMD between the (proxy) true KME $\mu_k$ of bag $k \in \intr{B}$ and its estimation $\estmu_k^{\texttt{m}}$, of form \eqref{eq:defmuhatomega}, performed by method \texttt{m} with weights $\pmb{\omega}_{k}^{\texttt{m}}$ is then used as error measure
\begin{align}\label{eq:expmmderror}
	\widehat{\text{MMD}}^2(\mu_k, \estmu_k^{\texttt{m}}) = &\sum_{\ell,\ell' \in \intr{B}} \omega_{k_\ell}^{\texttt{m}} \omega_{k_{\ell'}}^{\texttt{m}} \sum_{i=1}^{N_\ell} \sum_{i'=1}^{N_{\ell'}} \frac{\kappa(Z_i^{(\ell)}, Z_{i'}^{(\ell')})}{N_\ell N_{\ell'}}
	- \sum_{\ell \in \intr{B}} 2 \, \omega_{k_\ell}^{\texttt{m}} \sum_{i=1}^{N_\ell} \sum_{j=1}^{M_{k}} \frac{\kappa(Z_i^{(\ell)}, Y_j^{(k)})}{N_\ell M_{k}} \nonumber \\
	&+ \sum_{j \neq j'}^{M_k} \frac{\kappa(Y_j^{(k)}, Y_{j'}^{(k)})}{M_k (M_{k}-1)},
\end{align}
where $Y^{(k)}, Z^{(k)} \sim \mathbb{P}_k$ independent with $|Y^{(k)}_{\bullet}| = M_k \gg N_k = |Z^{(k)}_{\bullet}|$ for all $k \in \intr{B}$, so that $Y^{(k)}_{\bullet}$ can be used to calculate the proxy and $Z^{(k)}_{\bullet}$ for the estimation.
Each method is validated on the same data to guarantee comparability.
This estimation error is averaged over multiple trials $\overline{\text{MMD}}^2(\mu_k, \estmu_k^{\texttt{m}})$ and its decrease compared to the naive estimation $\muNE$ is reported for all experiments
\begin{equation*}
	\left( \paren[1]{ \, \overline{\text{MMD}}^2(\mu_k, \estmu_k^{\NE}) - \overline{\text{MMD}}^2(\mu_k, \estmu_k^{\texttt{m}}) \,} \, / \,\, \overline{\text{MMD}}^2(\mu_k, \estmu_k^{\NE})\right) \cdot 100 \, \, \, [\%].
\end{equation*}

\subsection{Artificial Gaussian data}
The toy data sets are Gaussian distributed in $\mathbb{R}^2$ with fixed means and randomly rotated covariance matrices.
For $k \in \intr{B}$ and $B=50$
\begin{equation*}
	Z_\bullet^{(k)}, Y_\bullet^{(k)} \sim \mathcal{N} \left( m_k, R(\theta_k) \Sigma {R(\theta_k)}^T \right) = \mathbb{P}_k \, , \,\,\, \,\,\,\,
	\theta_k \sim \mathcal{U} \left( -\frac{\pi}{4}, \frac{\pi}{4} \right),
\end{equation*}
where the rotation matrix $R(\theta_k)$ rotates the matrix $\Sigma=\text{diag}(1,10)$ according to angle $\theta_k$.
We generate $|Y^{(k)}_{\bullet}| = 1000$ data for the ``proxy truth''.
A Gaussian RBF kernel, with a kernel width set to the average feature-wise standard deviation of the data, maps the data from the two-dimensional input space to the infinite dimensional RKHS.
Two setups are tested:
\begin{enumerate}[label=(\alph*)]
	\item 	Clustered Bags: $N_k = 50$ for all $k \in \intr{B}$.
	In the input space, each ten bags form a cluster where the cluster centres ($= m_k$) lie equally spaced on a circle.
	The radius of that circle varies between $0$ and $3$, which creates different amount of overlap between the clusters.
	\item Imbalanced Bags: $m_k = 0$ for all $k \in \intr{B}$.
	The bags $Z^{(k)}_\bullet$ are highly imbalanced, i.e. $N_k \in [10,300]$.
	Because the tasks only vary in the rotation of their covariance matrices, we know that their KMEs lie on a low dimensional manifold in the RKHS.
	Because of the different bag sizes, the individual KMEs have different naive risks.
\end{enumerate}

\begin{figure}[t]
	\includegraphics[width=\textwidth]{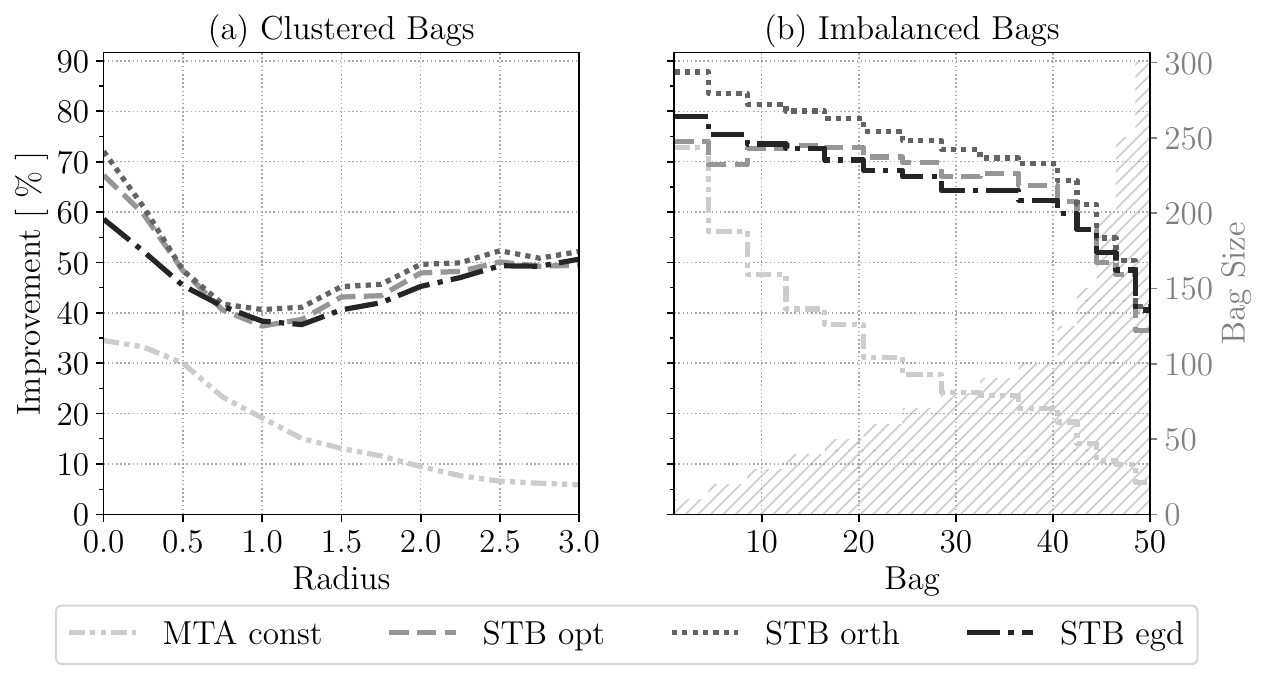}
	\caption{
		Decrease in average quadratic estimation error compared to \NE{} in percent on Gaussian data settings (a) and (b) resp.
		Higher is better.
		The hashed histogram bars in (b) show the bag sizes for the bags $1$ to $50$, which vary between $10$ and $300$  (right axis).}
	\label{fig:exp_art_results}
\end{figure}

The experiments are repeated for $100$ trials; the results of the methods with default parameter choices are shown in Fig.~\ref{fig:exp_art_results}. Figure~\ref{fig:apx_exp_art_results} in the Supplemental shows the results of the remaining methods.
We also report the performances with tuned parameters (optimised on i.i.d. training data) in Supplemental~\ref{apx:sec_defaultParams}, Figure~\ref{fig:apx_exp_trueopt}.

All methods provide an improvement over \NE{}, which is most significant for bags with few samples.
This was already observed in other multi-task learning problems, e.g., see \citet{marienwald2021high} or \citet{feldman2014revisiting}.
The constant similarity assumption of \MTAconst{} leads to an inadequate estimation for large radii or large bags.
Namely, a KME with large bag size is shrunk to the grand empirical mean of all bags even though it includes high-variance (low sample size) or distant bags.
This impairs the improvement.
This effect is alleviated by the proposed \texttt{STB} approaches, that define the shrinkage according to the variance of and the distances between the KMEs.
They show high performances for the tested settings.
For $0.5 < \text{radius} < 2$, the similarity test might mistake a bag of another cluster for a neighbour due to the strong overlap between the clusters, which explains the slight performance dip.
All the proposed methods provide similarly accurate results.
Despite its unrealistic orthogonality assumption, \STBorth{} performs best on the artificial data.

\subsection{Flow cytometry data}
Flow cytometry is fundamental to biomedical research and clinical practice.
It provides a multiparametric, single-cell analysis of a suspension or sample.
The flow cytometer analyses the size, shape and internal complexity of cells and can detect the presence and amount of different fluorochromes (which in turn reveal insights about the presence of proteins or structures within the cell).
These characteristics might then be used to classify the cells into different populations.
Applications are vast, but well-known examples are differential blood count, or immunophenotyping of leukemia or in HIV infections \citep{adan2017flow, mckinnon2018flow}.

The data set we use, corresponds to the T-cell panel of the Human ImmunoPhenotyping Consortium \citep{finak2016standardizing}.
Seven laboratories were asked to perform a flow cytometry analysis of three replicates of blood samples of three patients.
All laboratories were asked to follow the same experimental protocol and used the same seven markers to characterise the cells ($d=7$).
Based on the observed characteristics the cells were then classified into ten different populations or cell types.
We use this structure (laboratory, replicate, patient, cell type) to divide the data into bags.
We excluded bags with less than $1000$ data points, which leads to $424$ bags in total.
Each data point $Z_i^{(k)} \in \mathbb{R}^7$ in a bag $k$ corresponds to one cell.
As the number of cells varies, the bags are highly imbalanced.
We use a Gaussian RBF kernel with kernel width of~$950$ to map the cell features to a RKHS. 
The kernel choice and width are in accordance with \citet{dussap2023label}.
The (proxy) true KME is approximated by a naive estimation based on $Y^{(k)}_\bullet$ with $|Y^{(k)}_\bullet|=1000$ (bags with more samples are capped).
The sizes of the bags that are used for the estimation are chosen proportional to the bag sizes of the original input data, $N_k \in [7,125]$, to mimic a realistic setting.
In each one of the $100$ trials, a subset of samples $Z^{(k)}_\bullet$ with $|Z^{(k)}_\bullet|=N_k$ is drawn randomly from $Y^{(k)}_\bullet$, on which the methods perform their estimation.
We conducted experiments on each cell type separately so that $B \in [43, 62]$, and on all cell types jointly ($B=424$). Cell types~$5, 6$ and $10$, for which $B < 7$, are excluded for the separate but included in the joint analysis.
\begin{figure}[t]
	\includegraphics[width=\textwidth]{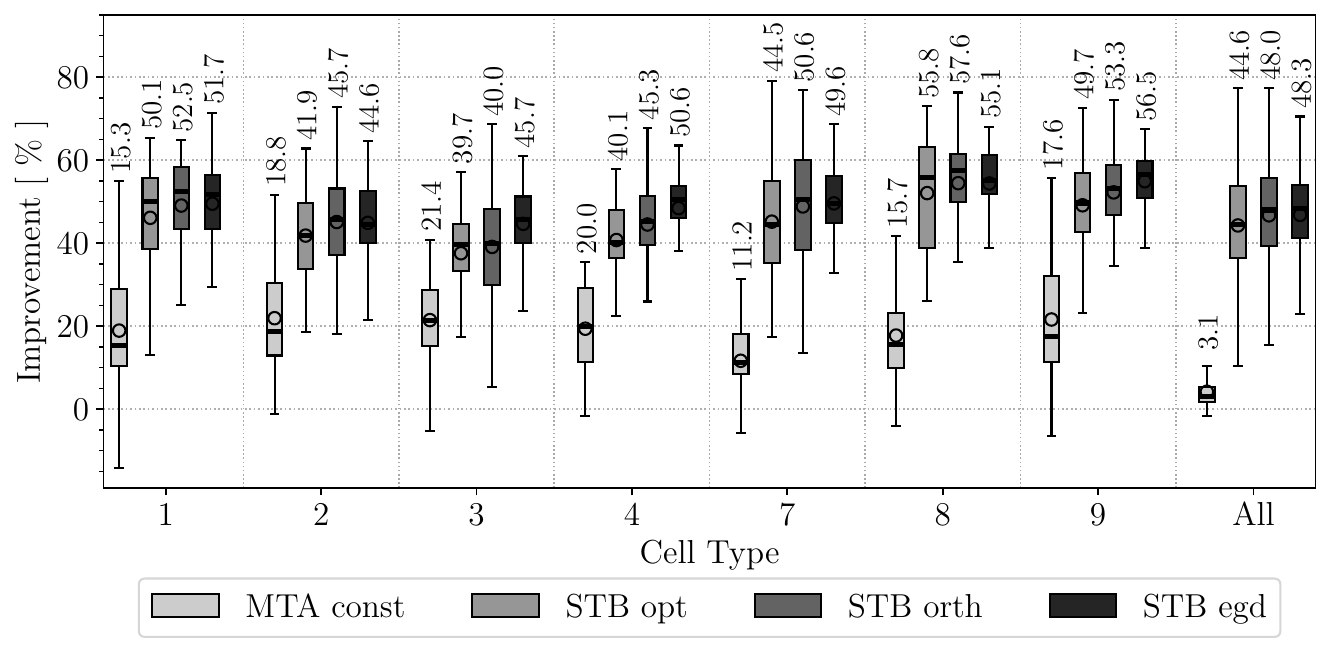}
	\caption{
		Decrease in estimation error compared to \NE{} in percent on the flow cytometry data.
		Higher is better.
		The number next to the boxplot quantifies the median, which is also depicted as a line.
		The mean is visualised as a circle.
		From left to right: results on individual cell types $1,2,3,4,7,8,9,$ and all cell types taken jointly.}
	\label{fig:exp_cyto_results}
\end{figure}

The results are depicted in Fig.~\ref{fig:exp_cyto_results} and in Supplemental Fig.~\ref{fig:apx_exp_cyto_results} for the additional methods.
On average, all methods provide an improvement over \NE.
For some trials, \MTAconst{} gives worse estimations than \NE{} (negative improvement), see e.g., cell type~1.
When all cell types are considered jointly, its performance drops significantly.
The \texttt{STB} approaches give more accurate estimations than \MTAconst{} and provide an improvement of $\approx50\%$ for all cell types.
\STBegd{} gives the most accurate and stable estimations across the different settings but also has high computational complexity.

In summary, our presented methods provide an improvement over the naive estimation and over other state-of-the-art methods.
Although \RKMSE{} or \MTAconst{} give more accurate estimations than the sample average, the provided improvements vary whether a shrinkage towards a common reference point or the grand mean resp. complies with the underlying data.
In contrast, our proposed methods identify inhomogeneous task similarities and are applicable to imbalanced data sets (which, therefore, surpass our previously introduced method \STBweight).
While \STBegd{} provides in most cases the highest improvement with least variance, it also requires the most computational complexity.
\STBorth{} and \STBopt{} provide a good trade-off.

\Revision{{\bf Practical recommendations.} Based on our empirical findings, we recommend pairing the $Q$-aggregation methods with a preceding neighbouring test to guard against distant or high-variance tasks, leading to more accurate and robust estimates (e.g., \STBorth, \STBegd). \STBorth{} offers a strong performance with a minimal computational cost. \STBopt{} is slightly less accurate but more robust to parameter choices, and still outperforms \NE. \STBegd{} yields the most accurate and stable estimates, with resilience to parameter misspecification, but at significantly higher computational cost due to exponentiated gradient descent and the estimation of additional variables. Supplemental~\ref{apx:praxisrecommends} provides more detailed recommendations on the choice of method.}

\section{Relation and comparison to previous work}
\label{se:previouswork}

We review related literature grouped along two axes: the first is rooted in statistics, compound decision rules and
the empirical Bayes point of view, and secondly a more recent one related to multitask learning.
We first emphasise again the seminal importance of the James-Stein (JS) estimator \citep{js1961} for a single vector mean,
which can be seen as a particular setting of model~\eqref{eq:mainmodel}.
Historically important is the realization that the sample average $\muNE_i := \frac{1}{N_i} \sum_{k=1}^{N_i} X_k^{(i)}$, despite being MLE (in the Gaussian model) and BLUE, is inadmissible and dominated by the shrinkage-based JS estimator.
\citet{Pin80} should be credited for an early ``dimensional asymptotics'' point of view, analysing the minimax risk if the
mean vector belongs to a ball of $\mbr^d$ as a by-product of his celebrated minimax analysis of estimators in Sobolev ball models (see, e.g., \citealp{Nussbaum96} for a discussion).
The risk of the JS estimator is asymptotically close to that minimax in the isotropic Gaussian model if $d \rightarrow \infty$, as well as adaptive to the radius of the ball \citep{Beran1996};
see more details in Supplemental~\ref{se:JS}.

\subsection{Empirical Bayes and compound decision point of view}
The celebrated series of works by \citet{efron1972empirical,efron1973stein,efron1976multivariate} advocated for an
interpretation of the JS estimator as a compound decision problem  and an empirical
Bayes point of view \citep{robbins1951asymptotically,robbins1964empirical,zhang2003compound}:
the problem of estimating a single mean vector in $\mbr^B$ with
standard Gaussian noise is better seen as $B$-many estimations of one-dimensional means observed with independent
observation noise (which in model~\eqref{eq:mainmodel} corresponds to $B>1$ means in dimension $d=1$).
The authors compare the performance of the JS estimator to that of a Bayesian model, i.e.,
the means are themselves drawn from a centred Gaussian prior.
The Bayes rule under the fully Gaussian model (prior and observations)
is solely determined by the prior variance, which is usually unknown, hence, called ``oracle'' in the present discussion.
The JS estimator can then be interpreted as being empirically Bayes as it replaces the oracle (prior) variance with an empirically estimated counterpart.
The compound risk is shown to converge to the oracle Bayes risk, as $B$ grows.

\citet{efron1976multivariate}  generalised this analysis to the multidimensional case which is an instance of
model~\eqref{eq:mainmodel} for arbitrary $d$ and Gaussian task distributions with
identical covariances.
They proposed a multidimensional version of the JS estimator.
Similarly to the one-dimensional case, this is interpreted as
an empirical Bayes procedure with a multidimensional Gaussian prior, whose unknown covariance is replaced by
an empirically estimated counterpart. 
If $(d+2)/ B \rightarrow 0$, then the risk of the multidimensional JS estimator approaches that of the oracle Bayes rule.

The nonparametric empirical Bayes estimator developed by~\citet{brown2009nonparametric} (see also \citealp{jiang2009} for a closely related, independent
work) is in the same line of thought, but considers
a completely arbitrary prior on the means (in dimension $d=1$).
In that situation, the oracle Bayes procedure can be expressed in terms of the marginal, nonparametric mixture density of the observations across tasks and of
its derivative (to establish this, Gaussian partial integration is used, thus, relying heavily on the assumption of isotropic Gaussian tasks). The proposed estimator replaces the true density with a kernel density estimate
(while \citealp{jiang2009} adopt a Generalised Maximum Likelihood Empirical Bayes estimator to
estimate the prior).
For a Gaussian kernel and as $B\rightarrow \infty$, 
this estimator approaches the oracle Bayes rule.

Similar to our approach, \citet{george1986minimax} proposed a weighted combination of shrinkage estimators, e.g., multiple JS estimators. The weights are assumed to be known but can adapt to the data to some extent.
He showed that an aggregation of Bayes rules is again Bayes on a mixture prior where the weights naturally translate to prior probabilities.

We emphasise the following key differences of this important line of work to the present one:
\begin{enumerate}[label=(\alph*)]
	\item The above approaches focus on the {\em compound} risk, while we analyse the risk of each individual task. The compound relative risk, analysed in Section~\ref{sec:avgrisk},
	is a different quantity from the ratio between compound risk and oracle Bayes risk.
	\item In the empirical Bayes framework, the focus lies on asymptotics as the number of independent tasks $B$ grows large, while ours is on the growing (effective) dimension. 
	Consequently, the choice of ``oracle'' reference for analyzing risk ratios differs between the two perspectives. 
	Within the empirical Bayes paradigm, the compound oracle Bayes risk serves as the reference. 
	We adopt a task-specific oracle improvement relative to the naive estimator. 
	Thus, the theoretical outcomes derived from these divergent approaches are not readily comparable.
	
	Concerning the role of the dimension, consistency with the oracle Bayes reference requires $d/B \rightarrow 0$ for the parametric approach of \citet{efron1976multivariate} and presumably an even more stringent condition for the nonparametric approaches of \citet{brown2009nonparametric} or \citet{jiang2009}. In fact they only considered the case $d=1$, but since both works rely on metric entropy estimates on appropriate function spaces, one would expect those to suffer of the curse of dimensionality.
	
	Consistency with the oracle, as considered in our paper, requires roughly $\mathrm{polylog}(B)/d \rightarrow 0$, thereby accommodating a broader spectrum of regimes. 
	For instance, when $B = \Theta( d^\alpha)$ for arbitrary $\alpha>0$, our approach ensures consistency with our oracle improvement, yet fails to achieve consistency with the oracle Bayes with a Gaussian prior if $\alpha \leq 1$. 
	Conversely, the regime where $B\rightarrow \infty$ while $d$ remains fixed, which is pertinent to empirical Bayes analyses, does not yield meaningful results in our framework (though, allowing the dimension to increase at an arbitrary small power of $B$ remains viable).
	
	In summary, our perspective is tailored towards high-dimensional scenarios, with possibly non-isotropic covariance structures, whereas the empirical Bayes methodology is not inherently designed for such settings. 
	Moreover, we emphasise the minimax property of our oracle improvement across suitable models as the dimension grows.
	\item We allow non-Gaussian data.
	\item We allow strong task heterogeneity (e.g., the covariances are not shared across tasks).
\end{enumerate}

\subsection{Multitask learning point of view}

\citet{feldman2014revisiting} viewed the many means estimation problem~\eqref{eq:mainmodel} as a multi-task learning problem
\citep{caruana1997multitask, zhang2021survey},
which gave rise to the term multi-task averaging. Also inspired by the JS estimator, the proposed approach extends the empirical compound risk minimization with a regularization term that favours the alignment of mean estimations for ``related'' tasks. The notion of ``task relatedness'' is encoded as a similarity matrix
considered as {\em a priori} information. In absence of specific information, the similarities are taken constant across tasks
and the method reduces to shrinkage towards the grand mean. The theoretical analysis focused mainly
on the low-dimensional setting and the oracle weights when $B=2$. Their data-driven similarity estimation yielded inconclusive results.
\citet{martinez2013multi} mitigated the default constant similarities in the absence of information by first clustering the tasks into different groups and then applying the approach of \citet{feldman2014revisiting} on each cluster separately;
but a theoretical analysis of this approach was not conducted.
In our work, we also propose to assimilate estimators of related tasks and thereby define an appropriate shrinkage direction.
We eliminate the disadvantage of both approaches, i.e., constant or known similarities, by estimating them solely based on the available data. We also extend significantly our preliminary work
\citep{marienwald2021high} which was limited to the testing approach unfit for heterogeneous tasks,
and with less precise theoretical results.

Recent work of \citet{duan2023adaptive} considers a general multi-task learning setting which includes
the multiple mean estimation problem as a special case. Comparable to that
of  \citet{feldman2014revisiting}, their estimators are determined by compound empirical risk
minimization with a regularization term measuring alignment to a predetermined model of task relatedness, e.g.,
the means form $K$ clusters or are close to a linear subspace
of dimension $K$. The proposed estimators depend on the considered task relatedness and on $K$.
Once interpreted in terms of relative squared risk, the theoretical bounds
obtained by \citet{duan2023adaptive} are not bounded by a constant but can grow as $\mathcal{O}(K^2)$ in the worst case where the fit to the
posited task relatedness is poor. For the relative risk to be significantly less than $\mathcal{O}(1)$,
the bounds require the condition $\delta \lesssim \rnrisk_1/K$, where $\delta$ represents closeness
to the model (cluster radius resp. distance to linear subspace). By contrast, in our analysis we do
not posit a particular task relatedness or value of $K$ to define the estimators; those are adaptive to the most advantageous
grouping model, including cluster number and size, describing the structure of the true parameters (see Section~\ref{se:minimax}). 
Our relative risk bounds are worst-case bounded (and even bounded close to 1), and
show a significant improvement in favourable cases even for the number of groups $K$ growing with the number of tasks $B$. On the other hand, our approach won't result in a significant risk improvement if the task means
belong to a low-dimensional subspace but are very far apart from each other. Still, using
the covering complexity point of view discussed in Section~\ref{se:minimax}, an improvement
can be shown if the tasks increase in number and are drawn, say, from a fixed a priori distribution having a low-dimensional support while the ambient dimension grows.

\section{Conclusion}

Considering the estimation of multiple mean  vectors in high dimensions from independent samples, we
focused on estimators formed as convex combinations of empirical averages of each sample. We proposed a test-then-aggregate method generalizing the approach of \citet{marienwald2021high}, and a direct $Q$-aggregation approach where the weights are found by minimization of an adequate objective.
From a theoretical perspective, we established asymptotic convergence to an oracle risk in an
appropriate ``dimensional asymptotics'' sense, as the effective dimensionality grows.
This oracle risk was proved to be exactly minimax under certain homogeneity conditions for the single-task
risk, and minimax up to a fixed factor for the compound relative risk (without homogeneity conditions).
One advantage of the $Q$-aggregation method is its theoretical adaptivity with respect to parameters
that have to be user-provided for the testing approach. We demonstrated the efficacy of the proposed
methods on showcase experiments for estimating multiple kernel mean embeddings
on controlled artificial datasets and real-world flow cytometry data.

Future investigations will aim to address the discrepancy between the lower and upper bounds for the single mean estimation in extremely inhomogeneous cases (we suspect the minimax lower bound could be too
conservative in such a case because it does not take into account the problem of neighbour detection).
Another important open direction is the integration in the multiple-mean estimation setting
of recent advances on single-mean estimation in high dimension,
achieving sub-Gaussian performance even under heavy-tailed distributions or samples that were adversarially corrupted, e.g., the median of means estimator (\citealp{lugosi2019subgaussian, lugosi2020multivariate}, see \citealp{lugosi2019mean, fathi2020relaxing} for an overview), or efficiently computable estimators (e.g., \citealp{cheng2019high, depersin2022robust}). Finally, a significant future avenue is to extend
our approach from mean estimation to more general high-dimensional multi-task learning problems such as those
considered by \citet{duan2023adaptive}.

\printbibliography

\section*{Acknowledgements and funding}
GB gratefully acknowledges funding from the grants ANR-21-CE23-0035 (ASCAI) and ANR-19-CHIA-0021-01 (BISCOTTE) of the French National Research Agency ANR. Part of this research was done when GB received
support from DFG SFB1294 ``Data Assimilation'', as ``Mercator fellow'' at Universit{\"a}t Potsdam, Germany.
HM gratefully receives funding from the German Federal Ministry of Education and Research under the grant BIFOLD24B.
Part of this research was done when HM was hired at Universit{\"a}t Potsdam, Germany, and funded by the German Ministry for Education and Research as BIFOLD (01IS18025A and 01IS18037A).

\newpage
	\markleft{G. BLANCHARD, J-B. FERMANIAN, H. MARIENWALD}
	\markright{SUPPLEMENTAL MATERIAL}
	\appendix
	\section{Nomenclature}\label{se:notation}
	\begin{longtable}[H]{ll}
		$B$ & number of tasks, Sec.~\ref{sec:intro}\\
		$\cB(\tau,\nu)$ & oracle relative risk, \eqref{eq:defBtaunu} \\ 
		(\BS) & bounded setting, Sec.~\ref{sec:dist_ass} \\
		$\Cvect$ & $J$-partition of $\intr{B}$, Def.~\ref{def:clustering}\\
		$d$ & ambient dimension, Sec.~\ref{sec:intro} \\
		$\deamm_k $ & effective dimension, \eqref{eq:effd} \\
		$\deff_k$ & effective dimension, \eqref{eq:effd} \\
		$\diam(\Cvect,\muvect)$ & diameter of partition $\Cvect$ of $\muvect$, \eqref{eq:def_diameter} \\
		$\Delta_k$ & difference between $\mu_k$ and $\mu_1$, Sec.~\ref{se:ne_aggregation} \\
		(\ECSS) & equal covariances and sample sizes, Sec.~\ref{sec:dist_ass} \\
		$\eta$ & relative estimation error of $s^2_k$, \eqref{eq:eventA} \\
		(\GS) & Gaussian setting, Sec.~\ref{sec:dist_ass} \\
		(\HT) & heavy-tailed assumption, Sup.~\ref{ap:est_schatten_norm} \\
		$J$ & nr. of parts of the partition, Def.~\ref{def:clustering}\\
		$k$ & index of task, Sec.~\ref{sec:intro}\\
		(\KC) & known covariances, Sec.~\ref{sec:dist_ass} \\
		$L_k(\estmu)$ & loss of estimator $\estmu$, \eqref{eq:lossrisk}\\
		$L_k(\omvect)$ & loss of aggregation estimator $\estmu_{\bm{\om}}$, \eqref{se:ne_aggregation} \\
		$\wh{L}_k(\omvect)$ & estimator for cond. risk, \eqref{eq:defwhL}, Sec.~\ref{sec:intro} \\
		$\cL^*\paren{\bm{\rnrisk}, \Cvect,\bm{\zeta}}$ & compound oracle risk, \eqref{eq:defLstar} \\
		$M$ & radius of ball in which the bounded data lies, Sec.~\ref{sec:QaggBS} \\
		$\mu_k$ & expectation of distribution $k$\\
		$\estmu_k$ & estimator of $\mu_k$\\ 
		$\muNE_k$ & naive estimation (empirical average) of $\mu_k$, Sec.~\ref{sec:intro}\\ 
		$\estmu_{\bm{\om}}$ & aggregation estimator, \eqref{eq:defmuhatomega}\\
		$\intr{n}$ & integers $1$ to $n$, Sec.~\ref{sec:intro}\\ 
		$N_k$ & number of samples (bag size) of task $k$, Sec.~\ref{sec:intro}\\
		$\norm{a}$ & Euclidian or Hilbert norm of the vector $a$, Sec.~\ref{se:setting}\\
		$\norm{\Sigma}_p$ & Schatten norm of matrix $\Sigma$, Sec.~\ref{sec:highdimasy}\\
		$\norm{\Sigma}_\infty$ & operator norm of matrix $\Sigma$, Sec.~\ref{sec:highdimasy}\\
		$\nu(U)$ & relative aggregated variance, \eqref{eq:def_s2v} \\
		$\omvect$ & aggregation weights, \eqref{eq:defmuhatomega} \\
		$\mbp_k$ & $k$-th task (probability distribution), Sec.~\ref{sec:intro}\\
		$\multimodel(\Cvect,\zetavect,\Sigma,\bm{\rnrisk})$ & class of distributions, Def.~\ref{def:modelcompound}\\
		$\singlemodel(\tau,V,\Sigma,\bm{\nrisk})$ & class of distributions, Def.~\ref{def:modelsingle} \\
		$\wh{Q}_1(\omvect)$ & prob. upper bound on $(\wh{L}_1(\omvect) - L_1(\omvect))$, \eqref{eq:defwhq} \\
		$\wh{Q}_1^\BS(\omvect)$ & additional penalization for (\BS), \eqref{eq:def_whQBS} \\
		$R_k(\estmu)$ & risk of estimator $\estmu$, \eqref{eq:lossrisk}\\
		$R_k(\omvect)$ & risk of aggregation estimator $\estmu_{\bm{\om}}$, \eqref{se:ne_aggregation} \\
		$s^2(U)$ & harmonic mean of the risks of the tasks in $U$\eqref{eq:def_s2v} \\
		$\nrisk_k$ & naive risk, \eqref{eq:defnaiverisk} \\
		$\Spx_B$ & (for $B \in \mbn_{>0}$) $(B-1)$-dimensional simplex, Sec.~\ref{se:ne_aggregation} \\
		$\Spx_V$ & (for $V \subset \intr{B}$) set of convex weights of support incl. in $V$, Sec.~\ref{se:ne_aggregation}\\
		$\cteW$ & threshold for $W_{(\cteW)}$, \eqref{eq:setW} \\
		$\Sigma_k$ & covariance matrix of $k$-th task, Sec.~\ref{se:setting} \\
		$\wt{T{\text{\tiny}}}^{(\tau)}_k$, $\dwt{T{\text{\tiny}}}^{(\tau)}_k$ & empirical similarity test using independent copy data, \eqref{eq:defTk}-\eqref{eq:setWandtkdtilde}\\
		$\tau, \tau^k_{\min}, \tau^\circ_{\min}, \tau^\pm$ & thresholds for similarity test, \eqref{eq:Vtau}-\eqref{eq:defqtest}-\eqref{eq:testsandwichv}-\eqref{eq:testsandwichv}\\
		$\tau/(1 + \tau)$ & best potential improvement \\
		$\wt{U}_k$ & unbiased estimator for $\norm{\Delta_k}^2$, \eqref{eq:def_testU} \\
		$\wt{V}$ & estimation of $V_\tau$ using independent copy data, \eqref{eq:eventA} \\
		$V_\tau$ & set of indices $\tau$-neighbouring tasks, \eqref{eq:Vtau} \\
		$V_{\tau, \cteW}$ & trimmed $V_\tau$, \eqref{eq:Vtauc} \\
		$V^*$ & subset of $V_\tau$, \eqref{eq:eventA} \\
		$W_{(\cteW)}$ & set of tasks with bounded relative variance, \eqref{eq:setW} \\
		$X_\bullet^{(k)}$ & $k$-th bag, \eqref{eq:mainmodel}\\
		$\wt{X}_\bullet^{(k)}$ & independent copy of $k$-th bag, Sec.~\ref{sec:ortoemp}\\
		$\bm{\zeta}$ & bound on the diameter of the $J$-partition, Def.~\ref{def:modelcompound} \\
	\end{longtable}

\section{Summary of the theoretical results}\label{se:summary_results}

\Revision{
	To aid the reader’s navigation, we summarize the main results below. Table~\ref{tab:summary_bounds} presents the upper and lower bounds on the relative risk across various settings. The testing approach builds on results from closeness testing, which are summarized in Table~\ref{tab:summary_bounds_testing}.
	
	We also point the reader to Proposition~\ref{prop:boundgammakernel} (Appendix~\ref{ap:notes_gamma}), which analyzes the behavior of the constant $\ratiobs$ coming into play in the risk
	bounds for the kernel setting (special case of (\BS)).

	\begin{table}[hb]
		\centering
		\caption{\Revision{Summary of upper and lower bounds on the relative risk under the different approaches and settings.}}
		\label{tab:summary_bounds}
		\Revision{
			\begin{tabular}{@{} lp{6.2cm}ll @{}}
				\toprule
				
				\textbf{Approach} & \textbf{Description} & \textbf{Setting} & \textbf{Result} \\
				\midrule		
				
				\multirow{2}{*}{Oracle results} & \multicolumn{2}{p{6.5cm}}{Oracle relative risk bound $\cB$ with \newline known $\tau$-neighbours and variances} & Lemma~\ref{lem:oraclebound} \\ 
				&  \multicolumn{2}{p{8cm}}{Relative risk w.r.t. $\cB$ using plug-in of \newline approximate $\tau$-neighbours and variances}  & Lemma~\ref{prop:boundSTB1}\\
				
				\midrule
				
				\multirow{2}{3.5cm}{Testing approach \newline (fully empirical)} 
				& \multirow{2}{=}{Relative risk w.r.t. $\cB$ } & (\GS) & Theorem~\ref{prop:fullemp} \\
				&  & (\BS) & Theorem~\ref{prop:fullemp_bnd} \\
				\midrule
				\multirow{5}{3.5cm}{Q-aggregation \newline approach \newline (fully empirical)}
				& \multirow{2}{*}{Oracle inequality for relative risk} & (\GS) & Theorem~\ref{prop:qaggreg_gauss} \\
				& & (\BS) & Theorem~\ref{thm:oracle_ineq_BS} \\
				& \multirow{2}{*}{Relative risk w.r.t. $\cB$} & (\GS) & Corollary~\ref{cor:aistats} \\
				&  & (\BS) & Theorem~\ref{prop:qaggreg_bounded} \\
				& Compound relative risk & (\GS) & Corollary~\ref{prop:uppbnd_all} \\
				& James-Stein setting & (\GS), (\JS) & Corollary~\ref{cor:stein} \\
				
				\midrule
				
				\multirow{2}{*}{Lower bounds} 
				& Single task relative risk & (\GS) & Theorem~\ref{prop:lowbnd_1} \\
				& Compound relative risk & (\GS) & Theorem~\ref{prop:cpdlowbnd} \\
				
				\bottomrule
		\end{tabular}}
	\end{table}
	
	\begin{table}[b]
		\centering
		\caption{\Revision{Summary of the results on concentration, level and power control of the neighbour detection by testing. In the (\BS) setting, the neighbor detection results
				are not given as separate statements but wrapped into Theorem~\ref{prop:fullemp_bnd}.
				For the (\HT) setting, a similar canvas can be applied to derive type I-II error control but
				was not written explicitly.}}
		\label{tab:summary_bounds_testing}
		\Revision{
			\begin{tabular}{@{} p{8cm}ll @{}}
				\toprule
				\textbf{Description} & \textbf{Setting} & \textbf{Result} \\
				\midrule
				\multirow{3}{*}{Concentration results}  & (\GS) & Propositions~\ref{prop:estnrisk}, \ref{prop:tr_sig}, \ref{prop:conc_sqrt_t_trsigma2} \\
				& (\BS) & Propositions~\ref{prop:tr_sig_bnd}, \ref{prop:conc_sqrt_t_trsigma2_bnd} \\
				& (\HT) & Propositions~\ref{prop:HTconcU}, \ref{prop:HTconcZ} \\
				Generic concentration-to-test result & [agnostic] & Proposition~\ref{prop:simplifiedtestresultJB} \\
				Type I-II error control for: & & \\
				\qquad Neighbour detection & (\GS), (\KC) & Proposition~\ref{prop:simplifiedtestresult} \\
				\qquad Whittled down neighbour detection & (\GS), (\KC) & Corollary~\ref{cor:testcor} \\
				\qquad Unknown covariance setting & (\GS) & Proposition~\ref{prop:pluginall} \\
				\bottomrule
		\end{tabular}}
	\end{table}
}

\section{Properties of the James-Stein estimator in the large dimension regime}\label{se:JS}
In this section, we provide a concise overview of the properties of the James-Stein estimator in high-dimensional settings and a comparison with the Q-aggregation approach.
Let us first cast the standard James-Stein problem as a particular limiting case of our general setting~\eqref{eq:mainmodel} with only two bags, the second of which with known mean equal to 0 and serving as a reference point:
\begin{assumption}[\protect{\JS}, James-Stein setting]  $B=2$,
	$\mu_2=0$, formally $N_2=\infty$ and $\nrisk_2=0$.
\end{assumption}
Since only $\mu_1$ is of interest, we drop the index 1 everywhere from the notation in what follows.
In that case, identifying $\omvect$ with its first weight, renoted as $\omega\in[0,1]$, $\estmu_\om = \om \muNE$ (defined in \eqref{eq:defmuhatomega}) is simply a shrinkage estimator towards $0$.

This is the type of estimator that \citet{stein1956} used to demonstrate that the empirical mean is not admissible.
Indeed, in an isotropic Gaussian setting ($\Sigma = \sigma^2I_d$), a shrinkage estimator $\mu^{\JS+}$ \citep{js1961} outperforms the empirical mean by shrinking towards a chosen reference point $\mu_2=0$. Let us denote $\sigma^2_N = \sigma^2/N$ and:
\begin{equation}\label{eq:defJS}
	\wh{\mu}^{\JS+} = \paren{ 1 - \frac{\sigma^2}{N} \frac{(d-2)}{\|\muNE\|_2^2} }_+ \muNE\,.
\end{equation}
The estimator is minimax for means inside a ball of radius $\tau\sigma^2_Nd$ but beats the empirical mean in general. For the model $\cG_d(\tau, \sigma^2) = \set{ \cN(\mu,\sigma^2I_d)^{\otimes N}, \|\mu\|^2 \leq \tau \sigma^2_Nd }$, the class of $N$-samples of an isotropic Gaussian distribution with bounded mean vector, \citet{Pin80} shows that (see also \citealp{Beran1996,Nussbaum96,Tsy08}):
\begin{equation}\label{eq:pinskerbound}
	\lim_{d\to \infty} \inf_{\wh{\mu}} \sup_{ \mbp \in \cG_d(\tau,\sigma^2) } \frac{\e{\norm{\wh{\mu}- \mu}^2}}{d\sigma^2_N} = \frac{\tau}{1+ \tau}, \qquad \frac{\e{ \norm{ \wh{\mu}^{\JS+} - \mu}^2}}{d\sigma^2_N} \leq \frac{\tau}{1+\tau} + \frac{4}{d}.
\end{equation}
In a non isotropic setting ($\Sigma$ is not necessary equal to $\sigma^2 I_d$ but remains known), a similar estimator achieving the same bounds can be constructed by replacing $d$ by $\deff$ and $\sigma^2$ by $\|\Sigma\|_{\infty}$ in \eqref{eq:defJS} and \eqref{eq:pinskerbound}.\\

We review three interpretations of the James-Stein estimator which relates it to our approaches in
the general model: the oracle and testing ones in Section~\ref{se:testapproach} and the Q-aggregation one in Section~\ref{se:Qaggreg}. \\
\textbf{Oracle interpretation:} the James-Stein shrinkage factor can be seen as an approximation of an oracle weight defined as the minimiser of the risk of $\estmu_\omega$:
\[R(\estmu_\omega) = (1-\omega)^2 \norm{\mu}^2 + \om^2\sigma^2_N d,
\]

which is minimised by $\omega^* 
=  \frac{\|\mu\|^2}{d\sigma^2_N+ \|\mu\|^2}$. By remarking that $\omega^* = 1 - \frac{d\sigma^2_N}{d\sigma^2_N+ \|\mu\|^2}$, a natural estimation of the oracle weight is obtained in \eqref{eq:defJS}.\\
\textbf{Test interpretation:} As in Section~\ref{se:testapproach}, we could first want to detect if $\mu$ is close to $\mu_2 = 0$. Knowing the variance, a very simple test is $T = {\bf 1}\set{ \|\muNE\|^2_2\leq (d-2)\sigma^2_N }$ for the hypothesis:
\begin{equation*}
	(H_0) : \|\mu\|^2 \geq d\sigma^2_N \qquad \text{against} \qquad (H_1) : \mu = 0.
\end{equation*}
If $(H_0)$ is rejected, then we choose $0$ as an estimator of $\mu$. \\
\textbf{Regularization interpretation:} Consider the following estimation by regularization, for $\lambda >0$:
\begin{equation}\label{eq:defmulanbda}
	\wh{\mu}^\lambda \in \arg \min_{\mu \in \mbr^d} \frac{1}{n} \sum_{i=1}^{n} \|\mu-X_i \|^2 + 2\lambda \|\mu\|.
\end{equation}
Then $\wh{\mu}^\lambda = \paren{ 1 - \frac{\lambda}{\|\muNE\|_2}}_+ \muNE$ is a minimiser. Using Stein's Lemma, we recover \eqref{eq:defJS} by choosing $\lambda = \frac{\sigma^2}{N}\frac{d-2}{\|\muNE\|_2}$ (see for example Lemma 3.8 of \citealp{Tsy08}). \\

In the (\JS) setting, the Q-aggregation method exhibits the same asymptotic behaviour as the James-Stein estimator $\estmu^{\JS+}$ without knowing the covariance $\Sigma$. Corollary~\ref{cor:stein} is deduced from Theorem~\ref{prop:qaggreg_gauss} and is proven in Supplemental~\ref{se:proofQaggreg}.
\begin{corollary}\label{cor:stein}
	Assume (\JS) and (\GS), let  $N \geq 7$, $(N-1)/2\geq u_0 \geq 3$, 
	and $\wh{\omvect}$ as defined in \eqref{eq:def_whomega}. Then: 
	\begin{equation}
		\frac{R(\wh{\omvect})}{\rnrisk^2}\leq  \frac{\|\mu\|^2}{\rnrisk^2+ \|\mu\|^2}\paren{1+Ce^{-u_0/2} }+  C \sqrt{\frac{u_0}{\deff}}  \,,
	\end{equation}
	where $C>0$ is some absolute constant.
\end{corollary}
The first term is, up to the multiplicative factor, Stein's error $\tau/(1+\tau)$ with $\tau=\norm{\mu^2}/\nrisk$.
In the dimensional asymptotic $\deff \rightarrow \infty$, assume
$u_0 {\rightarrow} \infty$ such that $u_0 = o(\deff)$
and suppose the mean
satisfies $\|\mu\|^2 \leq \tau \rnrisk^2$, then the estimator attains the Pinsker bound \eqref{eq:pinskerbound}:
\begin{equation*}
	\lim_{\deff \to \infty} \sup_{\substack{\mu, \rnrisk:\\ \|\mu\|^2 \leq \tau \nrisk}} \frac{ R(\wh{\omega}) }{\rnrisk^2} \leq \frac{\tau}{1 + \tau}\,.
\end{equation*}

We recover the same phenomenon in the bounded setting directly from Theorem~\ref{prop:qaggreg_bounded} as $V_{\tau,\cteW} = \{1,2\}$ for all $\cteW>0$ and the relative aggregated variance is null: $\nu(V_{\tau,\cteW}) =0$.\\

\JB{Among the numerous works related to the James-Stein estimator, we  mention the work of \citet{green1991james}, who consider two bags of Gaussian distributions $\cN(\mu_1,\sigma^2I_d)$ and $\cN(\mu_1+ \xi,v^2I_d)$. Like in the work of James and Stein, their approach
	uses an analysis strongly based on the Gaussian assumption,
	that does not generalize naturally to non-Gaussian data nor to $B>2$. In particular, their method relies on the estimation of optimal weights of the theoretical risk, the lack of a simple explicit form for those when $B>2$ prevents a direct estimation.
	It is worth noting that the risk bound derived for their estimator of $\mu_1$ (their Eq.~(2.4)) is close to the minimax risk $\cB$ introduced in our work:}
\[ d\sigma^2 -\frac{\sigma^4(d-2)^2}{\norm{\xi}^2+d(\sigma^2+v^2)} = d\sigma^2\brac{\cB\paren{ \frac{\norm{\xi}^2}{d\sigma^2}, \frac{\sigma^2}{\paren{\frac{1}{\sigma^2}+ \frac{1}{v^2}}^{-1}}} + O\paren{d^{-1}} }\,.\]

\JB{
	\section{On the optimality of Q-aggregation relative risk upper bound}\label{se:optimality_uppbound}
	
	The upper bounds we have obtained on the relative risk converge to the minimax risk at a rate $C/\sqrt{d}$ for some different notions of effective dimension (Theorem~\ref{prop:fullemp}, Theorem~\ref{prop:qaggreg_gauss}, Corollary~\ref{cor:aistats} and Theorem~\ref{prop:qaggreg_bounded}). In this section, we investigate the sharpness of the bound of the Q-aggregation procedure in a simplified Gaussian setting and show experimentally that our upper bounds are sharp.\\ 
	
	We consider $(\GS)$ with $B=50$, $\Sigma_k =I_d$ and $N_k =10$ for all $k\in \intr{B}$. We take $\mu_1 = 0$ and for $k\neq 1$, $\mu_k \sim \cN(0, \delta^2I_d )$. In that setting, the optimal relative risk for the estimation of $\mu_1$ is 
	\[r^* = \frac{N_1\delta^2+1}{N_1\delta^2+T} = \inf_{\omvect \in \Spx_B} \frac{R_1(\omvect)}{d/N_1}.\]
	We estimate the excess of relative risk $\wh{R}_1(\wh{\omvect})/(d/N_1)- r^*$ with $\wh{\omvect}$ constructed with \AGGe, for $10\leq d \leq 400$ and $\delta \in \set{0,1,3,6,10}$.\\
	
	\begin{figure}[ht]
		\includegraphics[width=0.6\textwidth]{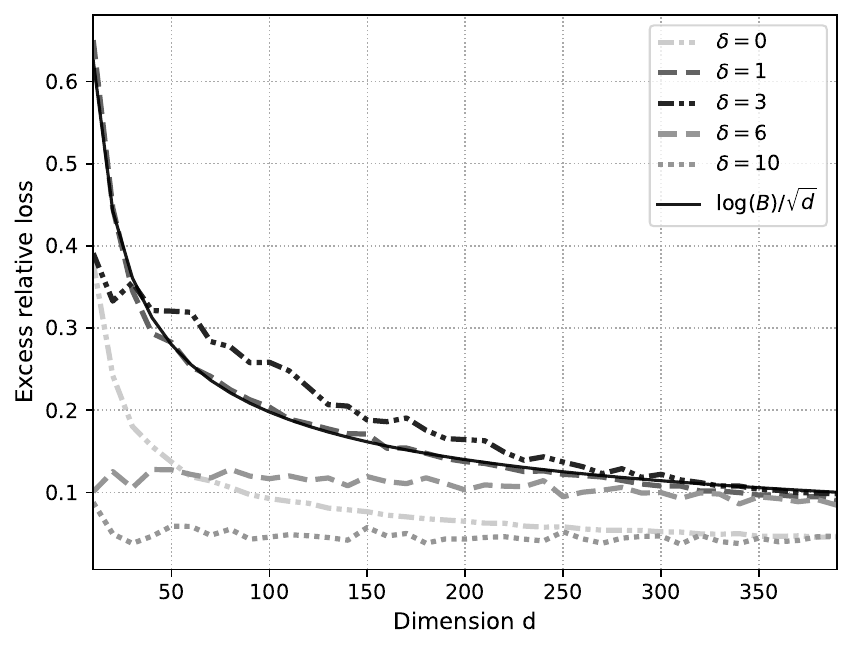}
		\caption{Excess relative risk for the estimation of $\mu_1$ using {\AGGe} with $c_q = \sqrt{\log(B)}$, $c_1=c_2 =c_{bs}=0$, in function of the dimension. Each curve corresponds to a different value of $\delta$. Each point is the mean of $500$ realizations. Data generation is detailed above.}
		\label{fig:relloss}
	\end{figure}
	
	The worst case for estimating $\mu_1$ is when the other averages are close enough to be useful but far enough apart to be difficult to detect. Hence, the excess relative risk is maximal for $\delta$ around $3$. When the others means are far away from the target mean, for instance $\delta \in \set{6,10}$ in Figure~\ref{fig:relloss}, our estimator is close to the optimal one, which only considers the target sample. In that case, we do not pay the error of $O(d^{-\nicefrac{1}{2}})$. However, in the other cases, the excess risk behaves as $Cd^{-\nicefrac{1}{2}}$, which aligns with our upper bounds, at least in
	the (\GS) setting for the Q-aggregation procedure. The slope of the linear regression of the logarithm of the excess risk in function of the logarithm of the dimension is, for $\delta \in \set{0,1,3}$, close to $-0.5$ (see Table~\ref{tab:reg_excessrisk}).  
	
	\begin{table}[h]
		\caption{Slope of $\log(\wh{R}_1(\omvect)/(d/N_1)-r^*)$ in function of $\log(d)$.}\label{tab:reg_excessrisk}
		\begin{tabular}{c|ccc} 
			$\delta$ & 0 & 1 & 3  \\
			\toprule
			Slope & -0.56 & -0.53 & -0.46 
		\end{tabular}
\end{table}}

\section{Proofs for Section~\ref{se:testapproach}}
\label{se:proofstests}

\subsection{Proof of Lemma~\ref{lem:oraclebound}}
The weights $\omvect^*$ are obtained by minimizing the upper bound \eqref{eq:riskboundV} using KKT conditions, for instance. However, to verify the bound \eqref{eq:oracletestbound}, it suffices to substitute the weights \eqref{eq:STBoptweights} into \eqref{eq:riskboundV}. Let us denote $\nu = \frac{\nrisk(V)}{\nrisk_1}$, from \eqref{eq:oracletestbound}:
\begin{align*}
	\frac{R_1(\omvect^*)}{\nrisk_1}&\leq \tau  (1-\omega_1^*)^2 + \sum_{k \in V} (\omega^*_k)^2 \frac{\nrisk_k}{\nrisk_1} \\
	&= \tau \lambda^2(1-\nu)^2 + (1-\lambda)^2 + 2 \lambda(1-\lambda) \nu + \lambda^2 \nu.
\end{align*}
By substituting $\lambda$ with its value from Equation \eqref{eq:STBoptweights}, we obtain:
\begin{align*}
	\frac{R_1(\omvect^*)}{\nrisk_1}&\leq \frac{\tau(1-\nu)^2+\tau^2(1-\nu)^2 + 2\tau (1-\nu)\nu + \nu}{(1+\tau(1-\nu))^2} \\
	&= \frac{\tau(1-\nu)((1-\nu)+ \tau(1-\nu) + \nu) + \nu(\tau(1-\nu)+1)}{(1+\tau(1-\nu))^2} \\
	&= \frac{\tau(1-\nu)+ \nu}{1+ \tau(1-\nu)} = \cB(\tau,\nu) = \cB\paren{ \tau, \frac{\nrisk(V)}{\nrisk_1} }.
\end{align*}
Thus, the inequality holds as claimed. \qed
\subsection{Proof of Lemma~\ref{prop:boundSTB1}}
Recall that we assume the following event holds:
\begin{align} \tag{\ref{eq:eventA}}
	\begin{cases}
		1 \in V^* \subseteq \wt{V} \subseteq V_\tau,
		\\
		\abs{\wt{\rnrisk}^2_k - \nrisk_k}  \leq \eta \nrisk_k, \text{ for all } k \in \wt{V},
	\end{cases}
\end{align}
for quantities $\wt{V},\wt{\rnrisk}$ which are considered as nonrandom for this proof (e.g.,
they are computed from an independent sample and we argue conditionally to that sample).
Denote
\[
\ol{R}_1(\wt{V},\omvect) := \tau \nrisk_1 (1-\om_1)^2 + \sum_{k \in \wt{V}} \om_k^2 \nrisk_k
\]
the risk upper bound from~\eqref{eq:riskboundV} wherein we used the index set $\wt{V}$. Due to the first
$\wt{V} \subset V_\tau$
the same argument leading up to~\eqref{eq:riskboundV}, it
holds $R_1(\omvect)\leq \ol{R}_1(\wt{V},\omvect)$ for all $\omvect \in \Spx_{\wt{V}}$.
Denoting now
\[
\wt{R}_1(\wt{V},\omvect) := \tau \wt{\rnrisk}^2_1 (1-\om_1)^2 + \sum_{k \in \wt{V}} \om_k^2 \wt{\rnrisk}^2_k
\]
the plug-in version of $\ol{R}_1(\wt{V},\omvect)$, we have, putting $\eps_k := \abs{
	\nrisk_k - \wt{\rnrisk}_k^2}$:
\[
\forall \omvect \in \Spx_{\wt{V}} :
\abs{\ol{R}_1(\wt{V},\omvect) - \wt{R}_1(\wt{V},\omvect)}
\leq \tau \eps_1 (1-\om_1)^2 + \sum_{k\in \wt{V}} \om^2_k \eps_k
\leq \paren{\max_{k \in \wt{V}} \frac{\eps_k}{\nrisk_k}} \ol{R}_1(\wt{V},\omvect),
\]
which entails, from the second part of event~\eqref{eq:eventA}:
\[
(1-\eta) \ol{R}_1(\wt{V},\omvect) \leq \wt{R}_1(\wt{V},\omvect) \leq (1+\eta) \ol{R}_1(\wt{V},\omvect)
\]
Since $\wt{\omvect}$ is a minimiser of $\wt{R}_1(\wt{V},\omvect)$, it holds for any other $\omvect \in \Spx_{\wt{V}}$:
\[
R_1(\wt{\omvect})\leq \ol{R}_1(\wt{V},\wt{\omvect})
\leq (1-\eta)^{-1} \wt{R}_1(\wt{V},\wt{\omvect})
\leq (1-\eta)^{-1} \wt{R}_1(\wt{V},\omvect)
\leq \paren{\frac{1+\eta}{1-\eta}} \ol{R}_1(\wt{V},{\omvect}).
\]
Minimizing the latter inequality over $\omvect$ yields (from Lemma~\ref{lem:oraclebound}):
\[\frac{R_1(\wt{\omvect})}{\nrisk_1}\leq \paren{\frac{1+\eta}{1-\eta}} \cB\paren[1]{\tau,\nu(\wt{V})} \leq
\paren{\frac{1+\eta}{1-\eta}} \cB\paren[1]{\tau,\nu(V^*)},\]
due to $V^* \subseteq \wt{V}$ and the monotonicity properties of $\nu$, $\cB$.

\subsection{Proofs for Section~\ref{se:knownvar}}

We start with a generic result linking concentration of the test statistic to the properties of the associated
test. It will allow to handle different distributional settings as particular cases.

We recall that $\wt{U}_k$ is the test U-statistic given by~\eqref{eq:def_testU} using independent ``tilde'' data.
\begin{assumption}[\TSC, Test Statistic Concentration]  Assume that for all $k \in \intr{B}$ and $\alpha \in (0,1)$, there exists $q_k(\alpha)$:
	\begin{equation}
		\label{eq:assmptconc}
		\prob{ \abs{\wt{U}_k - \|\Delta_k\|^2 } \geq \|\Delta_k\| q_k(\alpha) + c^2_0 q^2_k(\alpha) } \leq \alpha\,.
	\end{equation}
	where $c_0\geq 2$ is a numerical constant.
\end{assumption}

Put $u_\alpha := \log(8/\alpha)$, it is established that:
\begin{itemize}
	\item The assumption is satisfied under (\GS) for $q^2_k(\alpha) = 2\paren{ \frac{\norm{\Sigma_1}_2}{N_1} +  \frac{\norm{\Sigma_k}_2}{N_k}}u_\alpha $ and $c_0 =4$.  (Proposition~6 in \citet{BlaFer23})
	\item The assumption is satisfied under (\BS) for $q^2_k(\alpha) = 16\paren{ \frac{\norm{\Sigma_1}_2}{N_1} +  \frac{\norm{\Sigma_k}_2}{N_k}}u_\alpha+ 4 \frac{M^2 u_\alpha^2}{N_1^2 \wedge N_k^2} $ and $c_0=31$. (Proposition~9 in \citet{BlaFer23}) 
	\item The assumption is satisfied under (\HT) by an estimator of type "median-of-mean" for $q^2_k(\alpha) =  16\paren{ \frac{\|\Sigma_1\|_2}{N_1} + \frac{\|\Sigma_i\|_2}{N_i} }u_\alpha$ and $c_0=2$ but for $\alpha\geq 8e^{-N_1\wedge N_k}$. (see Proposition~\ref{prop:HTconcU}).
\end{itemize}

\begin{proposition}
	\label{prop:simplifiedtestresultJB}
	Grant assumption (\,\TSC) and let $\alpha\in (0,1)$, $\tau>0$ be fixed. Let $k\in \intr{B}$ and $\wt{T}_k$ be given by
	\begin{equation}
		\label{eq:defTk2}
		\wt{T}_k := \ind{ \wt{U}_k \leq \tau \nrisk_1}.
	\end{equation}
	Define $\tau^k_{\min} := 2 c_0^2 \invnrisk_1 q_k^2 (\alpha)$, then it holds:
	\begin{align}
		\label{eq:typeIerr} \text{ if } \norm{\mu_1-\mu_k} &> (\sqrt{\tau} + {\textstyle \sqrt{\tau^k_{\min}}}) \rnrisk_1:   && \prob{\wt{T}_k=1} \leq \alpha;\\
		\label{eq:typeIIerr}  \text{ if } \norm{\mu_1-\mu_k} & \leq  (\sqrt{\tau} - {\textstyle \sqrt{\tau^k_{\min}}})\rnrisk_1: &&  \prob{\wt{T}_k=0} \leq \alpha.
	\end{align}
\end{proposition}

{\bf Proof of Prop.~\ref{prop:simplifiedtestresultJB}}
Let $k\in \intr{B}$, we assume for the rest of the proof that
\[
\abs{\wt{U}_k - \|\Delta_k\|^2 } \leq \|\Delta_k\| q_k(\alpha) + c^2_0 q^2_k(\alpha)
\]
holds, which according to Assumption~(\TSC) is the case with probability at least $1-\alpha$. As $k$ is fixed, we will denote $\tau_{\min} := \tau_{\min}^k$.

Using $q_k^2(\alpha) \invnrisk_1 = \tau_{\min} c_0^{-2}/2$ and putting $x:=\frac{\norm{\Delta_k}}{\sqrt{\tau} \rnrisk_1}$,
the above inequality entails
\begin{equation}
	\abs{\frac{\wt{U}_k}{\tau \nrisk_1} - x^2 }
	\leq x \sqrt{\frac{\tau_{\min}}{2\tau}} c_0^{-1} + \frac{\tau_{\min}}{2\tau}
	\leq x \frac{\eps_\tau}{2\sqrt{2}}  + \frac{\eps_\tau^2}{2}, \label{eq:quadbd}
\end{equation}
where we have used $c_0 \geq 2$ and where $\eps_\tau:= \sqrt{\tau_{\min}/\tau}$.
This entails
\[
\tau^{-1}\invnrisk_1(\wt{U}_k - \tau\nrisk_1) \leq x^2 + x \frac{\eps_\tau}{2\sqrt{2}}  + \frac{\eps_\tau^2}{2} -1.
\]
Assuming $\eps_\tau \leq 1$,  the largest root of the quadratic polynomial on the right-hand-side above is lower bounded as
\[
x_+ = -\frac{\eps_\tau}{4\sqrt{2}} + \sqrt{1 - \frac{15}{32} \eps_\tau^2}
\geq 1 - \eps_\tau,
\]
using $\sqrt{1-a} \geq 1 -\sqrt{a}$ for $a\in[0,1]$.
Thus, $0 \leq x \leq  1 - \eps_\tau$ is
a sufficient condition ensuring $\wt{T}_k =1$, implying~\eqref{eq:typeIIerr}
since $(1-\eps_\tau)^2 \tau = \paren{\sqrt{\tau}-\sqrt{\tau_{\min}}}^2$.
(The case $\eps_\tau>1$ is trivial since the statement is void in that configuration.)

Similarly, \eqref{eq:quadbd} entails
\[
\tau^{-1}\invnrisk_1(\wt{U}_k - \tau\nrisk_1) \geq x^2 - x \frac{\eps_\tau}{2\sqrt{2}}  - \frac{\eps_\tau^2}{2} -1;
\]
the largest root of the quadratic polynomial on the right-hand-side above is upper bounded as
\[
x'_+ = \frac{\eps_\tau}{4\sqrt{2}} + \sqrt{1 + \frac{17}{32} \eps_\tau^2} \leq 1 + \eps_\tau,
\]
using 
$\sqrt{1+a} \leq 1 +\sqrt{a}$.
Thus, $x> 1 + \eps_\tau$ is a sufficient condition ensuring $\wt{T}_k =0$, implying~\eqref{eq:typeIerr}
since $(1+\eps_\tau)^2 \tau = \paren{\sqrt{\tau}+\sqrt{\tau_{\min}}}^2$.

{\bf Proof of Prop.~\ref{prop:simplifiedtestresult}}. Proposition~6 of \citet{BlaFer23} states that under
(\GS) it holds with probability at least $1-\alpha$ that
\begin{equation}
	\label{eq:concU}
	\abs{\wt{U}_k -\norm{\Delta_k}^2} \leq \norm{\Delta_k} q'_k \sqrt{u_\alpha} + 16q_k^2 u_\alpha,
\end{equation}
where \[
q_k^2 = 2\paren{ \frac{\norm{\Sigma_1}_2}{N_1} +  \frac{\norm{\Sigma_k}_2}{N_k}}
= 2 \nrisk_1 \paren{ \frac{1}{\sqrt{\deamm_1}} +  \frac{\nrisk_k/\nrisk_1}{\sqrt{\deamm_k}}},
\] 
and
\[
(q'_{k})^2 :=  2\paren{ \frac{\norm{\Sigma_1}_\infty}{N_1} +  \frac{\norm{\Sigma_k}_\infty}{N_k}} ;
\]
since $\norm{\Sigma}_\infty \leq \norm{\Sigma}_2$, we have $q'_k\leq q_k$,
so that assumption~\eqref{eq:assmptconc} is satisfied with $c_0=4$.
The claim is then a consequence of Proposition~\ref{prop:simplifiedtestresultJB}.

{\bf Proof of Cor.~\ref{cor:testcor}}.
For any $k\in \wt{V}$, we have $k\in W_{(c)}$, and since $\cteW \geq 1$, it holds
(with the notation used in Proposition~\ref{prop:simplifiedtestresult}, but using $\alpha/B$ in place of $\alpha$)
\[
\tau_{\min}^{(k)} 
=32 (u_\alpha + \log B) \paren{ \frac{1}{\sqrt{\deamm_1}} +  \frac{\nrisk_k}{\nrisk_1\sqrt{\deamm_k}}}
\leq \frac{\tau^\circ_{\min}}{2} + \frac{\cteW\tau^\circ_{\min}}{2} 
\leq  \cteW\tau^\circ_{\min},
\]
and
the result is a direct consequence of Proposition~\ref{prop:simplifiedtestresult}
(combined with a union bound over $k \in \intr{B}$).

\subsection{Proofs for Section~\ref{se:unknowncov}: estimating Schatten norms and plug-in estimates} \label{ap:est_schatten_norm}
We will be concentrating on one bag at a time
and for this reason omit the task index $k$ in the next results.
Thus, we assume $\wt{X}_1,\ldots,\wt{X}_N$ (with $N \geq 4$)
are i.i.d. data points in $\mbr^d$ with
expectation $\mu$ and known covariance matrix $\Sigma$.
We start with estimators for the Schatten norms $\norm{\Sigma}_p$, $p=1,2$.

We can use the natural unbiased estimator for any fixed $\norm{\Sigma}_1=\tr \Sigma$,
\begin{equation}
	\label{eq:esttrace}
	\wt{Z}^{(1)} := \frac{1}{N-1} \sum_{i=1}^{N} \|\wt{X}_i - \wt{\mu}\|^2 = \frac{1}{2N(N-1)} \sum_{i\neq j} \norm{\wt{X}_i-\wt{X}_j}^2,
\end{equation}
where $\wt{\mu} = N^{-1} \sum_{i=1}^{N} \wt{X}_i$ is the empirical mean of the (sub-)sample.

\subsubsection{Gaussian setting}\label{apx:test_gs}
We have the following error control in the Gaussian setting:
\begin{proposition}\label{prop:tr_sig}
	Assume (\GS) holds. For $u\geq 1$, if $N\geq 2$
	\begin{equation*}
		\prob{\abs{\wt{Z}^{(1)} - \tr \Sigma } \geq 4 \sqrt{\frac{2\tr \Sigma^2}{N}}u } \leq 2 e^{-u}\,.
	\end{equation*}
\end{proposition}

\begin{proof}
	Let $\bm{X} = (\wt{X}_1 - \wt{\mu},\ldots, \wt{X}_N-\wt{\mu}) \in \mbr^{dN}$. Then $\bm{X}$ is a centred Gaussian vector with covariance matrix $\bm{\Sigma} := \Gamma \otimes \Sigma$ where $\Gamma =I_N - \frac{1}{N}\bm{1}_N\bm{1}_N^T \in \mbr^{N\times N}$, $\bm{1}_N = (1,\ldots, 1) \in \mbr^N$ and $\otimes$ denotes the Kronecker product.
	Note that it holds $\tr \Gamma=(N-1), \bm{\Sigma}^2 = \Gamma^2 \otimes \Sigma^2 = \Gamma \otimes \Sigma^2$, $\tr{\bm{\Sigma}} = \tr \Gamma \tr \Sigma = (N-1) \tr \Sigma$, and $\tr(\bm{\Sigma}^2) =(N-1) \tr \Sigma^2$.
	Then, according to Lemma~\ref{lem:concnorm2gauss}, for $u \geq 1$, with probability greater than $1-2e^{-u}$:
	\begin{align*}
		\|\bm{X}\|_2^2 & \leq \tr \bm{\Sigma} +2 \sqrt{\tr \bm{\Sigma}^2u} + 2\|\bm{\Sigma}\|_{\infty}u \leq (N-1) \tr \Sigma + 4 \sqrt{(N-1)\tr \Sigma^2}u\,, \\
		\|\bm{X}\|_2^2 & \geq \tr \bm{\Sigma} -2 \sqrt{\tr \bm{\Sigma}^2u} \geq (N-1) \tr \Sigma - 2 \sqrt{(N-1)\tr \Sigma^2}u\,.
	\end{align*}
	We have used that $\sqrt{u} \leq u$ for $u\geq 1$. We conclude by remarking that $\|\bm{X}\|_2^2 = (N-1) \wt{Z}^{(1)}$.
\end{proof}

Following \citet{BlaFer23}, we can estimate $\norm{\Sigma}_2 = \sqrt{\tr \Sigma^2}$ using the following
U-statistic, which is an unbiased estimator of $\tr \Sigma^2$:
{\begin{equation} \label{eq:def_t_trsigma2}
		(\wt{Z}^{(2)})^2 :=  \frac{1}{4N(N-1)(N-2)(N-3)} \sum_{i \neq j \neq k \neq \ell } \inner{ \wt{X}_i-\wt{X}_k ,\wt{X}_j- \wt{X}_\ell}^2\,.
\end{equation}}
\begin{proposition}[\citealp{BlaFer23}, Prop. 12] \label{prop:conc_sqrt_t_trsigma2}
	Assume (\GS) holds and $N\geq 4$.
	Then for all $u \geq 0 $:
	\begin{equation}\label{eq:conc_sqrt_t_trsigma2}
		\prob{ \abs{\wt{Z}^{(2)} - \sqrt{\tr \Sigma^2} } \geq  30 \sqrt{\frac{\tr \Sigma^2}{N}}u^2 }   \leq e^4e^{-u}\,.
	\end{equation}
\end{proposition}

{\bf Proof of Prop.~\ref{prop:estnrisk}}
Proposition~\ref{prop:estnrisk} is a consequence of the above Proposition~\ref{prop:tr_sig}, using the union bound over $k \in \intr{B}$. \qed

Propositions~\ref{prop:tr_sig} and \ref{prop:conc_sqrt_t_trsigma2} can now be used to handle the plug-in
versions of the quantities considered in Section~\ref{se:knownvar} when covariances are unknown:

\begin{proposition} \label{prop:pluginall}
	Assume (\GS) holds, let $c\geq 1$ be a fixed number and let \HM{$\alpha\in (0,1/3)$}.
	Assume that we have estimates $\wt{Z}_1^{(1)}$ for $\norm{\Sigma_1}_1$ and $\wt{Z}_{k}^{(2)}$ for $\norm{\Sigma_k}_2$, $k \in \intr{B}$, (depending on the independent ``tilde'' data only) such that with probability $1-\alpha$ it holds
	simultaneously for some constants $\eta_1,\eta_2 \in (0,1)$:
	\begin{align}
		\label{eq:controlZ1}
		\abs{\wt{Z}_{1}^{(1)} - \norm{\Sigma_1}_1} & \leq \eta_1 \norm{\Sigma_1}_1;\\
		\label{eq:controlZ2}
		\abs{\wt{Z}_{k}^{(2)} - \norm{\Sigma_k}_2} & \leq \eta_2 \norm{\Sigma_k}_2, \text{ for all }
		k\in\intr{B}. 
	\end{align}
	Consider the following plug-in versions of the quantities appearing in~\eqref{eq:defTk}, \eqref{eq:setW}:
	\begin{equation}
		\label{eq:setWandtkdtilde2}
		\wt{W}_{(\cteW)} :=   \set{ k \in \intr{B} : \frac{\wt{Z}_{k}^{(2)}}{N_k} \leq \cteW \frac{\wt{Z}_{1}^{(2)}}{N_1}}, \qquad
		\dwt{T{\text{\tiny}}}^{(\tau)}_k := \ind{ \wt{U}_k \leq \tau \frac{\wt{Z}_{1}^{(1)}}{N_1}}.
	\end{equation}
	Then, defining
	\[
	\dwt{V}_{\tau,\cteW} := \set{ k \in \wt{W}_{(\cteW)}: \dwt{T{\text{\tiny}}}_k^{(\tau)}=1},
	\]
	with probability at least $1-3\alpha$ (with respect to the ``tilde'' data) it holds
	\begin{equation}
		\label{eq:empiricalsandwich}
		V_{\tau_-,\cteW /\beta} \subseteq \dwt{V}_{\tau,\cteW} \subseteq V_{\tau_+,\cteW\beta},
	\end{equation}
	where $\tau_- := \paren[1]{\sqrt{\tau(1 - \eta_1)} -  \sqrt{(\cteW/\beta)\taumino}}_+^2$,
	$\tau_+ := \paren[1]{\sqrt{\tau(1 + \eta_1)} + \sqrt{\beta\cteW\taumino}}^2$,
	with $\taumino = 64 (\log 8B\alpha^{-1})/\sqrt{\deamm_1}$,
	and $\beta := (1+\eta_2)/(1-\eta_2)$.
\end{proposition}
\begin{proof}
	Assume that \eqref{eq:controlZ1}-\eqref{eq:controlZ2} are satisfied. Then $W_{(\cteW/\beta)} \subseteq \wt{W}_{(\cteW)} \subseteq W_{(\beta \cteW)}$,
	with $\beta := (1 + \eta_2)/(1-\eta_2)$. Furthermore, recalling
	$ \wt{T}_k^{(\tau)} := \ind[1]{ \wt{U}_k \leq \tau \nrisk_1}, $
	then we have $\wt{T}_k^{((1-\eta_1)\tau)} \leq \dwt{T{\text{\tiny}}}_k^{(\tau)} \leq \wt{T}_k^{((1+\eta_1)\tau)}$; therefore
	\[
	\set{ k \in {W}_{(\cteW/\beta)}: \wt{T}^{((1-\eta_1)\tau)}_k=1} =: \wt{V}_- \subseteq \dwt{V} \subseteq
	\wt{V}_+:= \set{ k \in W_{(\beta \cteW)}: \wt{T}_k^{((1+\eta_1)\tau)}=1}.
	\]
	We can apply Corollary~\ref{cor:testcor} separately to $\wt{V}_- = \wt{V}_{(1-\eta_1)\tau, \cteW/\beta}$ and $\wt{V}_+= \wt{V}_{(1+\eta_1)\tau, \beta\cteW}$ and
	we get that with probability
	$1-3\alpha$ (accounting for the union bound together with event~\eqref{eq:controlZ1}-\eqref{eq:controlZ2}),
	\eqref{eq:empiricalsandwich} holds.
\end{proof}

{\bf Proof of Theorem~\ref{prop:fullemp}}.
From Proposition~\ref{prop:tr_sig} with $u=\log(4B\alpha^{-1})$ and a union bound over tasks,
with probability at least
$1-\alpha/2$ it holds
\begin{equation} \label{eq:zk1}
	\forall k \in \intr{B}: \qquad \abs{\wt{Z}_k^{(1)} - \norm{\Sigma_k}_1} \leq
	\norm{\Sigma_k}_2 \frac{\sqrt{32} \log (4B\alpha^{-1})}{\sqrt{N_k}}
	\leq \frac{1}{\sqrt{a}} \norm{\Sigma_k}_2,
\end{equation}
where for the last inequality we used the assumption $N_k \geq a(4 + \log(2B\alpha^{-1}))^4 \geq a(4+\log 6)^2(4+\log(2B\alpha^{-1}))^2 \geq 32a (\log(4B\alpha^{-1}))^2$ (also using $\alpha \leq 1/3$ in that estimate).

Similarly, from Proposition~\ref{prop:conc_sqrt_t_trsigma2} with $u=(4 +\log(2B\alpha^{-1}))$,
with probability at least
$1-\alpha/2$ it holds
\begin{equation}
	\label{eq:zk2}
	\forall k \in \intr{B}: \qquad \abs{\wt{Z}_k^{(2)} - \norm{\Sigma_k}_2} \leq
	30 \norm{\Sigma}_2 \frac{(4 + \log(2B\alpha^{-1}))^2}{\sqrt{N_k}}
	\leq \frac{30}{\sqrt{a}} \norm{\Sigma_k}_2.
\end{equation}
Therefore, conditions~\eqref{eq:controlZ1}-\eqref{eq:controlZ2} are satisfied simultaneously with probability $1-\alpha$, with $\eta_1 = \frac{1}{\sqrt{a}} \frac{1}{\sqrt{\deamm_1}}$ and $\eta_2 = \frac{30}{\sqrt{a}}$. Let $\beta = \frac{1+\eta_2}{1-\eta_2}$, as $a\geq 4400$, then $\beta \in (1, 3)$.

We apply Proposition~\ref{prop:pluginall}, but using the values $(\wt{\tau},3\cteW)$ given by~\eqref{eq:altplugin} in place of $(\tau,\cteW)$.
As a result we get with high probability the sandwiching property~\eqref{eq:empiricalsandwich},
\begin{equation}\label{eq:twistedsandwich}
	V_{\wt{\tau}_-,\cteW} \subseteq V_{\wt{\tau}_-,3\cteW/\beta} \subseteq \dwt{V}_{\wt{\tau},3\cteW} \subseteq V_{\wt{\tau}_+,3\beta\cteW}\subseteq V_{\wt{\tau}_+},
\end{equation}
denoting $\wt{\tau}_\pm$ the formula for $\tau_\pm$ of Proposition~\ref{prop:pluginall} where we replace $(\tau,\cteW)$ by $(\wt{\tau},3\cteW)$. We proceed to get bounds for $\wt{\tau}_\pm = (\sqrt{\wt{\tau}(1 \pm \eta_1)}\pm\sqrt{3\beta^{\pm 1}\cteW\tau^\circ_{\min}})^2_+$.

Let us start with bounding the estimation error of $\deamm_1$ by $\wt{\deamm_1}$:
it holds
\[
\sqrt{\wt{\deamm_1}} = \frac{N_1 \wt{\rnrisk}^2_1}{\wt{Z}_1^{(2)}} = \frac{\wt{Z}_1^{(1)}}{\wt{Z}_1^{(2)}}
\leq \frac{1+\eta_1}{1-\eta_2} \sqrt{\deamm_1} \leq 2 \sqrt{\deamm_1},
\]
where the last inequality holds if $a\geq 4400$. We deduce
\[
\wt{\tau}_{\min}^\circ = \frac{64 \log ( 8B\alpha^{-1})}{\sqrt{\wt{\deamm_1}}}
\geq \frac{1}{2} \cdot \frac{64 \log ( 8B\alpha^{-1})}{\sqrt{{\deamm_1}}} = \tau_{\min}^\circ/2,
\]
as defined in Proposition~\ref{prop:pluginall}. Furthermore, we have for $\eta_1 = \frac{1}{\sqrt{a\deamm_1}} \leq \frac{1}{\sqrt{a}}$ and $a\geq 4400$:
\[
\frac{1}{1-\eta_1} = 1 + \frac{\eta_1}{1-\eta_1} \leq 1 + \frac{1/\sqrt{a}}{1-1/\sqrt{a}} \frac{1}{\sqrt{\deamm_1}} \leq 1 + \frac{1}{60 \sqrt{\deamm_1}} \leq 1 + \frac{1}{30 \sqrt{\wt{\deamm_1}}}.
\]
Using the previous estimates we obtain
\[
\wt{\tau} := \paren{1 + \frac{1}{30\sqrt{\wt{\deamm_1}}}}\paren{\sqrt{\tau}+ \sqrt{6\cteW\wt{\tau}_{\min}^{\circ}}}^2 \geq \frac{1}{1-\eta_1} \paren{\sqrt{\tau} + \sqrt{3\cteW{\tau}_{\min}^\circ}}^2.
\]
It follows :
\[
\wt{\tau}_-  = \paren[1]{\sqrt{(1-\eta_1)\wt{\tau}} - \sqrt{3\beta^{-1}\cteW\tau_{\min}^\circ}}^2_+\geq  \paren[1]{\sqrt{(1-\eta_1)\wt{\tau}} - \sqrt{3\cteW\tau_{\min}^\circ}}^2_+  \geq \tau.
\]
Now to get an upper bound on $\wt{\tau}_+$, similarly to above we have
\[
\sqrt{\wt{\deamm_1}} \geq \frac{1-\eta_1}{1+\eta_2} \sqrt{\deamm_1} \geq  \frac{\sqrt{\deamm_1}}{2},
\]
and thus $\wt{\tau}_{\min}^\circ \leq 2 \tau^\circ_{\min}$.
It follows
\begin{equation*}
	\wt{\tau} = \paren{1 + \frac{1}{30\sqrt{\wt{\deamm_1}}}}\paren{\sqrt{\tau}+ \sqrt{6\cteW\wt{\tau}_{\min}^{\circ}}}^2 \leq \paren{1 + \frac{1}{15\sqrt{\deamm_1}}}\paren{\sqrt{\tau}+ 2\sqrt{3\cteW\tau_{\min}^{\circ}}}^2,
\end{equation*}
and then
\begin{align*}
	\wt{\tau}_+  \leq (1+\eta_1)\paren[1]{\sqrt{\wt{\tau}} + \sqrt{3\beta\cteW\tau_{\min}^\circ}}^2
	& \leq \paren{1 + \frac{1}{66\sqrt{{\deamm_1}}}}\paren{1 + \frac{1}{15\sqrt{{\deamm_1}}}}
	(\sqrt{\tau} + 4\sqrt{3\cteW \tau^\circ_{\min}})^2\\
	& = \xi \tau,
\end{align*}
where $\xi := (1+1/(15\sqrt{\deamm_1}))(1+1/(66\sqrt{\deamm_1}))(1 + 4\sqrt{3\cteW \tau^\circ_{\min}/\tau})^2.$

With these estimates in hand the sandwiching property~\eqref{eq:twistedsandwich} implies
\[
V_{\tau,\cteW} \subseteq \dwt{V}_{\wt{\tau},3\cteW} \subseteq V_{\xi\tau}.
\]
We use this property to apply Lemma~\ref{prop:boundSTB1} as earlier, and obtain
\[
\frac{R_1(\wt{\omvect})}{\nrisk_1} \leq \paren{\frac{1+\eta}{1-\eta}} \cB(\xi\tau,\nu(V_{\tau,\cteW}))
\leq \paren{1 + \frac{1}{25\sqrt{\min_k \deamm_k}}} \xi \cB\paren{\tau,\nu(V_{\tau,\cteW})}.
\]
Elementary estimates
lead to
\begin{align*}
	\paren{1 + \frac{1}{25\sqrt{\min_k \deamm_k}}} \xi
	&\leq \paren{1 + \frac{1}{7 \sqrt{\min_k \deamm_k}}}
	\paren{1+\frac{56\sqrt{\cteW\log(8B\alpha^{-1})}}{(\deamm_1)^{\frac{1}{4}}\sqrt{\tau}}}^2.
\end{align*}
\qed

\subsubsection{Bounded setting}\label{apx:test_BS}
Proposition~\ref{prop:tr_sig_bnd} and Proposition~\ref{prop:conc_sqrt_t_trsigma2_bnd} give concentration bounds for $\wt{Z}^{(1)}$ and $\wt{Z}^{(2)}$ in bounded setting.

\begin{proposition}\label{prop:tr_sig_bnd}
	Assume (\BS) holds. For $u\geq 1$, if $N\geq 2$
	\begin{equation*}
		\prob{\abs{\wt{Z}^{(1)} - \tr \Sigma } \geq  2\sqrt{2\frac{\var{\|X_1-\mu\|^2}}{N}u } + 32\frac{M^2u}{N} } \leq 4 e^{-u}\,.
	\end{equation*}
	
\end{proposition}

\begin{proof}
	Let us first remark that:
	\begin{equation*}
		\wt{Z}^{(1)} = \frac{1}{N-1} \sum_{i=1}^{N} \norm{ \wt{X}_i - \mu }^2 - \frac{N\|\wt{\mu} - \mu\|^2}{N-1}
	\end{equation*}
	Using Bernstein's inequality (Lemma~\ref{lem:bernstein}), with probability greater than $1-2e^{-u}$:
	\begin{equation*}
		\abs{\sum_{i=1}^{N} \norm{ \wt{X}_i - \mu }^2 -N\tr \Sigma } \leq \sqrt{2N \var{\|X_1- \mu\|^2} u} + 8M^2u\,.
	\end{equation*}
	Using McDiarmid's inequality \citep{McDiarmid1998,Bou04}, for $f(x_1, \ldots, x_N) = \|N^{-1} \sum_{i=1}^{N} (x_i -\mu)\|$, with probability greater than $1-2e^{-u}$:
	\begin{multline*}
		-\frac{4M^2}{N} \leq\|\wt{\mu} - \mu\|^2 - \frac{\tr \Sigma}{N} \leq \paren{\e{ \|\wt{\mu} - \mu\|} + \sqrt{\frac{2M^2u}{N}} }^2 - \frac{\tr \Sigma}{N} \\
		\leq \paren{\e{ \|\wt{\mu} - \mu\|}^2 - \frac{\tr \Sigma}{N}} +2\e{ \|\wt{\mu} - \mu\|} \sqrt{\frac{2M^2u}{N}}   + \frac{2M^2u}{N} \leq  8\frac{M^2u}{N}\,,
	\end{multline*}
	where we have used successively Jensen's inequality, that $\tr \Sigma \leq 4M^2$ and $u\geq 1$. It only stays to use that $ (N-1)^{-1} \leq 2N^{-1}$ for $N\geq 2$ and a triangle inequality to conclude the proof, with probability at least $1-4e^{-u}$:
	\begin{align*}
		\abs{\wt{Z}^{(1)} - \tr \Sigma} &\leq \frac{\sqrt{2N \var{\|X_1- \mu\|^2} u}}{N-1} + \frac{8M^2u}{N-1} + \frac{8M^2u}{N-1} \\
		&\leq  2\sqrt{2\frac{\var{\|X_1-\mu\|^2}}{N}u } + 32\frac{M^2u}{N}  \,.
	\end{align*}
\end{proof}

Similarly as in the Gaussian setting, we can estimate $\|\Sigma\|_2$ using the U-statistic \eqref{eq:def_t_trsigma2}:
\begin{proposition}[\citealp{BlaFer23}, Prop. 13] \label{prop:conc_sqrt_t_trsigma2_bnd}
	Assume (\BS) holds and $N\geq 4$.
	Then for all $u \geq 0 $:
	\begin{equation}\label{eq:conc_sqrt_t_trsigma2_bs}
		\prob{ \abs{\wt{Z}^{(2)} - \sqrt{\tr \Sigma^2} } \geq  12 M^2\sqrt{\frac{u}{N}} } \leq 2e^{-u}\,.
	\end{equation}
\end{proposition}
Thanks to these concentration results, we are able to give a bound on the estimation error of the test method for bounded data on the model of Theorem~\ref{prop:fullemp}.
\begin{theorem} \label{prop:fullemp_bnd}
	Assume (\BS) holds. Let $(X_\bullet^{(k)})_{k\in \intr{B}}$ and $(\wt{X}_\bullet^{(k)})_{k\in \intr{B}}$ be two independent datasets drawn from \eqref{eq:mainmodel} and $\alpha \in (0,1/3)$.
	Consider the set of estimated $\tau$-neighbours $\dwt{V}_{\tau,\cteW}$ defined in \eqref{eq:setVdtilde}, assume $N_k \geq a\ratiobs_k^2\deamm_k\log(8B\alpha^{-1})$ for all $k\in\intr{B}$, for a big enough constant $a$ ($a=576$ works),
	and where $\ratiobs_k := M^2/(\tr \Sigma_k)$.
	
	For fixed $\tau>0$, $\cteW \geq 1$, consider the
	weights $\wt{\omvect}^{\sharp}$ obtained by the {\em modified} plug-in $\paren[1]{\dwt{V}_{\wt{\tau},3\cteW},\wt{\bm{\rnrisk}}^2}$ for
	$(V,\bm{\nrisk})$ in~\eqref{eq:STBoptweights}, where
	\begin{equation} \label{eq:altplugin_bnd}
		\wt{\tau} := \paren{1 + \frac{1}{\sqrt{\wt{\deamm_1}}}}\paren{\sqrt{\tau}+ 3\sqrt{2\wt{\tau}_{\min}^{\circ}}}^2; \;\;\;\;
		\wt{\tau}_{\min}^{\circ} := \frac{80c_0^2\cteW \paren[1]{\log (8B\alpha^{-1})}}{\sqrt{\wt{\deamm_1}}}; \;\;\;\;
		\sqrt{\wt{\deamm_1}} := \frac{N_1 \wt{\rnrisk}^2_1}{ \wt{Z}_1^{(2)}}.
	\end{equation}
	and $c_0 = 31$. Then with probability at least $1-3\alpha$ over the draw of the ``tilde'' sample $(\wt{X}_\bullet^{(k)})_{k\in \intr{B}}$, it holds
	\[
	\frac{R_1(\wt{\omvect}^\sharp)}{\nrisk_1} \leq
	\paren{1 + \frac{5}{ \sqrt{\min_k \deamm_k}}}
	\paren{1+\frac{50^{2}\sqrt{\cteW\log(8B\alpha^{-1})}}{(\deamm_1)^{\frac{1}{4}}\sqrt{\tau}}}^2 \cB\paren[1]{\tau,\nu(V_{\tau,\cteW})},
	\]
	where the expected risk is with respect to the main sample $({X}_\bullet^{(k)})_{k\in \intr{B}}$.
\end{theorem}
{\bf Proof of Theorem~\ref{prop:fullemp_bnd}}.
From Proposition~\ref{prop:tr_sig_bnd} with $u=\log(8B\alpha^{-1})$ and a union bound over tasks,
with probability at least
$1-\alpha/2$ it holds
\begin{equation} \label{eq:zk1_bnd}
	\forall k \in \intr{B}: \qquad \abs{\wt{Z}_k^{(1)} - \norm{\Sigma_k}_1} \leq 2 \sqrt{2\frac{\norm{\Sigma_k}_1M^2u}{N_k}} + 32 \frac{M^2u}{N_k} \leq \frac{1}{3} \norm{\Sigma_k}_2
\end{equation}
where for the last inequality we used the assumption $N_k \geq 64a\ratiobs_k^2 \deamm_k \log(8B\alpha^{-1}) $.
Similarly, from Proposition~\ref{prop:conc_sqrt_t_trsigma2_bnd} with $u=\log(4B\alpha^{-1})$,
with probability at least $1-\alpha/2$ it holds
\begin{equation}
	\label{eq:zk2_bnd}
	\forall k \in \intr{B}: \qquad \abs{\wt{Z}_k^{(2)} - \norm{\Sigma_k}_2} \leq
	12 M^2 \sqrt{\frac{u}{N}} \leq  \frac{1}{6} \norm{\Sigma_k}_2.
\end{equation}
Therefore, as in the Gaussian case (see proof of Theorem~\ref{prop:fullemp}), with $\eta_2 = 1/6$ and $\beta = (1+\eta_2)/(1-\eta_2) \leq 3$:
\begin{equation*}
	W_{\cteW} \subseteq \wt{W}_{3\cteW} \subseteq W_{9\cteW}\,.
\end{equation*}
It follows, as in the proof of Proposition~\ref{prop:pluginall} that:
\begin{equation*}
	\set{ k \in {W}_{(\cteW)}: \wt{U}_k \leq (1-\eta_1)\tau \nrisk_1} =: \wt{V}_- \subseteq \dwt{V}_{\wt{\tau},3\cteW}\subseteq
	\wt{V}_+:= \set{ k \in W_{(9 \cteW)}: \wt{U}_k \leq (1+\eta_1)\tau \nrisk_1}.
\end{equation*} 
Let $\tau^\circ_{\min} = 80c_0^2 u(\deamm_1)^{-\nicefrac{1}{2}}$, one can check that $\tau^\circ_{\min}\cteW \geq \tau_{\min}^k$ for all $k \in V_{\tau,\cteW}$. Indeed, in the bounded setting:
\begin{equation*}
	\tau_{\min}^k \leq 2 c_0^2\paren{ 16u\frac{1+\cteW}{\sqrt{\deamm_1}} + 4u^2\rnrisk_1^{-2}\paren{ \frac{\ratiobs_1\nrisk_1}{N_1}+ \frac{\ratiobs_k\nrisk_k}{N_k}}} \leq 2 c_0^2\paren{ \frac{32  \cteW u}{\sqrt{\deamm_1}} + 4u \frac{1+\cteW}{\deamm_1}} \leq  \frac{80c_0^2\cteW u}{ \sqrt{\deamm_1}}
\end{equation*}
where we have used that $\cteW \geq 1$, the assumption on $N_k$ and the expression of $\tau_{\min}^k$ given by Proposition~\ref{prop:simplifiedtestresultJB}.
We apply Proposition~\ref{prop:simplifiedtestresultJB} separately to $\wt{V}_-$ and $\wt{V}_+$ and get that, with high probability:
\begin{equation}\label{eq:twistedsandwich_bnd}
	V_{\wt{\tau}_-,\cteW} \subseteq  \wt{V}_- \subseteq \dwt{V}_{\wt{\tau},3\cteW}\subseteq \wt{V}_+ \subseteq V_{\wt{\tau}_+, 9\cteW},
\end{equation}
where $\wt{\tau}_\pm = \paren{\sqrt{(1\pm \eta_1)\wt{\tau}}\pm 3\sqrt{\cteW\tau^\circ_{\min}}}$.  We proceed to get bounds for $\wt{\tau}_\pm$.

Let us start with bounding the estimation error of $\deamm_1$ by $\wt{\deamm_1}$:
it holds
\[
\sqrt{\wt{\deamm_1}} = \frac{N_1 \wt{\rnrisk}^2_1}{\wt{Z}_1^{(2)}} = \frac{\wt{Z}_1^{(1)}}{\wt{Z}_1^{(2)}}
\leq \frac{1+\eta_1}{1-\eta_2} \sqrt{\deamm_1} \leq 2 \sqrt{\deamm_1},
\]
where $\eta_1 = (\deamm_1)^{-1/2}/3\leq 1/3$. We deduce
\[
\wt{\tau}_{\min}^\circ = \frac{ 80c_0^2\cteW u}{\sqrt{\wt{\deamm_1}}}
\geq \frac{1}{2} \cdot \frac{ 80c_0^2\cteW u}{\sqrt{{\deamm_1}}} = \tau_{\min}^\circ/2,
\]
Furthermore, as $\eta_1 \leq 1/3$:
\[
\frac{1}{1-\eta_1} = 1 + \frac{\eta_1}{1-\eta_1} \leq 1 + \frac{1}{ 2\sqrt{\deamm_1}} \leq 1 + \paren{\wt{\deamm_1}}^{-\nicefrac{1}{2}}.
\]
Using the previous estimates we obtain
\[
\wt{\tau} := \paren{1 + \frac{1}{\sqrt{\wt{\deamm_1}}}}\paren{\sqrt{\tau}+ 3\sqrt{2\cteW\wt{\tau}_{\min}^{\circ}}}^2 \geq \frac{1}{1-\eta_1} \paren{\sqrt{\tau} + 3\sqrt{\cteW{\tau}_{\min}^\circ}}^2.
\]
It follows :
\[
\wt{\tau}_-  = (1-\eta_1)\paren[1]{\sqrt{\wt{\tau}} - \sqrt{3\cteW\tau_{\min}^\circ}}^2 \geq \tau.
\]
Now to get an upper bound on $\wt{\tau}_+$, similarly to above we have
\[
\sqrt{\wt{\deamm_1}} \geq \frac{1-\eta_1}{1+\eta_2} \sqrt{\deamm_1} \geq  \frac{\sqrt{\deamm_1}}{2},
\]
and thus $\wt{\tau}_{\min}^\circ \leq 2 \tau^\circ_{\min}$.
It follows
\begin{equation*}
	\wt{\tau} = \paren{1 + \frac{1}{\sqrt{\wt{\deamm_1}}}}\paren{\sqrt{\tau}+ 3\sqrt{2\cteW\wt{\tau}_{\min}^{\circ}}}^2 \leq  \paren{1 + \frac{2}{\sqrt{\deamm_1}}}\paren{\sqrt{\tau}+ 6\sqrt{\cteW\tau_{\min}^{\circ}}}^2 .
\end{equation*}
and then:
\begin{align*}
	\wt{\tau}_+  \leq (1+\eta_1)\paren[1]{\sqrt{\wt{\tau}} + 3\sqrt{\cteW\tau_{\min}^\circ}}^2
	& \leq \paren{1 + \frac{1}{3\sqrt{{\deamm_1}}}}\paren{1 + \frac{2}{\sqrt{{\deamm_1}}}}
	(\sqrt{\tau} + 9\sqrt{\cteW \tau^\circ_{\min}})^2\\
	& = \xi \tau,
\end{align*}
where $\xi := (1+1/(3\sqrt{\deamm_1}))(1+2/\sqrt{\deamm_1})(1 + 9\sqrt{\cteW \tau^\circ_{\min}/\tau})^2.$

With these estimates in hand the sandwiching property~\eqref{eq:twistedsandwich_bnd} implies
\[
V_{\tau,\cteW} \subseteq \dwt{V}_{\wt{\tau},3\cteW} \subseteq V_{\xi\tau}.
\]
We use this property to apply Lemma~\ref{prop:boundSTB1} as above, and obtain
\[
\frac{R_1(\wt{\omvect})}{\nrisk_1} \leq \paren{\frac{1+\eta}{1-\eta}} \cB(\xi\tau,\nu(V_{\tau,\cteW}))
\leq \paren{1 + \frac{1}{2\sqrt{\min_k \deamm_k}}} \xi \cB\paren{\tau,\nu(V_{\tau,\cteW})}.
\]
Elementary estimates lead to
\begin{align*}
	\paren{1 + \frac{1}{2\sqrt{\min_k \deamm_k}}} \xi
	&\leq \paren{1 + \frac{5}{\sqrt{\min_k \deamm_k}}}
	\paren{1+\frac{50^{2}\sqrt{\cteW\log(8B\alpha^{-1})}}{(\deamm_1)^{\frac{1}{4}}\sqrt{\tau}}}^2.
\end{align*}
\qed

\subsubsection{Heavy-tailed setting}\label{app:HT}
Similarly as in Sup.~\ref{apx:test_gs} and \ref{apx:test_BS}, we provide in this section estimators of $\|\Delta_k\|^2$, $\|\Sigma_k\|_1$ and $\|\Sigma_k\|_2$ but for heavy-tailed data. These estimators can be directly used to estimate the neighbours $V_{\tau,\cteW}$ and the oracle weights to then apply the testing approach in this setting.
\begin{assumption}[\protect{\HT}, Heavy-tailed setting] \label{ass:HT} For all $k\in \intr{B}$, $\mbp_k$ has a finite fourth moment.
\end{assumption}
Consider a statistic $T(N;x_1,\ldots x_N)$ in $\mbr$, the Median of Blocks statistics $\mom_b(T)$ for $b$ a divisor of $N$ is defined by the median of the statistics $T^a$, $1\leq a \leq b$ built from a $b$-partition of $x_1,\ldots x_N$ :
\[\mom_k(T) := \text{Median}( T^a, 1\leq a \leq b)
\]
where $T^a = T(N/b; x_{aN/b+1},\ldots x_{(a+1)N/b})$. If $b$ does not divide $N$, it suffices to partition the sample into sub-samples of size $\lfloor N/b \rfloor$ and $\lceil N/b \rceil$. If the original estimator is constructed from different samples (e.g., \eqref{eq:def_testU}), each sample is partitioned into $b$ subsamples.

\begin{proposition}\label{prop:HTconcU}
	Assume (\HT) holds, let $0 \leq u \leq N$ and $b = \lceil u \rceil$, let $U(X_\bullet^{(1)}, X_\bullet^{(k)})$ the estimator of $\|\Delta_k\|^2$ defined in \eqref{eq:def_testU}, then, with probability greater than $1-e^{-u/8}$:
	\begin{equation}\label{eq:conc_momU}
		\abs{\mom_b(U(X_\bullet^{(1)}, X_\bullet^{(k)})) - \|\Delta_k\|^2} \leq 4\sqrt{\Delta_k^T \paren{ \frac{\Sigma_1}{N_1 } + \frac{\Sigma_k}{N_k} }\Delta_ku} + 4\paren{ \frac{\|\Sigma_1\|_2}{N_1} + \frac{\|\Sigma_k\|_2}{N_k} }u \,.
	\end{equation}
\end{proposition}
In the kernel setting, the statistic $U(X_\bullet^{(1)}, X_\bullet^{(k)})$ is an estimator of the MMD distance between $\mbp_1$ and $\mbp_k$. 
(\citealp{Ler19} proposed a different robust estimator of this quantity called MONK, but we 
focus here on the $\mom$ estimator, which has the advantage to be easier to compute and to study.)

\begin{proposition}\label{prop:HTconcZ}
	Assume (\HT) holds, let $0 \leq u \leq N/4$ and $b = \lceil u \rceil$ :
	\begin{gather*}
		\prob{  \abs{ \mom_b(Z^{(1)}) - \tr \Sigma} \geq C\sqrt{\frac{\var{ \|X_1-\mu\|^2}u}{N}} + C\frac{\sqrt{\tr \Sigma^2}u}{N} } \leq e^{-u/8}\,, \\
		\prob{\abs{\sqrt{\mom_b(Z^{(2)})} - \sqrt{\tr \Sigma^2}} \geq C \sqrt{\frac{ M_X u}{N}}  }  \leq e^{-u/8}\,,
	\end{gather*}
	where $Z^{(1)}$ is defined in \eqref{eq:esttrace}, $Z^{(2)}$ in \eqref{eq:def_t_trsigma2}, $C>0$ is an absolute constant and $M_X = \e{\|X_1-\mu\|^4}$.
\end{proposition}
Proposition~\ref{prop:HTconcU} and Proposition~\ref{prop:HTconcZ} are different consequences of Lemma~\ref{lem:HTconc} below. Some more refined concentration bounds can be derived for MOB-type statistics (see, e.g., \citealp{Dev16,MinStra18}), but the present results are sufficient to show that in the (\HT) setting
suitable statistics satisfy Assumption (\TSC) and \eqref{eq:controlZ1}-\eqref{eq:controlZ2}. \\
{\bf Proof of Proposition~\ref{prop:HTconcU}.}
According to Lemma~\ref{lem:HTconc}, we only need compute the variances of the statistics $\wt{U}_a$,
\begin{align*}
	\var{\wt{U}_k}&=  4\sqrt{\Delta_k^T \paren{ \frac{\Sigma_1}{N_1 } + \frac{\Sigma_k}{N_k} }\Delta_k} + 2\tr\paren{\frac{\Sigma_1}{N_1}+ \frac{\Sigma_k}{N_k} }^2+ 2\paren{ \frac{\|\Sigma_1\|_2}{N_1^2(N_1-1)} + \frac{\|\Sigma_i\|_2}{N_i^2(N_i-1)} }  \\
	& \leq 4\sqrt{\Delta_k^T \paren{ \frac{\Sigma_1}{N_1 } + \frac{\Sigma_i}{N_i} }\Delta_k} + 4\paren{ \frac{\|\Sigma_1\|_2}{N_1} + \frac{\|\Sigma_k\|_2}{N_k} } =: \wt{v}(N_1,N_i)
\end{align*}
We apply Lemma~\ref{lem:HTconc} with $N=N_1+N_i$ and $v(N/u) := \wt{v}(N_1/u,N_k/u)$.\qed \\
{\bf Proof of Proposition~\ref{prop:HTconcZ}.}

For $Z^{(1)}$ the concentration bound is deduced directly from the variance:
\begin{align*}
	\var{Z^{(1)}} &= \frac{\var{\|X-\mu\|^2}}{N} + \frac{2\|\Sigma\|_2^2}{N(N-1)} \,.
\end{align*}

For $Z^{(2)}$ we can first assume w.l.g. than $X$ is centred. Then $Z^{(2)}$ can be developed as:
\begin{equation*}
	(Z^{(2)})^2= \frac{1}{N^{(2)}}\sum_{i\neq j} \inner{ X_i ,X_j } ^2  - \frac{2}{N^{(3)}} \sum_{i \neq j \neq k} \inner{ X_i ,X_j } \inner{ X_i , X_k }   - \frac{1}{N^{(4)}} \sum_{i \neq j \neq k \neq q}  \inner{ X_i ,X_j } \inner{ X_k , X_q }\,.
\end{equation*}
where $n^{(p)} = n(n-1)\ldots(n-p+1)$ for $n\geq p \in \mbn$. Let us first compute $\var{(Z^{(2)})^2}$:
\begin{align*}
	\var{(Z^{(2)})^2} &\leq \frac{2}{N^{(2)}} \e{ \inner{X, X'}^4 }  + \frac{4(N-2)}{N^{(2)}} \e{ \paren{X^T \Sigma X}^2} \\
	&+ \frac{4}{N^{(3)}} \paren{ (3!)M_X^2 + 2(N-3) \tr \Sigma^4 } + \frac{4!}{N^{(4)}} M_X^2  \\
	& \leq C\frac{\|\Sigma\|^2_{\infty} M_X}{N} + C\frac{M_X^2}{N^2}
\end{align*}
where $C>0$ is some absolute constant. Then according to Lemma~\ref{lem:HTconc}, for $u\leq N/4$, with probability grater than $1-e^{-u/8}$:
\begin{equation}\label{eq:concZ22}
	\abs{\mom_b((Z^{(2)})^2) - \tr \Sigma^2} \leq C\|\Sigma\|_{\infty} \sqrt{\frac{ M_Xu}{N}} + C\frac{M_Xu}{N}\,,
\end{equation}
Using that $\abs{\sqrt{(a^2+b)_+}-a} \leq \min\paren{ \sqrt{|b|}, \frac{b}{a}}$ for $a\in \mbr_+$ and $b\in \mbr$, (see, e.g., Lemma~15 of \citealp{BlaFer23}), assuming \eqref{eq:concZ22}, then
\begin{align*}
	\abs{\mom_b(Z^{(2)}) - \sqrt{\tr \Sigma^2}} &\leq \max_{\eps \in \{-1,1\}}\abs{ \sqrt{\tr \Sigma^2 + \eps C\|\Sigma\|_{\infty} \sqrt{\frac{M_Xu}{N}}}   -\sqrt{\tr \Sigma^2}} + C \sqrt{\frac{M_Xu}{N}} \\
	&\leq C \frac{\|\Sigma\|_{\infty}}{\sqrt{\tr \Sigma^2}} \sqrt{\frac{M_Xu}{N}} +C \sqrt{\frac{M_Xu}{N}} \leq C \sqrt{\frac{M_Xu}{N}}\,.
\end{align*}
\qed
\begin{lemma}\label{lem:HTconc}
	Let $T(N;x_1,\ldots x_N)$ a statistic build from $N$ i.i.d. random variables such that for all $N \geq N_0$:
	\begin{gather*}
		\e{T(N;X_1,\ldots,X_N)} = \e{T}, 
		\qquad \var{T(N;X_1,\ldots,X_N)} \leq v(N),
	\end{gather*}
	where $v:\mbr_+ \to \mbr_+$ is nonincreasing.  Let $1 \leq u \leq N/(N_0+1)$ and $b= \lceil u \rceil$, then
	\begin{equation*}
		\prob{ \abs{\mom_b(T) - \e{T}} \geq \sqrt{4 v\paren{\frac{N}{4u}}} } \leq e^{-u/8}\,.
	\end{equation*}
\end{lemma}

{\bf Proof of Lemma~\ref{lem:HTconc}.} \\
First assume that $b| N$. Let us denote $T_a := T( N/b ; x_{(a-1) N/b+1},\ldots x_{a N/b})$ for $ a \in \intr{b}$. Then for all $a \in \intr{b}$, by Markov's inequality:
\begin{equation}\label{eq:HTmarkov}
	\prob{ \abs{ T_a - \e{T}} \geq \sqrt{4v(N/k)}} \leq \frac{1}{4}\,.
\end{equation}
Then, $\abs{ \mom_b(T) - \e{T} } \geq  \sqrt{4v(N/b)}$ implies that at least $b/2$ of $T_a$ satisfies
\begin{equation*}
	\abs{ T_a - \e{T} } \geq  \sqrt{4v(N/b)}\,.
\end{equation*}
By independence of the $T_a$ and Hoeffding's inequality:
\begin{align*}
	\prob{|\mom_b(T) - \e{T}|>  \sqrt{4 v\paren{N/b}}} \leq \prob{\text{Bin}\paren{b,\frac{1}{4}} \geq \frac{b}{2}} \leq e^{-b/8}\,,
\end{align*}
where $\text{Bin}$ denotes the Binomial distribution. Because $u \leq b \leq u+1$ and $v$ is a noninccreasing
function, we can conclude:
\begin{equation*}
	e^{-b/8} \leq e^{-u/8}\,, \qquad v\paren{ \frac{N}{b} } \leq v\paren{ \frac{N}{u+1} } \leq v\paren{ \frac{N}{4u} } \,.
\end{equation*}
If $b \nmid N$, equation~\eqref{eq:HTmarkov} is still verified with $v\paren{ \left\lfloor \frac{N}{b} \right\rfloor }$ instead of $v\paren{ \frac{N}{b}  }$ and:
\begin{align*}
	&\left\lfloor \frac{N}{\lceil u\rceil} \right\rfloor \geq  \frac{N}{\lceil u\rceil} -1 \geq \frac{N}{2\lceil u\rceil}&  \text{ if } \lceil u\rceil\leq N/2 \\
	&\left\lfloor \frac{N}{\lceil u \rceil} \right\rfloor = 1 \geq \frac{N}{2\lceil u\rceil} & \text{ if } N \geq\lceil u\rceil > N/2\,.
\end{align*}
We conclude using that $\lceil u\rceil \leq (u+1) \leq 2u$ for $u\geq 1$.
\qed

\section{Proofs for Section~\ref{se:Qaggreg}}
\label{se:proofQaggreg}
\subsection{Proof of Theorem~\ref{prop:qaggreg_gauss}}
Let $\wh{\omvect}\in \arg \min_{\omvect\in \Spx_B} \paren[1]{\wh{L}_1(\omvect) + 16 \sqrt{u_0} \wh{Q}_1(\omvect)}$.
Denote $\cX^{(-1)}=(X_\bullet^{(k)})_{k\neq 1}$ the observed bag data except for the first bag, which corresponds to the target task.

{\bf First step : bound in conditional probability.}
As a first step, we obtain a high-probability bound for $L_1(\wh{\omvect})$.
For $x\geq 1$, define the event $A(x)$:
\begin{equation*}\label{eq:defAxy}
	A(x):= \left\{
	\begin{array}{ccr}
		\sqrt{{\mathfrak{q}_k}}  \leq c_1(x) \sqrt{{\wh{q}_k}}
		+ C \frac{\nrisk_1}{\deff_1} 
		\sqrt{N_1 x}, 
		&  2 \leq k \leq B,
		&(\text{a}) \\
		\sqrt{{\wh{q}_k}}  \leq \paren{ 1+ \sqrt{\frac{2x}{N_{1}-1}} }\paren{\sqrt{{\mathfrak{q}_k} +
				\frac{\rnrisk_1^4}{\deamm_1} N_1}
			+ \frac{\nrisk_1}{\deff_1} 
			\sqrt{2N_1 x}}, 
		&  2 \leq k \leq B,
		&(\text{b})\\
		\abs{\whnrisk_1 - \nrisk_1}  \leq C\frac{\rnrisk_1^2}{\sqrt{\deamm_1N_1}}x  \,, & & (\text{c}) \\
		\|\muNE_1 - \mu_1\|^2  \leq \rnrisk_1^2 + C \frac{\rnrisk_1^2}{\sqrt{\deamm_1}}x \,,&  & (\text{d})\\
		\abs{\inner{ \muNE_k - \mu_1,\muNE_1 - \mu_1 }}  \leq \sqrt{2\frac{\mathfrak{q}_k}{N_1}x },
		&  2 \leq k \leq B,
		& (\text{e}) \\
	\end{array} \right\}
\end{equation*}
where $\mathfrak{q}_k =(\muNE_k-\mu_1)^T\Sigma_1(\muNE_k-\mu_1) $ and $c_1(x) = \sqrt{e}\exp(x/(N_1-1))$. For the whole proof, the notation $C$ will denote an absolute numeric constant whose value can change between lines. The probability of
the event $A$ conditionally to $\cX^{(-1)}$ is bounded as: 
\begin{equation}\label{eq:pAc}
	\prob{A^c(x) | \cX^{(-1)}} \leq (6B+4)e^{-x}.
\end{equation}
We combine a union bound with estimates for each individual bound: bounds (a) and (b) are consequences of Proposition~\ref{prop:concwhqgauss} with $\nu = \muNE_k$. For (a), we have used that $\sqrt{\mathfrak{q}_k} \leq  \sqrt{\mathfrak{q}_k+ \tr\Sigma_1^2/N_1}$. Bound (c) is a rewriting of Proposition~\ref{prop:tr_sig}. Bound (d) is a consequence of Lemma~\ref{lem:concnorm2gauss} with $X =\muNE_1 - \mu_1$, $\mu = 0$, $\Sigma = \Sigma_1/N_1$; bounding $\sqrt{x}$ by $x$, and $\|\Sigma_1\|_{\infty}$ by $\sqrt{\tr \Sigma^2_1}$. Finally (e) is deduced from Lemma~\ref{lem:concnormal} with $X = \inner{\muNE_k - \mu_1, \muNE_1 - \mu_1}$, $m= 0$ and $\sigma^2 = \mathfrak{q}_k$. \JB{We point out that $\mathfrak{q}_k$ is non-random conditionally to $\cX^{(-1)}$.}

From now on, assume that event $A(x)$ holds.
Then, 
\begin{align}
	L_1(\wh{\omvect})
	&  = \norm[3]{\sum_{k=1}^B \wh{\om}_k (\muNE_k - \muNE_1) + (\muNE_1-\mu_1)}^2
	\notag \\
	& =\norm[3]{\sum_{k=2}^B \wh{\om}_k (\muNE_k - \muNE_1)}^2
	+ 2 \sum_{k=2}^B\wh{\om}_k \inner{\muNE_k - \muNE_1,\muNE_1 -\mu_1} + \norm{\muNE_1-\mu_1}^2
	\notag \\
	& = \norm[3]{\sum_{k=2}^B \wh{\om}_k (\muNE_k - \muNE_1)}^2 + 2\sum_{k=2}^{B} \wh{\omega}_k \inner{ \muNE_k - \mu_1,\muNE_1 - \mu_1 } + (2\wh{\omega}_1-1) \|\muNE_1-\mu_1\|^2 \notag\\
	& = \wh{L}_1(\wh{\omvect}) + 2\sum_{k=2}^{B} \wh{\omega}_k \inner{ \muNE_k - \mu_1,\muNE_1 - \mu_1 } + (2\wh{\omega}_1-1) \paren{\|\muNE_1-\mu_1\|^2 - \nrisk_1} \notag\\
	&\qquad + (2\wh{\omega}_1-1) \paren{ \nrisk_1  - \whnrisk_1}  \notag
\end{align}
Using (e) and then (a) for the second term, (d) for the third and (c) for the last, we get: 
\begin{align}
	L_1(\wh{\omvect})	&\leq \wh{L}_1(\wh{\omvect}) + 2c_1(x)\sqrt{2x}\sum_{k=2}^{B} \wh{\omega}_k \sqrt{\frac{\wh{q}_k}{N_1}}+C\rnrisk_1^2\paren{\frac{x}{\sqrt{\deamm_1}}+\frac{x}{\deff_1}} \notag\\
	& \leq \paren{ 1 \vee \frac{c_1(x)\sqrt{2x}}{8\sqrt{u_0}}}\min_{\omvect \in \Spx_B}
	\paren{  \wh{L}_1(\omvect) + 16\sqrt{u_0}\sum_{k=2}^{B} \omega_k \sqrt{\frac{\wh{q}_k}{N_1}} }+C\rnrisk_1^2\frac{x}{\sqrt{\deamm_1}}\,. \label{al:aux_qaggreg_gauss2}
\end{align}
The appearance of the minimum is a consequence of the definition of $\wh{\omvect}$.

{\bf Second step : conditional bound in expectation.}
We can now deduce, from the previous step, a bound in expectation conditionally to all samples except the first one. For any fixed $\om \in\Spx_B$, we first want to compare $\wh{L}_1(\omvect)$ to its conditional expectation $\e{\wh{L}_1(\omvect) \big| \cX^{(-1)}}$ which is equal to the conditional expectation of the loss $L_1$:
\begin{equation*}
	\e{\wh{L}_1(\omvect) \big| \cX^{(-1)}} = \norm[3]{ \sum_{k=2}^B \omega_k(\muNE_k - \mu_1 )}^2 + \omega_1^2 \nrisk_1 = \e{L_1(\omvect) \big|\cX^{(-1)}}.
\end{equation*}
For any fixed $\omvect\in \Spx_B$, 
as $x \geq 1$:
\begin{align}
	\wh{L}_1(\omvect)
	&= \norm[3]{ \sum_{k=2}^B \omega_k(\muNE_k - \muNE_1 )}^2 + (2\omega_1-1) \whnrisk_1 \notag \\
	&  = \norm[3]{\sum_{k=2}^B \om_k (\muNE_k - \mu_1) + (1-\om_1)(\mu_1 - \muNE_1)}^2 + (2\omega_1-1) \whnrisk_1 \notag \\
	&=\e{ L_1(\omvect)|\cX^{(-1)}} + 2 (1-\omega_1)\sum_{k=2}^{B} \omega_k\inner{ \muNE_k - \mu_1, \mu_1 - \muNE_1}  \notag \\
	& \qquad + (1-\omega_1)^2\paren{ \|\muNE_1- \mu_1\|^2 - \rnrisk_1^2} + (2\omega_1-1)(\whnrisk_1 - \nrisk_1) \notag \\
	& \leq \e{ L_1(\omvect) | \cX^{(-1)}} + 2\sqrt{2x} \sum_{k=2}^{B} \omega_k \sqrt{\frac{\mathfrak{q}_k}{N_1}} + C \frac{\rnrisk_1^2x}{\sqrt{\deamm_1}},
	\label{al:aux_qaggreg_gauss1}
\end{align}
using (c), (d), (e) again.
From (b), for all $k\in \intr{B}$, and using again $x \geq 1$: 
\begin{align}
	\sqrt{\wh{q}_k}&\leq
	\paren{ 1+ \sqrt{\frac{2x}{N_{1}-1}} }\paren{\sqrt{{\mathfrak{q}_k} +
			\frac{\rnrisk_1^4}{\deamm_1} N_1}
		+ \frac{\nrisk_1}{\deff_1} 
		\sqrt{2N_1 x}}
	\notag \\
	&\leq \paren{1+ \sqrt{\frac{2x}{N_1-1}}}\sqrt{\mathfrak{q}_k} + C\paren{ \sqrt{x}+ \frac{x}{\sqrt{N_1-1}}} \sqrt{\frac{N_1}{\deamm_1}}\nrisk_1. \label{eq:aux_qaggreg_gauss3}
\end{align}
Then, 
plugging \eqref{al:aux_qaggreg_gauss1} and \eqref{eq:aux_qaggreg_gauss3} into \eqref{al:aux_qaggreg_gauss2} , for all $\omvect\in \Spx_B$, 
as $x\geq 1$:
\begin{align*}
	L_1(\wh{\omvect}) &\leq  \paren{ 1 \vee \frac{c_1(x)\sqrt{2x}}{8\sqrt{u_0}}} \Bigg[ \e{L_1(\omvect)|\cX^{(-1)} }  \\
	& \;\; + \paren{ 2\sqrt{2x} + 16\sqrt{u_0}\paren{ 1+ \sqrt{\frac{2x}{N_1-1}}}} \sum_{k=2}^{B} \omega_k \sqrt{\frac{\mathfrak{q}_k}{N_1}} \\
	&\;\; +  \frac{ C\rnrisk_1^2}{\sqrt{\deamm_1}} \paren{ x +   C\sqrt{u_0}\paren{ \sqrt{x}+ \frac{x}{\sqrt{N_1-1}}}}\Bigg]. 
\end{align*}
By rearranging the terms and using that $u_0 \leq N_1-1$ and $x\geq 1$: 
\begin{multline*}
	L_1(\wh{\omvect}) \leq \paren{ 1 \vee \frac{c_1(x)\sqrt{2x}}{8\sqrt{u_0}}} \Bigg[ \e{L_1(\omvect)|\cX^{(-1)} } \\+ C\paren{ \sqrt{u_0}+ \sqrt{x}  } \sum_{k=2}^{B} \omega_k \sqrt{\frac{\mathfrak{q}_k}{N_1}}
	+  \frac{ C\rnrisk_1^2}{\sqrt{\deamm_1}} \paren{   \sqrt{u_0x} +x } \Bigg] =: \psi(x)P(x)
\end{multline*}
where $\psi(x) := 1 \vee \frac{c_1(x)\sqrt{2x}}{8\sqrt{u_0}}$ and $P$ is a degree 2 polynomial in $\sqrt{x}$ with coefficients that are constant conditionally to $\cX^{(-1)}$. We will denote the shifted version of $\psi$ and $P$ by $\psi_s$ and $P_s$, for $v \geq 0$:
\begin{equation}
	\psi_s(v) := \psi(v+\log(6B+4))\,, \quad P_s(v) = P(v+\log(6B+4)) \,.
\end{equation}
Both notations will be used depending on the case for the sake of readability. So for all $v \geq 0$
\[\prob{  L_1(\wh{\omvect}) \geq \psi_s(v)P_s(v) | \cX^{(-1)}}\leq e^{-v}.
\]
thanks to \eqref{eq:pAc} after taking $x = v + \log(6B+4) \geq 1$. Then there exists a random variable $\xi$ following an exponential distribution of parameter $1$ conditionally to $\cX^{(-1)}$, such that $ L_1(\wh{\omega}) \leq \psi_s(\xi)P_s(\xi)$ almost surely.
Let us first simplify the expression of $\psi$, recalling that by assumption $(N_1-1)/2\geq u_0 \geq \log(17B) \geq 1/2+ \log(6B+4) \geq 1/2+ \log(10) \geq 5/2$, then for $x \leq u_0$:
\begin{equation*}
	\sqrt{2}c_1(x) \leq \sqrt{2}\exp\paren{\frac{1}{2} + \frac{u_0}{N_1-1}} \leq  \sqrt{2}\exp\paren{\frac{1}{2} + \frac{1}{2}} \leq \sqrt{2}e \leq 8.
\end{equation*}
Thus, for $x \leq u_0$, $\psi(x)=1$. For $x \geq u_0$, it holds $c_1(x) \geq \sqrt{e} \geq 1$, so that:
\begin{align}
	\psi(x) &\leq  \frac{c_1(x)\sqrt{x}}{\sqrt{u_0}} \leq \exp\paren{ \frac{1}{2}+ \frac{x-u_0}{N_1-1} + \frac{u_0}{N_1-1}} \sqrt{\frac{x}{u_0}}\notag \\
	&\leq e\exp\paren{  \frac{x-\log(6B+4)}{5} }  \sqrt{\frac{x}{u_0}}.\label{al:aux_qaggreg_gauss5}
\end{align}
We can now bound the conditional expectation $\e{L_1(\wh{\omvect})| \cX^{(-1)}} $ separating the values before and after $u_0$:
\begin{align}
	\e{L_1(\wh{\omvect})| \cX^{(-1)}} &\leq \e{\psi_s(\xi)P_s(\xi) |\cX^{(-1)}} \notag \\
	&= \e{\psi_s(\xi)P_s(\xi)( \bm{1}_{\xi+\log(6B+4)\leq u_0}+ \bm{1}_{\xi+\log(6B+4)>u_0} )|\cX^{(-1)}} \notag\\
	& \leq P_s(u_0 - \log(6B+4))+ \e{\psi_s(\xi)P_s(\xi)\bm{1}_{\xi+\log(6B+4)>u_0} |\cX^{(-1)}} \notag \\
	& \leq P(u_0) + \e{ e\exp\paren{\xi/5} \sqrt{\frac{\xi+\log(6B+4)}{u_0}} P_s(\xi)\bm{1}_{\xi+\log(6B+4)>u_0} |\cX^{(-1)}} \,.\label{al:aux_qaggreg_gauss4}
\end{align}
We have used that $P$ (and $P_s$) is increasing on $\mbr_+$ ($P$ is a polynomial with positive coefficients) and the bound \eqref{al:aux_qaggreg_gauss5}. The second term in \eqref{al:aux_qaggreg_gauss4} can be upper bounded using Lemma~\ref{lem:aux_expo}, as $\sqrt{\xi+\log(6B+4)} P_s(\xi)$ can be seen as a polynomial of degree $3$ evaluated in $\sqrt{\xi + \log(6B+4)}$.
We apply \eqref{eq:lem_aux_polynomial} to this polynomial with $a = \log(6B+4)$, $\delta = u_0 - \log(6B+4)$, $\rho = 1/5$, $d=3$ and $\gamma = 1/2$. As $a \geq \log(10) \geq 2$ and $\delta \geq 1/2$, the condition required
by Lemma~\ref{eq:lem_aux_polynomial} is satisfied: $(\delta+a)(1-\rho)\geq 2\geq 3/2 =  \gamma d$. Then it holds:
\begin{multline}\label{al:aux_qaggreg_gauss6}
	\e{\exp\paren{\xi/5} \sqrt{\frac{\xi+\log(6B+4)}{u_0}} P_s(\xi)\bm{1}_{\xi+\log(6B+4)>u_0} |\cX^{(-1)}}  \\
	\leq C \sqrt{\frac{u_0}{u_0}} P_s(u_0 -\log(6B+4) ) e^{-(4/5)(u_0-\log(6B+4) )} \leq C P(u_0) B e^{-u_0/2}\,.
\end{multline}
Combining \eqref{al:aux_qaggreg_gauss4} and \eqref{al:aux_qaggreg_gauss6} and replacing $P(u_0)$ by its value, we obtain:
\begin{multline*}
	\e{L_1(\wh{\omvect}) | \cX^{(-1)}} \leq\e{L_1(\omvect)|\cX^{(-1)} }(1+CBe^{-u_0/2}) + C\sqrt{u_0} \sum_{k=2}^{B} \omega_k \sqrt{\frac{\mathfrak{q}_k}{N_1}}+ C\rnrisk_1^2  \frac{u_0}{\sqrt{\deamm_1}}.
\end{multline*}

{\bf Third step : unconditional bound.}
We now simply take the expectation with respect to $\cX^{(-1)}$. From the previous bound, using Jensen's inequality, for all $\omvect \in \Spx_B$:
\begin{align*}
	\e{L_1(\wh{\omvect})} \leq \e{L_1(\omvect) }(1+CBe^{-u_0/2}) + C\sqrt{u_0} \sum_{k=2}^{B} \omega_k \sqrt{\frac{\e{\mathfrak{q}_k}}{N_1}}+ C\rnrisk_1^2 \frac{u_0}{\sqrt{\deamm_1}}.
\end{align*}
We obtain  \eqref{eq:Qoracle} as $\e{\mathfrak{q}_k} = q_k$. \qed

\begin{lemma}\label{lem:aux_expo}
	Let $\xi \sim \cE(1)$ be an exponential random variable, and $\rho,a,\delta$ be positive real numbers. 
	Then for all $p\geq0$ such that $p< (\delta+a)(1-\rho)$, it holds:
	\begin{equation}\label{eq:lem_aux_expo}
		\e{ (\xi+a)^pe^{\rho\xi}\bm{1}_{\xi \geq \delta}} \leq  \paren{1-\rho - \frac{p}{a+\delta}}^{-1}(\delta+a)^p e^{- \delta(1-\rho)}\,.
	\end{equation}
	Let $P$ a polynomial of degree  $d$ and $\gamma>0$ such that $\gamma d< (\delta+a)(1-\rho)$, then:
	\begin{equation}\label{eq:lem_aux_polynomial}
		\e{ P((\xi+a)^\gamma)e^{\rho\xi}\bm{1}_{\xi \geq \delta}} \leq  \paren{1-\rho - \frac{d\gamma}{a+\delta}}^{-1}P((\delta+a)^\gamma) e^{- \delta(1-\rho)}\,.
	\end{equation}
\end{lemma}
\begin{proof}
	As $p< (\delta+a)(1-\rho)$, then $p< (\delta+a)(1-\rho-\varepsilon)$ for all $\varepsilon < 1-\rho-p/(a+\delta)$.
	The function $x \mapsto F(x) := (x+a)^pe^{(\rho - (1-\varepsilon))x}$ on $\mbr_+$ attains its maximum in $x_* := p\paren{ 1- \rho-\varepsilon}^{-1}-a$ and then decreases to $0$. As $x_* < \delta$,
	we have $F(x) \leq F(\delta)$ for all $x\geq \delta$, thus: 
	\begin{align*}
		\e{ (\xi+a)^pe^{\rho\xi}\bm{1}_{\xi \geq \delta}}
		= \e{F(\xi) e^{(1-\varepsilon)\xi}\bm{1}_{\xi \geq \delta}} \leq F(\delta)
		\e{e^{(1-\varepsilon)\xi}\bm{1}_{\xi \geq \delta}} = (\delta+a)^pe^{-(1-\rho)\delta}\varepsilon^{-1}\,.
	\end{align*}
	As the inequality is true for all $\varepsilon <   1-\rho-p/(a+\delta)$ we get \eqref{eq:lem_aux_expo}. Equation \eqref{eq:lem_aux_polynomial} is obtained by applying \eqref{eq:lem_aux_expo} to each of the monomials of degree $k\leq d$ as $k\gamma \leq d \gamma < (\delta+a)(1-\rho)$, upper bounding the first factor and summing.
\end{proof}

\subsection{Proofs of Corollary~\ref{cor:stein} and Corollary~\ref{cor:aistats}}

\paragraph{\textbf{Proof of Corollary~\ref{cor:stein}}}
According to Theorem~\ref{prop:qaggreg_gauss}, for $B=2$, $\mu_2 =0$ and $\Sigma_2=0$; for all $\omega_1 \in (0,1)$:
\begin{equation*}
	R_1(\wh{\omega}) \leq \paren{ (1-\omega_1)^2 \|\mu_1\|^2+ \omega_1\rnrisk_1^2   + 2(1-\omega_1)\eta }(1+ Ce^{-u_0/2}) + C\rnrisk_1^2 \sqrt{\frac{u_0}{\deff_1}} , \,,
\end{equation*}
where $\eta  = C \frac{\|\mu_1\|\rnrisk_1}{\sqrt{\deff_1}}\sqrt{u_0}$. Let us choose $\omega_1 = \min\paren{\frac{\|\mu_1\|^2+ \eta}{\|\mu_1\|^2+ \rnrisk_1^2},1}$. Then if $\eta \leq \rnrisk_1^2$:
\begin{align*}
	R_1(\wh{\omega})&\leq (1+ Ce^{-u_0/2})\frac{\|\mu_1\|^2\rnrisk_1^2+ 2\rnrisk_1^2 \eta - \eta^2}{\|\mu_1\|^2+ \rnrisk_1^2}+C\rnrisk_1^2 \sqrt{\frac{u_0}{\deff_1}}\\
	& \leq (1+ Ce^{-u_0/2})\frac{\|\mu_1\|^2\rnrisk_1^2}{\|\mu_1\|^2+ \rnrisk_1^2}+ C\nrisk_1\sqrt{\frac{u_0}{\deff_1}}\frac{2\|\mu_1\|\rnrisk_1}{\|\mu_1\|^2+\nrisk_1}+C\rnrisk_1^2 \sqrt{\frac{u_0}{\deff_1}} \\
	&\leq  (1+ Ce^{-u_0/2}) \frac{\|\mu_1\|^2\rnrisk_1^2}{\|\mu_1\|^2+ \rnrisk_1^2} +  C\rnrisk_1^2 \sqrt{\frac{u_0}{\deff_1}} \,,
\end{align*}
where we have used that $2ab \leq a^2+ b^2$. Otherwise, if $\eta \geq \rnrisk_1^2$:
\begin{align*}
	R_1(\wh{\omega}) &\leq \rnrisk_1^2 (1+ Ce^{-u_0/2})  + C\rnrisk_1^2 \sqrt{\frac{u_0}{\deff_1}} \\
	& \leq (1+ Ce^{-u_0/2}) \frac{\|\mu_1\|^2\rnrisk_1^2}{\|\mu_1\|^2+ \rnrisk_1^2}  + (1+ Ce^{-u_0/2}) \frac{\rnrisk_1^4}{\|\mu_1\|^2+ \rnrisk_1^2} + C\rnrisk_1^2 \sqrt{\frac{u_0}{\deff_1}}\,.
\end{align*}
We conclude using that $\rnrisk_1^2 \leq C \frac{\|\mu_1\|^2}{\deff_1}u_0$ in this case. \\

\textbf{Proof of Corollary~\ref{cor:aistats}.}
Let $\tau\geq 0,\cteW\geq 1 $ be fixed. Let $k$ be an element of $V_{\tau,\cteW} = W_{(\cteW)}\cap V_\tau$ with $k\neq 1$. We start by upper bounding $q_k$, with $q_k$ defined in~\eqref{eq:def_Qw}.
Since $k \in W_{(\cteW)}$, it holds $\tr \Sigma_k^2 \leq \cteW^2 \frac{N_k^2}{N_1^2} \tr \Sigma_1^2$, so that
\begin{align*}
	\tr \Sigma_1 \Sigma_k  \leq \frac{1}{2}\paren{ \frac{N_k}{N_1} \tr \Sigma^2_1+ \frac{N_1}{N_k}\tr \Sigma_k^2}
	& \leq \frac{1+\cteW^2}{2} \frac{N_k}{N_1}\tr \Sigma_1^2 \\
	& \leq \frac{N_k}{N_1}\frac{(1+\cteW^2)(\tr \Sigma_1)^2}{2\deamm_1} \\
	& = N_k N_1\frac{\cteW^2\rnrisk_1^4}{\deamm_1}.
\end{align*}
Since $k \in V_\tau$, it holds
\[
\frac{\Delta_k^T \Sigma_1 \Delta_k}{N_1}  \leq \frac{\norm{\Sigma_1}_\infty}{N_1} \norm{\Delta_k}^2
\leq \frac{\tr \Sigma_1}{N_1} \frac{1}{\deff_1} \tau \nrisk_1 = \frac{ \tau \rnrisk_1^4}{\deff_1}.
\]
Joining these estimates, we get
\[
\frac{q_k}{N_1} \leq \frac{\Delta_k^T \Sigma_1 \Delta_k}{N_1} + \frac{\tr \Sigma_1 \Sigma_k}{N_1 N_k} \\
\leq \rnrisk^4_1 \paren{ \frac{\tau}{\deff_1} + \frac{\cteW^2}{\deamm_1}}.
\]
Therefore, for $\omvect$ a vector of the simplex $\cS_B$ having support in $W^{(\cteW)} \cap V_\tau$, using $\deff_1 \leq \deamm_1$ it holds
\begin{equation} \label{eq:Qbound}
	Q_1(\omvect) =  \sum_{k\geq 2} \om_k \sqrt{\frac{q_k}{N_1}} \leq (1-\om_1)  \sqrt{\tau + \cteW^2 } \frac{\nrisk_1}{\sqrt{\deff_1}}.
\end{equation}
We now choose the weight vector $\omvect^* = \omvect^*_{V_{\tau,c}}$ given by the oracle weights of~\eqref{eq:STBoptweights}, for the set $V=V_{\tau,\cteW}$. From Lemma~\ref{lem:oraclebound}, this
gives rise to $R_1(\omvect^*) \leq \cB(\tau,\nu)$, where $\nu = \nu(V_{\tau,\cteW})$; furthermore we have the explicit expression
\[
(1-\om^*_1) = \lambda(1-\nu), \qquad \text{ where } \lambda = \frac{1}{1+\tau(1-\nu)},
\]
so that it holds (since $\nu \in [0,1]$)
\[
(1-\om^*_1) \sqrt{\tau} = \frac{(1-\nu)\sqrt{\tau}}{1 + \tau(1-\nu)} \leq \max \paren{\frac{\tau (1-\nu)}{1 + \tau(1-\nu)}, \frac{\sqrt{\tau}(1-\nu)}{1 + \sqrt{\tau}(1-\nu)}} \leq 1.
\]
Plugging this into~\eqref{eq:Qbound}, we get $Q_1(\omvect^*) \leq 2 \cteW \nrisk_1/\sqrt{\deff_1}$, then~\eqref{eq:aistats} since the obtained estimate holds
for any $\tau \geq 0, \cteW \geq 1$.

\subsection{Proof of Theorem~\ref{prop:qaggreg_bounded}}
\Revision{We first demonstrate an oracle inequality for the relative risk in the bounded setting (Theorem~\ref{thm:oracle_ineq_BS} below), and then will choose appropriate weights to prove Theorem~\ref{prop:qaggreg_bounded}.
	
	\begin{theorem}\label{thm:oracle_ineq_BS}
		Assume (\BS).
		Let $u_0 \geq 2\log N_1 + \log (B)$, and $\wh{\omvect}$ as defined in \eqref{eq:def_w_bnd}. Then it holds:
		\begin{align}
			\frac{R_1(\wh{\omvect})}{\nrisk_1} \leq& \frac{1}{\nrisk_1}\min_{\omvect \in \Spx_B} \paren{
				R_1(\omvect) + 4 \sqrt{2u_0} Q_1(\omvect) + 1424 u_0 \sum_{k=2}^{B} \omega_k \frac{M\paren{ \|\mu_k-\mu_1\|+\rnrisk_k }}{N_1}} \notag\\
			&+ C\frac{u_0+ \sqrt{u_0 }}{\sqrt{\deamm_1}} + C\sqrt{\frac{\ratiobs_1}{N_1}}u_0+  C\frac{\ratiobs_1}{N_1}u_0 ^2\,,  \label{al:thmbnd_oracle}
		\end{align}
		where $C>0$ is an absolute constant, $Q_1$ is defined in \eqref{eq:def_Qw} and $\ratiobs_1:= M^2/\tr \Sigma_1$.	
	\end{theorem}
}

\begin{proof}
	We follow the same general proceeding as in the proof of Theorem~\ref{prop:qaggreg_gauss}.
	
	{\bf First step : bound in conditional probability.}
	Let us recall the definitions of $Q^\BS(\omvect)$ and $\wh{q}_k$:
	\begin{equation*}
		\wh{Q}^\BS(\omvect) := \frac{M}{N_1} \sum_{k=2}^{B} \omega_k \|\muNE_k - \muNE_1 \|,\quad \wh{q}_k = \frac{1}{N_1-1} \sum_{p=1}^{N_1} \inner{ \muNE_k -\muNE_1, X_p^{(1)} - \muNE_1}^2\,.
	\end{equation*}
	We will need the following quantity $\wh{q}_k'$ which is close to $\wh{q}_k$ but easier to control: 
	\begin{equation*}
		\wh{q}_k' = \frac{1}{N_1-1} \sum_{p=1}^{N_1} \inner{ \muNE_k -\mu_1, X_p^{(1)} - \muNE_1}^2\,.
	\end{equation*}
	The estimated weight vector $\wh{\omvect}$ for the estimation of $\mu_1$ is chosen as
	\begin{equation*}
		\wh{\omvect} \in \argmin_{\omvect \in \Spx_B} \paren{ \wh{L}_1(\omvect)+ 4 \sqrt{2u_0} \wh{Q}_1(\omvect) + 1424 u_0 \wh{Q}^\BS(\omvect)}.
	\end{equation*}
	Let $u := u_0 - \log B$, and define the events:
	\begin{gather*}
		A_1 = \set{ \|\muNE_k-\mu_1\|_{\Sigma_1} \leq 2 \sqrt{\wh{q}'_k} + 711 \frac{\|\muNE_k-\mu_1\|M}{\sqrt{N_1}}(u+ \log B),
			2 \leq k \leq B
		}\,, \\
		A_2 =  \set{\abs{ \|\muNE_1-\mu_1\|^2 - \whnrisk_1} \leq C\frac{\rnrisk_1^2}{\sqrt{\deamm_1}}u + C\frac{M^2}{N_1^2}u ^2  }\,,
	\end{gather*}
	and
	\begin{multline*}
		A_3 = \Bigg\{ \inner{ \muNE_k - \mu_1,\muNE_1 - \mu_1 } \leq \sqrt{2\frac{u+ \log B}{N_1} }\|\muNE_k - \mu_1\|_{\Sigma_1} \\
		+ \frac{2\|\muNE_k - \mu_1\|M}{3N_1}(u + \log B),      2 \leq k \leq B \Bigg\},
	\end{multline*}
	where we recall that for $\nu$ a vector and $\Sigma$ an operator, $\|\nu\|^2_\Sigma := \inner{\nu,\Sigma \nu}$.
	For $i\in \{1,3\}$, $\prob{A_i|\cX^{(-1)} } \geq 1- e^{-u}$ and  $\prob{A_2|\cX^{(-1)} } \geq 1- 2e^{-u}$ because of Proposition~\ref{prop:concwhqbnd} for $A_1$, Lemma~\ref{lem:bernstein} for $A_3$ and for $A_2$, because $\|\muNE_1-\mu_1\|^2 - \whnrisk_1$ is a U-statistic:
	\begin{equation}\label{eq:ustat}
		\|\muNE_1-\mu_1\|^2 - \whnrisk_1 = \frac{1}{N_1(N_1-1)} \sum_{\ell\neq p=1}^{N_1} \inner{ X^{(1)}_\ell - \mu_1, X^{(1)}_p - \mu_1}\,,
	\end{equation}
	the concentration is a direct consequence of \citet{Hou03} (or see Proposition 9 in \citealp{BlaFer23} for this specific statistic). Then the event $A = A_1\cap A_2 \cap A_3$ conditionally to $\cX^{(-1)}$ is of probability greater than $1-4e^{-u}$. \\
	The differences between respectively $\wh{q}_k$ and $\wh{q}'_k$ for $k\in\intr{B}$ can be bounded independently of $k$ :
	\begin{equation}\label{eq:defdeltaq}
		\abs{ \sqrt{\wh{q}_k} - \sqrt{\wh{q}'_k}} \leq \sqrt{\frac{1}{N_1(N_1-1)} \sum_{p=1}^{N_1} \inner{\muNE_1 - \mu_1, X_p -\muNE_1}^2} =: \Deltaq.
	\end{equation}
	Assume $A$, then:
	\begin{align*}
		L_1(\wh{\omvect} )  &  = \wh{L}_1(\wh{\omvect}) + 2 \sum_{k=2}^{B} \wh{\omega}_k \inner{ \muNE_k - \mu_1, \muNE_1 - \mu_1 } + (2\wh{\omega}_1 - 1) \paren{ \|\muNE_1 - \mu_1 \|^2 - \whnrisk_1 } \\
		& \leq \wh{L}_1(\wh{\omvect}) + 2 \sum_{k=2}^{B}  \wh{\omega}_k \paren{\sqrt{2\frac{u+ \log B}{N_1} }\|\muNE_k - \mu_1\|_{\Sigma_1} + \frac{2\|\muNE_k - \mu_1\|M}{3N_1}(u + \log B)} \\
		&+ \frac{C\rnrisk_1^2}{\sqrt{\deamm_1}}u + \frac{CM^2}{N_1^2}u ^2,
	\end{align*}
	where we have used the events $A_2$ and $A_3$. Then using the event $A_1$, the bound \eqref{eq:defdeltaq} and a triangle inequality we get:
	\begin{align*}
		L_1(\wh{\omvect} )	& \leq \wh{L}_1(\wh{\omvect}) + 4 \sqrt{\log B +u} \sum_{k=2}^{B}  \wh{\omega}_k \sqrt{\frac{2\wh{q}_k}{N_1}} + 1424 \paren{\log B +u} \sum_{k=2}^{B}  \wh{\omega}_k \frac{M \|\muNE_k - \muNE_1\|}{N_1}  \\
		&+ C\frac{\Deltaq}{\sqrt{N_1}} \sqrt{\log B + u} + C\frac{\|\muNE_1- \mu_1\|M}{N_1} (\log B + u ) + \frac{C\rnrisk_1^2}{\sqrt{\deamm_1}}u + \frac{CM^2}{N_1^2}u ^2.
	\end{align*}
	Using the choice of $\wh{\omvect}$, conditionally to $A$:
	\begin{align*}
		L_1(\wh{\omvect} )	& \leq  \min_{\omvect \in \Spx_B}\paren{\wh{L}_1(\omvect) + 4 \sqrt{2u_0} \wh{Q}(\omvect) + 1424 u_0 \wh{Q}^\BS(\omvect)}  \\
		&+ C\frac{\Deltaq}{\sqrt{N_1}} \sqrt{\log B + u} + C\frac{\|\muNE_1- \mu_1\|M}{N_1} (\log B + u ) + \frac{C\rnrisk_1^2}{\sqrt{\deamm_1}}u + \frac{CM^2}{N_1^2}u ^2.
	\end{align*}
	{\bf Second and third steps: bound in expectation.} Let us bound some expectation using Jensen's inequality:
	\begin{equation}\label{eq:deltaq_expectation}
		\e{\sqrt{\wh{q}_k'}} \leq \sqrt{\|\mu_k - \mu_1\|^2_{\Sigma_1} + \frac{\tr(\Sigma_1\Sigma_k)}{N_k}}\,, \quad \e{\Deltaq} \leq \frac{M \sqrt{\tr \Sigma_1}}{N_1} + \frac{\sqrt{\tr \Sigma_1^2}}{\sqrt{N_1}}.
	\end{equation}
	The expectation of $\sqrt{\wh{q}_k}$ can be bounded using that $\sqrt{\wh{q}_k} \leq \sqrt{\wh{q}'_k} + \Deltaq$. We can now bound the risk. Let $\omvect \in \Spx_B$:
	\begin{align*}
		R_1(\wh{\omvect}) &\leq \e{ L_1(\wh{\omvect})\bm{1}_A} + M^2\prob{A^c} \\
		&\leq L_1(\omvect) + 4 \sqrt{2u_0} \sum_{k=2}^{B}\omega_k \frac{\e{\sqrt{\wh{q}_k} }}{\sqrt{N_1}}  + 1424 u_0 \sum_{k=2}^{B} \omega_k \frac{M\paren{ \|\mu_k-\mu_1\|+\rnrisk_1+\rnrisk_k }}{N_1}\\
		& + C\frac{\e{\Deltaq}}{\sqrt{N_1}} \sqrt{\log B + u} + C\frac{\rnrisk_1 M}{N_1} (\log B + u ) + C\frac{\rnrisk_1^2}{\sqrt{\deamm_1}}u + C\frac{M^2}{N_1^2}u ^2 + 3M^2e^{-u}
	\end{align*}
	Because $u \geq 2\log N_1$, the last term is upper bounded by the previous one. Using \eqref{eq:deltaq_expectation} and by bringing together the terms:
	\begin{align}
		R_1(\wh{\omvect}) &\leq R_1(\omvect) + 4 \sqrt{2(\log B +u)} Q_1(\omvect) + 1424 \paren{\log B +u} \sum_{k=2}^{B} \omega_k \frac{M\paren{ \|\mu_k-\mu_1\|+\rnrisk_k }}{N_1} \notag\\
		&+ C\frac{\rnrisk_1^2}{\sqrt{\deamm_1}}(u+ \sqrt{\log B +u }) + C\frac{M\rnrisk_1}{N_1}(\log B +u)+  C\frac{M^2}{N_1^2}u ^2\,,  \label{al:eqaux1}
	\end{align}
	where $Q_1$ is defined in \eqref{eq:def_Qw}. \Revision{After dividing by $\nrisk_1$ and using that $u\leq \log B + u \leq u_0$, we obtain \eqref{al:thmbnd_oracle}.}
\end{proof}

{\bf Proof of Theorem~\ref{prop:qaggreg_bounded}}

Let $\tau,\cteW >0$ and $\omvect^* = \omvect^*_{V_{\tau,\cteW}}$ be defined as in \eqref{eq:STBoptweights}. Then as in the proof of Corollary~\ref{cor:aistats}:
\begin{equation}\label{eq:aux1}
	R_1(\omvect^*) = \nrisk_1 \cB\paren{\tau,\nu \paren{V_{\tau,\cteW}}}\,,\quad  Q_1(\omvect^*) \leq C \sqrt{\frac{1+\cteW^2}{\deff_1}} \nrisk_1.
\end{equation}
Up to bound the third term in the upper bound \eqref{al:eqaux1}, let us bound $\nrisk_k$ for $k\in V_{\tau,\cteW}$. On the one hand:
\begin{equation*}
	\nrisk_k = \frac{\tr \Sigma_k}{N_k} \leq \frac{4M^2}{N_k} = 4\tr \Sigma_1 \frac{\ratiobs_1}{N_k} = 4 \nrisk_1 \frac{\ratiobs_1N_1}{N_k}\,.
\end{equation*}
On the other hand, as $k \in V_{\tau,\cteW} \subset W_{(\cteW)}$:
\begin{equation*}
	\nrisk_k = \frac{\tr \Sigma_k}{N_k} = \sqrt{\deamm_k} \frac{\sqrt{\tr \Sigma^2_k}}{N_k} \leq \sqrt{\deamm_k} \cteW \frac{\sqrt{\tr \Sigma^2_1}}{N_1} = \nrisk_1 \cteW \sqrt{\frac{\deamm_k}{\deamm_1}}\,.
\end{equation*}
Combining these two bounds:
\begin{align*}
	\nrisk_k \leq 4\nrisk_1 \min \paren{ \frac{\ratiobs_1N_1}{N_k}, \cteW \sqrt{\frac{\deamm_k}{\deamm_1}}}\,.
\end{align*}
As we assume $N_k \geq (\deamm_k)^\beta$, for $k\in V_{\tau,\cteW}$:
\[
\nrisk_k \leq 4\nrisk_1 \min \paren{ \frac{\ratiobs_1N_1}{(\deamm_k)^\beta}, \cteW \sqrt{\frac{\deamm_k}{\deamm_1}}} \leq 4\nrisk_1 \max_{d \geq 1}  \min \paren{ \frac{\ratiobs_1N_1}{d^\beta}, \cteW \sqrt{\frac{d}{\deamm_1}}} = 4 \nrisk_1 \paren{\ratiobs_1N_1}^{\frac{1}{1+2\beta}} \paren{\frac{\cteW}{\sqrt{\deamm_1}}}^{\frac{2\beta}{1+2\beta}} \,.
\]
We can now bound the third term in \eqref{al:eqaux1}. As $\omega^*_k=0$ for $k\notin V_{\tau,\cteW}$: 
\begin{align*}
	\sum_{k=2}^{B} \omega^*_k \frac{M\paren{ \|\mu_k-\mu_1\| + \rnrisk_k }}{N_1} & \leq 	\frac{M}{N_1}  (1-\omega^*_1) \paren{ \sqrt{\tau} \rnrisk_1 +   2 \rnrisk_1 \paren{\ratiobs_1N_1}^{\frac{1}{2(1+2\beta)}} \paren{\frac{\cteW}{\sqrt{\deamm_1}}}^{\frac{\beta}{1+2\beta}} } \\
	& \leq \nrisk_1 \paren{ (1-\omega_1^*)\sqrt{\frac{\tau\ratiobs_1}{N_1}}+2 \ratiobs_1^{\frac{1+\beta}{1+2\beta}} \paren{\frac{\cteW}{N_1\sqrt{\deamm_1}}}^{\frac{\beta}{1+2\beta}}}\,.
\end{align*}
As $N_1 \geq (\deamm_1)^{\beta}$ and $(1-\omega^*_1)\sqrt{\tau} \leq 1$ (by definition of $\omega_1^*$),
we get:
\begin{equation}\label{eq:aux2}
	\sum_{k=2}^{B} \omega^*_k \frac{M\paren{ \|\mu_k-\mu_1\| + \rnrisk_k }}{N_1} \leq 2\nrisk_1 \paren{\frac{\sqrt{\ratiobs_1}}{(\deamm_1)^{\beta/2}} +	\frac{\ratiobs_1^{\frac{1+\beta}{1+2\beta}} \cteW^{\frac{\beta}{1+2\beta}}}{(\deamm_1)^{\beta/2}}}.
\end{equation}
Injecting the bounds \eqref{eq:aux1} and \eqref{eq:aux2} into \eqref{al:eqaux1} leads to:
\begin{align*}
	\frac{R_1(\wh{\omvect})}{\nrisk_1} &\leq \min_{\tau>0,\cteW>0} \paren{ \cB\paren{\tau,\nu(V_{\tau,\cteW})} + C \cteW\sqrt{\frac{u_0}{\deff_1}} + Cu_0\frac{\ratiobs_1^{\frac{1+\beta}{1+2\beta}} \cteW^{\frac{\beta}{1+2\beta}}}{(\deamm_1)^{\beta/2}} } \\
	&+Cu_0\frac{\sqrt{\ratiobs_1}}{(\deamm_1)^{\beta/2}} + C\sqrt{\frac{u_0}{\deff_1}} + C \frac{u_0}{\sqrt{\deamm_1}} + C\frac{u_0 \sqrt{\ratiobs_1}}{\sqrt{N_1}} + C\frac{\ratiobs_1u^2}{N_1} \\
	& \leq \min_{\tau>0,\cteW>0} \paren{ \cB\paren{\tau,\nu(V_{\tau,\cteW})} + C\cteW \sqrt{\frac{u_0}{\deff_1}} + Cu_0\frac{\ratiobs_1^{\frac{1+\beta}{1+2\beta}} \cteW^{\frac{\beta}{1+2\beta}}}{(\deamm_1)^{\beta/2}} } + C\sqrt{\frac{u_0}{\deff_1}} + C \frac{u_0\ratiobs_1}{(\deamm_1)^{\beta/2}}.
\end{align*}
where we have used for the last inequality that $u\leq u_0 \leq \sqrt{N_1}$. 

As $\ratiobs_1^{\frac{1+\beta}{1+2\beta}} \cteW^{\frac{\beta}{1+2\beta}} \leq \max(\ratiobs_1,\cteW) \leq \ratiobs_1 +\cteW $, we obtain:
\begin{equation*}
	\frac{R_1(\wh{\omvect})}{\nrisk_1} \leq \min_{\tau>0,\cteW>0} \paren{ \cB\paren{\tau,\nu(V_{\tau,\cteW})} + C\cteW \max\paren{\sqrt{\frac{u_0}{\deff_1}}, \frac{u_0}{(\deamm_1)^{\beta/2}}}} + C\sqrt{\frac{u_0}{\deff_1}} + C \frac{u_0\ratiobs_1}{(\deamm_1)^{\beta/2}}.
\end{equation*}\qed

\subsection{Concentration inequalities}
\subsubsection{Concentration for $\wh{q}$.}
Consider first the Gaussian setting (\GS).
\begin{proposition}\label{prop:concwhqgauss}
	Let $X_1,\ldots, X_N$ i.i.d. Gaussian random vectors of distribution $\cN(\mu_1,\Sigma_1)$ and $\nu \in \mbr^d$. Let $\wh{q} = \frac{1}{N-1} \sum_{k=1}^{N} \inner{\muNE_1 - \nu, X_k - \muNE_1}^2$, then for all $x \geq 0$:
	\begin{equation}\label{eq:concwhqup}
		\prob{\sqrt{\wh{q}} \geq \paren{ 1+ \sqrt{\frac{2x}{N-1}} }\paren{\sqrt{\norm{\mu_1-\nu}_{ \Sigma_1}^2 + \frac{\tr \Sigma_1^2}{N} } + \|\Sigma_1\|_{\infty} \sqrt{\frac{2x}{N}}}  } \leq 2e^{-x}\,,
	\end{equation}
	and
	\begin{equation}\label{eq:concwhqlow}
		\prob{\sqrt{\wh{q}} \leq e^{-1/2-x/(N-1)}\paren{\sqrt{\norm{\mu_1-\nu}_{ \Sigma_1}^2+ \frac{\tr \Sigma_1^2}{N} } - 2\|\Sigma_1\|_{\infty} \sqrt{\frac{2x}{N}}}} \leq 2e^{-x}\,,
	\end{equation}
	where $\|\mu_1-\nu\|_{\Sigma_1}^2 = (\mu_1-\nu)^T\Sigma_1(\mu_1-\nu)$.
\end{proposition}

\begin{proof}
	Let us consider the random vector $Z \in \mbr^N$ with $Z_k = \inner{ \muNE_1-\nu, X_k - \muNE_1}$, then $\wh{q}= \norm{Z}^2_N/(N-1)$, where $\|\cdot \|_N$ is the Euclidian norm in $\mbr^N$. Conditionally to $\muNE_1$, $Z$ is a Gaussian vector of distribution $\cN(0, \ealpha \Gamma)$, where $\ealpha = (\muNE_1-\nu)^T \Sigma_1 (\muNE_1 - \nu)$ and $\Gamma = I_N - \mathbf{1}_N\mathbf{1}_N^T/N$ with $\mathbf{1}_N = (1, \ldots ,1) \in \mbr^N$. The eigenvalues of $\Gamma$ are $1$ with multiplicity $N-1$ and $0$. So $\|Z\|^2/\ealpha$ has a $\chi^2(N-1)$ distribution. Then conditionally to $\muNE_1$:
	\begin{equation*}
		\widehat{q} = \frac{\|Z\|^2}{N-1} \sim \frac{\ealpha}{N-1} \chi^2(N-1)\,.
	\end{equation*}
	Then according to Lemma~\ref{lem:concnorm2gauss} and Lemma~\ref{lem:conckhi2}, for all $x \geq 0$:
	\begin{equation*}
		\prob{ \sqrt{\frac{\wh{q}}{\ealpha}} \geq 1+ \sqrt{\frac{2x}{N-1}} \Big|\muNE_1 } \leq e^{-x}\,,\quad  \prob{ \sqrt{\frac{\wh{q}}{\ealpha}} \leq e^{-1/2}e^{-x/(N-1)} \Big| \muNE_1 } \leq e^{-x}.
	\end{equation*}
	Let $g = \Sigma_1^{1/2}(\muNE -\nu) \sim \cN( \Sigma_1^{1/2} (\mu_1 - \nu), \Sigma_1^2/N )$, as $\|g\|^2 = \ealpha$, from Lemma~\ref{lem:concnorm2gauss} with $ \Sigma_1^{1/2} (\mu_1 - \nu) \to \mu$ and $ \Sigma_1^2/N \to \Sigma$, we get that for all $x\geq 0$:
	\begin{gather*}
		\prob{ \sqrt{\ealpha} \geq \sqrt{(\mu_1-\nu)^T \Sigma_1 (\mu_1- \nu) + \frac{\tr \Sigma_1^2}{N} } + \|\Sigma_1\|_{\infty} \sqrt{\frac{2x}{N}}} \leq e^{-x}\,, \\
		\prob{ \sqrt{\ealpha} \leq \sqrt{(\mu_1-\nu)^T \Sigma_1 (\mu_1- \nu) + \frac{\tr \Sigma_1^2}{N} } - 2\|\Sigma_1\|_{\infty} \sqrt{\frac{2x}{N}}} \leq e^{-x}\,.
	\end{gather*}
	We have used that for all $\mu \in \mbr^d$, $\Sigma \in \mbr^{d\times d}$ and $x\geq 0$:
	\begin{align*}
		\paren{ \sqrt{\|\mu\|^2+ \tr \Sigma}+ \sqrt{2\|\Sigma\|_{\infty}x}}^2\geq \paren{\norm{\mu}^2 + \tr \Sigma} + 2 \sqrt{\paren{ \tr \Sigma^2 +2 \mu^T \Sigma \mu}x} + 2 \norm{\Sigma}_{\infty}x  \,, \\
		\paren{ \sqrt{\|\mu\|^2+ \tr \Sigma}- 2\sqrt{2\|\Sigma\|_{\infty}x}}_+^2
		\leq \paren{\paren{\norm{\mu}^2 + \tr \Sigma} - 2 \sqrt{\paren{ \tr \Sigma^2 +2 \mu^T \Sigma \mu}x}}_+\,,
	\end{align*}
	as $(a-b)_+^2 \leq (a^2-ab)_+$ for $a,b>0$. \\
	Equations~\eqref{eq:concwhqup} and \eqref{eq:concwhqlow} are obtained  by combining these concentration inequalities.
\end{proof}
In the bounded setting (\BS),
Proposition~\ref{prop:concwhqbnd} gives a concentration bound for $\wh{q}'$, which is a slightly different statistic from $\wh{q}$ because we consider $\mu_1-\nu$ known for $\wh{q}'$.
\begin{proposition}\label{prop:concwhqbnd}
	Assume (\BS), let $\nu \in \mbr^d$ and $\wh{q}' = \frac{1}{N-1} \sum_{k=1}^{N} \inner{\mu_1 - \nu, X_k - \muNE_1}^2$. Then for all $u \geq 1$:
	\begin{equation*}
		\prob{ 2\sqrt{\wh{q}'}\leq \sqrt{(\mu_1-\nu)\Sigma_1(\mu_1-\nu)} - 711\frac{\|\mu_1-\nu\|M}{\sqrt{N-1}}u} \leq e^{-u}.
	\end{equation*}
\end{proposition}

\begin{proof}
	Let us first denote $\delta := \mu_1 - \nu$ and $Z':= \sqrt{\wh{q}'}$. We are going to use Talagrand's inequality (Theorem~\ref{thm:talagrand}). So let us first rewrite $Z'$:
	\begin{align*}
		Z'& = \sup_{\|v\|_N=1} \frac{1}{\sqrt{N-1}} \sum_{k=1}^{N} v_k \inner{\delta ,X_k - \muNE_1 } \\
		&= \sup_{\|v\|_N=1} \frac{1}{\sqrt{N-1}} \sum_{k=1}^{N} \inner{\delta,X_k - \mu_1} \paren{ v_k- \frac{1}{N}\sum_{q=1}^{N}v_q }.
	\end{align*}
	
	Let $T = \{ v \in \mbr^N, \|v\|_N=1 \}$ (or a countable dense subset) and define for $v \in T$:
	\begin{equation*}
		X_k^v:= \frac{1}{\sqrt{N-1}}\inner{\delta ,X_k - \mu_1} \paren{ v_k- \frac{1}{N}\sum_{q=1}^{N}v_q },
	\end{equation*}
	then:
	\begin{align*}
		|X_k^v| \leq \frac{2\|\delta\|M}{\sqrt{N-1}} \,, \quad \sup_{v\in T} \sum_{k=1}^{N}\e{ (X_k^v)^2}  \leq \frac{\delta^T\Sigma\delta}{N-1}\leq\frac{4\|\delta\|^2M^2}{N-1} \,.
	\end{align*}
	Using Theorem~\ref{thm:talagrand}, with probability greater than $1-e^{-u}$, $u\geq 1$:
	\begin{equation*}
		Z' \geq \e{Z'}(1- \varepsilon) - C(\varepsilon) \frac{\|\delta\|M}{\sqrt{N-1}}u\,,
	\end{equation*}
	where $C(\varepsilon) = 8(2+\varepsilon^{-1})$ for some $\varepsilon >0$. We just need to lower bound $\e{Z'}$ by $\sqrt{\e{(Z')^2}} = \sqrt{\delta^T\Sigma_1\delta}$. For that, using again Talagrand's inequality, it exists an exponential random variable $\xi \sim \cE(1)$ such that:
	\begin{equation*}
		Z' \leq  \e{ Z' }(1+ \varepsilon) + C(\varepsilon) \frac{\|\delta\|M}{\sqrt{N-1}} \xi
	\end{equation*}
	Then:
	\begin{align*}
		\e{(Z')^2} &\leq \e{ \paren{\e{ Z' }(1+ \varepsilon) + C(\varepsilon) \frac{\|\delta\|M}{\sqrt{N-1}} \xi}^2} \\
		&\leq \paren{\e{ Z'}(1+ \varepsilon) + \sqrt{2}C(\varepsilon) \frac{\|\delta\|M}{\sqrt{N-1}} }^2,
	\end{align*}
	and we get that $(1+\varepsilon)\e{ Z'} \geq \sqrt{ \e{(Z')^2}} - \sqrt{2}C(\varepsilon) \frac{\|\delta\|M}{\sqrt{N-1}} $. Putting together the two bounds, we get a first lower bound for $Z'$: for $u\geq 1$ and probability greater than $1-e^{-u}$:
	\begin{equation}
		Z' \geq \sqrt{\delta^T\Sigma_1\delta}\frac{1- \varepsilon}{1+\varepsilon} - C(\varepsilon) \paren{ \sqrt{2}\frac{1-\varepsilon}{1+\varepsilon}+1} \frac{\|\delta\|M}{\sqrt{N-1}}u\,.
	\end{equation}
	Let us choose $\varepsilon = 1/3$ to conclude.
\end{proof}

\subsubsection{Classical concentration inequalities}
\paragraph{{\bf Concentration inequalities for Gaussian random variables.}}

\begin{lemma}\label{lem:concnormal}
	Let $X \sim \cN(m, \sigma^2)$, then for all $x\geq 0$:
	\begin{equation*}
		\prob{ \abs{X- m} \geq \sqrt{2\sigma^2x} }\leq 2e^{-x}
	\end{equation*}
\end{lemma}
\begin{proof}
	It is a direct consequence of the Chernoff bound \citep{Che52}.
\end{proof}

\begin{lemma}\label{lem:concnorm2gauss}[Concentration of Gaussian vectors]
	Let $X \sim \cN(\mu,\Sigma)$, then for all $x \geq 0$:
	\begin{gather*}
		\prob{ \norm{X}^2 \geq \paren{\norm{\mu}^2 + \tr \Sigma} + 2 \sqrt{\paren{ \tr \Sigma^2 +2 \mu^T \Sigma \mu}x} + 2 \norm{\Sigma}_{\infty}x } \leq e^{-x}\,, \\
		\prob{ \norm{X}^2 \leq \paren{\norm{\mu}^2 + \tr \Sigma} - 2 \sqrt{\paren{ \tr \Sigma^2 +2 \mu^T \Sigma \mu}x}} \leq e^{-x}\,, \\
	\end{gather*}
\end{lemma}
The above is a reformulation of Lemma~2 in \citet{laurent2012testing} and can be seen as a consequence of combining
the arguments of  Lemma~1 of \citet{LauMas00} and Lemma~8.1 of \citet{Bir01}.
%

\begin{lemma}\label{lem:conckhi2}[Lower bound for $\chi^2$]
	Let $Z \sim \chi^2(n)$, then for all $x \geq 0$:
	\begin{equation*}
		\prob{ Z \leq ne^{-(1+2x/n)} } \leq e^{-x}\,.
	\end{equation*}
\end{lemma}

\begin{proof}
	Let $\delta \in (0,1)$, $\lambda \in \mbr^+$:
	\begin{align*}
		\prob{ Z \leq n \delta} = \prob{ e^{-\lambda Z} \geq e^{-n\lambda\delta} } \leq \e{e^{-\lambda Z}}e^{n\lambda \delta} = \exp\paren{ -\frac{n}{2}\paren{ \log(1+2\lambda) -2\lambda \delta}}
	\end{align*}
	where the inequality is due to Markov. Fix $\lambda = (-1+ \delta^{-1})/2>0$, then:
	\[
	\prob{ Z \leq n \delta}  \leq \exp\paren{- \frac{n}{2}\paren{ - \log(\delta)+\delta -1 }} \leq \exp\paren{- \frac{n}{2}\paren{ - \log(\delta) -1 }}
	\]
	Let us choose $\delta = \exp\paren{-1-2x/n}$  to conclude the proof.
\end{proof}

\paragraph{{\bf Concentration inequalities for bounded random variables.}}

\begin{lemma}\label{lem:bernstein}[Bernstein's concentration inequality]
	Let $X_1,\ldots,X_N$ i.i.d. real centred random variables bounded by $M$ such that $ \e{X_1^2} \leq \sigma^2$, then for all $x\geq 0$:
	\begin{equation*}
		\prob{ \sum_{i=1}^{N}X_i \geq \sqrt{2N\sigma^2 x} + \frac{2Mx}{3} } \leq e^{-x}
	\end{equation*}
\end{lemma}
\begin{proof}
	See for instance \citet{Ver19}, Exercise 2.8.5.
\end{proof}

\begin{theorem}\label{thm:talagrand}[Talagrand's inequality]
	Let $X^t_1,...,X^t_n$ independant random variables indexed by $t\in T$ ($T$ countable) in $\mbr$ and $L> 0$ such that for all $t \in T $, $i\leq n$,
	\begin{align}\label{hypothesis_tal}
		\e{ X_i^t } = 0\,,\quad  |X_i^t | \leq L
	\end{align}
	Let
	\begin{equation*}
		Z := \underset{t \in  T}{\sup} \sum_{i=1}^{n} X_i^t\,, \quad  \sigma^2  = \sup_{t \in T} \sum_{i=1}^{n} \e{ (X_i^t)^2 }
	\end{equation*}
	then for all $x \geq 0$ and $\varepsilon \in (0,1)$:
	\begin{align}\label{talgrand_inequalities_propre2_eps}
		\prob{ Z \geq \e{Z}(1+ \varepsilon) + 2\sqrt{2\sigma^2x} + 2Lx (1 + 8 \varepsilon^{-1}) } \leq e^{-x}\\
		\prob{ Z \leq  \e{Z}(1 - \varepsilon) - 2\sqrt{4\sigma^2x} - 4Lx (1 + 8 \varepsilon^{-1}) } \leq e^{-x}
	\end{align}
\end{theorem}

\begin{proof}
	See for instance \citet{massart2000constants}.
\end{proof}

\section{Proofs for Section~\ref{se:minimax} }

\subsection{Proof of Theorem~\ref{prop:lowbnd_1}}
This proof follows the same scheme as the Pinsker's bound  (\citealp{Pin80} or see \citealp{Tsy08} for a recent version).

The proof is provided for $V=B$ but can be directly adapted for $V<B$ by assuming $\mu_k$ independent of $\mu_1$ for $k>V$ when constructing the distribution $\mbq$ \eqref{eq:def_m}.

Let us first restrict ourselves to the case where $\mu_1$ is in a ball around $0$:
\begin{align*}
	\inf_{\wh{\mu}_1} \sup_{ \mu_i \in B(\mu_1,\sqrt{\tau}\rnrisk_1) } R_1(\wh{\mu}_1) & \geq   \inf_{\wh{\mu}_1} \sup_{\substack{ \mu_1 \in B(0,\sqrt{\beta}\rnrisk_1) \\ \mu_i \in B(\mu_1,\sqrt{\tau}\rnrisk_1)} } R_1(\wh{\mu}_1).
\end{align*}
Then the infimum over the estimators is now attained for an estimator $\wh{\mu}_1$ bounded by $2\sqrt{\beta} \rnrisk_1$. Indeed, any estimator $\wh{\mu}$ further perform less well than the deterministic estimator $\wh{\mu}=0$. If $\|\wh{\mu}\| > 2\sqrt{\beta} \rnrisk_1$:
\begin{equation}\label{eq:Tbest}
	\| \wh{\mu} - \mu_1 \| \geq \|\wh{\mu}\| - \|\mu_1\| >\sqrt{\beta} \rnrisk_1 > \|0 - \mu_1\|\,.
\end{equation}
We introduce now the probability measure $\mbq$: 
\begin{equation}\label{eq:def_m}
	\mu_1 \overset{\mbq}{\sim} \cN(0, \alpha\beta \nrisk_1 \Sigma)\,, \quad \mu_2= \ldots = \mu_B = \mu_\circ \overset{\mbq}{\sim} \cN(\mu_1, \alpha\tau\nrisk_1 \Sigma)\,,
\end{equation}
where $\beta >0 $ and $\alpha \in (0,1)$. Let $A$ be the event $\{ \|\mu_1 \|^2 \leq \beta \rnrisk_1^2, \|\mu_\circ - \mu_1 \|^2 \leq \tau \rnrisk_1^2 \}$ and $\mbe_\mbq$ denote the expectation over the distribution $\mbq$, then:
\begin{align}
	\inf_{\wh{\mu}_1} \sup_{ \mu_i \in B(\mu_1,\tau\rnrisk_1)} R_1(\wh{\mu}_1)
	& \geq  \inf_{\wh{\mu}_1 : \|\wh{\mu}_1\| \leq 2\sqrt{\beta} \rnrisk_1} \sup_{ \substack{ \mu_1 \in B(0,\sqrt{\beta}\rnrisk_1) \\ \mu_i \in B(\mu_1,\sqrt{\tau}\rnrisk_1)}}  R_1(\wh{\mu}_1) \label{al:bnd1}\\
	& \geq  \inf_{\wh{\mu}_1 : \|\wh{\mu}_1\| \leq 2\sqrt{\beta} \rnrisk_1} \frac{1}{\mbq(A)} \int_{A}R_1(\wh{\mu}_1) d\mbq(\nu, \mu_1,\ldots ,\mu_B) \nonumber \\
	& \geq \inf_{\wh{\mu}_1}  \ee{\mbq}{R_1(\wh{\mu}_1)}  - \sup_{\wh{\mu}_1 : \|\wh{\mu}_1\| \leq 2\sqrt{\beta} \rnrisk_1} \ee{\mbq}{R_1(\wh{\mu}_1) \bm{1}_{A^c}} \nonumber\\
	&=: I - r \nonumber\,,
\end{align}
Let us now bound $I$ and $r$.

\paragraph{\textbf{Lower bound for $I$ :}}  The first infimum (term $I$) is attained for $\wh{\mu}_1 = \e{ \mu_1 | X_\bullet^{(1)},\ldots, X_\bullet^{(B)}}$. Let us calculate $\wh{\mu}_1$.
\begin{gather*}
	\e{ \mu_1 | \mu_\circ,X_\bullet^{(1)},\ldots, X_\bullet^{(B)}} = \e{ \mu_1 | \mu_\circ ,X_\bullet^{(1)}} =  \paren{ (\alpha\beta)^{-1}+ 1+ (\alpha\tau)^{-1}}^{-1}\paren{\muNE_1 + \frac{1}{\alpha\tau}\mu_\circ} \,, \\
	\e{ \mu_\circ |\mu_1, X_\bullet^{(1)},\ldots, X_\bullet^{(B)} } = \paren{ (\alpha\tau)^{-1} + \|\rho\|^2}^{-1}\paren{ \frac{1}{\alpha\tau} \mu_1+ \sum_{k=2}^{B} \rho_k^2 \muNE_k }
\end{gather*}
where $\rho = (\rnrisk_1/\rnrisk_k)_{k\neq 1}$ and $\|\rho\|^2 = \sum_{k=2}^{B}\rho_k^2$. Combining these two expressions we get:
\begin{multline*}
	\e{ \mu_1 | X_\bullet^{(1)},\ldots, X_\bullet^{(B)}} =\paren{ (\alpha\beta)^{-1}+1 + (\alpha\tau)^{-1}}^{-1}  \\
	\times\paren{\muNE_1 + \frac{1}{\alpha\tau}\paren{ (\alpha\tau)^{-1} + \|\rho\|^2}^{-1}\paren{ \frac{1}{\alpha\tau}  \e{ \mu_1 | X_\bullet^{(1)},\ldots, X_\bullet^{(B)}}+ \sum_{k=2}^{B} \rho_k^2 \muNE_k } }\,,
\end{multline*}
and then:
\begin{equation*}
	\e{ \mu_1 | X_\bullet^{(1)},\ldots, X_\bullet^{(B)}} =  \paren{ (\alpha\beta)^{-1}+1 + \frac{\|\rho\|^2}{1+ \alpha \tau \|\rho\|^2}}^{-1}\paren{\muNE_1 + \frac{1}{1+\alpha\tau \|\rho\|^2} \sum_{k=2}^{B} \rho_k^2\muNE_k}\,,
\end{equation*}
Let us first notice that:
\begin{multline*}
	\e{ \mu_1 | X_\bullet^{(\cdot)}} - \mu_1 = \paren{ (\alpha\beta)^{-1}+1 + \frac{\|\rho\|^2}{1+ \alpha \tau \|\rho\|^2}}^{-1}  \\
	\times\brac{ (\muNE_1-\mu_1) +  \frac{1}{1+\alpha\tau \|\rho\|^2} \sum_{k=2}^{B} \rho_k^2(\muNE_k-\mu_\circ)  + \frac{\|\rho\|^2}{1+\alpha\tau \|\rho\|^2} (\mu_\circ - \mu_1) - \frac{1}{\alpha\beta}\mu_1}
\end{multline*}
Using that $\muNE_1 - \mu_1$, $\muNE_k -\mu_\circ$ (for $k\neq 1$), $\mu_\circ -\mu_1$ and $\mu_1$ are pairwise independent we get that:
\begin{align*}
	\frac{\e{ \|\wh{\mu}_1 - \mu_1\|^2}}{\rnrisk_1^2} =& \paren{ (\alpha\beta)^{-1}+1 + \frac{\|\rho\|^2}{1+ \alpha \tau \|\rho\|^2}}^{-2} \\
	&\times\brac{ 1 +  \frac{1}{(1+\alpha\tau \|\rho\|^2)^2} \sum_{k=2}^{B} \rho_k^4\rho_k^{-2} + \frac{\alpha\tau\|\rho\|^4}{(1+\alpha\tau \|\rho\|^2)^{2}} + \frac{1}{\alpha\beta}}
\end{align*}
After simplification:
\begin{equation}\label{eq:boundI}
	I = \nrisk_1\paren{ (\alpha\beta)^{-1}+1 + \frac{\|\rho\|^2}{1+ \alpha \tau \|\rho\|^2}}^{-1}
\end{equation}

\paragraph{\textbf{Upper bound for $r$:}} Using the triangle and Cauchy-Schwartz inequalities we have:
\begin{align}
	r& = \sup_{\wh{\mu}_1 : \|\wh{\mu}_1\| \leq 2\sqrt{\beta}\rnrisk_1} \e{ \| \wh{\mu}_1(X_\bullet^{(k)}, k\in \intr{B})  - \mu_1 \|^2 \bm{1}_{A^c}} \label{al:bndr1}\\
	& \leq \e{ 2\paren{ 4\beta\rnrisk_1^2 + \|\mu_1\|^2 }  \bm{1}_{A^c}} \nonumber \\
	& \leq 8 \beta \rnrisk_1^2 \prob{A^c} + 2\sqrt{\e{ \|\mu_1\|^4} \prob{A^c} } \nonumber \\
	& \leq 2\rnrisk_1^2 \paren{ 4\beta +\sqrt{3}\alpha\beta  } \sqrt{\prob{A^c}} \leq 20 \beta \rnrisk_1^2 \sqrt{\prob{A^c}} \nonumber
\end{align}
It stays to show the exponential decrease of $\prob{A^c}$. Let $\xi \sim \cN(0,\Sigma)$:
\begin{multline*}
	\prob{ \|\mu_1\|^2 \geq \beta \rnrisk_1^2  } = \prob{ \|\mu_\circ-\mu_1\|^2 \geq \tau \rnrisk_1^2 } = \prob{ \|\xi\|^2\geq \alpha^{-1}  } \\
	\leq \exp\paren{ -\frac{\deff_1}{2} \paren{ \sqrt{\frac{2}{\alpha}-1}-1  } }\,.
\end{multline*}
This follows from the concentration of the norm of Gaussian vectors (Lemma~\ref{lem:concnorm2gauss}). By union bound we get that:
\begin{align*}
	r\leq 30 \rnrisk_1^2 \beta\exp\paren{ -\frac{\deff_1}{4} \paren{ \sqrt{\frac{2}{\alpha}-1}-1  } } \,.
\end{align*}

\paragraph{\textbf{Conclusion :}}

The lower bound finally obtained is :
\begin{multline*}
	\inf_{\wh{\mu}_1} \sup_{ \mu_i \in B(\mu_1,\tau\rnrisk_1) } \frac{R_1(\wh{\mu}_1)}{\rnrisk_1^2} \geq \paren{ (\alpha\beta)^{-1}+1 + \frac{\|\rho\|^2}{1+ \alpha \tau \|\rho\|^2}}^{-1} \\
	-30 \beta\exp\paren{ -\frac{\deff_1}{4} \paren{ \sqrt{\frac{2}{\alpha}-1}-1  } },
\end{multline*}
where $\alpha \in (0,1)$ and $\beta \in \mbr_+$ are two free parameters. 
We can choose $\beta = \deff_1/\log \deff_1$ and $\alpha = \frac{2}{1+ (1+8 \beta^{-1})^2}$, then:
\begin{gather*}
	\beta\exp\paren{ -\frac{\deff_1}{4} \paren{ \sqrt{\frac{2}{\alpha}-1}-1  } } = \beta \exp\paren{ -\frac{2\deff_1}{\beta}}=  \frac{1}{\deff_1 \log\deff_1} \\
	\paren{ (\alpha\beta)^{-1}+1 + \frac{\|\rho\|^2}{1+ \alpha \tau \|\rho\|^2}}^{-1} - \paren{ 1 + \frac{\|\rho\|^2}{1+ \alpha \tau \|\rho\|^2}}^{-1} \geq -(\alpha\beta)^{-1}\geq -41\frac{\log \deff_1}{\deff_1}.
\end{gather*}
and
\begin{align*}
	\paren{ 1 + \frac{\|\rho\|^2}{1+ \alpha \tau \|\rho\|^2}}^{-1} - \paren{ 1 + \frac{\|\rho\|^2}{1+ \tau \|\rho\|^2}}^{-1} &= -(1-\alpha) \frac{\tau \|\rho\|^2}{1+\tau \|\rho\|^2} \frac{\frac{\|\rho\|^2}{1+\alpha\tau \|\rho\|^2}}{1+ \frac{\|\rho\|^2}{1+\alpha\tau \|\rho\|^2}} \frac{1}{1+ \frac{\|\rho\|^2}{1+ \tau \|\rho\|^2}} \\
	& \geq - (1-\alpha) \geq -40 \frac{\log \deff_1}{\deff_1}
\end{align*}
where we recall $\|\rho\|^2 = \sum_{i=1}^{B} \frac{\rnrisk_1^2}{\rnrisk_i^2}-1 = \paren{ \nu(V_\tau)}^{-1}-1$. Hence:
\begin{equation*}
	\paren{ 1 + \frac{\|\rho\|^2}{1+ \tau \|\rho\|^2}}^{-1} = \cB\paren{\tau, \nu(V_\tau) }
\end{equation*}
By combining these three inequalities, we get that:
\begin{equation*}
	\inf_{\wh{\mu}_1} \sup_{ \mu_i \in B(\mu_1,\tau\rnrisk_1) } \frac{R_1(\wh{\mu}_1)}{\rnrisk_1^2} \geq \cB\paren{\tau, \nu(V_\tau) } -111 \frac{\log \deff_1}{\deff_1}
\end{equation*}

\subsection{Proof of Corollary~\ref{prop:uppbnd_all}}
Let $\Cvect$ be a fixed $J$-partition of the means $(\mu_k)_{k\in \intr{B}}$ and denote $\zetavect = \diam(\Cvect)$.
Let us focus first on a specific group $j \in \intr{J}$ and task $k \in \cC_j$. Denote $\tau_{j,k} = \zeta^2_j/\nrisk_k$ and
$\nu_{j,k} = \nrisk(\cC_j)/\nrisk_k$. Consider the vector of oracle weights $\omvect_k^*$ given by~\eqref{eq:STBoptweights},
wherein the target task 1 is replaced by $k$ everywhere, and the subset of neighbouring tasks is taken as $\cC_j \subseteq V_{\tau_{j,k}}$.
Lemma~\ref{lem:oraclebound} then states $R_k(\omvect_k^*)/\nrisk_k \leq \cB(\tau_{j,k}, \nu_{j,k})$. As a consequence,
according to Theorem~\ref{prop:qaggreg_gauss}, it holds
\[
\frac{R_k(\wh{\omvect}_k)}{\nrisk_k} \leq (1 + CB e^{-u_0})\paren{ \cB(\tau_{j,k}, \nu_{j,k}) + C\sqrt{u_0}
	\frac{Q_k(\omvect^*_k)}{\nrisk_k}} + C \frac{u_0}{\sqrt{\deamm_1}}.
\]
The rest of the proof is dedicated to bounding the terms
$Q_k(\omvect^*_k) \rnrisk^{-1}_k$ (and their sum over $k \in \cC_j$).
Denote $\om^*_{k,\ell}$ the $\ell$-th component of $\omvect^*_k$.
It holds
\begin{align}
	\frac{Q_k(\omvect_k^*)}{\nrisk_k}
	&= \rnrisk_k^{-2}\sum_{\ell \in \cC_j\backslash\{k\}} \ \omega_{k,\ell}^* \sqrt{\frac{(\mu_\ell - \mu_k)^T\Sigma_k(\mu_\ell - \mu_k)}{N_k}+ \frac{\tr \Sigma_\ell \Sigma_k}{N_\ell N_k}} \notag \\
	& \leq \rnrisk_k^{-2}\sum_{\ell \in \cC_j\backslash\{k\}} \omega_{k,\ell}^* \frac{\|\Sigma_k\|_{\infty}^{1/2}}{\sqrt{N_k}} \sqrt{\zeta_j^2+ \nrisk_\ell} \notag\\
	& \leq \frac{1}{\sqrt{\deff_k}} \paren[3]{(1 -\omega_{k,k}^* )\sqrt{\tau_{j,k}} +  \frac{\nu_{j,k}\rnrisk_k}{1+\tau_{j,k}(1-\nu_{j,k})}\sum_{\ell \in \cC_j\backslash\{k\}} \rnrisk^{-1}_\ell} \notag\\
	& \leq \frac{1}{\sqrt{\deff_k}} \paren[3]{(1 -\omega_{k,k}^* )\sqrt{\tau_{j,k}} +  \nu_{j,k}\rnrisk_k
		\sum_{\ell \in \cC_j} \rnrisk^{-1}_\ell}, \label{al:aux_cpb_1}
\end{align}
where we have used: $\|\mu_\ell - \mu_k\|\leq \zeta_j$ as tasks $k$ and $\ell$ are in the group $\cC_j$; $(\norm{\Sigma_k}_{\infty}/N_k)^{1/2}=\rnrisk_k/\sqrt{\deff_k}$; and
the explicit expression~\eqref{eq:STBoptweights} for the oracle weights $\om^*_{k,\ell}$
for group $\cC_j$.
For the first term of~\eqref{al:aux_cpb_1}, for all $k\in \cC_j$ we have:
\begin{equation*}
	(1 -\omega_{k,k}^* )\sqrt{\tau_{j,k}} = \frac{1- \nu_{j,k}}{1+ \tau_{j,k}(1-\nu_{j,k})}\sqrt{\tau_{j,k}} \leq \frac{\sqrt{\tau_{j,k}}}{1+ \tau_{j,k}} \leq 1\,.
\end{equation*}
For the second term of~\eqref{al:aux_cpb_1}, introduce
the vector $\rhovect:= (\rnrisk_\ell^{-1})_{\ell \in \cC_j}$ and
observe that $\nu_{j,k} = \rho_k^2 / \norm{\rhovect}_2^2$, thus,
when summing over $k \in \cC_j$:
\[
\sum_{k\in \cC_j} \paren[3]{\nu_{j,k}\rnrisk_k
	\sum_{\ell \in \cC_j} \rnrisk^{-1}_\ell} =
\sum_{k\in \cC_j} \rho_k \frac{\norm{\rhovect}_1}{\norm{\rhovect}_2^2}
= \frac{\norm{\rhovect}^2_1}{\norm{\rhovect}_2^2} \leq \abs{\cC_j}.
\]
We deduce from the above estimates:
\begin{equation*}
	\sum_{k\in \cC_j} \frac{Q_k(\omvect^*_k)}{\nrisk_k} 
	\leq  \frac{2 |\cC|_j}{\min_k (\deff_k)^{\nicefrac{1}{2}}}\,,
\end{equation*}
implying
\begin{equation*}
	\frac{1}{B} \sum_{k =1}^B \frac{Q(\omvect^*_k)}{\nrisk_k} \leq \frac{2}{\min_k (\deff_k)^{\nicefrac{1}{2}}}.
\end{equation*}
Therefore for any $J$-partition $\bm{\cC}$, since $\deamm_k \geq \deff_k$:
\begin{equation*}
	\frac{1}{B} \sum_{k=1}^B \frac{R_k(\wh{\omvect}_k)}{\nrisk_k} \leq \paren[3]{1+CBe^{-u_0}} \paren[3]{ \frac{1}{B} \sum_{j=1}^J \sum_{j\in \cC_j}\cB(\tau_{j,k},\nu_{j,k}) + C' \frac{u_0}{\min_{k\in \intr{B}} (\deff_k)^{\nicefrac{1}{2}}}}.
\end{equation*}\qed

\subsection{Proof of Theorem~\ref{prop:cpdlowbnd}}
The proof follows the same steps as the proof of Theorem~\ref{prop:lowbnd_1}. Let $\cC$ a $J$-partition of $\intr{B}$, $\bm{\zeta} \in \mbr_+^J$ and $\Sigma$ a definite positive matrix in $\mbr^{d\times d}$. W.l.g. we can assume that $\tr \Sigma =1$. In a first time, we are going to lower bound the minimax risk for the estimation of $\mu_1$ that we can assume to be in the cluster $1$ ($1 \in \cC_1)$.

If for $j\in \intr{J}$ the means of $\cC_j$ are in a ball of radius $\zeta_j/2$, then two means are at a distance at most $\zeta_j$:
\begin{equation*}
	\inf_{\wh{\mu}_1} \sup_{\mbp \in \multimodel( \cC,\bm{\zeta},\Sigma, \bm{\nrisk} )} R_1(\wh{\mu}_1) \geq \inf_{\wh{\mu}_1} \sup_{\substack{\exists \nu_1,\ldots,\nu_J \in \mbr^d \\ \mu_k \in B(\nu_j, \zeta_j/2), \forall k \in \cC_j }} R_1(\wh{\mu}_1).
\end{equation*}
For simplicity, the supremum over the vectors means $\mu_k$ is used to denote the supremum over the Gaussian distributions $\mbp_k = \cN(\mu_k,\nrisk_k\Sigma)$.\\
We can restrict ourself in the case where the centres $\nu_j$ are in a ball around $0$ of radius $\sqrt{\beta}$:
\begin{equation*}
	\inf_{\wh{\mu}_1} \sup_{\substack{\exists \nu_1, \ldots \nu_J \in \mbr^d \\ \mu_k \in B(\nu_j, \zeta_j/2), \forall k \in \cC_j }} R_1(\wh{\mu}_1) \geq  \inf_{\wh{\mu}_1} \sup_{\substack{\exists \nu_1,\ldots \nu_J\in B(0,\sqrt{\beta}) \\ \mu_k \in B(\nu_j, \zeta_j/2), \forall k \in \cC_j }} R_1(\wh{\mu}_1)
\end{equation*}
Let $\alpha\in (0,1),\beta >0$, we introduce now the probability measure $\mbq= \mbq(\alpha,\beta)$ on $(\mbr^d)^{B+J}$ such that a random vector $( \nu_1,\ldots, \nu_J, \mu_1, \ldots , \mu_B) \in (\mbr^d)^{B+J}$ follows the distribution $\mbq$ if:
\begin{equation*}
	\nu_j \overset{\mbq}{\sim} \cN(0, \alpha\beta \Sigma) \text{ for } k\in \intr{\cN}, \qquad \mu_k \overset{\mbq}{\sim} \cN(\nu_j, \alpha \frac{\zeta_j^2}{4} \Sigma ) \text{ for } k \in \cC_j.
\end{equation*}
Hence, considering the events $H_j := \{ \|\nu_j\|^2 \leq \beta, \|\mu_k - \nu_j\|^2 \leq \zeta_j^2/4 , k \in \cC_j \}$, $H := \cap_{j=1}^J H_j$, as in the equations~\eqref{al:bnd1}:
\begin{align*}
	\inf_{\wh{\mu}_1} \sup_{\mbp \in \multimodel( \cC,\bm{\zeta},\Sigma, \bm{\nrisk} )} R_1(\wh{\mu}_1) \geq \inf_{\wh{\mu}_1} \ee{\mbq}{R_1(\wh{\mu}_1) |H}.
\end{align*}
The distribution $\mbq$ can be decomposed into a product of $J$ probability measure: $\mbq = \bigotimes_{j=1}^J \mbq_j$ where $\mbq_j$ is the distribution of $(\nu_j,(\mu_k)_{k\in \cC_j} )$. By independence, the Bayes estimator of $\mu_1$ only consider the means of $\cC_1$ and following equations~\eqref{al:bnd1} we get:
\begin{align*}
	\inf_{\wh{\mu}_1} \sup_{\mbp \in \multimodel( \cC,\bm{\zeta},\Sigma, \bm{\nrisk} )} R_1(\wh{\mu}_1) \geq \inf_{\wh{\mu}_1} \ee{\mbq_1}{R_1(\wh{\mu}_1) |H_1} \geq \frac{1}{\mbq(H_1)}\paren{ I_1 - r_1}\,,
\end{align*}
where
\begin{equation}\label{eq:defI1r1}
	I_1 :=  \inf_{\wh{\mu}_1}\ee{\mbq_1}{R_1(\wh{\mu}_1)}\,, \qquad  r_1 := \sup_{\wh{\mu}_1 : \|\wh{\mu}_1\| \leq 2\sqrt{\beta} +  \zeta_1} \ee{\mbq}{R_1(\wh{\mu}_1) \bm{1}_{H_1^c}}
\end{equation}
We have used that the infimum is attained for an estimator $\wh{\mu}_1$ bounded by $2\sqrt{\beta} + \zeta_1$, because the estimator $\wh{\mu} = 0$ beats the estimators outside that ball (as in \eqref{eq:Tbest}).

\paragraph{\textbf{Lower bound for $I_1$ :}}  The infimum is attained for $\wh{\mu}_1 = \e{ \mu_1 | X_\bullet^{(k)} \, \, k \in \cC_1}$. Let us calculate $\wh{\mu}_1$. We will denote in the rest of the proof $\tilde{\zeta}_j := \zeta_j/2$:
\begin{gather*}
	\e{ \mu_1 | \nu_1, X_\bullet^{(k)} \, \, k \in \cC_1} = \e{ \mu_1 | \nu_1 ,X_\bullet^{(1)}} =  \frac{ \alpha\wt{\zeta}_1^2}{\rnrisk_1^2+\alpha\wt{\zeta}^2_1}\muNE_1 + \frac{\rnrisk_1^2}{\rnrisk_1^2+ \alpha\wt{\zeta}_1^2}\nu_1\,, \\
	\e{ \nu_1 |  X_\bullet^{(k)} \, \, k \in \cC_1 } = \paren{(\alpha\beta)^{-1} + \sum_{i\in \cC_1} \paren{ \alpha \wt{\zeta}_1^2 + \rnrisk_k^2}^{-1}}^{-1} \sum_{k\in \cC_1} \frac{1}{\alpha\wt{\zeta}_1^2 + \rnrisk_k^2} \muNE_k
\end{gather*}
Combining these two expressions:
\begin{multline*}
	\e{ \mu_1 | X_\bullet^{(k)} \, \, k \in \cC_1} =  \\ \frac{ \alpha\wt{\zeta}_1^2}{\rnrisk_1^2+\alpha\wt{\zeta}^2_1}\muNE_1
	+ \frac{\rnrisk_1^2}{\rnrisk_1^2+ \alpha\wt{\zeta}_1^2}\paren{(\alpha\beta)^{-1} + \sum_{k\in \cC_1} \paren{ \alpha \wt{\zeta}_1^2 + \rnrisk_k^2}^{-1}}^{-1} \sum_{k \in \cC_1} \frac{1}{\alpha\wt{\zeta}_1^2 + \rnrisk_k^2} \muNE_k
\end{multline*}
Let $\kappa_1 :=\paren{(\alpha\beta)^{-1} + \sum_{k\in \cC_1} \paren{ \alpha \wt{\zeta}_1^2 + \rnrisk_k^2}^{-1}}^{-1} $, we can first notice that:
\begin{align*}
	\e{ \mu_1 | X_\bullet^{(\cdot)}} - \mu_1 =& \brac{ \frac{\alpha\wt{\zeta}_1^2}{\rnrisk_1^2+ \alpha \wt{\zeta}_1^2} + \frac{\kappa_1 \rnrisk_1^2}{(\rnrisk_1^2+ \alpha\wt{\zeta}_1)^2}  } (\muNE_1 -\mu_1) \\
	&+ \frac{\kappa_1 \rnrisk_1^2}{\rnrisk_1^2+ \alpha \wt{\zeta}_1^2} \sum_{k \in \cC_1 \backslash \{1\}} \frac{1}{\alpha \wt{\zeta}_1^2+ \rnrisk_k^2}(\muNE_k - \nu_1) \\
	& - \frac{\rnrisk_1^2}{\rnrisk_1^2+ \alpha \wt{\zeta}_1^2} \paren{1-\frac{\kappa_1\rnrisk_1^2}{\rnrisk_1^2+ \alpha \wt{\zeta}_1^2}}( \mu_1 - \nu_1) - \frac{\kappa_1\rnrisk_1^2}{\rnrisk_1^2+ \alpha \wt{\zeta}_1^2} \frac{1}{\alpha\beta}\nu_1\,.
\end{align*}
Using that $\muNE_k -\nu_1$ for $k\in \cC_1\backslash\{1\}$, $\muNE_1 -\mu_1$, $\mu_1 - \nu_1$ and $\nu_1$  are pairwise independent we get that:
\begin{align*}
	\e{ \|\wh{\mu}_1 - \mu_1\|^2} &= \brac{ \frac{\alpha\wt{\zeta}_1^2}{\rnrisk_1^2+ \alpha \wt{\zeta}_1^2} + \frac{\kappa_1 \rnrisk_1^2}{(\rnrisk_1^2+ \alpha\wt{\zeta}_1)^2}  }^2 \rnrisk_1^2 + \frac{\kappa_1^2 \rnrisk_1^4}{(\rnrisk_1^2+ \alpha \wt{\zeta}_1^2)^2} \sum_{k \in \cC_1\backslash\{1\}} \frac{1}{\alpha \wt{\zeta}_1^2+ \rnrisk_k^2} \\
	& + \frac{\rnrisk_1^4}{(\rnrisk_1^2+ \alpha \wt{\zeta}_1^2)^2} \paren{1-\frac{\kappa_1\rnrisk_1^2}{\rnrisk_1^2+ \alpha \wt{\zeta}_1^2}}^2 \alpha \wt{\zeta}_1^2 + \frac{\kappa_1^2\rnrisk_1^4}{(\rnrisk_1^2+ \alpha \wt{\zeta}_1^2)^2} \frac{1}{\alpha\beta}\,.
\end{align*}
After simplification:
\begin{equation}\label{eq:boundI_cpd}
	\frac{I_1}{\rnrisk_1^2} = \frac{ \alpha\wt{\zeta}_1^2}{\rnrisk_1^2+\alpha\wt{\zeta}^2_1} + \frac{\kappa_1 \rnrisk_1^2}{(\rnrisk_1^2+\alpha\wt{\zeta}^2_1)^2}
\end{equation}
\paragraph{\textbf{Upper bound for $r_1$ :}}
By the same arguments of equations \eqref{al:bndr1}:
\begin{align*}
	\sup_{\wh{\mu}_1 : \|\wh{\mu}_1\| \leq 2\sqrt{\beta}+\zeta_1} &\e{ \| \wh{\mu}_1(X_\bullet^{(k)}, k\in \cC_1)  - \mu_1 \|^2 \bm{1}_{H_1^c}} \leq 20 (\beta+\zeta_1^2) \sqrt{\prob{H_1^c}}
\end{align*}
From Lemma~\ref{lem:concnorm2gauss}, for all $k\in \cC_1$:
\begin{equation*}
	\prob{ \|\nu_1\|^2 \geq \beta   } = \prob{ \|\mu_k-\nu_1\|^2 \geq \zeta_1^2/2 } \leq \exp\paren{ -\frac{\deff}{2} \paren{ \sqrt{\frac{2}{\alpha}-1}-1  } }\,,
\end{equation*}
and by union bound we get that  :
\begin{align*}
	r_1\leq 20  (\beta+\zeta_1^2)  \sqrt{|\cC_1|+1}\exp\paren{ -\frac{\deff}{4} \paren{ \sqrt{\frac{2}{\alpha}-1}-1  } } \,.
\end{align*}
where $\deff = \tr \Sigma/ \|\Sigma\|_{\infty}$.

\paragraph{\textbf{Compound bound}}
We recall that $\mbq~=~\bigotimes_{j=1}^J~\mbq_j$ where $\mbq_j$ is the distribution of $(\nu_j, \mu_k \text{ for } k\in \cC_j)$. Then let $\wh{\muvect} = (\wh{\mu}_k)_{k\in\intr{B}} \in (\mbr^d)^B$ be an estimator of the vectors $(\muvect_k)_{k\in\intr{B}}$:
\begin{align*}
	\inf_{\wh{\muvect}} \sup_{\mbp \in \multimodel( \cC,\bm{\zeta},\Sigma, \bm{\nrisk} )} \frac{1}{B} \sum_{k=1}^{B} \frac{R_k(\wh{\mu}_k)}{\rnrisk_k^2} &\geq \inf_{\wh{\muvect}} \frac{1}{\mbq(H)} \int_{H}\frac{1}{B} \sum_{k=1}^{B} \frac{R_k(\wh{\mu}_k)}{\rnrisk_k^2} d\mbq(\nu_1, \ldots , \nu_\cN, \mu_1, \ldots ,\mu_B) \\
	&= \inf_{\wh{\muvect}} \frac{1}{B}\sum_{j=1}^{J } \sum_{k\in \cC_j}  \frac{\mbq(H_{-j})}{\mbq(H)}\int_{H_j}\frac{R_k(\wh{\mu}_k)}{\rnrisk_k^2} d\mbq_j(\nu_j,(\mu_\ell)_{\ell\in \cC_j})
\end{align*}
where we recall $H_j =\{ \|\nu_j\|^2 \leq \beta , \|\mu_k - \nu_j\|^2 \leq \wt{\zeta}_j^2, \forall k \in \cC_j \}$, $H = \bigcap_{j=1}^J H_j$ and $H_{-j} = \bigcap_{\ell\neq j} H_\ell$. Using that $\mbq(H_{-j})/\mbq(H) = \mbq_j(H_j)^{-1} \geq 1$ and that the infimum over estimators $\wh{\muvect}$ of the sum is the sum of the infimum over estimators $\wh{\mu}_k$, we get that:
\begin{align}
	&\inf_{\wh{\muvect}}  \sup_{\mbp \in \multimodel( \cC,\bm{\zeta},\Sigma, \bm{\nrisk} )} \frac{1}{B} \sum_{k=1}^{B} \frac{R_k(\wh{\mu}_k)}{\rnrisk_k^2} \geq \frac{1}{B}\sum_{j=1}^{J } \sum_{k\in \cC_j } (I_k - r_k) \nonumber\\
	&\geq \frac{1}{B} \sum_{j=1}^{J} \sum_{k\in \cC_j} \frac{ \alpha\wt{\zeta}_j^2}{\rnrisk_k^2+\alpha\wt{\zeta}^2_j} + \frac{\kappa_j \rnrisk_k^2}{(\rnrisk_k^2+\alpha\wt{\zeta}^2_j)^2}- \frac{20}{B}\paren{ \sum_{j=1}^{J} |\cC_j|^{3/2}\frac{\beta + \wt{\zeta}_j^2}{\rnrisk^2(\cC_j)} }\exp\paren{ - \deff c(\alpha)} \label{al:bndalphabeta}
\end{align}
where $\kappa_j = \paren{(\alpha\beta)^{-1} + \sum_{k\in \cC_j} \paren{ \alpha \wt{\zeta}_j^2 + \rnrisk_k^2}^{-1}}^{-1} $ and $c(\alpha) =  \paren{ \sqrt{\frac{2}{\alpha}-1}-1  }/4 $\,.

\paragraph{\textbf{Conclusion :}}

Let $\deff \to \infty$ in \eqref{al:bndalphabeta}, then:
\begin{align*}
	\lim_{\deff\to \infty}    \inf_{\wh{\muvect}}  \sup_{\mbp \in \multimodel( \cC,\bm{\zeta},\Sigma, \bm{\nrisk} )} \frac{1}{B} \sum_{k=1}^{B} \frac{R_k(\wh{\mu}_k)}{\rnrisk_k^2}\geq  \frac{1}{B} \sum_{j=1}^{J} \sum_{k\in \cC_j} \frac{ \alpha\wt{\zeta}_j^2}{\rnrisk_k^2+\alpha\wt{\zeta}^2_j} + \frac{\kappa_j \rnrisk_k^2}{(\rnrisk_k^2+\alpha\wt{\zeta}^2_j)^2}
\end{align*}
Let $\alpha \to 1$ and $\beta \to \infty$, then:
\begin{equation}\label{eq:true_cpd_low}
	\lim_{\deff\to \infty}    \inf_{\wh{\muvect}}  \sup_{\mbp \in \multimodel( \cC,\bm{\zeta},\Sigma, \bm{\nrisk} )} \frac{1}{B} \sum_{k=1}^{B} \frac{R_k(\wh{\mu}_k)}{\rnrisk_k^2}\geq   \frac{1}{B} \sum_{j=1}^{J} \sum_{k\in \cC_j} \frac{ \wt{\zeta}_j^2}{\rnrisk_k^2+\wt{\zeta}^2_j} + \frac{ \rnrisk_k^2}{\rnrisk_k^2+\wt{\zeta}^2_j} \frac{1}{\sum_{\ell\in C_j} \frac{\rnrisk_k^2+ \tilde{\zeta}_j^2}{\rnrisk_\ell^2+ \tilde{\zeta_j}^2}}
\end{equation}
We conclude by remarking that for all $j\in \intr{J}$:
\begin{align*}
	\sum_{\ell\in \cC_j} \frac{\rnrisk_k^2+ \wt{\zeta}_j^2}{\rnrisk_\ell^2+\wt{\zeta}_j^2} =1+ \sum_{\ell \in \cC_j\backslash\{k\}} \frac{\rnrisk_k^2+ \wt{\zeta}_j^2}{\rnrisk_\ell^2+\wt{\zeta}_j^2} \leq 1+ \sum_{\ell \in \cC_j\backslash\{k\}} \frac{\rnrisk_k^2+ \wt{\zeta}_j^2}{\rnrisk_\ell^2}
\end{align*}
Then:
\begin{multline*}
	\frac{1}{B} \sum_{j=1}^{J} \sum_{k\in \cC_j} \frac{ \wt{\zeta}_j^2}{\rnrisk_k^2+\wt{\zeta}^2_j} + \frac{ \rnrisk_k^2}{\rnrisk_k^2+\wt{\zeta}^2_j} \frac{1}{\sum_{\ell\in C_j} \frac{\rnrisk_k^2+ \tilde{\zeta}_j^2}{\rnrisk_\ell^2+ \tilde{\zeta_j}^2}} \\
	\geq \frac{1}{B} \sum_{j=1}^{J} \sum_{k\in \cC_j} \frac{ \wt{\zeta}_j^2}{\rnrisk_k^2+\wt{\zeta}^2_j} + \frac{ \rnrisk_k^2}{\rnrisk_k^2+\wt{\zeta}^2_j} \frac{1}{1+ \sum_{\ell \in \cC_j\backslash\{k\}} \frac{\rnrisk_k^2+ \wt{\zeta}_j^2}{\rnrisk_\ell^2}} = \cL^*\paren{\bm{\rnrisk}, \Cvect, \zetavect/2}.
\end{multline*}
\qed
\subsection{Proof of Lemma~\ref{prop:cpd_upbnd}}
We start with the following elementary bounds on the function $\cB$ (for $\tau \geq 0, \nu \in [0,1]$)
\begin{equation}
	\label{eq:boundB}
	\cB(\tau,\nu)     \leq \frac{\tau + \nu}{1 + \tau} \leq \max\paren{1,\frac{\tau}{1+\tau} + \nu}.
\end{equation}
Now consider the quantity $A_j := \abs{\cC_j}^{-1}\sum_{k\in \cC_j} \cB(\tau_{j,k},\nu_{j,k} )$.
Observe that $\sum_{k\in \cC_j} \nu_{j,k}=1$ and
$\tau_{j,k} = \nu_{j,k} B_j$, where $B_j := \zeta^2_j / \nrisk(\cC_j)$. Thus
\[
A_j := \abs{\cC_j}^{-1}\sum_{k\in \cC_j} \cB(B_j \nu_{j,k},\nu_{j,k}) \leq (B_j +1) \abs{\cC_j}^{-1}
\sum_{k \in \cC_j} \frac{\nu_{j,k}}{1+ B_j \nu_{j,k}}.
\]
where we have used the first inequality in~\eqref{eq:boundB}.
By concavity of $t \mapsto t/(1+t)$ we conclude to
\[
A_j \leq \frac{B_j |\cC_j|^{-1} + |\cC_j|^{-1}}{1+ B_j |\cC_j|^{-1}}
= \frac{\bar{\tau}_j+|\cC_j|^{-1}}{1+\bar{\tau}_j},
\]
and thus to~\eqref{eq:cpd_upbnd_gen} by summation over $j\in \intr{J}$.
Now using the second inequality in~\eqref{eq:boundB}, we obtain
\[
\sum_{j \in \intr{J}} \frac{\abs{\cC_j}}{B} \frac{\bar{\tau}_j+|\cC_j|^{-1}}{1+\bar{\tau}_j}
\leq \sum_{j \in \intr{J}} \frac{\abs{\cC_j}}{B} \min\paren{1,\frac{\bar{\tau}_j}{1+\bar{\tau}_j}
	+|\cC_j|^{-1}} \leq \min\paren{1,\frac{\taubb  }{1+\taubb} + \frac{J}{B}},
\]
where we have used the second inequality in~\eqref{eq:boundB} and the biconcave character
of the function $(x,y) \mapsto \min(1,y + x/(1+x))$; thus establishing~\eqref{eq:estlglob}.
Assume now that all risks and the diameters are equal, i.e. $\nrisk_k = \nrisk$ and $\zeta_j = \zeta$ for $k\in \intr{B}$ and $j \in \intr{J}$. 
Then for all $j\in \intr{J}$ and $k \in \intr{B}$, $\ol{\rnrisk}^2(\cC_j) = \nrisk$, $\ol{\tau}_{j,k} = \zeta^2/\nrisk = \ol{\tau}$ and $\nu_{j,k} = |\cC_j|^{-1}$. Using the elementary bound
\[
\cB(\tau,\nu) \geq \frac{\tau}{1+\tau} + \frac{\nu}{(1+\tau)^2},
\]
we thus have in this case
\begin{align}
	\cL^*\paren{\bm{\rnrisk}, \Cvect,\bm{\zeta}} &= \frac{1}{B} \sum_{j =1}^{J} \sum_{k \in C_j}\cB\paren{ \tau_{j,k}, \nu_{j,k}}= \frac{1}{B} \sum_{j =1}^{J} \sum_{k \in C_j}\cB\paren{ \ol{\tau}, |\cC_j|^{-1}} \notag \\
	& \geq \frac{1}{B} \sum_{j =1}^{J} \sum_{k \in C_j} \paren{\frac{\ol{\tau}}{1+\ol{\tau}} + \frac{\abs{\cC_j}^{-1}}{(1+\ol{\tau})^2}} \notag\\
	& = \frac{\ol{\tau}}{1+\ol{\tau}} + \frac{J}{B}\frac{1}{(1+\ol{\tau})^2}, \label{eq:lowerboundL}
\end{align}

Finally, since for $\tau \geq 0,\nu \in[0,1]$:
\begin{align*}
	\frac{\tau}{1+\tau} + \frac{\nu}{(1+\tau)^2}
	& \geq \max\paren{\frac{\tau}{1+\tau}, \frac{1}{(1+\tau)^2} \paren{\frac{\tau}{1+\tau} +\nu}}\\
	& \geq \max\paren{\frac{\tau}{1+\tau},\frac{1}{(1+\tau)^2}}\min\paren{1,\frac{\tau}{1+\tau} + \nu}\\
	& \geq 0.38 \min\paren{1,\frac{\tau}{1+\tau} + \nu},
\end{align*}
we conclude that in the case of equal risks and diameters the upper bound~\eqref{eq:estlglob} and
the lower bound~\eqref{eq:lowerboundL} differ by a factor at most $1/0.38 \leq 2.7$.

\qed

\section{About the constant $\ratiobs$ in the translation-invariant kernel setting}\label{ap:notes_gamma}

In this section, we investigate the distribution-dependent constant $\ratiobs = M^2/(\tr \Sigma)$
in the (\BS) setting (i.e., for data bounded in norm by the constant $M$).
This constant comes into play in the risk bounds for our methods, in relation to sufficient sample
sizes, see e.g. Theorems~\ref{prop:fullemp_bnd}, and \eqref{prop:qaggreg_bounded}.
Rewriting $\tr \Sigma = \e{\norm{X-\e{X}}^2}$ yields a direct interpretation of $\ratiobs$, namely
it is the ratio between the known bound on $\norm{X}$ and the ``variance'' of $X$; in other
words, $\ratiobs$ is all the bigger as the variable $X$ is more concentrated in relation to the size of its support.

We are interested in an understanding more detailed than this simple observation
in the situation of kernel mean embedding (KME), which was our primary motivation for investigating
the (\BS) setting. Namely, in that situation the user might choose between different kernels and their
associated Hilbert space mappings, in particular choosing or tuning the ``kernel bandwidth''.
Even if kernels under consideration are all bounded by the same constant, different kernels may give
rise to different constants $\ratiobs$ for the same underlying data distribution.

We look into this issue under the following general conditions:
\begin{itemize}
	\item[(K1)] the original data takes values in $\cZ =\mbr^\ell$, and
	the data whose means we wish to estimate have been obtained via a Hilbert space mapping
	$X = \Phi_{\kappa}(Z)$, $\Phi_{\kappa}: \mbr^\ell \rightarrow \cH$, associated to the kernel $\kappa(z,z') = \inner{\Phi_{\kappa}(z),\Phi_{\kappa}(z')}$.
	\item[(K2)] $\kappa$ is a translation-invariant kernel on $\mbr^\ell$, of the form $\kappa(z,z') = F(z-z')$,
	where $F: \mbr^\ell \rightarrow \mbr$,
	with $M^2:=F(0)$.
	\item[(K3)] For any $u\in \mbr^\ell$, the function $\lambda \mapsto F(\lambda u)$ is nonincreasing on $\mbr_+$.
	Furthermore, there exist constants $h>0, c\leq 1$ such that
	\begin{equation}
		\label{eq:kerbandw}
		F(u) \leq
		M^2\paren{1 - c\frac{\norm{u}^2}{h^2}},  \text{ for all  } u\in \mbr^\ell \text{ s.t. } 0 \leq \norm{u}\leq h.
	\end{equation}
\end{itemize}
Observe that (K1)-(K2) imply that the mapped data $X$ satisfies (\BS); as for (K3), it means that the
kernel is locally upper bounded by a strongly concave function
in a neighbourhood of 0 of size $h$. The latter quantity can therefore
interpreted as a proxy bandwidth for the kernel; and if $F_1$ satisfies~\eqref{eq:kerbandw} for $h=1$ then
the rescaled kernel function $F_h(u) := F_1(u/h)$ satisfies~\eqref{eq:kerbandw} for the bandwidth parameter $h>0$.
The classical Gaussian, exponential, and Matérn kernels, for example, satisfy such conditions.

\begin{proposition}
	\label{prop:boundgammakernel}
	Assume (K1)-(K2)-(K3) hold, and that the distribution $P$ of the original data $Z$ in $\mbr^\ell$ satisfies the following norm moment condition
	for some $p \geq 1, C>0$:
	\begin{equation}
		\label{eq:smallball}
		\frac{\e{\xi^{2p}}}{\e{\xi^p}^2}\leq C,\qquad \text{ where } \xi := \norm{Z -\e{Z}}.
	\end{equation}
	Then it holds
	\[
	\ratiobs=\frac{M^2}{\e{\norm{X-\e{X}}^2}}
	\leq \frac{4.2^{\frac{2}{p} + 2p} C}{c} \max\paren{1,\frac{h}{2\e{\norm{Z-\e{Z}}^p}^{\frac{1}{p}}}}^2.
	\]
\end{proposition}
Assume $p=2$ to simplify (we allowed for other values of $p$ in the moment condition~\eqref{eq:smallball}
mainly with the possible value $p=1$ in mind, which makes the condition weaker; the discussion below
can be readily adapted to other values of $p$).
This result shows that, provided  the bandwidth parameter $h$ is chosen of the order of
$\sigma_Z := \e{\norm{Z-\e{Z}}^2}^{\frac{1}{2}}$ or smaller, the constant $\ratiobs$ for the mapped data
is bounded independently of $h$. The bound depends on (1) the
strong concavity parameter $c$ of the upper bound on the (unit scaled) kernel function in a neighbourhood of the origin, and (2) the norm moment ratio~\eqref{eq:smallball}
of the original data distribution. Since
$\e{\xi^{4}} \leq \e{\xi^2} \norm{\xi}^2_{L^\infty}$,
in the case where the original data is itself bounded in norm by a constant $R$, \eqref{eq:smallball} holds with
$C=(R/\sigma_Z)^2$. Thus, if the original $X$ data is bounded, the distribution of the mapped data $Z$ under the above conditions ``inherits''
the constant $\ratiobs$ from that of the original data, up to factors.
However, the norm moment condition is much milder than a boundedness condition and can also accommodate
unbounded distributions with heavy tails of the original data.

{\bf Proof of Proposition~\ref{prop:boundgammakernel}.}
For $Z,Z'\sim\mbp$ independent, denote  $D:=\norm{Z-Z'},$  $\theta:=\min\paren{\frac{h^p}{\e{D^p}},
	\frac{1}{2}},$ and $t^p:= \theta \e{D^p} = \min\paren{h^p,\frac{\e{D^p}}{2}}$,
it holds
\begin{align*}
	\norm{\e{\Phi_{\kappa}(Z)}}^2/M^2
	& = M^{-2} \e{\inner{\Phi_{\kappa}(Z),\Phi_{\kappa}(Z')}}\\
	& = M^{-2} \e{F(Z-Z')}\\
	& \leq 1- c\frac{t^2}{h^2} \prob{ D^p > t^p}\\
	& \leq 1- c\frac{\e{D^p}^{\frac{2}{p}}}{h^2} \theta^{\frac{2}{p}} \paren{1- \theta}^2 \frac{\e{D^p}^2}{\e{D^{2p}}}\\
	& \leq \frac{\e{\norm{\Phi_{\kappa}(Z)}^2}}{M^2} - \frac{c}{4} \min\paren{1,\frac{\e{D^p}^{\frac{2}{p}}}{2^{\frac{2}{p}}h^{2}}} \frac{\e{D^{p}}^2}{\e{D^{2p}}},
\end{align*}
where the first inequality stems from (K3); the second comes from the Paley-Zygmund inequality;
and we used $\theta \leq \frac{1}{2}$ for the third, as well as the fact that $\norm{\Phi_\kappa(Z)}=F(0)=M$ by (K2).
Since $\e{\norm{\Phi_{\kappa}(Z)}^2}-\norm{\e{\Phi_{\kappa}(Z)}}^2 = \e{\norm{X}^2} - \norm{\e{X}}^2 = \e{\norm{X-\e{X}}^2}$,
we deduce
\[
\frac{M^2}{\e{\norm{X-\e{X}}^2}} \leq \frac{4.2^{\frac{2}{p}}}{c} \max\paren{1,\frac{h}{\norm{D}_{L^p(P)}}}^2 
\paren{\frac{\e{D^{2p}}}{\e{D^p}^2}}.
\]
Finally, note that
\[
\e{D^{2p}} = \e{\norm{Z-Z'}^{2p}} \leq \e{ (\norm{Z-\e{Z}} + \norm{Z'-\e{Z'}})^{2p}}
\leq 2^{2p}  \e{ \norm{Z-\e{Z}}^{2p}},
\]
and by Jensen's inequality
\[
\e{\norm{Z-\e{Z}}^p} = \e{\norm{Z-\e{Z'}}^p}  \leq \e{\norm{Z-Z'}^p} = \e{D^p}.
\]
(Observe that the equality $\e{\norm{Z-Z'}^2} = 2\e{\norm{Z-\e{Z}}^2}$ holds, so the constants
in the first, resp. second inequality
above can be improved for the special cases $p=1$, resp. $p=2$.)
\qed

\section{Illustrative example: denoising of {MNIST} images}\label{sec:apx_mnist}~\HM{
	This section compares the testing and $Q$-aggregation approaches for estimating multiple real means using an illustrative toy data set. 
	The MNIST dataset, consisting of $28 \times 28$ grayscale images scaled to $[0,1]$ of handwritten digits ($0 - 9$), is used. 
	Figure~\ref{fig:mnist_examples} provides some examples.
	\begin{figure}[h]
		\includegraphics[width=\textwidth]{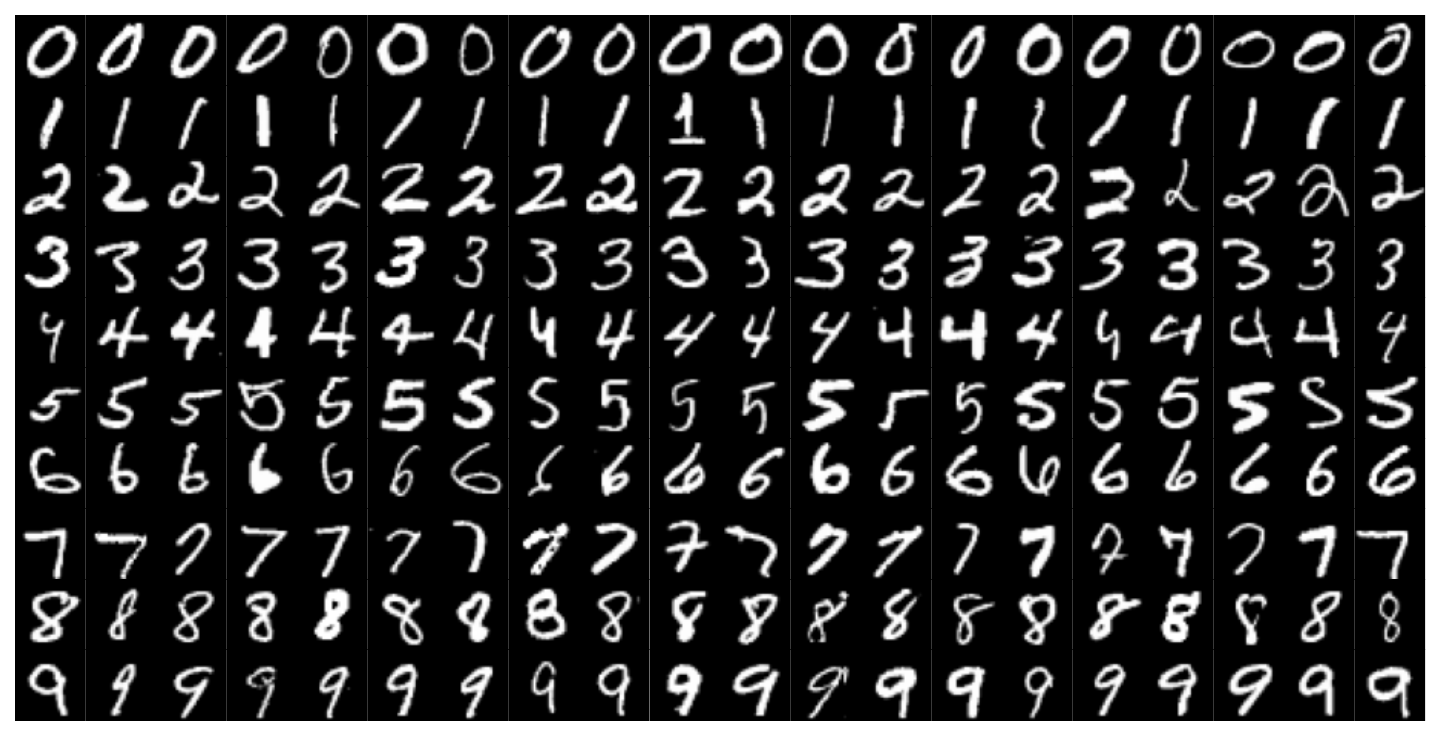}
		\caption{
			Example images of handwritten digits from the MNIST data set.}\label{fig:mnist_examples}
	\end{figure}
	$100$ images per digit are selected and corrupted with additive standard normal noise, each creating $20$ samples, resulting in $1000$ bags of bag size $20$. 
	\Revision{
		Each bag comprises 20 copies of the same underlying image but each altered by independent noise. Figure~\ref{fig:mnist_estimations} shows example instances. The objective is to recover the original image (the true mean), which in this synthetic setting corresponds to denoising via averaging out the standard normal noise. However, leveraging similarities across bags, such as those containing the same or similar digits, may further enhance the estimation accuracy.
	}
	Independent collections of bags are used for model parameter selection and testing, with experiments repeated $50$ times.
	
	This data set allows us to inspect the quality of the estimation, even though the means are high-dimensional ($28 \times 28 = 784$), and to showcase the validity of the neighbouring test as we can compare our understanding of similarly appearing digits with the outcome of the neighbouring test. 
	Figure \ref{fig:mnist_st} shows $V_{\tau,\cteW}$ for different $\tau$, revealing a block structure corresponding to digit-specific bags. 
	The block diagonal confirms similarity within the same digit. 
	Digit $0$ is distinct due to its circular shape but empty interior, while $1$ is similar to most digits due to its simple vertical line. 
	Digits $4$, $7$, and $9$ are judged as similar, aligning with their visual resemblance. 
	While smaller $\tau$ values yield more conservative neighbour sets, we recall that the neighbouring test only functions as a safeguard for \STBegd{} which can, thus, resort to a higher $\tau$. 
	It primarily excludes outliers like $0$ and identifies further similar pairs, such as $3$ and $5$ or $3$ and $8$.
	
	\begin{figure}[ht]
		\includegraphics[width=\textwidth]{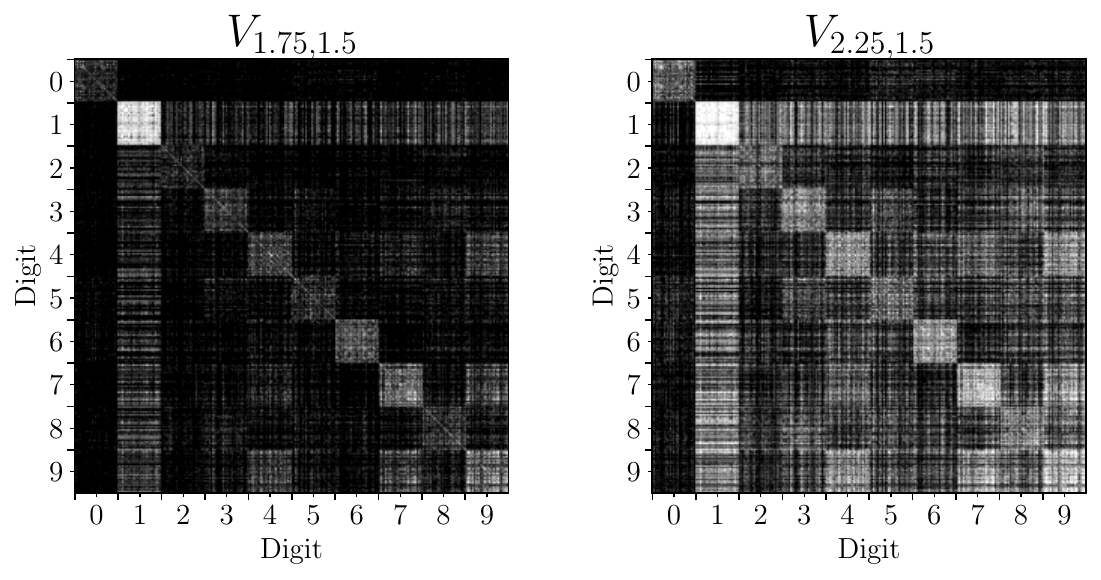}
		\caption{$V_{\tau,\cteW}$ for the optimal values of $\tau$ and $\cteW$ for \STBopt{} (left) and \STBegd{} (right) on the MNIST data set. $V_{\tau,\cteW}$ is not necessarily symmetric, and each row corresponds to the outcome of the neighbouring test for one mean, with white indicating `neighbours' and black `no neighbours'.}
		\label{fig:mnist_st}
	\end{figure}
	
	\begin{figure}[ht]
		\includegraphics[width=\textwidth]{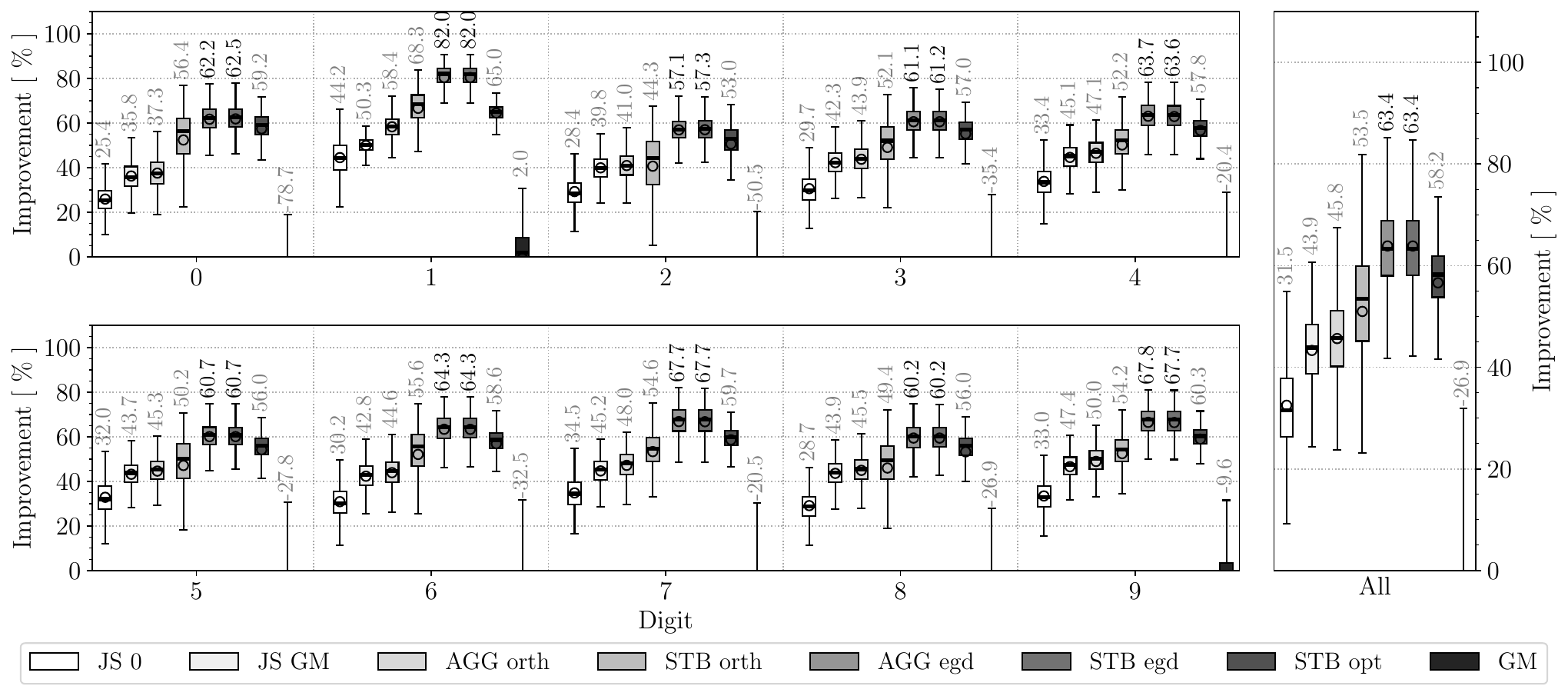}
		\caption{
			Decrease in estimation error compared to \NE{} in percent on the MNIST data set. Higher is better. The number next to the boxplot quantifies the median, which is also depicted as a line. The mean is visualised as a circle. Results are shown on each digit individually and on `All' jointly.}
		\label{fig:mnist_errors}
	\end{figure}
	
	Figure~\ref{fig:mnist_errors} shows the improvement over the naive estimation achieved by the proposed methods. As comparison we also denote the performances of the James-Stein estimator with a shrinkage towards zero (\texttt{JS 0}) or to the grand mean (\texttt{JS GM}). 
	The true instead of the empirically estimated (co-) variance is provided to the \texttt{JS} estimator, giving it an advantage over our fully empirical methods. 
	Nevertheless, the testing approach (\STBopt) performs significantly better than the \texttt{JS} estimator. In contrast to the experiments on the estimation of the kernel mean embeddings, performing the aggregation methods on a preselected set of neighbouring means only provides a small benefit for \AGGo.
	The $Q$-aggregation approach, \AGGe{} and \STBegd, surpasses all remaining methods. 
	
	\begin{figure}[ht]
		\includegraphics[width=0.4\textwidth]{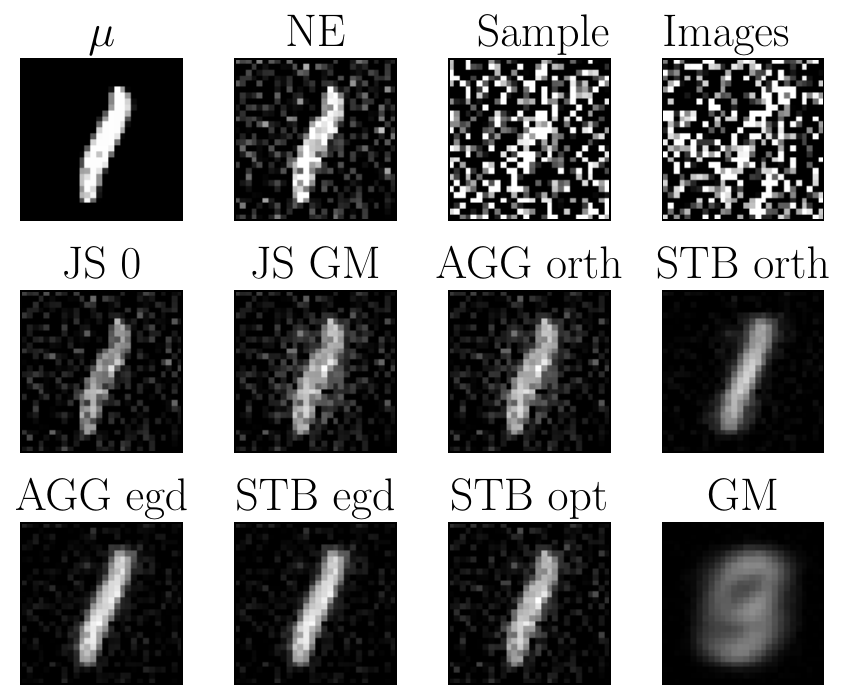}
		\hspace{0.02\textwidth}
		\includegraphics[width=0.4\textwidth]{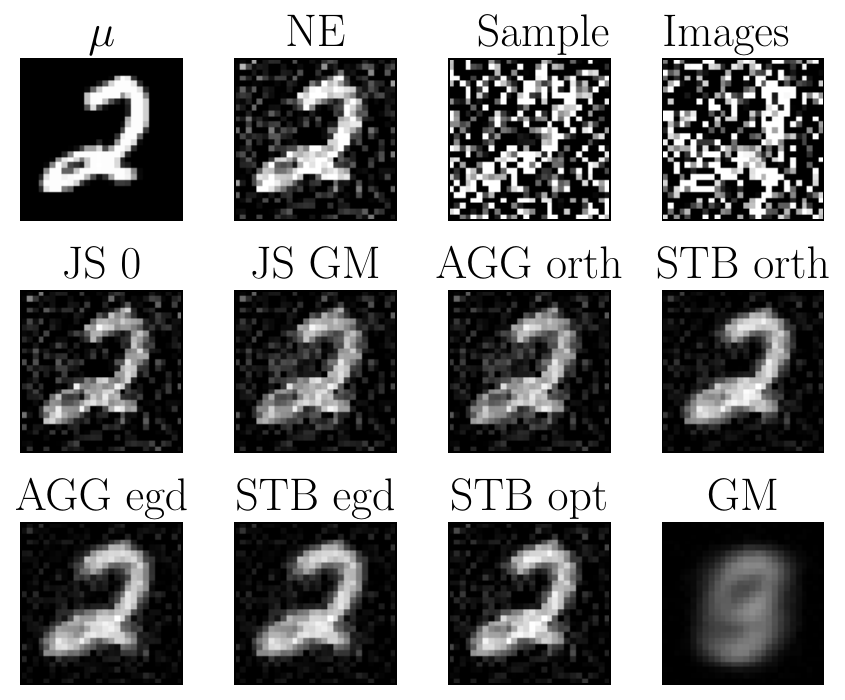}
		\caption{
			The true mean, some samples and the estimations of the mean of a bag of digit $1$ (left) and $2$ (right).}
		\label{fig:mnist_estimations}
	\end{figure}
	
	The denoising quality and the sharpness of the estimated images are depicted in Figure \Ref{fig:mnist_estimations}. 
	Because the grand mean (\texttt{GM}) averages all available images, it has the best denoising ability for background pixels but lacks to reproduce the desired digit. 
	Compared to the James-Stein estimations, our approaches provide an even background and a bright digit. 
	However, due to the convex combination of multiple images, transitions between background and digits are gradual (cf. black hole in curl of $2$), yet the overall image quality is best.}

\section{Details on the tested methods}\label{sec:apx_methods}~\subsection{Description of the methods}\label{sec:apx_desc_methods}~\HM{
	We consider convex combinations of naive empirical means (\NE) which results in estimators of the form
	\begin{equation*}
		\estmu_k^{\texttt{m}} := \sum_{\ell \in \intr{B}} \om_{k_\ell}^{\texttt{m}} \muNE_\ell \qquad \text{ s.t. } \qquad \forall k,\ell \in \intr{B}: \sum_{\ell \in \intr{B}} \om_{k_\ell}^{\texttt{m}} = 1 \,\,\, \text{and} \,\,\, \om_{k_\ell}^{\texttt{m}} \geq 0,
	\end{equation*}
	where the definition of the weighting $\omega^{\texttt{m}}$ depends on the applied method \texttt{m}.
	The details for all tested methods are provided next, both for the estimation of real means and kernel mean embeddings.
	For a description of the James-Stein estimator, see Supplemental~\ref{se:JS}.
	We start with a listing of empirical estimators required by the methods that were already introduced in the theoretical results. Some of them are rewritten or modified to be computable for large datasets and in the kernel framework.
	We emphasize that \emph{none} of the proposed methods perform sample splitting in practice. Instead all estimations are performed on the same samples and function fully empirically without any prior knowledge.
	
	\subsubsection{Empirical estimators}
	\begin{itemize}
		
		\item $\wh{\Sigma}_k$: In finite dimension, some of the following estimators employ the empirical covariance matrix,
		\begin{equation}\label{eq:emp_covariance}
			\wh{\Sigma}_k = \frac{1}{N_k -1} \sum_{i=1}^{N_k} \paren{X_i^{(k)} - \muNE_k}\paren{X_i^{(k)} - \muNE_k}^T\,.
		\end{equation}
		\item $\hat{s}^2$: 
		The computation of the naive risk requires the trace of the covariance matrix $\tr{\Sigma}$ which can either be estimated as the trace of the empirical covariance matrix $\tr{\hat{\Sigma}}$, if $\hat{\Sigma}$ as in \eqref{eq:emp_covariance} is available, or using \eqref{eq:esttrace} by
		\begin{equation*}
			\tr{\hat{\Sigma}_k} = {\hat{Z}_k}^{(1)} 
			= \frac{1}{N_k - 1}\sum_{i=1}^{N_k} 
			\norm{{X}_i^{(k)} - \muNE_k}^2.
		\end{equation*} 
		Proposition~\ref{prop:estnrisk} presents an estimator for the naive risk on an independent sample which in practice is calculated on $X$ instead
		\begin{equation}\label{eq:shatrme}
			\hat{s}_k^2 = \frac{\tr{\hat{\Sigma}_k}}{N_k} 
			= \frac{1}{N_k(N_k-1)}\sum_{i=1}^{N_k} \norm{X^{(k)}_i - \muNE_k}^2.
		\end{equation}
		For data in a RKHS, we propose the following estimator
		\begin{equation}\label{eq:shatkme}
			\hat{s}_k^2 = \frac{1}{2 N_k^2 (N_k - 1)} \sum_{i \neq j}^{N_k} \kappa( Z_{i}^{(k)},Z_{i}^{(k)} ) - 2\kappa( Z_{i}^{(k)},Z_{j}^{(k)} ) + \kappa( Z_{j}^{(k)},Z_{j}^{(k)} ).
		\end{equation}
		
		\item $\widehat{\tr{(\Sigma^2)}}$:
		An unbiased estimator for $\tr{(\Sigma^2)}$ is given by \eqref{eq:def_t_trsigma2} which is equivalent to
		\begin{align}\label{eq:trE2_rme}
			\hat{Z}_k^{(2)^2} = &\frac{(N_k-1)^2}{N_k(N_k-3)}\tr{(\hat{\Sigma}_k^2)}
			- \frac{1}{(N_k-2)(N_k-3)} \sum_{i=1}^{N_k} \norm{{X}_i^{(k)} - \muNE_k}^2 \nonumber \\
			&+ \frac{N_k-1}{N_k(N_k-2)(N_k-3)}\tr{(\hat{\Sigma}_k)}^2.
		\end{align}
		In a RKHS, \eqref{eq:def_t_trsigma2} can be translated to 
		\begin{align}\label{eq:trE2_kme}
			&\frac{1}{N_k (N_k - 1)} \sum_{{n_1} \neq {n_2}}^{N_k} {\kappa(Z_{n_1}^{(k)}, Z_{n_2}^{(k)})}^2 \nonumber \\
			&- \frac{2}{N_k (N_k - 1) (N_k - 2)} \sum_{{n_1} \neq {n_2} \neq {n_3}}^{N_k} \kappa(Z_{n_1}^{(k)}, Z_{n_2}^{(k)}) \kappa(Z_{n_1}^{(k)}, Z_{n_3}^{(k)}) \\
			&+ \frac{1}{N_k (N_k - 1) (N_k - 2) (N_k - 3)} \sum_{{n_1} \neq {n_2} \neq {n_3} \neq {n_4}}^{N_k} \kappa(Z_{n_1}^{(k)}, Z_{n_2}^{(k)}) \kappa(Z_{n_3}^{(k)}, Z_{n_4}^{(k)}), \nonumber
		\end{align}
		for $N_k \geq 4$. 
		However this estimator has computational complexity $\mathcal{O}(N_k^4)$ and is infeasible in practice for large $N_k$.
		Instead, we propose in Algorithm~\ref{alg:approx_trE} a subsampling strategy that gives a less accurate estimate but operates in $\mathcal{O}(N_k)$.
		\begin{algorithm}[ht]
			\caption{Subsampling estimate of $\widehat{\tr{\Sigma^2}}$ for the kernel setting}\label{alg:approx_trE}
			\begin{algorithmic}[1]
				\Require data $Z^{(k)}_{\bullet}$, bag size $N_k$, number of repetitions $r$
				\State \algorithmiccomment{ initialise }
				\State t$_1 \leftarrow 0$
				\State t$_2 \leftarrow 0$
				\State t$_3 \leftarrow 0$
				\State \algorithmiccomment{ first term can be calculated directly in linear time}
				\State t$_1 \leftarrow \sum_{i,j}^{N_k} \kappa(Z_i^{(k)},Z_{j}^{(k)})^2 - \sum_{i}^{N_k} \kappa(Z_i^{(k)},Z_i^{(k)})^2$
				\State \algorithmiccomment{ other terms are approximated in $r$ iterations}
				\For{$1$ to $r$}
				\State \algorithmiccomment{ select four distinct samples }
				\State $n_1, n_2, n_3, n_4 \leftarrow \text{randint}(1,N_k,4)$
				\State \algorithmiccomment{ approximate second and third term }
				\State t$_2 \leftarrow \text{t}_2 + \kappa(Z_{n_1}^{(k)},Z_{n_2}^{(k)}) \cdot \kappa(Z_{n_1}^{(k)},Z_{n_3}^{(k)})$
				\State t$_3 \leftarrow \text{t}_3 + \kappa(Z_{n_1}^{(k)},Z_{n_2}^{(k)}) \cdot \kappa(Z_{n_3}^{(k)},Z_{n_4}^{(k)})$
				\EndFor
				\State \algorithmiccomment{ normalise and add}
				\State $\widehat{\tr{\Sigma^2}} \leftarrow \text{t}_1 / {(N_k (N_k - 1))} - 2 \text{t}_2 / r + \text{t}_3 / r$
				\State \Return $\widehat{\tr{\Sigma^2}}$
			\end{algorithmic}
		\end{algorithm} 
		
		\item $U$: 
		An unbiased estimator for $\norm{\Delta_{\ell}^{(k)}}^2 = \norm{\mu_{\ell} - \mu_k}^2$ on an independent sample is provided in \eqref{eq:def_testU} with $\norm{\mu_{\ell} - \mu_k}^2$ corresponding to $U_{k, \ell}$ which becomes
		\begin{align}\label{eq:apxU}
			U_{k,\ell} &= \sum_{b \in \lbrace k, \ell \rbrace}\left(
			\sum_{i \neq j}^{N_b} \frac{\ip{X_i^{(b)}}{X_j^{(b)}}}{N_b (N_b - 1)} \right)
			- 2 \sum_{i=1}^{N_k} \sum_{j=1}^{N_{\ell}} \frac{\ip{X_i^{(k)}}{X_j^{(\ell)}}}{N_k N_{\ell}} 
			\nonumber \\
			&= \sum_{b \in \lbrace k, \ell \rbrace} \left( \frac{N_b^2 \ip{\muNE_b}{\muNE_b} - \sum_{i=1}^{N_b} \ip{X_i^{(b)}}{X_i^{(b)}}}{N_b (N_b - 1)} \right) - 2 \ip{\muNE_k}{\muNE_{\ell}}.
		\end{align}
		For the kernel setting, the maximum mean discrepancy can be employed for an unbiased estimation of $\norm{\mu_{\ell} - \mu_k}^2_{\mathcal{H}}$,
		\begin{align}\label{eq:unmmd}
			\hat{\text{MMD}}^2(\mu_k, \mu_\ell) =& \sum_{b \in \lbrace k, \ell \rbrace}\left(
			\sum_{i \neq j}^{N_b} \frac{\kappa(Z_i^{(b)}, Z_j^{(b)})}{N_b (N_b - 1)} \right)
			- 2 \sum_{i=1}^{N_k} \sum_{j=1}^{N_{\ell}} \frac{\kappa(Z_i^{(k)}, Z_j^{(\ell)})}{N_k N_{\ell}}.
		\end{align}
		
		\item $\hat{q}$: The $Q$-aggregation approach requires an estimation for $q$ as provided in \eqref{eq:defwhq},
		\begin{equation}\label{eq:apx_q_rme}
			\hat{q}^{(k)}_{\ell} = \frac{1}{N_k - 1} \sum_{i=1}^{N_k}
			\ip{\muNE_k - \muNE_{\ell}}{X_i^{(k)} - \muNE_k}^2.
		\end{equation}
		For data in a RKHS, we propose a biased estimate 
		\begin{align}\label{eq:apx_dedbiased}
			\hat{q}^{(k)}_{\ell} =&\frac{1}{N_k-1} \sum_{i=1}^{N_k} {\left( \frac{1}{N_\ell} \sum_{j=1}^{N_\ell} \kappa(Z_{i}^{(k)},Z_{j}^{(\ell)}) -
				\frac{1}{N_k} \sum_{j=1}^{N_k} \kappa(Z_{i}^{(k)},Z_{j}^{(k)}) \right)}^2 \nonumber\\
			&- \frac{N_k}{N_k - 1} {\left(
				\frac{1}{N_k N_\ell} \sum_{i=1}^{N_k} \sum_{j=1}^{N_\ell}\kappa(Z_{i}^{(k)},Z_{j}^{(\ell)})
				- \frac{1}{N_k N_k} \sum_{i=1}^{N_k} \sum_{j=1}^{N_k}\kappa(Z_{i}^{(k)},Z_{j}^{(k)})
				\right)}^2.
		\end{align}
		For translation invariant kernels we use a "less biased" estimate
		\begin{align*}
			\hat{q}^{(k)}_{\ell} =&\frac{1}{N_k-1} \sum_{i=1}^{N_k} {\left( \frac{1}{N_\ell} \sum_{j=1}^{N_\ell} \kappa(Z_{i}^{(k)},Z_{j}^{(\ell)}) -
				\frac{1}{N_k-2} \sum_{j=1}^{N_k} \kappa(Z_{i}^{(k)},Z_{j}^{(k)}) \right)}^2 \nonumber\\
			&- \frac{N_k}{N_k - 1} {\left(
				\frac{1}{N_k N_\ell} \sum_{i=1}^{N_k} \sum_{j=1}^{N_\ell}\kappa(Z_{i}^{(k)},Z_{j}^{(\ell)})
				- \frac{1}{N_k (N_k-2)} \sum_{i=1}^{N_k} \sum_{j=1}^{N_k}\kappa(Z_{i}^{(k)},Z_{j}^{(k)})
				\right)}^2.
		\end{align*}
		Its computational complexity is in $\mathcal{O}(N_k^2)$.
		
		\item $\hat{\Lambda}_{\ell,\ell^\prime}^{(k)} := \ip{\muNE_\ell - \muNE_k}{\muNE_{\ell^\prime} - \muNE_k}$: The inner product of the distances between the empirical means might be understood as a biased estimation of the Gram matrix ${(\ip{\Delta_\ell^{(k)}}{\Delta_{\ell^\prime}^{(k)}})}_{\ell,\ell^\prime \in \intr{B}}$. For the kernel setting this translates to
		\begin{equation}\label{eq:apx_lambdabiased}
			\hat{\Lambda}_{\ell,\ell^\prime}^{(k)} = \begin{cases}
				0
				, &\text{for   } k = \ell \text{, or } k = \ell^\prime \text{, or } k = \ell = \ell^\prime\\

				\begin{subarray}{1}
					\frac{1}{N_\ell N_{\ell^\prime}}\sum_{j}^{N_\ell} \sum_{j^\prime}^{N_{\ell^\prime}} \kappa(Z_j^{(\ell)}, Z_{j^\prime}^{(\ell^\prime)}) \\
					\,\,\,\,\,\, - \frac{1}{N_\ell N_k}\sum_{j}^{N_\ell} \sum_{i}^{N_k} \kappa(Z_j^{(\ell)}, Z_{i}^{(k)}) \\
					\,\,\,\,\,\, - \frac{1}{N_k N_{\ell^\prime}}\sum_{i}^{N_{k}} \sum_{j^\prime}^{N_{\ell^\prime}} \kappa(Z_i^{(k)}, Z_{j^\prime}^{(\ell^\prime)}) \\
					\,\,\,\,\,\, + \frac{1}{N_k N_k}\sum_{i}^{N_{k}} \sum_{i^\prime}^{N_{k}} \kappa(Z_i^{(k)}, Z_{i^\prime}^{(k)})
				\end{subarray}
				, &\text{otherwise.}
			\end{cases}
		\end{equation}
		
		\item $\hat{W}_{\cteW}, \hat{V}_{\tau, \cteW}$: The (pre-) selection of neighbours is performed on the same sample as used for the estimation and not, as required by the theoretical results, on an independent sample (cf. \eqref{eq:setWandtkdtilde}, \eqref{eq:setVdtilde}). We obtain
		\begin{gather}
			\hat{W}_{\cteW}^{(k)} = \left\lbrace \ell \in \intr{B} : \frac{\hat{Z}^{(2)}_\ell}{N_\ell} \leq \cteW \frac{\hat{Z}^{(2)}_k}{N_k}\right\rbrace
			\nonumber \\
			\hat{T}^{(\tau)}_{k,\ell} := \ind{ U_{k,\ell} \leq \tau \hat{\rnrisk}^2_k}, \quad
			\hat{V}_{\tau,\cteW}^{(k)} := \set[2]{ \ell \in \hat{W}_{(\cteW)}: \hat{T}^{(\tau)}_{k,\ell}=1},
			\label{eq:apx_rstb}
		\end{gather}
		with $\hat{Z}^{(2)}_k$ as in \eqref{eq:trE2_rme}, $U_{k,\ell}$ in \eqref{eq:apxU}, and $\hat{\rnrisk}^2_k$ in \eqref{eq:shatrme}.
		For the kernel setting, the notation translates to
		\begin{gather}
			\hat{W}_{\cteW}^{(k)} = \left\lbrace \ell \in \intr{B} : \frac{\sqrt{\widehat{\tr{(\Sigma^2_\ell)}}}}{N_\ell} \leq \cteW \frac{\sqrt{\widehat{\tr{(\Sigma^2_k)}}}}{N_k}\right\rbrace
			\nonumber \\
			\hat{T}^{(\tau)}_{k,\ell} := \ind{ \hat{\text{MMD}}^2(\mu_k, \mu_\ell) \leq \tau \hat{\rnrisk}^2_k}, \quad
			\hat{V}_{\tau,\cteW}^{(k)} := \set[2]{ \ell \in \hat{W}_{(\cteW)}: \hat{T}^{(\tau)}_{k,\ell}=1},
			\label{eq:apx_kstb}
		\end{gather}
		with $\widehat{\tr{(\Sigma^2_\ell)}}$ as in \eqref{eq:trE2_kme}, $\hat{\text{MMD}}^2(\mu_k, \mu_\ell)$ in \eqref{eq:unmmd}, and $\hat{\rnrisk}^2_k$ in \eqref{eq:shatkme}.
	\end{itemize}
	These estimations can be plugged in the computation of the convex combination weights $\omega^{\texttt{m}}$ of method \texttt{m} as presented next.
	\subsubsection{State-of-the-art approaches}\label{sec:apx_desc_sota}
	\begin{enumerate}[label=(\roman*)]
		\item \NE{} considers each bag individually.
		\begin{equation*}
			\omega^{\NE}_{k_\ell} = \begin{cases}
				1, &\text{for   } k = \ell \\
				0, &\text{otherwise.}
			\end{cases}
		\end{equation*}
		
		\item \RKMSE{} \citep{muandet2016stein} estimates each KME individually but shrinks it towards $\mathbf{0}$.
		It corresponds to a James-Stein estimator in a RKHS (cf. Supplemental~\ref{se:JS} for a presentation of the original JS-estimator in $\mathbb{R}^d$).
		The amount of shrinkage is data dependent
		\begin{equation*}
			\omega^{\RKMSE}_{k_\ell} = \begin{cases}
				1-\frac{\lambda_k}{1+\lambda_k}, &\text{for   } k = \ell \\
				0, &\text{otherwise}
			\end{cases}
		\end{equation*}
		where
		\begin{equation*}
			\lambda_k = \frac{\varrho_k - \rho_k}{(\nicefrac{1}{N_k} - 1)\varrho_k + (N_k - 1)\rho_k}
		\end{equation*}
		with $\varrho_k = \nicefrac{1}{N_k} \sum_{i=1}^{N_k} \kappa(Z_i^{(k)},Z_i^{(k)})$ and
		$\rho_k = \nicefrac{1}{N_k^2} \sum_{i,j=1}^{N_k} \kappa(Z_i^{(k)},Z_{j}^{(k)})$.
		
		\item \MTAconst{} \citep{feldman2014revisiting} was initially proposed for the estimation for multiple real means.
		We adapted the approach such that it can be applied to the estimation of multiple kernel means
		\begin{equation*}
			\omega^{\MTAconst}_{k_\ell} = \left({\left(I_B + \frac{\gamma}{B} \hat{S}^2 \cdot \mathcal{L}(A)\right)}^{-1}\right)_{k_\ell}.
		\end{equation*}
		Here, $I_B$ denotes the $B$-dimensional identity, $\hat{S}^2 = \text{diag}\left( (\hat{s}_\ell^2)_{\ell \in \intr{B}}\right)$ with $\hat{s}_\ell^2$ as in \eqref{eq:shatkme}, 
		and $\mathcal{L}(A)$ denotes the graph Laplacian of task-similarity matrix $A$.
		For \MTAconst{} the similarity is assumed to be constant, i.e.,
		$A = a \cdot (\mathbf{1} \mathbf{1}^T)$
		with $a = \frac{1}{B (B-1)} \sum_{k,\ell \in \intr{B}} \norm{\estmu_k^{\NE} -\estmu_\ell^{\NE}}^2_\mathcal{H}$.
		The optimal value for model parameter $\gamma$ may be found using model optimization.
	\end{enumerate}
	
	\subsubsection{\texttt{AGG} approaches}\label{sec:apx_desc_agg}
	The aggregation approaches form a convex combination of possibly all bags whose weights are found directly by minimization of quantities related to the squared risk.
	\begin{enumerate}[label=(\roman*)]
		\setcounter{enumi}{3}
		\item \AGGo{} is based on the constrained optimization problem
		\begin{equation*}
			\pmb{\omega}_{k} = \underset{\mathbf{w}_{k}}{\text{argmin}} \left\lbrace \mathbb{E} \norm{\sum_{\ell \in \intr{B}} w_{k_\ell} \estmu_\ell^{\NE} - \mu_k}_\mathcal{H}^2 \right\rbrace \text{ s.t. } \forall k,\ell \in \intr{B} : \sum_{\ell \in \intr{B}} \omega_{k_\ell} = 1 \, , \,\,  \omega_{k_\ell} \geq 0.
		\end{equation*}
		Using Lagrangian multipliers the optimal solution can be derived as
		\begin{equation}\label{eq:nT_optproblem}
			\pmb{\omega}_{k} \simeq {\left( S^2 + \Lambda^{(k)} \right)}^{(-1)} \mathbf{1}
		\end{equation}
		where $S^2 = \text{diag}\left( ({s}_\ell^2)_{\ell \in \intr{B}}\right)$, 
		Gram matrix $\Lambda_{\ell,\ell^\prime}^{(k)} = \ip{\Delta_\ell^{(k)}}{\Delta_{\ell^\prime}^{(k)}}$, and $\mathbf{1}$ denotes a $B$-dimensional one vector.
		Simplifying assumption of \AGGo{} is $\Lambda_{\ell,\ell^\prime}^{(k)} = 0$ for all $\ell \neq \ell^\prime$ such that $\Lambda^{(k)}$ becomes a diagonal matrix with \[\Lambda^{(k)} = \text{diag}\left({\left({\norm{\mu_\ell - \mu_k}}^2 \right)}_{\ell \in \intr{B}} \right).\]
		This (purely heuristic) assumption states that all means $\ell, \ell^\prime$ spread pairwise orthogonal around mean $k$ in their centre, i.e., their differences are orthogonal.
		Because this has to hold for all distinct $k, \ell, \ell^\prime \in \intr{B}$, it is unlikely to hold yet it simplifies the optimization problem.
		All quantities are replaced by their empirical estimates, and \eqref{eq:nT_optproblem} reduces to
		\begin{equation*}
			\omega_{k_\ell}^{\AGGo} = \frac{1}{\hat{s}_\ell^2 + \gamma U_{k,\ell}}, \quad \text{or} \quad \omega_{k_\ell}^{\AGGo} = \frac{1}{\hat{s}_\ell^2 + \gamma \hat{\text{MMD}}^2(\mu_k, \mu_\ell)},
		\end{equation*}
		with $\hat{s}_\ell^2$ as in \eqref{eq:shatrme} or \eqref{eq:shatkme}, and $U_{k,l}$ in \eqref{eq:apxU} or $\hat{\text{MMD}}^2(\mu_k, \mu_\ell)$ in \eqref{eq:unmmd}.
		We add a multiplicative constant (model parameter) $\gamma$ for more flexibility.
		If the distances between bags is inhomogeneous, e.g., the data set contains close but also far distant unrelated bags, higher values of $\gamma$ might be advisable.
		Finally the weights are normalised such that they sum to one.
		
		\item \AGGe{} implements the $Q$-aggregation approach presented in Section~\ref{se:Qaggreg},
		\begin{equation*}
			\pmb{\omega}_{k}^{\AGGe} = \underset{\mathbf{w}_{k}}{\text{argmin}} \left\lbrace \hat{L}_k(\mathbf{w}_{k})+ c_q \hat{Q}_k(\mathbf{w}_{k})+ c_{bs} \hat{Q}^{\BS}_k(\mathbf{w}_{k})\right\rbrace .
		\end{equation*} 
		As there is no closed form analytic solution, the optimal solution is found by exponentiated gradient descent (egd) \citep{kivinen1997exponentiated} with gradient
		\begin{equation*}
			\nabla \pmb{\omega}_{k} = 2 \ip{\hat{\Lambda}^{(k)}}{\pmb{\omega}_{k}} + 2 \hat{S}_{k,\cdot}^2 + c_q \sqrt{\frac{\hat{\mathbf{q}}^{(k)}}{N_k}} + c_{bs}\frac{M}{N_k}\sqrt{\text{diag}(\hat{\Lambda}^{(k)})},
		\end{equation*} 
		where $\hat{S}_{k,\cdot}^2$ denotes the $k$-th column of the diagonal matrix $\hat{S} = \text{diag}\left( (\hat{s}_\ell^2)_{\ell \in \intr{B}}\right)$ with $\hat{s}_\ell^2$ as in \eqref{eq:shatrme}, while $\text{diag}(\hat{\Lambda}^{(k)})$ retrieves the diagonal of matrix $(\hat{\Lambda}^{(k)})_{\ell, \ell^\prime} = \ip{\muNE_\ell - \muNE_k}{\muNE_{\ell^\prime} - \muNE_k}$.
		Vector $\hat{\mathbf{q}}^{(k)}$ is defined in \eqref{eq:apx_q_rme}.
		Model parameters $c_q$ and $c_{bs}$ control the impact of $\hat{Q}$ or $\hat{Q}^{\textbf{BS}}$ respectively. 
		Exponentiated gradient descent updates the weights iteratively and for each bag separately. Using the gradient, an update step in iteration $t$ is performed as
		\begin{align*}
			\pmb{\omega}_{k}^{(t+1)} &= \pmb{\omega}_{k}^{(t)} \cdot \exp{\lbrace -\eta^{(t)} \cdot \nabla \pmb{\omega}_{k}^{(t)}\rbrace} \\
			\pmb{\omega}_{k}^{(t+1)} &= \frac{\pmb{\omega}_{k}^{(t+1)}}{\ip{\mathbf{1}}{\pmb{\omega}_{k}^{(t+1)}}}.
		\end{align*}
		$\eta^{(t)}$ denotes the learning rate and decreases over time, e.g., $\eta^{(t)} = \nicefrac{\eta}{(1+(t/B))}$, and $\mathbf{1}$ is a $B$-dimensional one vector. 
		
		In the kernel setting, we introduce a modified version 
		\begin{gather*}
			\pmb{\omega}_{k}^{\AGGe} = \underset{\mathbf{w}_{k}}{\text{argmin}}
			\set[3]{
				\hat{L}_k(\mathbf{w}_{k})+ c_q \hat{Q}_k(\mathbf{w}_{k})
				+ c_1 \sum_{\ell \in \intr{B}} w_{k_\ell} \vartheta_\ell
				+ c_2 \sum_{\ell \in \intr{B}} w_{k_\ell}^2 \vartheta_\ell
			},\\
			\hat{L}_k(\mathbf{w}_{k}) = \norm{\sum_{\ell \in \intr{B}} w_{k_\ell} \left( \estmu_\ell^{\NE} - \estmu_k^{\NE} \right)}_\mathcal{H}^2 + \hat{s}_k^2 \left( 2 w_{k_k} - 1 \right)\,, \\
			\hat{Q}_k(\mathbf{w}_{k}) = \sum_{\ell \in \intr{B}} w_{k_\ell} \sqrt{\frac{\hat{q}^{(k)}_{\ell}}{N_k}}, \qquad
			\vartheta_\ell = \frac{\sqrt{\widehat{\tr{(\Sigma^2_\ell)}}}}{N_\ell}.
		\end{gather*}
		Two regularisation terms are added.
		The $c_1$ term favours sparse results whereas the $c_2$ regularisation leads to diffuse, small weights.
		Their effect can be compared to that of 
		\JB{elastic net} regularisation. \JB{These regularisation terms are, relative to the naive risk, of order $O\paren{ (\deamm)^{-1/2}}$. From a theoretical point of view, their effect on the relative risk is not significant, as the bounds obtained in Section~\ref{se:Qaggreg} already contain an error of this order. In practice, we observed that they 
			improve the overall results. }
		
		\JB{For the estimation of kernel means, \AGGe{} has no instantiation of the regularisation term $\wh{Q}^\BS$ \eqref{eq:def_whQBS},
			which in the theory was introduced for the (BS) setting.
			Instead, we introduce \STBegd{} which applies an additional ``safeguard'' pre-testing step serving the same purpose,
			namely excluding far distant tasks (see description of \STBegd{} below for more details.)}
		
		The optimization over the probability simplex is done by exponentiated gradient descent with gradient
		\begin{equation*}
			\nabla \pmb{\omega}_{k} = 2 \left( {\hat{\Lambda}}^{(k)} + c_2 \, \text{diag}(\pmb{\vartheta}) \right) \pmb{\omega}_{k} + 2 \hat{S}_{k, \cdot}^2 + c_q \sqrt{\nicefrac{\hat{\mathbf{q}}^{(k)}}{N_k}} + c_1 \pmb{\vartheta},
		\end{equation*}
		where $\hat{S}_{k \cdot}$ denotes the $k$-th column of the diagonal matrix $\hat{S} = \text{diag} ( {(\hat{s}_\ell^2)}_{\ell \in \intr{B}} )$ with $\hat{s}_\ell^2$ as in \eqref{eq:shatkme}. 
		$\text{diag}(\pmb{\vartheta})$ denotes a diagonal matrix with vector ${(\vartheta)}_{\ell \in \intr{B}} = \nicefrac{\sqrt{\widehat{\tr{(\Sigma^2_\ell)}}}}{N_\ell}$ on the main diagonal, where $\widehat{\tr{(\Sigma^2_\ell)}}$ can be estimated as in \eqref{eq:trE2_kme} or by Algorithm~\ref{alg:approx_trE}.
		Vector $\hat{\mathbf{q}}^{(k)}$ is defined in \eqref{eq:apx_dedbiased}.
		The final procedure of \AGGe{} is shown in Algorithm~\ref{alg:agge}.
		\begin{algorithm}[t]
			\caption{\AGGe{} for the $k$-th KME}\label{alg:agge}
			\begin{algorithmic}[1]
				\Require $\hat{\Lambda}^{(k)} \in \mathbb{R}^{B \times B}$, $\pmb{\vartheta}\in \mathbb{R}^B$, $\hat{\mathbf{q}}^{(k)} \in \mathbb{R}^B$, $\hat{\mathbf{s}}^2 \in \mathbb{R}^B$, model parameters $c_q, c_1, c_2$, learning rate $\eta$, maximum nr. of iterations $t_{\max}$
				\State \algorithmiccomment{ initialise }
				\State $\omega_{i \cdot}^{(0)} \leftarrow \mathbf{1}/B$
				\State \algorithmiccomment{ until maximum nr$.$ of iterations or convergence}
				\While{$t \leq t_{\max}$ and $\left(\omega_i^{(t - 1)} - \omega_i^{(t)}\right)^2 > 10^{-8}$}
				\State \algorithmiccomment{ compute gradient }
				\State $\nabla \pmb{\omega}_{k}^{(t-1)} \leftarrow 2 \left( {\hat{\Lambda}}^{(k)} + c_2 \, \text{diag}(\pmb{\vartheta}) \right) \pmb{\omega}_{k}^{(t-1)} + 2 \hat{S}_{k, \cdot}^2 + c_q \sqrt{\nicefrac{\hat{\mathbf{q}}^{(k)}}{N_k}} + c_1 \pmb{\vartheta},$
				\State \algorithmiccomment{ perform exponentiated gradient descent  }
				\State $\pmb{\omega}_{k}^{(t)} \leftarrow \pmb{\omega}_{k}^{(t-1)} \cdot \exp{\lbrace -\eta^{(t)} \cdot \nabla \pmb{\omega}_{k}^{(t-1)} \rbrace} $
				\State \algorithmiccomment{ normalise }
				\State $\pmb{\omega}_{k}^{(t)} \leftarrow \frac{\pmb{\omega}_{k}^{(t)}}{\ip{\mathbf{1}}{\pmb{\omega}_{k}^{(t)}}}$
				\EndWhile
				\State \algorithmiccomment{ estimated optimal weight vector for bag $k$ }
				\State \Return $\pmb{\omega}_{k}^{(t)}$
			\end{algorithmic}
		\end{algorithm}
	\end{enumerate}
	
	\subsubsection{\texttt{STB} approaches}\label{sec:apx_desc_stb}
	The similarity test based approaches shrink the estimation only towards neighbouring means $\hat{V}_{\tau, \cteW}^{(k)}$, estimated by  \eqref{eq:apx_rstb} or \eqref{eq:apx_kstb}, such that
	\begin{equation*}
		\estmu_k^{\texttt{STB}} := \sum_{\ell \in \hat{V}_{\tau, \cteW}^{(k)}} \om_{k_\ell}^{\texttt{STB}} \muNE_\ell \qquad \text{ s.t. } \qquad \forall k,\ell \in \intr{B}: \sum_{\ell \in \intr{B}} \om_{k_\ell}^{\texttt{STB}} = 1 \,\,\, \text{and} \,\,\, \om_{k_\ell}^{\texttt{STB}} \geq 0.
	\end{equation*}
	
	\begin{enumerate}[label=(\roman*)]
		\setcounter{enumi}{5}
		\item \STBweight{} \citep{marienwald2021high} assigns a uniform weight to all neighbours except for $\omega_{k_k}$ which is higher
		\begin{equation*}
			\omega_{k_\ell}^{\STBweight} = \begin{cases}
				\gamma + \frac{1-\gamma}{\vert \hat{V}_{\tau, \cteW}^{(k)} \vert}, &\text{for   } k = \ell \\
				\frac{1-\gamma}{\vert \hat{V}_{\tau, \cteW}^{(k)} \vert}, &\text{for   } k \neq \ell,  \ell \in \hat{V}_{\tau, \cteW}^{(k)} \\
				0, &\text{otherwise.}
			\end{cases}
		\end{equation*}
		\STBweight{} was proposed for balanced bags and under independence of test and data, for which it has strong theoretical results (cf. Theorem 3.1 of \citet{marienwald2021high}).
		Larger values of $\tau$ allow higher distances between $\muNE_k$ and its neighbours, thus, potentially increase the number of neighbours and the bias of the estimation.
		Higher $\gamma$ values put emphasis on $\muNE_k$, i.e., $\omega_{k_k} > \omega_{k_\ell}$ for $k \neq \ell$, and the solution reduces to \NE{} for $\gamma = 1$.
		
		\item \STBopt{} corresponds to the optimal weights derived in Lemma~\ref{lem:oraclebound} that minimise an upper bound on the risk
		\begin{equation*}
			\pmb{\omega}_{k}^{\STBopt} = \underset{\mathbf{w}_{k}}{\text{argmin}}
			\left\lbrace
			\tau s_k^2 {(1-w_{k_k})}^2 + \sum_{\ell \in V_{\tau}^{(k)}} w_{k_\ell}^2 s_\ell^2
			\right\rbrace.
		\end{equation*}
		The optimal solution is derived in \eqref{eq:STBoptweights} using Lagrangian multipliers. 
		As a fully empirical procedure we propose
		\begin{equation*}
			\omega_{k_\ell}^{\STBopt} = \begin{cases}
				\lambda_k \cdot \hat{\nu}(k, \hat{V}_{\tau, \cteW}^{(k)}) + (1-\lambda_k), &\text{for   } k = \ell \\
				\lambda_k \cdot \hat{\nu}(\ell, \hat{V}_{\tau, \cteW}^{(k)}),                 &\text{for   } k \neq \ell,  \ell \in \hat{V}_{\tau, \cteW}^{(k)} \\
				0, &\text{otherwise.}
			\end{cases}
		\end{equation*}
		with $\hat{\nu}(\ell, \hat{V}_{\tau, \cteW}^{(k)}) := \hat{s}_\ell^{-2}/{(\sum_{\ell^\prime \in \hat{V}_{\tau, \cteW}^{(k)}} \hat{s}_{\ell^\prime}^{-2})}$, the aggregated variance relative to task $\ell$ (cf.~\eqref{eq:def_s2v}) with $\hat{s}_k^{2}$ as in \eqref{eq:shatrme} or \eqref{eq:shatkme},  and $\lambda_k := {(1+ \gamma \tau (1-\hat{\nu}(k, \hat{V}_{\tau, \cteW}^{(k)})))}^{(-1)}$.
		The additional multiplicative constant $\gamma$ allows for more flexibility and tends to put emphasis on $\omega_{k_k}$.
		
		\item \STBorth{} performs the similarity test and applies \AGGo{} on neighbouring means
		\begin{equation*}
			\omega_{k_\ell}^{\STBorth} = \begin{cases}
				\omega_{k_\ell}^{\AGGo},                 &\text{for   } \ell \in \hat{V}_{\tau, \cteW}^{(k)} \\
				0, &\text{otherwise.}
			\end{cases}
		\end{equation*}
		The similarity test merely functions as a safeguard here and excludes \JB{largely distant means} and does not play such a central role as for the other \texttt{STB} methods.
		Therefore, $\tau$ can be fixed to a large value, e.g., $\tau := 5$.
		Even though $\omega_{k_\ell}^{\AGGo}$ is reduced when ${\norm{\muNE_k - \muNE_\ell}}^2$ is high, \AGGo{} does not perform well when there are many \JB{largely distant means}.
		Their weights accumulate and reduce the weights of important bags because of the normalization step.
		The similarity test alleviates this problem.
		
		Compared to \STBweight{} and \STBopt{}, $\tau$ can be chosen larger.
		Because of this safeguard, $\gamma$, which penalises large distances, can be reduced.
		
		\item \STBegd{} performs the similarity test and applies \AGGe{} on neighbouring means
		\begin{equation*}
			\omega_{k_\ell}^{\STBegd} = \begin{cases}
				\omega_{k_\ell}^{\AGGe},                 &\text{for   } \ell \in \hat{V}_{\tau, \cteW}^{(k)} \\
				0, &\text{otherwise.}
			\end{cases}
		\end{equation*}
		Analogous to the discussion of \STBorth{} the similarity test functions as a safeguard to exclude \JB{largely distant means.}
		It can also be seen as an instrument to replace the missing $\wh{Q}^\BS$ in the implementation of \AGGe{} \JB{for kernel means} (see also discussion of \AGGe).
		This preselection of neighbouring means constitutes a stricter selection than the (implicit) one performed by the $\wh{Q}^\BS$-term. 
		\STBegd{} relies on several model parameters.
		Compared to the model parameter values of \AGGe{}, diffuse weights should be favoured whereas regularization based on the distances ($c_q$) or sparse weights ($c_1$) become less important because of the preselection of neighbouring means.
	\end{enumerate}
}

\subsection{Experimental results of additional methods}\label{sec:apx_expresults}
The results of \RKMSE{} \citep{muandet2016stein}, \STBweight{} \citep{marienwald2021high}, \AGGo, and \AGGe{} can be found in Figure~\ref{fig:apx_exp_art_results} and Figure~\ref{fig:apx_exp_cyto_results}.
\begin{figure}[ht]
	\includegraphics[width=\textwidth]{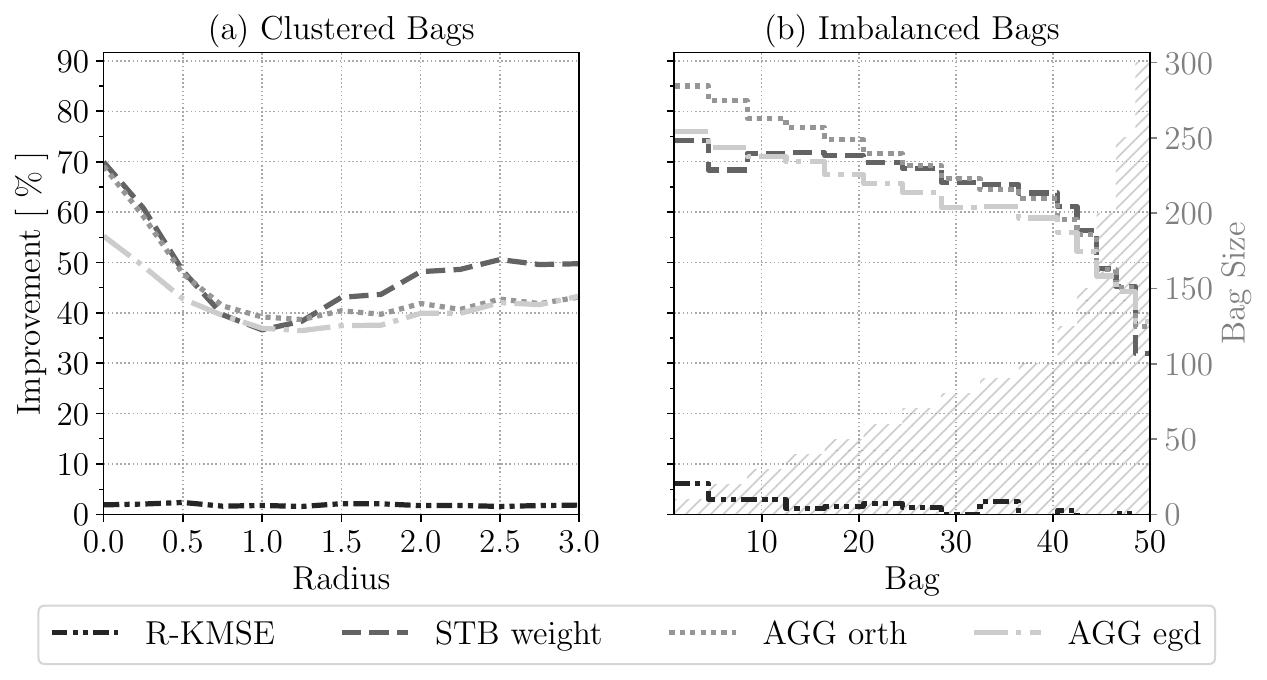}
	\caption{
		Decrease in estimation error compared to \NE{} in percent on Gaussian data settings (a) and (b) resp.
		Higher is better.
		The bars (right axis) in (b) show the bag sizes for the bags $1$ to $50$ which vary between $10$ and $300$.
		Compare with Fig.~\ref{fig:exp_art_results}}
	\label{fig:apx_exp_art_results}
\end{figure}
\begin{figure}[ht]
	\includegraphics[width=\textwidth]{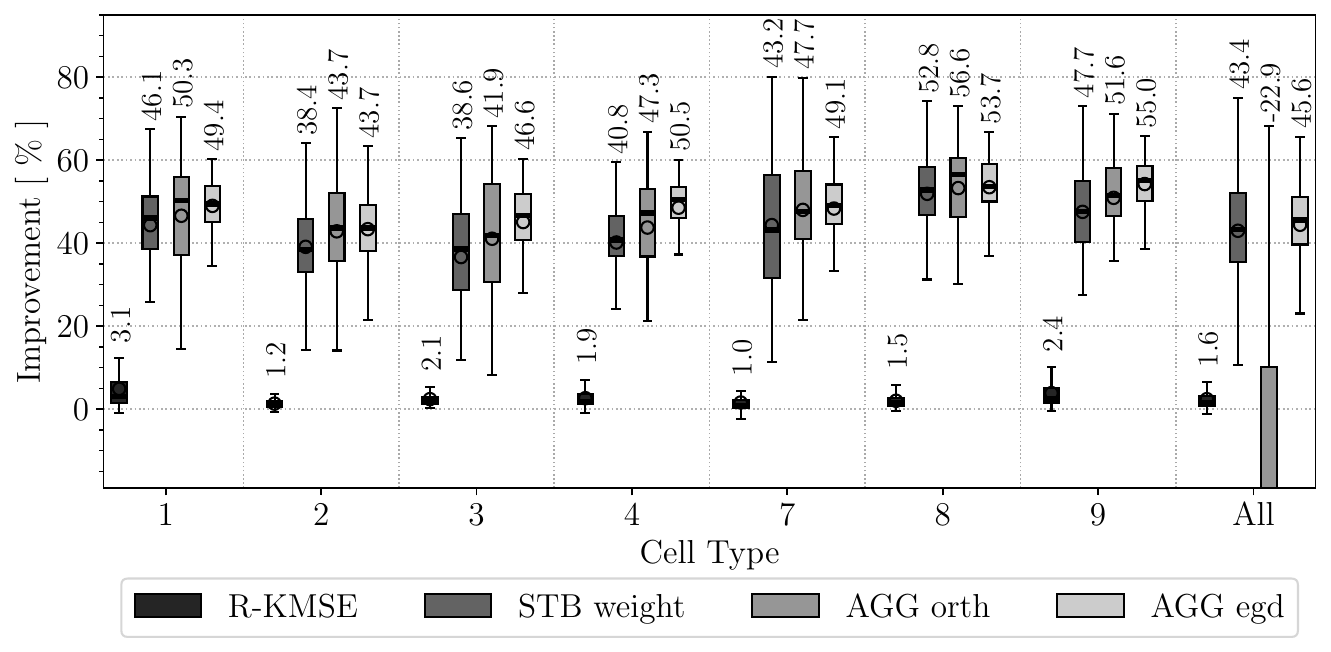}
	\caption{
		Decrease in estimation error compared to \NE{} in percent on the flow cytometry data.
		Higher is better.
		The number next to the boxplot quantifies the median which is also depicted as a line.
		The mean is visualised as a circle.
		From left to right: results on individual cell types $1,2,3,4,7,8,9$, and 'All' cell types jointly.
		The performance of \AGGo{} is partly occluded which allows a more detailed display of the remaining results.
		Its mean is $\approx -30$, $Q_1 \approx -60$, and its lower whisker ends at $\approx -140$.
		Compare with Fig.~\ref{fig:exp_cyto_results}}
	\label{fig:apx_exp_cyto_results}
\end{figure}
\RKMSE{} estimates each KME individually and only provides marginal improvement over \NE.
The performance of \STBweight{} is comparable to that of \STBopt, however, it gives less accurate estimations of large bags in setting (b).
\STBweight{} assumes equal variances of the estimations, thus, can not handle imbalanced bags.
Its weights estimation of large bags might be corrupted by small bags with high variance.
Comparing the performances of the aggregation methods \AGGo{} and \AGGe{} with \STBorth{} and \STBegd{} resp. shows that the similarity test is beneficial and functions as a safeguard to discard distant bags.

The same observations can be made on the cytometry data.
Especially on 'All' cell types, the importance of the similarity test can be noted.
\AGGo{} considers the distance between the bags in the definition of the weights.
However, due to the normalization $\sum_\ell \omega_{k_\ell}^{\AGGo} = 1$ small weights $\omega_{k_\ell}^{\AGGo}$ of many distant bags $\ell$ can accumulate and thereby reduce the impact of important weights.
The similarity test preselects only close means and eliminates this problem.

\subsection{Default model parameter values}\label{apx:sec_defaultParams}
The presented models have up to four data dependent model parameters and three hyperparameters.
Parameter tuning is possible, whenever the user wishes to estimate the (kernel) means of bags of size $N$ but also has bags of much larger size $N’ \gg N$.
For the optimization, e.g., in form of cross-validation, subsets of size $N$ are sampled from the $N’$ bags.
A method-specific parameter combination is then tested on all bags of size $N$, while the test error is (only) computed wrt. the $N’$ bags (again, a proxy true mean can be estimated using the complete $N’$ samples).

In most practical applications this scheme is not possible.
For this reason we propose default values for the estimation of kernel means that we observed to perform well in various settings.
Table~\ref{tab:apx_default} shows an overview.
\begin{table}[hb]\label{tab:apx_default}
	\caption{Summary of the default parameter values of each method for the estimation of kernel means.}
	\begin{tabular}{lllllllll} \toprule
		\phantom{0} & Method   & \multicolumn{2}{l}{Parameter(s)}      & & \phantom{0} & Method    & \multicolumn{2}{l}{Parameter(s)} \\ \midrule
		\multicolumn{4}{l}{SOTA}                           &   & \multicolumn{4}{l}{STB}                             \\
		& \NE       & $\emptyset$  &                     &   &    & \STBweight & $\tau$       & $= 2.2$             \\
		& \RKMSE    & $\emptyset$  &                     &   &    &           & $\gamma$     & $= 0.2$             \\
		& \MTAconst & $\gamma$     & $= 1.$              &   &    & \STBopt    & $\tau$       & $= 2.2$             \\
		&          &              &                     &   &    &           & $\gamma$     & $= 0.2$             \\
		\multicolumn{4}{l}{AGG}                            &   &    & \STBorth   & $\tau$       & $:= 5.$             \\
		& \AGGo     & $\gamma$     & $= 13.$             &   &    &           & $\gamma$     & $= 3.$             \\
		& \AGGe     & $c_q$        & $= 1.4$              &   &    & \STBegd    & $\tau$       & $:= 5.$             \\
		&          & $c_1$        & $:= 1.$             &   &    &           & $c_q$        & $:= 1.$              \\
		&          & $c_2$        & $= 4.$              &   &    &           & $c_1$        & $:= 1.$              \\
		&          & $r$          & $:= 100$             &   &    &           & $c_2$        & $= 5.$              \\
		&          & $t_{\max}$   & $:= 500$             &   &    &           & $r$          & $:= 100$             \\
		&          & $\eta^{(t)}$ & $:= 50 / (1+ (t/B))$ &   &    &           & $t_{\max}$   & $:= 500$             \\
		&          &              &                     &   &    &           & $\eta^{(t)}$ & $:= 50 / (1+ (t/B))$ \\ \bottomrule
	\end{tabular}
\end{table}
To determine these values we considered the Clustered setting (a).
We generated additional independent training samples on which we ran cross-validation to determine for each radius an optimal set of parameter values (25 repetitions).
Most default values correspond to their optimal choices for radius$=1.5$.
We chose a radius of $1.5$ because it presents a balance between overlapping and yet distinct clusters.
It also corresponds to a setting which is close to practice and for which multi-task averaging approaches generally promise an improvement over the naive estimation.

The methods based on egd form a peculiarity as the selection procedure did not uncover suitable parameter values.
For \STBegd{} we fixed $c_q = c_1 = 1$ because the values selected by CV were too low ($c_q$) or even zero ($c_1$) at radius $1.5$.
Because the similarity test selects only close neighbours, the distance controlled by parameter $c_q$, i.e., $\mathbf{q}^{(k)}$, is small.
Still, we found the term to be important in practice and, hence, fix $c_q = 1$.
Regularization by the $c_1$-term leads to sparse weights.
For the artificial data, the found neighbours are relatively homogeneous, therefore, no sparse solution must be acquired.
However, we found $c_1 = 1$ to be better in more general settings.
The same argument holds for \AGGe{} for which we had to fix $c_1 = 1$ as well.

We fixed $\tau := 5$ for \STBorth{} and \STBegd , because their similarity test merely functions as a precaution to discard high-distant bags.
As shown by the theoretical discussion, the \texttt{AGG} methods are equivalent to the \texttt{STB} methods with optimally selected $\tau$.

We did not optimise for the hyperparameters $r, t_{\max}$ and $\eta^{(t)}$.
For larger values of $r$ the approximation in Alg.~\ref{alg:approx_trE} becomes more accurate.
It also leads to a higher computational complexity.
We found $r=100$ to be a good trade-off.
The same trade-off can be observed for $t_{\max}$.
However, we found that for the artificial data egd converged usually before $t_{\max} = 500$ was reached.
For the selection of $\eta^{(t)}$ we follow the recommendation given in \citet{collins2008exponentiated, lecun1998gradient}.
\begin{figure}[ht]
	\includegraphics[width=\textwidth]{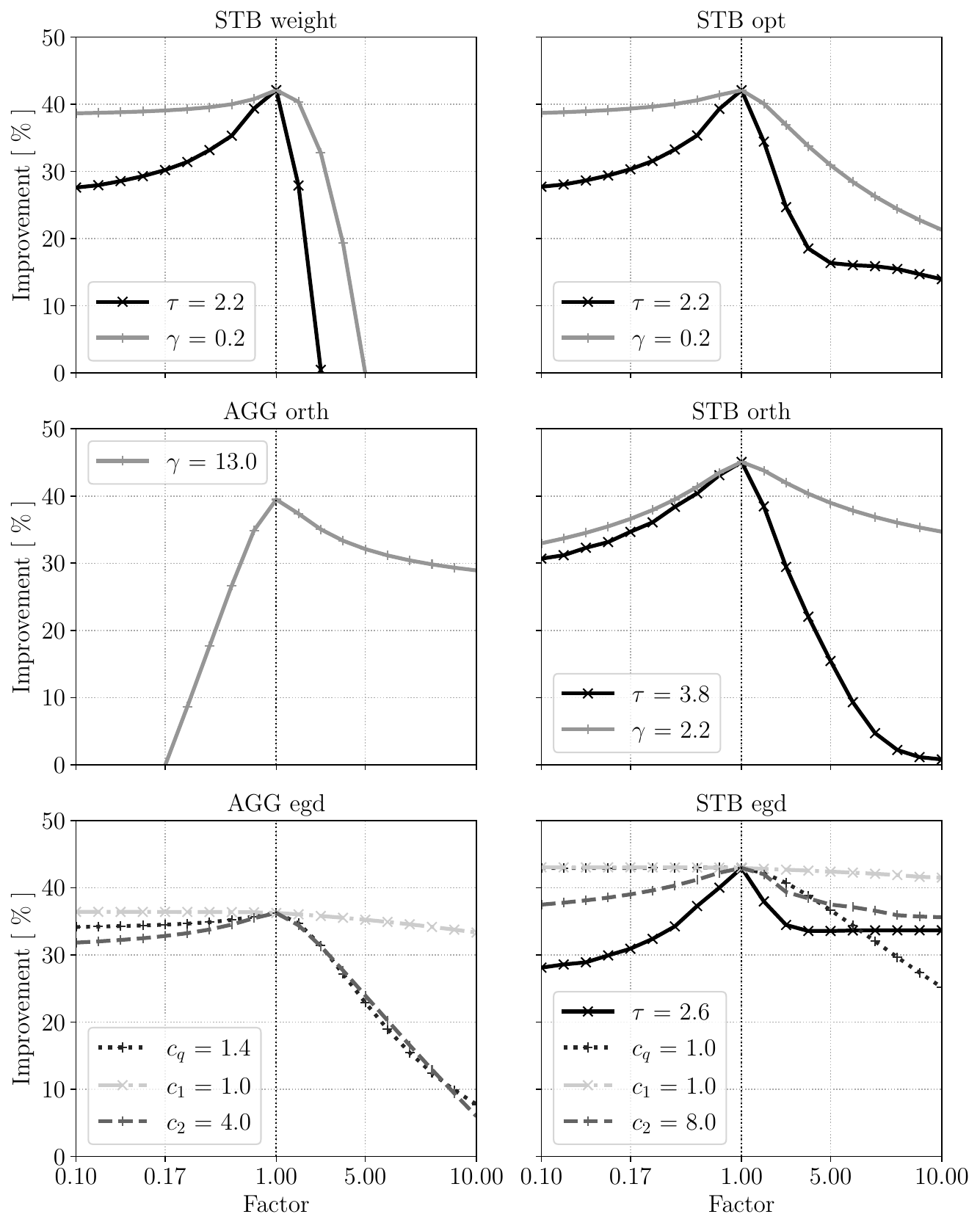}
	\caption{
		See Figure~\ref{fig:apx_modelParams2}.}
	\label{fig:apx_modelParams}
\end{figure}
\begin{figure}[ht]
	\begin{minipage}[t]{0.5\textwidth}
		\vspace{0pt}
		\includegraphics[width=\textwidth]{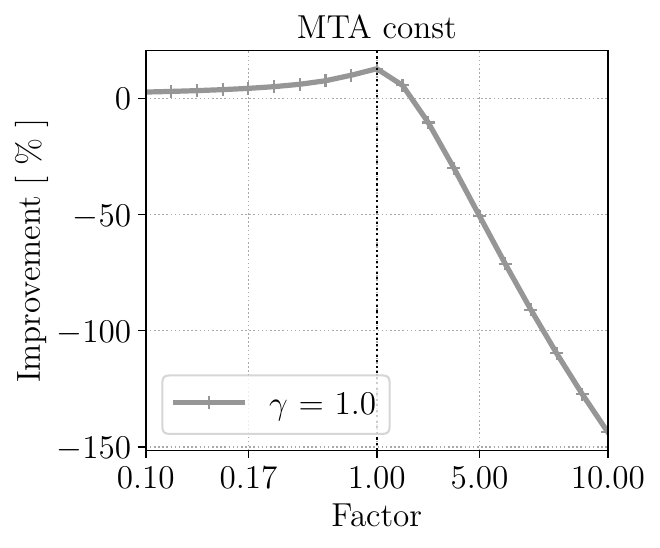}
	\end{minipage}\hspace{-0.1\textwidth}\hfill
	\begin{minipage}[t]{0.6\textwidth}
		\vspace{0pt}
		\caption{
			Decrease in estimation error compared to \NE{} on Clustered, radius $=1.5$, for different model parameter values.
			Each parameter is considered individually and changed by factors $\times0.1 - \times10$.
			The remaining parameters are kept at their optimal value.
			(Note that the y-axis is different for \MTAconst. Some settings lead to worse results than \NE .)
		} \label{fig:apx_modelParams2}
	\end{minipage}
\end{figure}

Figures~\ref{fig:apx_modelParams}-\ref{fig:apx_modelParams2} show the effect of the model parameter choice on the performance for each method.
\STBweight{} is highly affected by improper parameter choices, whereas the new \texttt{STB} methods are more stable.
Specifically, the effect of $\tau$ plateaus.
Even for small $\tau$, an improvement over \NE{} can be observed.
When the \texttt{AGG} methods are complemented with a preselection of bags (similarity test as safeguard), the methods become more robust.
Compared to the $c_2$- and $c_q$-terms of \AGGe{} and \STBegd , the $c_1$-term has only little effect.
This lets us conclude that $\ell_1$-regularization is not as important as $\ell_2$-regularization of the weights.
Furthermore, it seems that an improper value of one parameter can be alleviated if the remaining parameters are chosen correctly.

For completeness, we also report the optimised (again by cross-validation on iid training data) model parameter values for the Imbalanced setting in Table~\ref{tab:apx_optimbalanced}.
\begin{table}[hb]\label{tab:apx_optimbalanced}
	\caption{Summary of the optimised parameter values of each method for the imbalanced setting.}
	\begin{tabular}{lllllllll} \toprule
		\phantom{0} & Method   & \multicolumn{2}{l}{Parameter(s)}      & & \phantom{0} & Method    & \multicolumn{2}{l}{Parameter(s)} \\ \midrule
		\multicolumn{4}{l}{SOTA}                           &   & \multicolumn{4}{l}{STB}                             \\
		& \NE       & $\emptyset$  &                     &   &    & \STBweight & $\tau$       & $= 1.4$             \\
		& \RKMSE    & $\emptyset$  &                     &   &    &           & $\gamma$     & $= 0.05$             \\
		& \MTAconst & $\gamma$     & $= 1.1$              &   &    & \STBopt    & $\tau$       & $= 1.4$             \\
		&          &              &                     &   &    &           & $\gamma$     & $= 0.1$             \\
		\multicolumn{4}{l}{AGG}                            &   &    & \STBorth   & $\tau$       & $= 3.9$             \\
		& \AGGo     & $\gamma$     & $= 5.$             &   &    &           & $\gamma$     & $= 2.$             \\
		& \AGGe     & $c_q$        & $= 0.$              &   &    & \STBegd    & $\tau$       & $= 2.7$             \\
		&          & $c_1$        & $= 3.5$             &   &    &           & $c_q$        & $= 0.$              \\
		&          & $c_2$        & $= 14.$              &   &    &           & $c_1$        & $= 2.9$              \\
		&          & $r$          & $:= 100$             &   &    &           & $c_2$        & $= 34.$              \\
		&          & $t_{\max}$   & $:= 500$             &   &    &           & $r$          & $:= 100$             \\
		&          & $\eta^{(t)}$ & $:= 50 / (1+ (t/B))$ &   &    &           & $t_{\max}$   & $:= 500$             \\
		&          &              &                     &   &    &           & $\eta^{(t)}$ & $:= 50 / (1+ (t/B))$ \\ \bottomrule
	\end{tabular}
\end{table}
\begin{figure}[ht]
	\includegraphics[width=\textwidth]{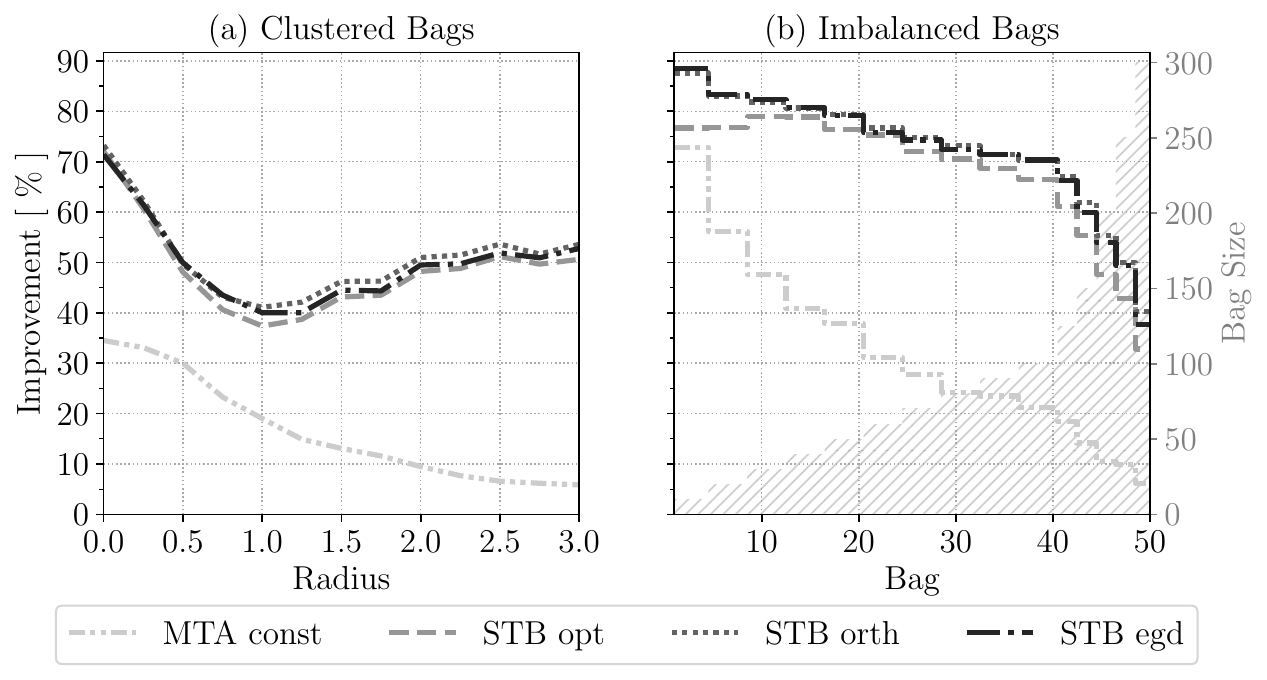}
	\includegraphics[width=\textwidth]{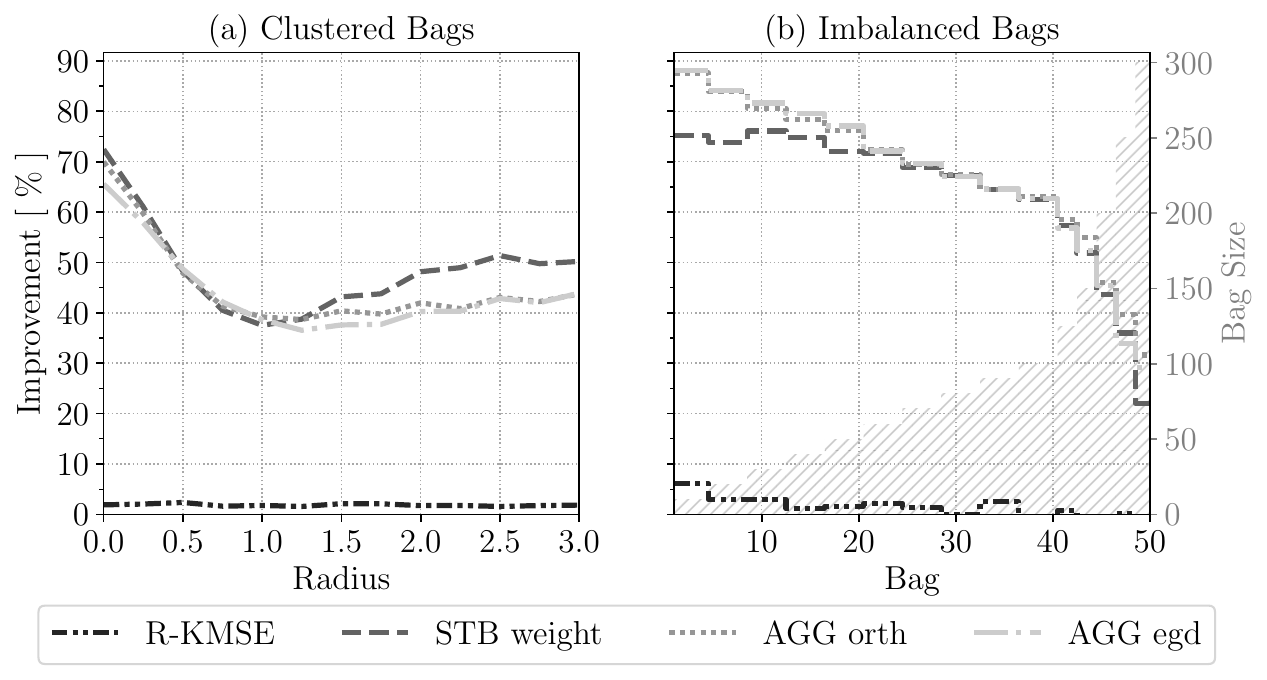}
	\caption{
		Decrease in estimation error compared to \NE{} in percent on Gaussian data settings (a) and (b) resp.
		Higher is better.
		The model parameters of the methods were optimised on iid training samples using cross-validation.
		The bars (right axis) in (b) show the bag sizes for the bags $1$ to $50$ which vary between $10$ and $300$.
		Compare with Fig.~\ref{fig:exp_art_results}, \ref{fig:apx_exp_art_results}.}
	\label{fig:apx_exp_trueopt}
\end{figure}

Finally, Figure~\ref{fig:apx_exp_trueopt} shows the results on the artificial Gaussian data with optimised model parameters which can be compared to the results with the default parameters shown in Figures~\ref{fig:exp_art_results} and \ref{fig:apx_exp_art_results}.
The performance of \MTAconst{} is unchanged because the optimised and default values of $\gamma$ are similar.
For \STBweight, \STBopt{} and \STBorth{} their optimised model parameters differ from their default counterparts for small radii, so that we observe a performance difference respectively.
Even though the cross-validation values of \STBorth{} vary with the radius, the performance difference is only small which suggests that \STBorth{} is stable for its parameter choice.
The range of suitable values of \AGGo{} is large as can be seen in the marginal performance difference even though the value of $\gamma$ fluctuates between $7$ and $14$.
However, when we also consider the poor performance on the cytometry data, it can be seen that it is sensitive to improper choices of $\gamma$ (very larger values for $\gamma$, e.g., $\gamma = 85$, improve the performance on cytometry).
Unsurprisingly, the performance difference is most noticeably for \AGGe{} and \STBegd{} for which we manually forced the values of $c_q$ or $c_1$.
Note, however, that $c_1$ is of clear importance for the Imbalanced data (optimal values are non-zero), and we observe an improvement on the cytometry data if $c_q$ and $c_1$ are non-zero.
This justifies the decision to choose non-zero values for $c_q$ and $c_1$ in more general settings.

\subsection{Computational complexity}\label{sec:apx_complexity}
Based on Section~\ref{sec:apx_desc_methods}, we can give estimations of the computational complexity of each method.
We first identify the complexities of prerequisites on which the methods are based on (items (a)-(g)).
Then we analyse each approach individually, where the stated complexity relates to the calculation of the weighting matrix, i.e., all pairwise weights (items (i)-(ix)).
Table \ref{tab:apx_complexity} summarises the total complexities as the combination of all required operations.
\subsubsection{Prerequisites}

\begin{enumerate}[label=(\alph*)]
	\item Intra-task kernel matrices $\kappa(Z^{(k)}_{\bullet}, Z^{(k)}_{\bullet})$ and inter-task kernel matrices $\kappa(Z^{(k)}_{\bullet}, Z^{(\ell)}_{\bullet})$: \\
	Requires: - \\
	Complexity: $\mathcal{O}(B^2 N^2 D)$\\
	Before we analyse the computations of the weights, we note that the kernel mean embedding is not computed explicitly.
	Instead, it occurs only in terms of inner products with other KMEs (kernel trick).
	A kernelised multiple instance problem is usually also not completed with the computation of the KMEs.
	The computations are required as intermediate step as part of, e.g., statistical testing or distributional learning.
	We conclude that the computation of the intra-bag kernel matrices $\kappa(Z^{(k)}_{\bullet}, Z^{(k)}_{\bullet})$ and that of the inter-bag kernel matrices $\kappa(Z^{(k)}_{\bullet}, Z^{(\ell)}_{\bullet})$ for all pairs $k,\ell \in \intr{B}$ is done anyway, and not just required for the computation of the weights of a multi-task averaging approach.
	
	For simplicity we assume that $N = \max {(N_k)}_{k \in \intr{B}}$.
	The complexity of the computation of all kernel matrices depends on the kernel choice.
	If it is mostly defined by the complexity of matrix multiplications of the data matrices, e.g., linear, Gaussian kernel, etc., it can be assumed to be in $\mathcal{O}(B^2 N^2 D)$.
	
	Except for the estimation of $\text{Tr}(\Sigma^2)$ (Algorithm~\ref{alg:approx_trE}) and the estimations of the naive variance $s^2$ \eqref{eq:shatkme} and $\mathbf{q}^{(k)}$ (Eq~\eqref{eq:apx_dedbiased}), the methods rely on the kernel matrices only in term of their sums, $\sum_{i=1}^{N_k} \sum_{j=1}^{N_\ell}\kappa(Z_{i}^{(k)},Z_{j}^{(\ell)})$ and $\sum_{i \neq j}^{N_k} \kappa(Z_{i}^{(k)},Z_{j}^{(k)})$.
	This reduces the memory consumption from $B^2 N^2$ to $B^2+B$ because the kernel matrices do not have to be stored but only their sums.
	
	\item Distance matrix $\hat{\text{MMD}}^2(\mu_k, \mu_\ell)$, Eq.~\eqref{eq:unmmd}: \\
	Requires: (a)\\
	Complexity: $\mathcal{O}(B^2)$\\
	Once (a) is precomputed, the distances $\hat{\text{MMD}}^2(\mu_k, \mu_\ell)$, Eq.~\eqref{eq:unmmd}, can be computed in linear time for all pairs $k,\ell \in \intr{B}$.
	The complexity is then in $\mathcal{O}(B^2)$.
	
	\item Naive risk $\hat{s}_k^2$, Eq.~\eqref{eq:shatkme}: \\
	Requires: (a) intra-task kernel matrices \\
	Complexity: $\mathcal{O}(B N^2)$\\
	Performs linear operations, e.g., sum and elementwise multiplication of $N \times N$ kernel matrices for every bag individually.
	
	\item $\widehat{\tr{\Sigma^2}}$, Algorithm~\ref{alg:approx_trE}: \\
	Requires: (a) intra-task kernel matrices \\
	Complexity: $\mathcal{O}(B N^2 r)$ \\
	As seen in Algorithm~\ref{alg:approx_trE}, the first term can be calculated explicitly (line 6) in $\mathcal{O}(N^2)$ and the other terms must be approximated in $r$ iterations.
	
	\item $\mathbf{q}^{(k)}$, Eq.~\eqref{eq:apx_dedbiased}: \\
	Requires: (a)\\
	Complexity: $\mathcal{O}(B^2 N^2)$\\
	The first sum of~\eqref{eq:apx_dedbiased} is computed explicitly on $N\times N$ kernel matrices for every pair $k,\ell \in \intr{B}$ individually, $\mathcal{O}(B^2 N^2)$. The second sum can be computed more efficiently, as it operates only on the sums of the kernels, which requires $\mathcal{O}(B^2)$.
\end{enumerate}

\subsubsection{State-of-the-art approaches}
\begin{enumerate}[label=(\roman*)]
	\item \NE: \\
	Requires: - \\
	Complexity: - \\
	
	\item \RKMSE: \\
	Requires: (a) intra-task kernel matrices \\
	Complexity: $\mathcal{O}(B)$\\
	Performs linear operations on the precomputed sums of the intra-task kernel matrices.
	
	\item \MTAconst: \\
	Requires: (b), (c)\\
	Complexity: $\mathcal{O}(B^3)$\\
	The similarity matrix is constant but data dependent and computes the sum of the naive distance matrix, i.e., is in $\mathcal{O}(B^2)$.
	The calculation of the Laplacian also operates linearly on $B \times B$ matrices.
	Finally, the computation of the weighting matrix performs a matrix multiplication and a matrix inversion which both require $\mathcal{O}(B^3)$.
\end{enumerate}

\subsubsection{\texttt{AGG} approaches}
\begin{enumerate}[label=(\roman*)]
	\setcounter{enumi}{3}
	\item \AGGo: \\
	Requires: (b), (c) \\
	Complexity: $\mathcal{O}(B^2)$ \\
	Only elementwise multiplications and sums are required for the computation of the weights.
	
	\item \AGGe: \\
	Requires: (a), (c), (d), (e) \\
	Complexity: $\mathcal{O}(B^4 t_{\max})$ \\
	Algorithm~\ref{alg:agge} is performed for each bag individually.
	The computation of $\hat{\Lambda}^{(k)}$, Eq.~\eqref{eq:apx_lambdabiased}, can be computed in $\mathcal{O}(B^2)$.
	The weights are then iteratively updated at most $t_{\max}$ times.
	In each iteration the gradient is computed in $\mathcal{O}(B^2)$ (matrix-vector multiplication), and egd and normalization are performed in $\mathcal{O}(B)$.
	In total, the computational complexity of the weighting matrix is then in $\mathcal{O} \left( B \cdot \left( B^2 + t_{\max} \cdot \left( B^2+B \right) \right) \right) = \mathcal{O}(B^3 t_{\max})$.
\end{enumerate}

\subsubsection{\texttt{STB} approaches}
\begin{enumerate}[label=(\alph*)]
	\setcounter{enumi}{5}
	\item \texttt{STB} safeguard $\hat{W}_{\cteW}^{(k)}$, Eq.~\eqref{eq:apx_kstb}: \\
	Requires: (d)\\
	Complexity: $\mathcal{O}(B^2)$\\
	Before neighbours are found, bags with a very large variance are excluded as a safeguard. Because the inequality is checked for every pair of bags, the complexity is in $\mathcal{O}(B^2)$.
	
	\item \texttt{STB} neighbours $\hat{V}_{\tau,\cteW}^{(k)}$, Eq.~\eqref{eq:apx_kstb}: \\
	Requires: (b), (c) \\
	Complexity: $\mathcal{O}(B^2)$ \\
	The similarity test is performed for every pair of bags such that the complexity is in $\mathcal{O}(B^2)$.
\end{enumerate}

\begin{enumerate}[label=(\roman*)]
	\setcounter{enumi}{5}
	\item \STBweight: \\
	Requires: (f), (g) \\
	Complexity: $\mathcal{O}(B^2)$ \\
	The computation of the weighting matrix only requires elementwise operations, i.e., sums and multiplications.
	
	\item \STBopt: \\
	Requires: (c), (f), (g)\\
	Complexity: $\mathcal{O}(B^2)$\\
	The computation of the weighting matrix only requires elementwise operations, i.e., sums and multiplications.
	
	\item \STBorth:\\
	Requires: (f), (g), (iv)\\
	Complexity: $\mathcal{O}(B^2)$\\
	\STBorth{} combines the similarity test with \AGGo.
	Its computational complexity is the sum of both approaches accordingly.
	
	\item \STBegd:\\
	Requires: (f), (g), (v)\\
	Complexity: $\mathcal{O}(B^3  t_{\max})$\\
	The computational complexity is composed of finding the neighbours and of \AGGe.
\end{enumerate}
\begin{table}[hb] \label{tab:apx_complexity}
	\caption{Summary of the individual and total computational complexities of the methods and their prerequisites in the kernel setting.
		The total complexity is the sum of the complexities of all required computations.
		Task (a) does not affect the total complexity because its computation is required not only for the estimation of the KMEs.}
	\begin{tabular}{lllll} \toprule
		& Task & Individual        & Required Computations   & Total                                 \\ \midrule
		\multicolumn{5}{l}{PREREQUISITES}                                                             \\ \midrule
		& \color{gray}{(a)}  & \textcolor{gray}{$\mathcal{O}(B^2 N^2 D)$}    & \color{gray}{(a)}                     & \textcolor{gray}{$\mathcal{O}(B^2 N^2 D)$}                        \\
		& (b)  & $\mathcal{O}(B^2)$          & \textcolor{gray}{(a)}, (b)                & $\mathcal{O}(B^2)$                              \\
		& (c)  & $\mathcal{O}(B N^2)$        & \textcolor{gray}{(a)}, (c)                & $\mathcal{O}(B N^2)$                            \\
		& (d)  & $\mathcal{O}(B N^2 r)$      & \textcolor{gray}{(a)}, (d)                & $\mathcal{O}(B N^2 r)$                          \\
		& (e)  & $\mathcal{O}(B^2 N^2)$      & \textcolor{gray}{(a)}, (e)                & $\mathcal{O}(B^2 N^2)$                          \\ \midrule
		\multicolumn{5}{l}{SOTA}                                                                      \\ \midrule
		& (i)  & -                 & (i)                     & -                                     \\
		& (ii)  & $\mathcal{O}(B)$            & \textcolor{gray}{(a)}, (ii)                & $\mathcal{O}(B)$                                \\
		& (iii)  & $\mathcal{O}(B^3)$          & (b), (c), (iii)           & $\mathcal{O}(B N^2 + B^3)$                      \\ \midrule
		\multicolumn{5}{l}{AGG}                                                                       \\ \midrule
		& (iv)  & $\mathcal{O}(B^2)$          & (b), (c), (iv)           & $\mathcal{O}(B N^2 + B^2)$                      \\
		& (v)  & $\mathcal{O}(B^3 t_{\max})$ & \textcolor{gray}{(a)}, (c), (d), (e), (v) & $\mathcal{O}(B N^2 r + B^2 N^2 + B^3 t_{\max})$ \\ \midrule
		\multicolumn{5}{l}{STB}                                                                       \\ \midrule
		& (f)  & $\mathcal{O}(B^2)$          & (d), (f)                & $\mathcal{O}(B N^2 r + B^2)$                    \\
		& (g)  & $\mathcal{O}(B^2)$          & (b), (c), (g)           & $\mathcal{O}(B N^2 + B^2)$                      \\
		& (vi)  & $\mathcal{O}(B^2)$          & (f), (g), (vi)           & $\mathcal{O}(B N^2 r + B^2)$                    \\
		& (vii)  & $\mathcal{O}(B^2)$          & (c), (f), (g), (vii)      & $\mathcal{O}(B N^2 r + B^2)$                    \\
		& (viii)  & $\mathcal{O}(B^2)$          & (f), (g), (iv), (viii)      & $\mathcal{O}(B N^2 r + B^2)$                    \\
		& (ix)  & $\mathcal{O}(B^3 t_{\max})$ & (f), (g), (v), (ix)      & $\mathcal{O}(B N^2 r + B^2 N^2 + B^3 t_{\max})$ \\ \bottomrule
	\end{tabular}
\end{table}

\subsection{Practical Recommendations}\label{apx:praxisrecommends}
\Revision{
	We summarise our findings on the computational complexities (Sup.~\ref{sec:apx_complexity}), the parameter robustness (Sup.~\ref{apx:sec_defaultParams}), and the numerical performances (Sec.~\ref{se:application}, Sup.~\ref{sec:apx_expresults}) in the following recommendations:
	
	\AGGo{} has the lowest computational complexity and involves only a single model parameter. However, its estimation performance is highly sensitive to the choice of $\gamma$. In particular, the method may perform worse than \NE{} in the presence of many distant means, which can distort the estimation. In such cases, the parameter $\gamma$ must be increased to reduce the influence of these distant means and improve the results. Parameter tuning, e.g., cross-validation, will be needed for determining a suitable value for $\gamma$.
	
	\STBorth{} extends \AGGo{} by incorporating a neighbouring test to exclude distant means, which significantly reduces sensitivity to the choice of $\gamma$. As a result, \STBorth{} strictly dominates \AGGo{} and is therefore preferable. Depending on the choice of the safeguard threshold $\tau$, the method interpolates between \NE{} (for small $\tau$) and \AGGo{} (for large $\tau$). If fast and reliable estimates are required, we recommend \STBorth .
	
	\STBopt{} has a computational complexity comparable to that of \STBorth{} (i.e., remains efficient), and typically yields satisfactory estimates. The influence of its two model parameters on the estimation accuracy tends to plateau, indicating robustness. Its performance is slightly below that of \STBorth{} and may occasionally exhibit higher variance.
	
	Despite its high computational complexity, \AGGe{} is outperformed by both \STBorth{} and \STBopt{} in terms of estimation accuracy. Accordingly, we do not recommend the use of \AGGe{} without a preceding neighbouring test (cf. \STBegd).
	
	\STBegd{} incurs the highest computational cost, due to the use of exponentiated gradient descent and the need for several auxiliary statistics. Nevertheless, it consistently provides the most accurate mean estimates with a low variance in performance across trials. If computational resources are not a constraint, \STBegd{} is recommended.
}     
\end{document}